\documentclass[journal,onecolumn]{IEEEtran} 
\usepackage[numbers]{natbib}
\usepackage{hyperref}
\usepackage{amsmath}
\usepackage{amsfonts}
\usepackage{lastpage}
\usepackage{graphicx}
\usepackage{grffile}
\usepackage{subfigure}
\usepackage{multirow,diagbox}
\usepackage{tabu}
\usepackage{essay-def}
\usepackage{clrscode}
\usepackage{caption}
\usepackage{mathrsfs}
\usepackage{algorithm}
\usepackage{algorithmic}
\usepackage{indentfirst}
\usepackage{natbib}

\newcommand{\yifei}[1]{\textcolor{blue}{\textbf{YW:} #1}}
\usepackage{ifthen}
\usepackage{siunitx}
\usepackage{tikz}
\usetikzlibrary{hobby} 
\usepackage{pgfplots}
\pgfplotsset{compat=1.11}
\usetikzlibrary{arrows.meta} 
\tikzset{>={Latex[length=4,width=4]}} 
\usetikzlibrary{calc,intersections,decorations.markings}
\usepackage{siunitx}
\usepackage{xcolor} 
\usepackage{lmodern}
\newenvironment{qaquote}{\begin{quotation}\hspace*{-1\parindent}\setlength{\parindent}{0em}\relax}{\end{quotation}}

\colorlet{mylightblue}{blue!20}
\colorlet{myblue}{blue!50!black}
\colorlet{mydarkblue}{blue!30!black}
\colorlet{mylightred}{red!10}
\colorlet{myred}{red!50!black}
\colorlet{mydarkred}{red!60!black}
\colorlet{mydarkgreen}{green!30!black}

\tikzset{
  midarr/.style={decoration={markings,mark=at position #1 with {\arrow{stealth}}},postaction={decorate}},
  midarr/.default=0.5
}

\usepackage[textsize=tiny]{todonotes}

\newcommand{\Bern}{\mathbf{Bern}}
\newcommand{\sgn}{\mathrm{sgn}}

\newcommand{\normal}{\mathrm{N}}
\newcommand{\cone}{\mathrm{cone}}

\newcommand{\bone}{\boldsymbol{1}}
\newcommand{\balpha}{\boldsymbol{\alpha}}
\newcommand{\interior}{\mathbf{int}}
\newcommand{\eigmax}{\mathrm{eigmax}}
 
\DeclareMathSymbol{\mydash}{\mathord}{AMSa}{"39} 

\renewcommand{\xi}{\chi}

\ifCLASSINFOpdf
\else
\fi
\begin{document}
%
\title{Overparameterized ReLU Neural Networks Learn the Simplest Model: Neural Isometry and Phase Transitions}
%
%
%
\author{Yifei~Wang$^1$,
       Yixuan~Hua$^2$,
       Emmanuel~J.~Cand\`es$^3$,
       Mert~Pilanci$^1$~\IEEEmembership{Member,~IEEE}
\thanks{$^1$Yifei~Wang and Mert~Pilanci are with the Department
of Electrical Engineering, Stanford University, Stanford, CA USA. E-mail: \texttt{\{wangyf18,pilanci\}@stanford.edu}}
\thanks{$^2$Yixuan~Hua is with the Department
of Operations Research \& Financial Engineering, Princeton University, Princeton, NJ USA.}
\thanks{$^3$Emmanuel~J.~Cand\`es is with the Department
of Statistics and Department of Mathematics, Stanford University, Stanford, CA USA.}%
}

\maketitle

\begin{abstract}
The practice of deep learning has shown that neural networks generalize remarkably well even with an extreme number of learned parameters. This appears to contradict traditional statistical wisdom, in which a trade-off between model complexity and fit to the data is essential. We aim to address this discrepancy by adopting a convex optimization and sparse recovery perspective. We consider the training and generalization properties of two-layer ReLU networks with standard weight decay regularization. Under certain regularity assumptions on the data, we show that ReLU networks with an arbitrary number of parameters learn only simple models that explain the data. This is analogous to the recovery of the sparsest linear model in compressed sensing. For ReLU networks and their variants with skip connections or normalization layers, we present isometry conditions that ensure the exact recovery of planted neurons. For randomly generated data, we show the existence of a phase transition in recovering planted neural network models, which is easy to describe: whenever the ratio between the number of samples and the dimension exceeds a numerical threshold, the recovery succeeds with high probability; otherwise, it fails with high probability. Surprisingly, ReLU networks learn simple and sparse models that generalize well even when the labels are noisy . The phase transition phenomenon is confirmed through numerical experiments. 

\end{abstract}


%
\IEEEpeerreviewmaketitle

\section{Introduction}
%
%
%
%
\IEEEPARstart{R}{}ecent work has shown that neural networks (NNs) exhibit extraordinary generalization abilities in many machine learning tasks. Although NNs employed in practice are often over-parameterized, meaning the number of parameters exceeds the sample size, they generalize to unseen data and perform well.
In this work, we study the problem of generalization in such over-parameterized models from a convex optimization and sparse recovery perspective. We are interested in the setting where a NN with an arbitrary number of neurons is trained on datasets that have a simple structure, for instance, when the label vector is the output of a simple NN, which may have additive noise. A natural question arises: under which conditions, a neural network with an arbitrary number of neurons, in which the number of trainable parameters can be quite large, recovers the planted model exactly and thus achieve perfect generalization? We uncover a sharp phase transition in the behavior of NNs in the recovery of simple planted models. Our results imply that the weight decay regularization solely controls whether the NN recovers the underlying simple model planted in the data or fails by overfitting a more complex model, regardless of the number of parameters in the NN.

Our findings are close in spirit to classical results on sparse recovery and compressed sensing. It is known that there exists a sharp phase transition in recovering a sparse planted vector from random linear measurements. To be specific, the probability of successful recovery will be close to one when the sample number exceeds a certain threshold. Otherwise, the probability of recovery is close to zero. This can be shown by analyzing the intersection probability of a convex cone with a random subspace, which undergoes a sharp phase transition as the statistical dimension of the cone changes with respect to the ambient dimension \cite{amelunxen2014living}. By leveraging recently discovered connections between ReLU NNs and Group Lasso models \cite{nnacr,ergen2020convex,ergen2021global,wang2022convexgeometryofbp,wang2020hidden}, we show that a calculation involving statistical dimensions of convex cones implies a phase transition in two-layer ReLU NNs for recovering simple planted models. 

In addition, we consider deterministic training data and derive analogues of the \emph{irrepresentability condition} \citep{zhao2006model} and \emph{Restricted Isometry Property} \cite{candes2005decoding,candes2008restricted}, which play an important role in the recovery of sparse linear models. We develop the notion of \emph{Neural Isometry Conditions} to characterize non-random training data that allow exact recovery of planted neurons. We further show that random i.i.d. Gaussian, sub-Gaussian, and Haar distributed random matrices satisfy Neural Isometry Conditions with high probability when the number of samples is sufficiently high. The random Gaussian data assumption is widely used in practice. For instance, the subproblem in training diffusion models \cite{sohl2015deep,ho2020denoising} is equivalent to learn the function over Gaussian random data.


Although neural networks lead to non-convex optimization problems which are challenging to analyze, a recent line of work \citep{nnacr, ergen2020convex, ergen2021global} showed that regularized training problems of multilayer ReLU networks can be reformulated as convex programs. Based on the convex optimization formulations, \citep{wang2020hidden} further gives the exact characterizations of all global optima in two-layer ReLU networks. More precisely, it was shown that all globally optimal solutions of the nonconvex training problem are given by the solution set of a simple convex program up to permutation and splitting. In other words, we can find the set of optimal NNs for the regularized training problem by solving a convex optimization problem. The convex optimization formulations of NNs were also extended to NNs with batch normalization layers \citep{ergen2021demystifying}, convolutional NNs (CNNs) \citep{ ergen2020implicit, ergen2021global}, polynomial activation networks \citep{bartan2021neural}, transformers with self-attention layers \cite{sahiner2022attention} and Generative Adversarial Networks (GANs) \cite{sahiner2021hidden}.



%

\begin{table*}[!t]
    \centering
    \begin{tabular}{|c|c|c|c|c|c|}
    \hline
    \multirow{2}*{Model}&\multirow{2}*{Data}&\multicolumn{2}{c|}{Success} & \multirow{2}*{Failure} & \multirow{2}*{Results} \\\cline{3-4}
    &&strong recovery&weak recovery&&\\\hline
         \multirow{2}*{linear }& Haar& $n>7.613 d$ &$n>2d$ &$n<2d$ & Thm. \ref{thm:relu_skip_phase} Thm. \ref{thm:white_strong} \\\cline{2-6}
         &Gaussian&$n>O(d\log(n))$&$n>2d$&$n<2d$&Thm. \ref{thm:gauss_success}, Thm. \ref{prop:skip_gauss_success}\\\hline
         ReLU-normalized &Gaussian&-&$n>2d$&-&Thm. \ref{thm:relu_normal}\\\hline
         ReLU &Gaussian&-&$n\to\infty$&-&Prop. \ref{prop:asymp_relu}\\\hline
         two ReLUs-normalized &Gaussian&-&$n\to\infty$&-& Prop. \ref{prop:asymp_two_neu}\\\hline
         $k$ ReLUs &-&-&Neural Isometry Condition \eqref{eq:irr_cond}&-& Prop. \ref{prop:imply_multi}\\\hline
         $k$ ReLUs-normalized &-&-&-&-& Prop. \ref{prop:imply_multi_normal}\\\hline
         noisy linear model &sub-Gaussian& -&$n>O(d\,\log(n))$ &-& Thm. \ref{thm:noisy}\\\hline
    \end{tabular}
    \caption{Summary of our results for the success or failure of the recovery of planted neurons. Gaussian data refers to training matrices with i.i.d. standard Gaussian entries, Haar data refers to uniformly drawn training matrices from the set of orthogonal matrices. Strong recovery refers to the recovery of all possible planted neurons, whereas weak recovery refers to the recovery of fixed neurons.}
    \label{tab:summary of results}
\end{table*}

We study recovery properties of optimal two-layer ReLU neural networks by considering their equivalent convex formulations and leveraging connections to sparse recovery and compressed sensing. We also consider variants of the basic two-layer architecture with skip connections and normalization layers, which are basic building blocks of modern DNNs such as ResNets \cite{resnet}. 
We show the existence of a sharp phase transition in the recovery of simple models via ReLU NNs. To be more specific, for a data matrix $\mfX\in\mbR^{n\times d}$, there exists a critical threshold for the ratio $n/d$, above which 
a planted network with few neurons will be the unique optimal solution of the convex program of two-layer networks with probability close to $1$, as long as $n$ and $d$ are moderately high. Otherwise, this probability will be close to $0$. The same conclusion applies to the non-convex training of the two-layer network with any number of neurons, up to permutation and neuron splitting (see Appendix \ref{sec:permutation}).  Moreover, our results highlight the importance of skip-connections and normalization layers in the success of recovery as we show in Section \ref{sec:relu_normal}.  We also provide deterministic isometry conditions that guarantee the recovery of an arbitrary number of ReLU neurons. We summarize these results in Table \ref{tab:summary of results}. 

\subsection{Related works}
Recent work have investigated linearized neural network models trained with gradient descent from a kernel based learning perspective \cite{Jacot2018, Du2018, allen2018learning, liu2020linearity}. When the width of the neural network approaches infinity, it is known that NNs can fit all the training data \cite{zou2018stochastic, allen2019convergence, allen2019learning}. However, in this regime, the analysis shows that almost no hidden neurons move from their initial values to learn features \cite{Chizat2018}.
Further experiments also confirm that infinite width limits and linearized kernel approximations are insufficient to give an adequate explanation for the success of non-convex neural network models as its width goes to infinity \cite{arora2019fine}.
Due to the non-linear structure of neural networks and non-convexity of the training problem, only a few works consider the role of over-parameterization when the width of the neural network is finite \cite{chen2019much, neyshabur2018role, liu2020linearity}.

It was conjectured that models trained with simple iterative methods such as Stochastic Gradient Descent (SGD) approach flat local minima \citep{hochreiter1997flat, jastrzkebski2017three, wu2017towards} that generalize well. However, it was also shown that the behavior of the trained model heavily depends on the choice of the specific optimization algorithm and its hyper-parameters \citep{goyal2017accurate, hoffer2017train,jastrzkebski2017three,luo2020many}.

Developing algorithms for the recovery of planted two-layer neural networks has been studied in the literature. In \cite{ge2018learning, bakshi2019learning}, the authors design spectral methods for learning the weight matrices of a planted two-layer ReLU network with $k$ neurons. In contrast, our work studies the recovery properties of over-parameterized NNs with arbitrary many neurons that minimize the training objective and sheds light into the optimization landscape.  
\cite{tian2017analytical, soltanolkotabi2017learning, li2017convergence, zhang2019learning} analyze the recovery of two-layer ReLU neural networks using the gradient descent method, while \cite{zhong2017recovery} extends the analysis to other activation functions, including leaky ReLU. In comparison, our results apply to the case where neural networks have many more neurons than the planted model. 

\subsection{Notation}
We introduce notations used throughout the paper. We use the notation $[n]$ to represent the set $\{1,\dots,n\}$. We use the notation $\mbI(\cdot)$ for the $0$-$1$ valued indicator function which takes the value $1$ when its argument is a true logical statement and 0 otherwise. We reserve boldcase lower-case letters for vectors, boldcase upper-case letters for matrices and plain lower-case letters for scalars. For a vector $\mfw\in\mbR^d$, we use $\|\mfw\|_p:=\left( \sum_{i=1}^d |w_i|^{p} \right)^{1/p}$ to represent its $\ell_p$ norm. The notation $\cos\angle(\mfw,\mfv):=\frac{\mfw^T\mfv}{\|\mfw\|_2\|\mfv\|_2}\in[-1,1]$ represents the cosine angle between two vectors $\mfw\in\mathbb{R}^{d}$ and $\mfv\in\mathbb{R}^{d}$. For a matrix $\mfX\in\mbR^{n\times d}$, we use $\|\mfX\|_2$ to denote its operator norm. We use $\|\mfX\|_{\infty,\infty}$ for the matrix infinity norm, i.e., the elementwise maximum absolute value. We use $\mfI_n\in\mbR^{n\times n}$ to denote the identity matrix with size $n\times n$. We denote $\lambda_{\max}(\mfX)$ for the maximum eigenvalue of a symmetric matrix $\mfX$. Similarly, the notion $\eigmax(\mfX)$ represents the subspace spanned by the eigenvectors corresponding to the maximum eigenvalue for a symmetric matrix $\mfX$. We use the shorthand $\diag(x_1,...,x_n)$ to represent a diagonal matrix with entries $x_1,...,x_n$ on the diagonal. The notation $\mfx_i^\mathrm{row}\in\mbR^d$ and $\mfx_j^\mathrm{col}\in\mbR^n$ denoted the $i$-th row and $j$-th column of an $n\times d$ matrix $\mfX \in\mathbb{R}^{n\times d}$ respectively. We use the notation \textrm{poly}$(n)$ to denote a polynomial function of the variable $n$. We call a centered random variable $X$ sub-Gaussian with variance proxy $\sigma^2>0$ if it holds that $\mbE[e^{sX}\le e^{\frac{1}{2}\sigma^2s^2}]$ for any $s\in\mathbb{R}$.

\subsection{Organization}
In Section \ref{sec:preview}, we present a preview of our results for the special case of linear neuron recovery with ReLU NNs. We introduce variants of the ReLU NN architecture in Section \ref{sec:neural_recovery}. In Section \ref{sec:neu_iso} we develop deterministic conditions, termed \textbf{Neural Isometry Conditions}, for the recovery of linear and non-linear neurons using different NN architectures. In Section \ref{sec:relu_skip}, we investigate random ensembles of training data matrices, for which we show the existence of a sharp phase transition in satisfying these deterministic conditions via non-asymptotic probabilistic bounds. In Section \ref{sec:asymp}, we develop asymptotic results on the recovery probability for the case of multiple ReLU neurons. We present numerical simulations to corroborate our theoretical results in Section \ref{num_res:main}. We present our conclusions in Section \ref{sec:conclusion}. 
\section{A preview: an exact characterization of linear neuron recovery}\label{sec:preview}
In this section, we present a preview of our results on two-layer ReLU networks with skip connection in the special case of linear neuron recovery. We generalize our result to the recovery of nonlinear neurons using different architectures later in Sections \ref{sec:neu_iso}-\ref{sec:asymp}. We invite the reader to refer to Appendix \ref{app:review_linear} for the background on the isometry conditions in compressed sensing, which share important parallels with our analysis. We start with the following two-layer ReLU network model with a linear skip connection. 
\begin{equation*}
    f(\mfX;\Theta) =\mfX\mfw_{1,1}w_{2,1}+\sum_{i=2}^m (\mfX\mfw_{1,i})_+w_{2,i},
\end{equation*}
where $\Theta = \{\mfW_1,\mfw_2\}$ denotes trainable parameters including first layer weights $\mfW_1\in \mbR^{d\times m}$ and second layer weights $\mfw_2\in \mbR^{m}$. We will first consider the minimum norm interpolation problem
\begin{equation}
\min_{\Theta} \underbrace{\|\mfW_1\|_F^2+\|\mfw_2\|_2^2}_{\|\Theta\|_F^2}, \text{ s.t. } f(\mfX;\Theta)=\mfy. \label{prob:min_nrm_preview}
\end{equation}
where the objective $\|\Theta\|_F^2$ stands for the weight-decay regularization term.


\subsection{Hyperplane Arrangements}
We now introduce an important concept from combinatorial geometry called hyperplane arrangement patterns, in order to introduce convex optimization formulations of ReLU network training problems.
\begin{definition}[Diagonal Arrangement Patterns]\label{def:hyper}
We define $\{0,1\}$ valued diagonal matrices $\mfD_1,\dots,\mfD_p$ that contain an enumeration of the set of hyperplane arrangement patterns of the training data matrix $\mfX$ as follows. Let us define
\begin{equation}\label{equ:hyperplane}
    \mathcal{H}:=\Big \{\diag(\mbI(\mfX\mfh\geq 0))\,\,\big |\,\,\mfh\in\mbR^d,\mfh\neq 0 \Big \}.
\end{equation}
We call $\mfD_1,\dots,\mfD_p\in \mathcal{H}$, an enumeration of the elements of $\mathcal{H}$ in an arbitrary fixed order, \emph{diagonal arrangement patterns} associated with the training data matrix $\mfX$.
\end{definition}
The $\{0,1\}$ valued patterns on the diagonals of $\mfD_1,\dots,\mfD_p$ encodes a partition of $\mathbb{R}^d$ by hyperplanes passing through the origin that are perpendicular to the rows of $\mfX$.
The number of such distinct patterns is the cardinality of the set $\mathcal{H}$ and is bounded as follows
\[ p := | \mathcal{H} | \le 2 \sum_{k=0}^{r-1} {n-1 \choose k} \le 2r\Big(\frac{e(n-1)}{r}\Big)^r\,,\]
where $r=\textrm{rank}(\mfX)$, see \citep{cover1965geometrical}.
A $2$-dimensional example of diagonal arrangement patterns with three hyperplanes, i.e., $n=3$, is presented in Figure \ref{fig:diag}. Note that there are $p = 2 \big ( {2 \choose 0} + {2 \choose 1} \big ) = 6$ regions and corresponding patterns associated with this configuration. For every fixed dimension $d$ (or rank $r$), the number of patterns $p$ is bounded by \textrm{poly}$(n)$. In \cite{ergen2020implicit}, it was shown that convolutional neural networks (CNNs) have a small fixed rank. For instance, a typical convolutional layer of size $3\times 3\times 512$, e.g. $512$ filters of size $3\times 3$ implies $r\le 3\times 3=9$.
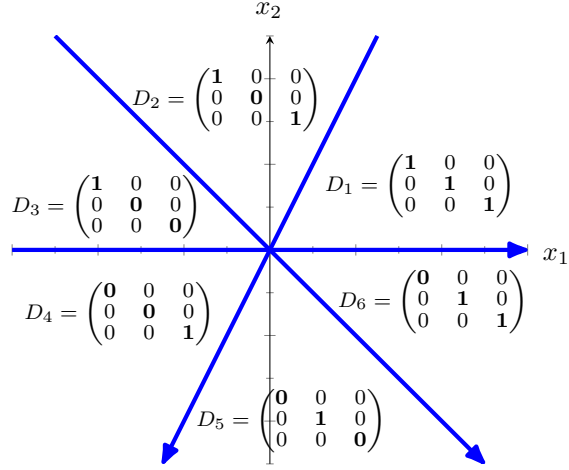
\begin{figure}
\centering
\begin{minipage}{.8\textwidth}
\begin{center}
\begin{tikzpicture}
  \begin{axis}[ymin=-5,ymax=5,xmax=6,xmin=-6,xticklabel=\empty,yticklabel=\empty,
               minor tick num=1,axis lines = middle,xlabel=$x_1$,ylabel=$x_2$,label style ={at={(ticklabel cs:1.1)}}]
\addplot [<-,>={Latex[round]},domain=-2.5:2.5,samples=2,ultra thick,color=blue] {2*x};
\addplot [->,>={Latex[round]},domain=-5:5,samples=2,ultra thick,color=blue] {-x};
\addplot [->,>={Latex[round]},domain=-6:6,samples=2,ultra thick,color=blue] {0};
\node[font=\fontsize{8}{8}\selectfont] at (axis cs:3.5,1.5) {$D_1=\begin{pmatrix}\mathbf{1}&0&0\\0&\mathbf{1}&0\\0&0&\mathbf{1}\end{pmatrix}$};
\node[font=\fontsize{8}{8}\selectfont] at (axis cs:-1,3.5) {$D_2=\begin{pmatrix}\mathbf{1}&0&0\\0&\mathbf{0}&0\\0&0&\mathbf{1}\end{pmatrix}$};
\node[font=\fontsize{8}{8}\selectfont] at (axis cs:-3.8,1) {$D_3=\begin{pmatrix}\mathbf{1}&0&0\\0&\mathbf{0}&0\\0&0&\mathbf{0}\end{pmatrix}$};
\node[font=\fontsize{8}{8}\selectfont] at (axis cs:-3.5,-1.5) {$D_4=\begin{pmatrix}\mathbf{0}&0&0\\0&\mathbf{0}&0\\0&0&\mathbf{1}\end{pmatrix}$};
\node[font=\fontsize{8}{8}\selectfont] at (axis cs:0.5,-4) {$D_5=\begin{pmatrix}\mathbf{0}&0&0\\0&\mathbf{1}&0\\0&0&\mathbf{0}\end{pmatrix}$};
\node[font=\fontsize{8}{8}\selectfont] at (axis cs:3.8,-1.2) {$D_6=\begin{pmatrix}\mathbf{0}&0&0\\0&\mathbf{1}&0\\0&0&\mathbf{1}\end{pmatrix}$};
  \end{axis}
\end{tikzpicture}
\end{center}
\end{minipage}
\caption{A $2$-dimensional example of diagonal arrangement patterns.}\label{fig:diag}
\end{figure}
\subsection{Convex Reformulations}
The non-convex optimization problem \eqref{prob:min_nrm_preview} of ReLU networks with skip connection is equivalent\footnote{Under the condition that the NN has sufficiently many neurons (see e.g., Theorem 1 in \cite{nnacr})}  to a convex program:
\begin{equation}\label{min_nrm:relu_skip}
\begin{aligned}
    \min_{\mfw_0,\left(\mfw_{j}, \mfw_{j}'\right)_{j=1}^{p}}\quad &  \sum_{j=1}^{p}\left(\left\|\mfw_{j}\right\|_{2}+\left\|\mfw_{j}'\right\|_{2}\right)\\
        \text{s.t.}\quad &\mfX\mfw_0+\sum_{j=1}^{p} \mfD_{j} \mfX\left(\mfw_{j}-\mfw_{j}'\right)=\mfy,\\
        & (2\mfD_j-\mfI_n)\mfX\mfw_j\ge0, (2\mfD_j-\mfI_n)\mfX\mfw_j'\ge0, j\in[p].
\end{aligned}
\end{equation}
Indeed, the global optimal set of the non-convex program \eqref{prob:min_nrm_preview} can be characterized by the optimal solutions of the convex program \eqref{min_nrm:relu_skip}. The next result is an extension of the earlier results \cite{wang2022convexgeometryofbp,wang2020hidden,ergen2020convex} to NNs with a skip connection.
\begin{lemma}\label{thm:landscape_skip}
All globally optimal solutions of the non-convex problem \eqref{prob:min_nrm_preview} of ReLU networks with skip connection can be found (up to splitting and permutation) via the optimal solution set of the convex program \eqref{min_nrm:relu_skip} when $m\ge m^*$. Here $m^*:=\sum_{j=1}^p \mathbb{I}[w^*_j\neq 0]+\mathbb{I}[w^{*\prime}_j\neq 0]$ is the number of non-zero neurons in the minimum norm optimal solution of \eqref{min_nrm:relu_skip}.    
\end{lemma}

\subsection{Linear Neural Isometry Condition}
We now consider the case when the labels are generated by a planted linear model, i.e., $\mfy=\mfX\mfw^*$ and ask the following question: 

\begin{qaquote}
\emph{Can ReLU NNs with a linear skip connection learn a planted linear relation as effectively as a linear model?}
\end{qaquote}
In order to prove that the linear model can be 
recovered by solving the non-convex problem \eqref{prob:min_nrm_preview} or its convex reformulation \eqref{min_nrm:relu_skip}, we introduce the following isometry condition on the training data. 

\begin{definition}[Linear Neural Isometry Condition]
The linear neural isometry condition for recovering the linear model $\mfy=\mfX\mfw^*$ from \eqref{min_nrm:grelu_skip} is given by:
\begin{equation}\label{irrep:grelu_skip}
    \norm{\mfX^T \mfD_j \mfX\pp{\mfX^T\mfX}^{-1}\hat \mfw}_2<1, \forall  j\in[p], \tag{NIC-L}
\end{equation}
where $\hat w :=\frac{\mfw^*}{\|\mfw^*\|_2}$.
Note that \eqref{irrep:grelu_skip} holds uniformly for all $\mfw^*$ only if 
\begin{equation}
    \norm{\mfI_d-\sum_{{k:\mfx_k^T\mfw< 0}}\mfx_k\mfx_k^T\pp{\sum_{k=1}^n\mfx_k\mfx_k^T}^{-1}}_{2}<1, \forall \mfw: \mfw\neq 0.
\end{equation}
The above is a spectral isometry condition on the empirical covariance
$\hat \Sigma := \sum_{k=1}^n\mfx_k\mfx_k^T$ and its subsampled version $ \hat \Sigma_{\mfw} := \sum_{{k:\mfx_k^T\mfw< 0}}\mfx_k\mfx_k^T$ that excludes samples activated by a ReLU neuron. Intuitively, for the above condition to hold, the empirical covariance should be relatively stable when samples that lie on a halfspace are removed, and consequently $\norm{\mfI_d - \hat \Sigma_{\mfw} \hat \Sigma^{-1}}_2=\norm{\big(\hat \Sigma - \hat \Sigma_{\mfw} \big) \hat \Sigma^{-1}}_2<1$. 
\end{definition}
In the following proposition, we show that the linear neural isometry condition implies the recovery of the planted linear model by solving  \eqref{min_nrm:relu_skip}.

\begin{proposition}\label{prop:imply_linear}
Suppose that $n>d$ and the neural network contains an arbitrary number of neurons, i.e., $m\ge 1$. Let $\mfy=\mfX\mfw^*$. 
Suppose that the linear neural isometry condition \eqref{irrep:grelu_skip} holds. Then, the unique optimal solution to \eqref{prob:min_nrm_preview} and \eqref{min_nrm:relu_skip} (up to permutation and splitting\footnote{Please see Appendix \ref{sec:permutation} for a precise definition of the notion of permutation and splitting.}) is given by the planted linear model, i.e., $\hat\mfW=\big\{\hat\mfw_0,\{\hat \mfw_j,\hat \mfw_j'\}_{j=1}^p\big\}$, where $\hat \mfw_{0}=\mfw^*$ and $\hat \mfw_j=\hat \mfw_j'=0$ for $j\in[p]$.
\end{proposition}
The above result implies that minimizing the $\ell_2^2$ objective subject to the interpolation condition uniquely recovers the ground truth model by setting all the neurons except the skip connection to zero, regardless of the number of neurons in the NN. Remarkably, a NN with an arbitrary number of neurons, i.e., containing arbitrarily many parameters, optimizing the criteria \eqref{prob:min_nrm_preview} achieves perfect generalization when the ground truth is the linear model. The exact recovery follows from the equivalence of the problem \eqref{prob:min_nrm_preview} to the group sparsity minimization problem \eqref{min_nrm:relu_skip}, however, this sparsity inducing regularization is hidden in the typical non-convex formulation with weight decay regularization, i.e., $\|\Theta\|_F^2$.

\subsection{Sharp Phase Transition}
Our second main result in this paper is that there exists a sharp phase transition in training ReLU NNs when the data is generated by a random matrix ensemble and the observations are produced by a planted linear model. We precisely identify the relation between the number of samples $n$ and the feature dimension $d$ under i.i.d. training data and planted model assumptions. We first summarize our results in this section informally, and then present detailed theorems in later sections. 

\begin{theorem}[informal]
Suppose that the training data matrix $\mfX\in\mbR^{n\times d}$ is i.i.d. Gaussian, and $f(\mfX;\Theta)$ is a two-layer ReLU network containing arbitrarily many neurons with skip connection. Assume that the response is a noiseless linear model $\mfy=\mfX\mfw^*$. 
The condition $n>2d$ is sufficient for ReLU networks with skip connections or normalization layers to recover the planted model exactly with high probability. Furthermore, when $n<2d$, the recovery fails with high probability.
\end{theorem}

Interestingly, in the regime where $n\in(d,2d)$, fitting a simple linear model instead of a NN recovers the planted linear model. In contrast, the ReLU network fails to recover the linear model due to the richness of the model class. We formalize this observation as follows:
\begin{corollary}
Suppose that the observations are given by $\mfy=\mfX\mfw^*$, where $\mfw^*\in\mbR^d$ is a fixed, unknown parameter. If $n\in(d,2d)$, a simple linear model $f^\mathrm{lin}(\mfX;\mfw)=\mfX\mfw$ recovers the planted linear neuron $\mfw^*$ exactly for any $\mfw^*$, while the ReLU network with skip connection and $m\ge 1$ neurons fitted via either \eqref{prob:reg_preview} with any $\beta\geq 0$ or \eqref{prob:min_nrm_preview} fails with high probability for all $\mfw^*$ such that $\mfX\mfw^*\ge 0$.
\end{corollary}
This corollary is illustrated in Figure \ref{fig:tikz} as a phase diagram in the $(n,d)$ plane. This result clearly shows that the model complexity of ReLU networks can hurt generalization compared to simpler linear models when the number of samples is limited, but information-theoretically sufficient for exact recovery. On the other hand, ReLU networks learn the true model and close the gap in generalization when twice as many samples are available, i.e., $n>2d$.

\begin{figure}[t!]
    \centering
\subfigure{
\begin{tikzpicture}[scale=0.5]
  \message{Phase diagrams}
  
  \def\xtick#1#2{\draw[thick] (#1)++(0,.2) --++ (0,-.4) node[below=-.5pt,scale=0.7] {#2};}
  \def\ytick#1#2{\draw[thick] (#1)++(.2,0) --++ (-.4,0) node[left=-.5pt,scale=0.7] {#2};}
  
  \coordinate (O) at (0,0);
  \coordinate (N1) at (2,10);
  \coordinate (N2) at (11,10);
  \coordinate (NE) at (12,10);
  \coordinate (NW) at (0,10);
  \coordinate (SE) at (12,0);
  \coordinate (W) at (0,5);
  \coordinate (S) at (6,0);
  \coordinate (C) at (10.2,7.8); 
  \coordinate (T) at (4,3); 
  
  \coordinate (NM) at (6,10);

  \def\nd{(O) -- (NM)}
  \def\ntwod{(O) -- (NE)}
    \def\SL{(T) -- (N1)}
  
  \fill[blue!5] \nd -- (NW) -- cycle;
  \fill[mylightred] \ntwod -- (NM) -- (O) -- cycle;
  \fill[mylightblue] \ntwod -- (SE) -- cycle;
  \node at (6.1,8) {linear};
  \node at (5.7,7) {model};
  \node at (5.3,6) {recovers};
  \node at (8.5,4) {ReLU network};
  \node at (8,3) {with skip connection};
    \node at (8,2) {recovers};
  \node at (3,9) {$n<d$};
  \node at (2,7.5) {no recovery};
  \node at (8,9) {$d<n<2d$};
  \node at (10,1) {$n>2d$};
  
  
  \draw[thick] \nd;
  \draw[thick] \ntwod;
  
  \draw[thick] (O) rectangle (NE);
  \node[left=3pt,above,rotate=90] at (W) {$d$};
  \node[below=3pt] at (S) {$n$};
  
 \end{tikzpicture}
}
\subfigure[]{
      \centering
      \includegraphics[width=0.46\textwidth]{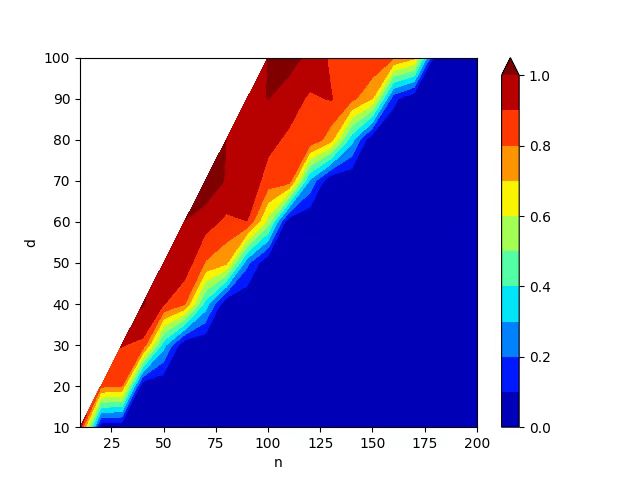}  
    }
    
\caption{Phase transition in recovering a linear neuron. Left: when $n \in (d,2d)$, ReLU network fails to recover a planted linear model, while a simple linear model succeeds in recovery. Right: Empirical generalization error in recovering a linear neuron by solving the convex program \eqref{min_nrm:relu_skip} numerically. }\label{fig:tikz}
\end{figure}
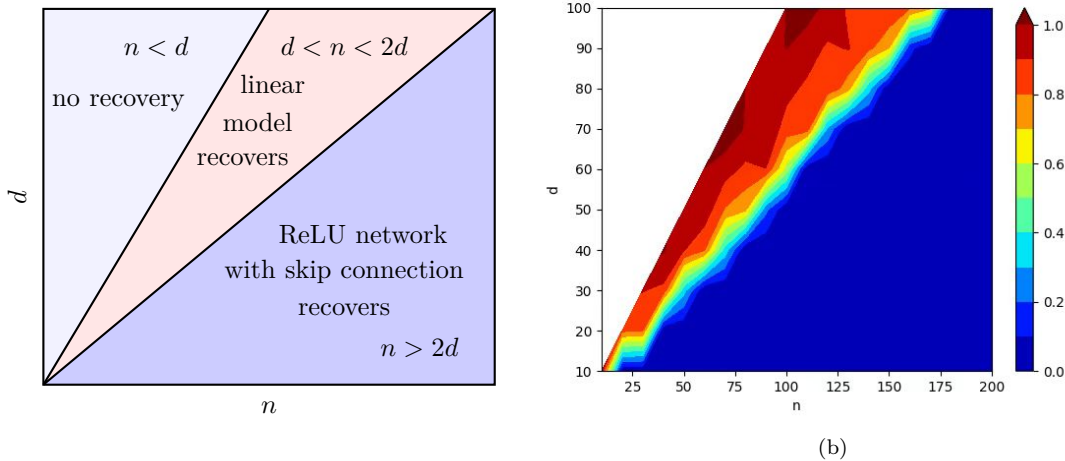

\subsection{Noisy Observations}
We also develop theoretical results when the observation vector $\mfy$ is noisy and the following regularized version of the training problem is solved. We consider the regularized training problem 
\begin{equation}\label{prob:reg_preview}
    \min_{\Theta} \frac{1}{2}\|f(\mfX;\Theta)-\mfy\|_2^2+\frac{\beta}{2}R(\Theta),
\end{equation}
where $R(\Theta)=\|\Theta\|_F^2$ is the weight-decay regularization term.
\begin{theorem}[informal]\label{thm:noisy_inform}
Suppose that the training data matrix $\mfX\in\mbR^{n\times d}$ is i.i.d. sub-Gaussian, and $f(\mfX;\Theta)$ is a two-layer ReLU network with $m$ neurons, skip connection and a normalization layer applied before the ReLU layer. Assume that the observation $\mfy$ is the sum of a linear neuron and an arbitrary disturbance term $\mfz$, i.e., $\mfy=\mfX\mfw^*+\mfz$. Then, there exists a range of values for the regularization parameter $\beta$, such that the non-convex optimization problem in \eqref{prob:reg_preview} and its corresponding convex formulation (see \eqref{reg:normal_before_relu_skip} in Section \ref{sec:relu_skip}) exactly recovers a linear model with high probability when $n$ is sufficiently large for all values of the number of neurons $m\ge 1$.
Additionally, the $\ell_2$ distance between the learned linear neuron $\mfw$ and the planted neuron is bounded by
\[\norm{\mfw -\mfw^*}_2
 \le \mathcal{O}(\beta) + \mathcal{O}(\norm{\mfz}_2).
\]
\end{theorem}
Considering the non-convex optimization problem \eqref{prob:reg} with weight decay, i.e., $\ell_2$-regularization
, with a proper choice of the regularization parameter $\beta$, the optimal NN with arbitrary number of neurons consists of only a linear weight. 
The key to understand this result involves techniques from sparse recovery combined with the convex re-formulation of the regularized training problem \eqref{prob:reg_preview} as group $\ell_1$ regularization. 

The above result shows the bias of two-layer ReLU networks towards simple models even when the relation between the data and labels is not exactly linear. Our proof also shows that skip connection and normalization layers in these two-layer networks are essential, which are the building blocks of the ResNets popularly applied in the practice. The details of Theorem \ref{thm:noisy_inform} can be found in Section \ref{sec:noisy}.

\begin{corollary}
By combining the above proposition with Lemma \ref{thm:landscape_skip}, we note that the global optima of the nonconvex problem \eqref{prob:min_nrm} consists of exactly the planted linear neuron up to permutation and splitting.
\end{corollary}

\section{Convex Formulations of ReLU NNs with Skip Connection and Normalization}\label{sec:neural_recovery}
In this section, we introduce important variants of the simple ReLU architecture and their corresponding convex reformulations. As it will be shown next, architectural choices of these models, such as an addition of a normalization layer play a significant role in their recovery properties.   

Suppose that $\mfX\in \mbR^{n\times d}$ is the training data matrix and $\mfy\in \mbR^d$ is the label vector. We will focus on the following two-layer ReLU NNs:
\begin{itemize}
  \setlength\itemsep{0.5em}
    \item plain ReLU networks:
    \begin{equation*}
    f^\mathrm{ReLU}(\mfX;\Theta) =(\mfX\mfW_1)_+\mfw_2 = \sum_{i=1}^m (\mfX\mfw_{1,i})_+w_{2,i}, \quad \Theta = (\mfW_1,\mfw_2),
\end{equation*}
where $\mfW_1\in \mbR^{d\times m}$ and $\mfw_2\in \mbR^{m}$.
\item ReLU networks with skip connections:
\begin{equation*}
    f^\mathrm{ReLU-skip}(\mfX;\Theta) =\mfX\mfw_{1,1}w_{2,1}+\sum_{i=2}^m (\mfX\mfw_{1,i})_+w_{2,i},
\end{equation*}
where $\Theta = (\mfW_1,\mfw_2)$, $\mfW_1\in \mbR^{d\times m}$ and $\mfw_2\in \mbR^{m}$.
\item ReLU networks with normalization layers:
\begin{equation*}
    f^\mathrm{ReLU-norm}(\mfX;\Theta) =\sum_{i=1}^m \normal_{\alpha_i}((\mfX\mfw_{1,i})_+)w_{2,i},
\end{equation*}
where $\Theta = (\mfW_1,\mfw_2,\balpha)$ and the normalization operator $\normal_\alpha(\mfv)$ is defined by
\begin{equation*}
    \normal_{\alpha}(\mfv) = \frac{\mfv}{\|\mfv\|_2}\alpha, \mfv\in\mbR^n,\alpha\in \mbR.
\end{equation*}
\end{itemize}
We note that the normalization layer employed in the model $f^\mathrm{ReLU-norm}$ is a variant of the well-known batch normalization (BN) with a trainable normalization-correction variable. However, convex formulations of ReLU NNs with exact Batch Normalization layers can be derived as shown in \citep{ergen2021demystifying}. 
We now focus on the regularized training problem 
\begin{equation}\label{prob:reg}
    \min_{\Theta} \frac{1}{2}\|f(\mfX;\Theta)-\mfy\|_2^2+\frac{\beta}{2}R(\Theta).
\end{equation}
For plain ReLU networks and ReLU networks with skip connections, we consider weight decay regularization. This is given by $R(\Theta)=\|\Theta\|_F^2=(\|\mfW_1\|_F^2+\|\mfw_2\|_2^2)$, while for ReLU networks with normalization layers, we have $R(\Theta)=(\|\mfW_1\|_F^2+\|\mfw_2\|_2^2+\|\balpha\|_2^2)$. When $\beta\to 0$, the optimal solution of the above problem approaches the following minimum norm interpolation problem:
\begin{equation}\label{prob:min_nrm}
    \min_{\Theta} R(\Theta) \,\,\text{ s.t. } f(\mfX;\Theta)=\mfy.
\end{equation}
\subsection{Convex formulations for plain ReLU NNs}
According to the convex optimization formulation of two-layer ReLU networks in \citep{nnacr}, the minimum norm problem \eqref{prob:min_nrm} of plain ReLU networks, i.e., the model $f^\mathrm{ReLU}(\mfX;\Theta)$, is equivalent to a convex program:{}
\begin{equation}\label{min_nrm:relu}
\begin{aligned} 
    \min_{\left\{\mfw_{j}, \mfw_{j}'\right\}_{j=1}^{p}}\quad &  \sum_{j=1}^{p}\left(\left\|\mfw_{j}\right\|_{2}+\left\|\mfw_{j}'\right\|_{2}\right)\\
        \text{s.t.}\quad &\sum_{j=1}^{p} \mfD_{j} \mfX\left(\mfw_{j}-\mfw_{j}'\right)=\mfy,\\
        & (2\mfD_j-\mfI_n)\mfX\mfw_j\ge0, (2\mfD_j-\mfI_n)\mfX\mfw_j'\ge0, j\in[p].
\end{aligned}
\end{equation}
Here the matrices $\mfD_1,\dots,\mfD_p$ are diagonal arrangement patterns as defined in Definition \ref{def:hyper}.
Compared to the formulation in \citep{nnacr}, we exclude the hyperplane arrangement induced by the zero vector without loss of generality as justified in Appendix \ref{app:hyper}.
Analogous to the results in \citep{nnacr,wang2020hidden}, the global optimal set of the non-convex program \eqref{prob:min_nrm} can be characterized by the optimal solutions of the convex program \eqref{min_nrm:relu}.
\begin{theorem}\label{thm:landscape}
All globally optimal solutions of the non-convex problem \eqref{prob:min_nrm} of ReLU networks can be computed via the optimal solutions of the convex program \eqref{min_nrm:relu} up to splitting and permutation as soon as the number of neurons $m$ exceeds a critical threshold $m^*$. 
\end{theorem}

The inequality constraints in \eqref{min_nrm:relu} render the analysis of the uniqueness of the optimal solution directly via \eqref{min_nrm:relu} difficult. However, we show that we can instead work with a relaxation without any loss of generality. By dropping all inequality constraints, we obtain a relaxation of the convex program \eqref{min_nrm:relu} which reduces to the following group $\ell_1$-minimization problem:
\begin{equation}\label{min_nrm:grelu}
\begin{aligned}
    \min_{\left\{\mfw_{j}\right\}_{j=1}^{p}}\quad &  \sum_{j=1}^{p}\left\|\mfw_{j}\right\|_{2}~~
        \text{s.t.}\quad &\sum_{j=1}^{p} \mfD_{j} \mfX \mfw_{j}=\mfy.\\
\end{aligned}
\end{equation}
The above form corresponds to the well-known Group Lasso model \cite{yuan2006model} studied in high-dimensional variable selection and compressed sensing. 
We note that a certain unique solution to the relaxation in \eqref{min_nrm:grelu} satisfying the constraint in the original problem \eqref{min_nrm:relu} implies that the original problem \eqref{min_nrm:relu} also has the same unique solution.

Another interesting observation is that the above group $\ell_1$-minimization problem corresponds to the minimum norm interpolation problem with gated ReLU activation, for which the NN model is given by
\begin{equation}\label{equ:grelu}
    f^\mathrm{gReLU}(\mfX;\Theta) = \sum_{i=1}^m\diag(\mbI(\mfX\mfh_i\geq 0))(\mfX\mfw_{1,i})w_{2,i},
\end{equation}
where $\Theta=(\mfW_1,\mfw_2,\mfH)$. This derivation is provided in Appendix \ref{app:grelu}.



\subsection{Convex formulations for ReLU networks with skip connection}
The minimal problem \eqref{prob:min_nrm} of ReLU networks with skip connection, i.e., the model $f^\mathrm{ReLU-skip}(\mfX;\Theta)$, is equivalent to a convex program: 
\begin{equation}
\begin{aligned}\label{min_nrm:relu_skip_convex}
    \min_{\mfw_0,\left(\mfw_{j}, \mfw_{j}'\right)_{j=1}^{p}}\quad &  \sum_{j=1}^{p}\left(\left\|\mfw_{j}\right\|_{2}+\left\|\mfw_{j}'\right\|_{2}\right)\\
        \text{s.t.}\quad &\mfX\mfw_0+\sum_{j=1}^{p} \mfD_{j} \mfX\left(\mfw_{j}-\mfw_{j}'\right)=\mfy,\\
        & (2\mfD_j-\mfI_n)\mfX\mfw_j\ge0, (2\mfD_j-\mfI_n)\mfX\mfw_j'\ge0, j\in[p].
\end{aligned}
\end{equation}
By dropping all inequality constraints, the convex program \eqref{min_nrm:relu_skip} reduces to the following group $\ell_1$-minimization  problem:
\begin{equation}\label{min_nrm:grelu_skip}
\begin{aligned}
    \min_{\left\{\mfw_{j}\right\}_{j=0}^{p}}\quad &  \sum_{j=0}^{p}\left\|\mfw_{j}\right\|_{2}
        \quad\text{s.t.}\quad & \mfX\mfw_0 + \sum_{j=0}^{p} \mfD_{j} \mfX \mfw_{j}=\mfy.\\
\end{aligned}
\end{equation}
Similarly, this group $\ell_1$-minimization  problem is equivalent to the minimum norm interpolation problem of gated ReLU networks with skip connection.

\subsection{Convex formulations for ReLU networks with normalization layer}\label{sec:relu_normal}
The minimum norm problem \eqref{prob:min_nrm} of ReLU networks with normalization layer, i.e., the model $f^\mathrm{ReLU-norm}(\mfX;\Theta)$. is equivalent to the following convex program:
\begin{equation}\label{min_nrm:relu_normal}
\begin{aligned}
    \min_{\left\{\mfw_{j}, \mfw_{j}'\right\}_{j=1}^{p}}\quad &  \sum_{j=1}^{p}\left(\left\|\mfw_{j}\right\|_{2}+\left\|\mfw_{j}'\right\|_{2}\right)\\
        \text{s.t.}\quad &\sum_{j=1}^{p} \mfU_j\left(\mfw_{j}-\mfw_{j}'\right)=\mfy,\\
        & (2\mfD_j-\mfI_n)\mfX\mfV_j^T\bSigma_j^{-1}\mfw_j\ge0, (2\mfD_j-\mfI_n)\mfX\mfV_j^T\bSigma_j^{-1}\mfw_j'\ge0, j\in[p],
\end{aligned}
\end{equation}
where $\mfD_j\mfX=\mfU_j\bSigma_j\mfV_j$ is the compact singular value decomposition (SVD) of $\mfD_j\mfX$ for $j\in[p]$. Again, the global optima of the non-convex program \eqref{prob:min_nrm} can be characterized by the optimal solutions of the convex program \eqref{min_nrm:relu_normal}.
\begin{theorem}\label{thm:landscape_normal}
Suppose that $n>d$. 
All globally optimal solutions of the non-convex problem \eqref{prob:min_nrm} of ReLU networks with normalization layer can be found by the optimal solutions of the convex program \eqref{min_nrm:relu_normal} (up to splitting and permutation) when $m$ is greater than some threshold $m^*$. 
\end{theorem} By dropping all inequality constraints, the convex program \eqref{min_nrm:relu_normal} reduces to a group $\ell_1$-minimization  problem: 
\begin{equation}\label{min_nrm:grelu_normal}
\begin{aligned}
    \min_{\left\{\mfw_{j}\right\}_{j=1}^{p}}\quad &  \sum_{j=1}^{p}\left\|\mfw_{j}\right\|_{2},
        \text{s.t.}\quad &\sum_{j=1}^{p} \mfU_j \mfw_{j}=\mfy.\\
\end{aligned}
\end{equation}
Analogously, the above group $\ell_1$-minimization  problem corresponds to the minimum norm interpolation problem of gated ReLU network with normalization layer.

\section{Neural isometry conditions and recovery of nonlinear neurons}\label{sec:neu_iso}

In this section, we introduce conditions on the training data, called neural isometry, that guarantee the recovery of planted models via solving the non-convex problem \eqref{prob:min_nrm} or the convex problems (\ref{min_nrm:grelu},~\ref{min_nrm:grelu_skip},~\ref{min_nrm:grelu_normal}). 

\subsection{Recovery of a single ReLU neuron using plain ReLU NNs}
In this section we present recovery results on ReLU networks. Suppose that $\mfy=(\mfX\mfw^*)_+$ is the output of a planted ReLU neuron\footnote{Here we ignore the second layer weight without loss of generality. Positive weights can be absorbed into the first layer weight, while for negative weights we can consider to use $-y$ as the label.}, where $\mfw^*\neq 0$.  Let $\mfD_{i^*}=\diag(\mbI(\mfX\mfw^*\geq 0))$ be the diagonal arrangement pattern corresponding to the planted neuron. 
Based on these diagonal arrangement patterns, we introduce a regularity condition on the data which is analogous to the irrepresentability condition of sparse recovery, called the \emph{Neural Isometry Condition} (NIC). For the recovery of a single neuron, the NIC is defined as follows.
\begin{definition}[Neural Isometry Condition for a single ReLU neuron] A sufficient condition for recovering a single neuron $\mfy=(\mfX\mfw^*)_+$ via the problem \eqref{min_nrm:grelu} is given by:
\begin{equation}\label{irrep:grelu}
\norm{\mfX^T\mfD_{j}\mfD_{i^*}\mfX\pp{\mfX^T\mfD_{i^*}\mfX}^{-1}\hat \mfw}_2<1, \forall j\neq i^*, \tag{NIC-1}
\end{equation}
where $\hat \mfw = \frac{\mfw^*}{\|\mfw^*\|_2}$.
\end{definition}
We assume that the matrix $\mfX^T\mfD_{i^*}\mfX = \sum_{k:\, \mfx_k^T\mfw^*\geq 0 }\mfx_k\mfx_k^T$ is invertible whenever the above condition holds.
The condition above can be equivalently stated as
\begin{equation}
    \Bigg\|\sum_{{k: \mfx_k^T\mfw^*\geq 0,\, \mfx_k^T\mfw\geq 0}}\mfx_k\mfx_k^T\Big(\sum_{k:\, \mfx_k^T\mfw^*\geq 0 }\mfx_k\mfx_k^T\Big)^{-1}\hat \mfw\Bigg\|_2<1, \forall \mfw: \mbI(\mfX\mfw\geq0)\neq \mbI(\mfX\mfw^*\geq0).
\end{equation}
In the following proposition, we show that the neural isometry condition \eqref{irrep:grelu} implies the recovery of the planted model by solving the convex reformulation \eqref{min_nrm:relu}, or equivalently the non-convex problem \eqref{prob:min_nrm}.

\begin{proposition}\label{prop:imply}
Suppose that $n\geq d$. Let $\mfy=(\mfX\mfw^*)_+$. 
Suppose that \eqref{irrep:grelu} holds. Then, the unique optimal solution to \eqref{min_nrm:relu} is given by the planted ReLU neuron $w^*$, i.e, $\hat\mfW=\{(\hat \mfw_j,\hat \mfw_j')|j\in[p]\}$, where $\hat \mfw_{i^*}=\mfw^*$, $\hat \mfw_{i^*}'=0$ and $\hat \mfw_j=\hat \mfw_j'=0$ for $j\neq i^*$.
\end{proposition}

\begin{corollary}
By combining the above proposition with Theorem \ref{thm:landscape}, we note that all global optima of the nonconvex problem \eqref{prob:min_nrm} consist of permuted and split versions of the planted ReLU neuron when the condition \eqref{irrep:grelu} holds.
\end{corollary}

\subsection{Recovery of a single ReLU neuron using ReLU networks with normalization layer}
We now consider the case where $\mfy=\frac{(\mfX\mfw^*)_+}{\|(\mfX\mfw^*)_+\|_2}$, which is the output of a single-neuron ReLU network followed by a normalization layer. Let $\mfD_{i^*}=\diag(\mbI(\mfX\mfw^*\geq0))$ and denote $\tilde \mfw^*=\frac{\bSigma_{i^*}\mfV_{i^*}\mfw^*}{\|\bSigma_{i^*}\mfV_{i^*}\mfw^*\|_2}$. Here,  $\mfU_{i^*}, \bSigma_{i^*},\mfV_{i^*}$ are the SVD factors of $\mfD_{i^*}\mfX$ as defined in subsection \ref{sec:relu_normal}. We note the simplified expression $$
\mfU_{i^*}\tilde \mfw^*=\frac{\mfU_{i^*}\bSigma_{i^*}\mfV_{i^*}\mfw^*}{\|\bSigma_{i^*}\mfV_{i^*}\mfw^*\|_2}
=\frac{(\mfX\mfw^*)_+}{\|(\mfX\mfw^*)_+\|_2}=\mfy.
$$
We introduce the following normalized Neural Isometry Condition.
\begin{definition}[Normalized Neural Isometry Condition]
The normalized neural isometry condition for recovering the planted model $\mfy=\frac{(\mfX\mfw^*)_+}{\|(\mfX\mfw^*)_+\|_2}$ from \eqref{min_nrm:grelu_normal} is given by:
\begin{equation}\label{irrep:grelu_normal}
    \norm{\mfU_j^T\mfU_{i^*}\tilde \mfw^*}_2<1, \forall j\in[p], j\neq i^*.
    \tag{NNIC-1}
\end{equation}
\end{definition}
Similarly, the normalized neural isometry condition implies the the recovery of the planted model via solving the convex formulation for ReLU NNs with normalization layer given in \eqref{min_nrm:relu_normal} or the corresponding non-convex problem \eqref{prob:min_nrm}. 
\begin{proposition}\label{prop:imply_normal}
Let $\mfy=\frac{(\mfX\mfw^*)_+}{\|(\mfX\mfw^*)_+\|_2}$.  
Suppose that the $\mathrm{NNIC} \mydash 1$ given in \eqref{irrep:grelu_normal} holds.
Then, the unique optimal solution to \eqref{min_nrm:relu_normal} is given by the planted normalized ReLU neuron, i.,e., $\hat\mfW=(\hat \mfw_j,\hat \mfw_j')_{j=1}^p$, where $\hat \mfw_{i^*}=\frac{\bSigma_{i^*}\mfV_{i^*}\mfw^*}{\|\bSigma_{i^*}\mfV_{i^*}\mfw^*\|_2}$, $\mfw_{i^*}'=0$ and $\hat \mfw_j=\hat \mfw_j'=0$ for $j\neq i^*$.
\end{proposition}

Similarly, by combining the above proposition with Theorem \ref{thm:landscape_normal}, we note that all global optima of the nonconvex problem \eqref{prob:min_nrm} consist of split and permuted versions of the planted neuron.

\begin{remark}
Our analysis reveals that normalization layers play a key role in the recovery  conditions. Note that in \eqref{irrep:grelu}, the matrices $\{\mfD_{1}\mfX, \hdots, \mfD_{p}\mfX\}$ are replaced by their whitened versions, effectively cancelling the matrix inverse $\big( {\sum_{k:\, \mfx_k^T\mfw^*\geq 0 }\mfx_k\mfx_k^T} \big)^{-1}$ in \eqref{irrep:grelu_normal}. Therefore, the conditioning of the matrices are improved by the addition of a normalization layer, which applies implicit whitening to the data blocks $\{\mfD_{1}\mfX, \hdots, \mfD_{p}\mfX\}$. As a result, it can be deduced that normalization layers help NNs learn simple models more efficiently from the data.
\end{remark}
\subsection{Recovering Multiple ReLU neurons using plain ReLU NNs}
We now extend the Neural Isometry Condition to the recovery of $k>1$ ReLU neurons, starting with plain ReLU NNs. Suppose that the label vector is given by $$\mfy=\sum_{i=1}^k (\mfX\mfw^*_i)_+r_i^*,$$ where $\mfw^*_i\in\mbR^d$, $r_i^*\in\{-1,+1\}$ for $i\in[k]$, and $\diag(\mbI(\mfX\mfw_i^*\geq 0))$ are distinct for $i\in[k]$. Suppose that an enumeration of the diagonal arrangement patterns corresponding to the planted neurons are given by $\mfD_1,...,\mfD_p$. We denote $\mfD_{s_i}=\diag(\mbI(\mfX\mfw_i^*\geq 0))$ for $i\in[k]$, where $S=\{s_1,\dots,s_k\}\subseteq [p]$ contain the indices of planted neurons in the enumeration of arrangement patterns $\{1,...,p\}$ according to any fixed order. 
\begin{definition}[Neural Isometry Condition for $k$ neurons]
For recovering $k$ ReLU neurons from the observations $\mfy=\sum_{i=1}^k (\mfX\mfw^*_i)_+r_i^*$ via the optimization problem \eqref{min_nrm:relu_normal}, we introduce the multi-neuron Neural Isometry Condition:
\begin{equation}\label{eq:irr_cond}
    \norm{\mfX^T\mfD_j
    \bmbm{\mfX^T\mfD_{s_1}\\\vdots\\\mfX^T\mfD_{s_k}}
    ^{\dagger}\bmbm{ \hat \mfw_1 \\\vdots\\ \hat \mfw_k}}_2<1, \forall j\in[p], j\notin S,
    \tag{NIC-k}
\end{equation}
where $\hat \mfw_i := r_i^*\mfw_{i}^*/\|\mfw_{i}^*\|_2\,\forall i \in [k]$. 
\end{definition}
In the following proposition, we show that \eqref{eq:irr_cond} implies the recovery of the planted model with $k$ neurons by solving the non-convex problem \eqref{prob:min_nrm} or its convex reformulation \eqref{min_nrm:relu}. 
\begin{proposition}\label{prop:imply_multi}
Suppose that \eqref{eq:irr_cond} is satisfied. Then, the unique optimal solution to \eqref{min_nrm:relu} is given by the planted neurons, i.e., $\hat\mfW=(\hat \mfw_j,\hat \mfw_j')_{j=1}^p$, where we let $\hat \mfw_{s_i}=\frac{\mfw_i}{\|\mfw_{i}\|_2}$, $\mfw_{s_i}'=0$ for $r_i^*=1$, $\mfw_{s_i}=0$, $\hat \mfw_{s_i}'=\frac{\mfw_i}{\|\mfw_{i}\|_2}$, for $r_i^*=-1$ and $\hat \mfw_j=\hat \mfw_j'=0$ for $j\neq i^*$.
\end{proposition}
As a corollary, by combining the above proposition with Theorem \ref{thm:landscape}, we deduce that all globally optimal solutions of the non-convex problem \eqref{prob:min_nrm} with plain ReLU NNs consist of the planted model up to permutation and splitting.
\subsection{Recovering Multiple ReLU neurons using ReLU NNs with Normalization Layer}
Next, we consider ReLU NNs with normalization layers. Suppose that the label vector takes the form $$\mfy=\sum_{i=1}^k\frac{(\mfX\mfw^*_i)_+}{\|(\mfX\mfw^*_i)_+\|_2}r_i^*,$$ where $\mfw^*_i\in\mbR^d$ are first layer weights, $r_i^* \in \mathbb{R}$ are second layer weights for $i\in[k]$ and $\diag(\mbI(\mfX\mfw_i^*\geq 0))$ are distinct for each $i\in[k]$. For simplicity, denote $\mfD_{s_i}=\diag(\mbI(\mfX\mfw_i^*\geq 0))$ for $i\in[k]$, where $S=\{s_1,\dots,s_k\}\subseteq [p]$. Let $\mfD_i\mfX=\mfU_i\bSigma_i\mfV_i$ be the singular value decomposition for $i\in[p]$.
\begin{definition}[Normalized Neural Isometry Condition for $k$ neurons]
For recovering $k$ normalized ReLU neurons from the observations $\mfy=\sum_{i=1}^k\frac{(\mfX\mfw^*_i)_+}{\|(\mfX\mfw^*_i)_+\|_2}r_i^*$ via the optimization problem \eqref{min_nrm:relu_normal}, the normalized neural isometry condition is given by
\begin{equation}\label{eq:irr_cond_normal}
    \norm{\mfU_j^T\bmbm{\mfU_{s_1}^T\\\vdots\\\mfU_{s_k}^T}^{\dagger}
    \bmbm{\tilde \mfw_1\\\vdots\\\tilde \mfw_k}}_2<1, \forall j\in[p], j\notin S,
    \tag{NNIC-k}
\end{equation}
where $\tilde \mfw_i := r_i^*\bSigma_{s_i}\mfV_{s_i}\mfw_i^*/\|\bSigma_{s_i}\mfV_{s_i}\mfw_i^*\|_2\,\forall i\in[k]$.
\end{definition}
Suppose that $\mfD_{s_1},\dots,\mfD_{s_j}$ further satisfy that $\mfD_{s_i}\mfD_{s_j}=0$ for $i\neq j$. Then, we can simplify the neural isometry condition to the following form
\begin{equation} 
    \norm{\mfU_j^T\sum_{i=1}^k\frac{(\mfX\mfw^*_i)_+}{\|(\mfX\mfw^*_i)_+\|_2}r_i^*}_2 = \norm{\mfU_j^T\mfy}_2<1, \forall j\in[p]/S.
\end{equation}
In the following proposition, we show that \eqref{eq:irr_cond_normal} implies the recovery of the $k$ planted normalized ReLU neurons via the optimization problem \eqref{min_nrm:relu_normal}. 

\begin{proposition}\label{prop:imply_multi_normal}
Suppose that \eqref{eq:irr_cond_normal} is satisfied. Then, the unique optimal solution to \eqref{min_nrm:relu_normal} is given by the planted neurons, i.e., $\hat\mfW=(\hat \mfw_j,\hat \mfw_j')_{j=1}^p$, where $\hat \mfw_{s_i}=\frac{\bSigma_{s_i}\mfV_{s_i}\mfw_i}{\|\bSigma_{s_i}\mfV_{s_i}\mfw_{i}\|_2}$, $\hat\mfw_{s_i}'=0$ for $r_i^*=1$, $\hat\mfw_{s_i}=0$, $\hat \mfw_{s_i}'=\frac{\bSigma_{s_i}\mfV_{s_i}\mfw_i}{\|\bSigma_{s_i}\mfV_{s_i}\mfw_{i}\|_2}$ for $r_i^*=-1$ and $\hat \mfw_j=\hat \mfw_j'=0$ for $j\neq i^*$.
\end{proposition}
\begin{corollary}
By combining the above proposition with Theorem \ref{thm:landscape}, we note that all globally optimal solutions of the non-convex problem \eqref{prob:min_nrm} consist of the split and permuted version of the planted model.
\end{corollary}


\section{Sharp Phase Transitions}\label{sec:relu_skip}
One of our major results in this paper is that there exists a sharp phase transition in the success probability of recovering planted neurons in the $(n,d)$ plane for certain random ensembles for the training data matrix. We start illustrating this phenomenon with the case of recovering linear neurons through an application of the kinematic formula for convex cones.

\begin{figure}[H]
\centering
\begin{minipage}[t]{0.48\textwidth}
\centering
\includegraphics[width=0.7\linewidth]{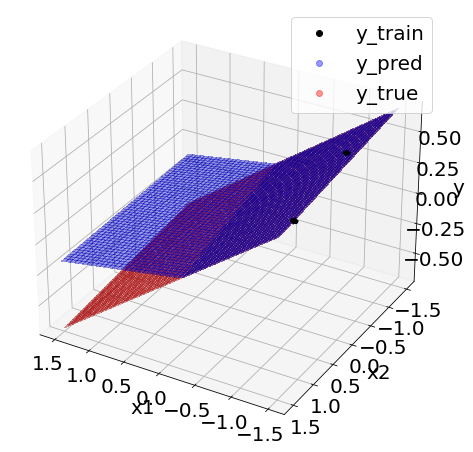}
\end{minipage}
\begin{minipage}[t]{0.48\textwidth}
\centering
\includegraphics[width=0.7\linewidth]{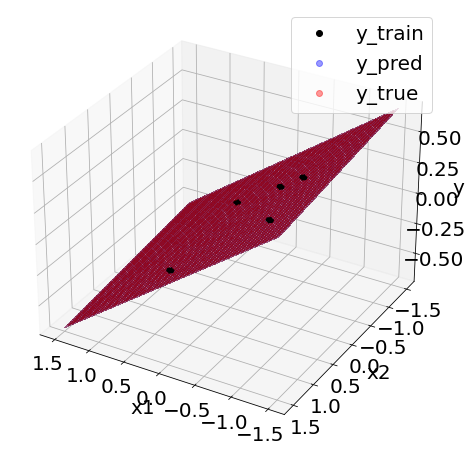}
\end{minipage}
\caption{Optimal ReLU NNs found via the convex program \eqref{min_nrm:relu_skip}. Left: A ReLU neuron is fitted to the observations generated from a linear model when $n=2, d=2$. Right: Only a linear neuron is fitted to the observations generated from a linear model when $n=5,d=2$.}
\end{figure}



\subsection{Kinematic formula}
We introduce an important result from \citep{amelunxen2014living} that will be used in proving the phase transitions in the recovery of planted neurons via ReLU NNs. 
\begin{lemma}\label{lem:kinematic}
For a convex cone $K\subseteq \mbR^n$, define the statistical dimension of the cone $K$ by
\begin{equation}
    \delta (K)=\mbE\bb{\|\Pi_K(\mfg)\|_2^2},
\end{equation} 
where $\mfg$ is a standard Gaussian random vector and $\Pi_K$ is the Euclidean projection onto the cone $K$. Define $\alpha:=\frac{(n-\delta (K)-d)^2}{64 n^2}$. Then, we have
\begin{equation}\label{equ:kine}
    P(\exists\mfw\neq0: \mfX\mfw\in K)\begin{cases}
    \begin{aligned}
            &\leq 4e^{-n\alpha}, &\delta (K)+d<n,\\
        &\geq 1-4e^{-n\alpha},&\delta (K)+d>n.
    \end{aligned}
    \end{cases}
\end{equation}
\end{lemma}

Consider the positive orthant $K=\mbR^n_+$, which is a convex cone. It can be easily computed that
\begin{equation}
\delta(K)=\mbE[\|\pi_K(\mfg)\|^2]=\mbE\bb{\sum_{i=1}^n \max\{0,g_i\}^2}=\frac{n}{2}.   
\end{equation}
A direct corollary of the kinematic formula is as follows.
\begin{proposition} \label{prop:kinematic}
Let $\mfX\in \mbR^{n \times d}$ be a matrix whose entries are i.i.d. random variables following $\mcN(0,1)$. Then, we have
\begin{equation*}
    P(\exists \mfh\neq 0: \mfX\mfh\geq 0)\begin{cases}
    \begin{aligned}
            &\leq 4e^{-n\alpha}, &d<n/2,\\
        &\geq 1-4e^{-n\alpha},&d>n/2,
    \end{aligned}
    \end{cases}
\end{equation*}
where $\alpha=\frac{(n/2-d)^2}{64 n^2}$.
\end{proposition}
This implies that for $n>2d$, the probability that the set of diagonal arrangement patterns contain the identity matrix, i.e., a vector $\mfh\neq 0$ such that $\mbI(\mfX\mfh\geq 0)=\bone$, is close to $1$, while for $n<2d$, this probability is close to $0$. 
\subsection{Recovering Linear Models}
\subsubsection{Gaussian random data}
For Gaussian random data, our next result shows that the ratio $n/d$ controls whether the Linear Neural Isometry Condition for linear neuron recovery holds or fails with high probability. Therefore, neural networks do not overfit when the number of samples is above a critical threshold regardless of the number of neurons.
\begin{theorem}[Phase transition in Gaussian Data: Success]\label{prop:skip_gauss_success}
Suppose that each entry $x_{i,j}$ of the data matrix $\mfX\in \mbR^{n\times d}$ is an i.i.d. random variable following the Gaussian distribution $\mcN(0,1/n)$. For $n>2d$, $\mathrm{NIC}\mydash\mathrm{L}$ shown in \eqref{irrep:grelu_skip} holds with probability at least $1-\exp(-\alpha n)$ where $\alpha=\frac{(n/2-d)^2}{64n^2}$. Consequently, the unique optimal ReLU NN with skip connection found via the convex reformulation \eqref{min_nrm:relu_skip_convex} uniquely recovers planted linear models up to permutation and splitting.
\end{theorem}
\begin{remark}
This implies that when the planted model is $\mfy=\mfX\mfw^*$, the linear term $\mfw_0=\mfw^*$ is the unique optimal solution to the convex program 
\eqref{min_nrm:relu_skip}. It is also the unique globally optimal solution to the non-convex problem \eqref{prob:min_nrm}. 
\end{remark}
We next show that when the ratio $n/d$ is below a certain critical threshold, $\mfW=(\mfw^*,0,\dots,0)$ is not the unique optimal solution to the convex problem \eqref{min_nrm:relu_skip_convex} (or equivalently to the corresponding non-convex problem \eqref{prob:min_nrm}). Consequently, neural networks overfit when the number of samples is below this critical threshold.
\begin{theorem}[Phase transition in Gaussian Data: Failure]\label{prop:negative_gaussian}
Suppose that $n<2d$. Assume that each entry of $\mfX$ is an i.i.d. random variable following the Gaussian distribution $\mcN(0,1/n)$. Suppose that the planted parameter $\mfw^*$ satisfies $\mfX\mfw^*\geq 0$. Then, with probability $1$, $\mfW=(\mfw^*,0,\dots,0)$ is not the unique optimal solution to the minimum norm interpolation problem of ReLU NNs with skip connection \eqref{prob:min_nrm} or its convex reformulation \eqref{min_nrm:relu_skip_convex}.
\end{theorem}
\begin{proof}
When $n<2d$, according to the kinematic formula, we have
\begin{equation}
    P(\exists \mfh\in\mbR^d:\mfX\mfh\geq 0,\mfh\neq 0)\geq 1-e^{-n\alpha},
\end{equation}
where $\alpha$ is as defined in Proposition \ref{prop:kinematic}.
Suppose that there exists $\mfh\in \mbR^d$ such that $\mfh\neq 0$, $\mfX\mfh\geq 0$.  This implies that the identity matrix is among the diagonal arrangement patterns, i.e., there exists $j\in[p]$ such that $\mfD_j=\mfI_n$. 
Assume that the planted neuron $\mfw^*$ further satisfies that $\mfX\mfw^*\geq 0$.
In this case, let $\tilde \mfw_i=\mfw_i'=0$ if $i\neq j$, $\tilde \mfw_j=\mfw^*$ and $\tilde \mfw_j'=0$. Then, $\mfW'=(0,\tilde \mfw_1,\tilde \mfw_1',\dots,\tilde \mfw_p,\tilde \mfw_p')$ is also a feasible solution to \eqref{min_nrm:grelu_skip}. This implies that for $n<2d$, with probability close to $1$, there exists an optimal solution which consist of at least one non-zero ReLU neuron. $\IEEEQEDhere$
\end{proof}

However, if the condition $\mfX\mfw^*\geq 0$ is not satisfied, we next show that $\mfW=(\mfw^*,0,\dots,0)$ is the unique optimal solution to the minimum norm interpolation problem \eqref{prob:min_nrm} even when $\mfI_n\in H$.
\begin{proposition}\label{prop:positive_nd}
Suppose that $d<n<2d$. Assume that each entry of $\mfX$ is an i.i.d. random variable following the Gaussian distribution $\mcN(0,1/n)$. Suppose that $\mfX\mfw^*\geq 0$ does not hold. Then, with probability $1$, $\mfW=(\mfw^*,0,\dots,0)$ is the unique optimal solution to the minimum norm interpolation problem \eqref{prob:min_nrm} for ReLU NNs with skip connection, or the corresponding convex reformulation \eqref{min_nrm:relu_skip_convex}.
\end{proposition}
In Figure \ref{fig_main:lin_phase_w0}, we numerically verify the phase transition by solving the convex program \eqref{min_nrm:relu_skip_convex} and its relaxation given by the group $\ell_1$-minimization problem \eqref{min_nrm:grelu_skip}, which drops the linear inequality constraints in \eqref{min_nrm:relu_skip}. 
\begin{figure}[H]
\setcounter{subfigure}{0}
    \subfigure{
      \centering
      \includegraphics[width=.45\textwidth]{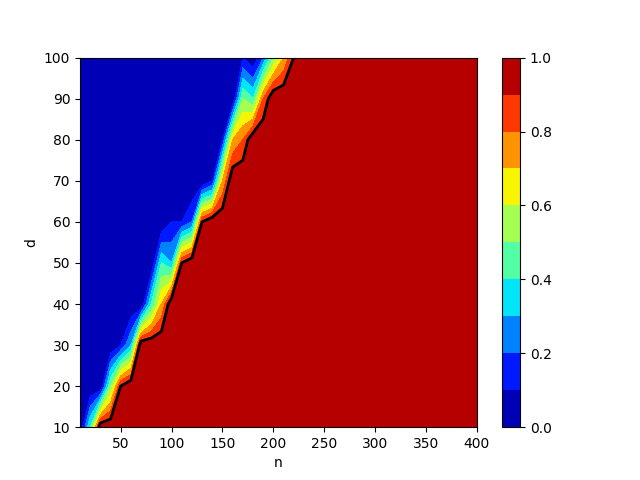}  
    }
    \subfigure{
      \centering
      \includegraphics[width=.45\textwidth]{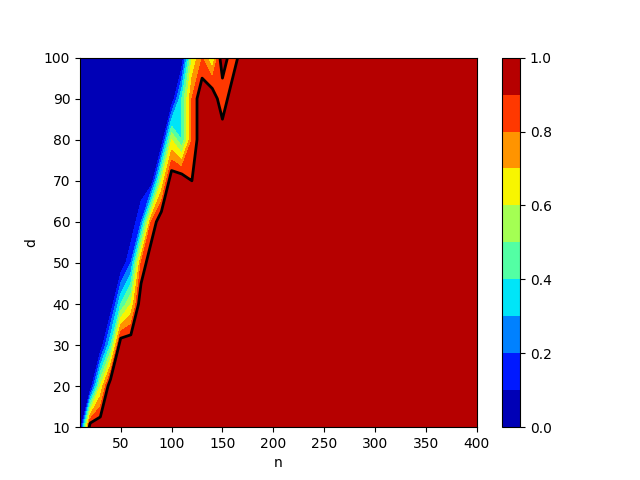}  
    }
    \caption{The empirical probability of successful recovery of the planted linear neuron estimated over $5$ independent trials. Left panel: solving the group $\ell_1$-minimization problem \eqref{min_nrm:grelu_skip} as a relaxation of \eqref{min_nrm:relu_skip}. Right panel: solving the convex NN program \eqref{min_nrm:relu_skip}. Red (blue) region shows the region where exact recovery probability is close to one (zero).   The black lines represent the boundaries of successful recovery with probability $1$.}\label{fig_main:lin_phase_w0}
\end{figure}

\subsubsection{Haar random data}

We now investigate the case where the training data $\mfX\in\mbR^{n\times d}$ is a Haar distributed random matrix. More precisely, $\mfX$ is uniformly sampled from the set of column orthogonal matrices $\big \{\mfX\in\mbR^{n\times d}\,|\, \mfX^T\mfX=\mfI_d\big\}$ for $n\geq d$. 
In this case, since $(\mfX^{T}\mfX)^{-1} = \mfI_d$, \eqref{irrep:grelu_skip} for the recovery of a linear neuron reduces to
\begin{equation}\label{equ:event_hu}
    \hspace{1cm} \max_{\mfh\in\mbR^d,\,\mfh\neq 0}\norm{\mfX^T\diag(\mbI(\mfX\mfh\geq 0))\mfX\hat\mfw}_2<1,
    \tag{orth-NIC-L}
\end{equation}
where $\hat \mfw :=\frac{\mfw^*}{\|\mfw^*\|_2}$.


Based on the simpler form of \eqref{equ:event_hu}, we derive the phase transition results on the recovery of a planted linear neuron. Moreover, this form enables the analysis of a stronger recovery condition that ensures the simultaneous recovery of all planted linear neurons.

\begin{theorem}[Phase transition in Haar data]\label{thm:relu_skip_phase}
Suppose that the planted neuron $\mfw^*\in \mbR^d$, $\mfw^*\neq 0$ is fixed and $\mfX\in\mbR^{n\times d}$ is a Haar distributed random matrix. Denote $\mfW^*=(\mfw^*,0,\dots,0)$. For $n>2d$, with probability at least $1-\exp(-\alpha n)$, $\mfW^*$ is the unique optimal solution to \eqref{min_nrm:relu_skip_convex}, where $\alpha=\frac{(n/2-d)^2}{64n^2}$. For $n<2d$, with probability at least $1-\exp(-\alpha n)$, $\mfW^*$ is not the unique optimal solution to \eqref{min_nrm:relu_skip_convex}. 
\end{theorem}


Similar to our previous analysis for the Gaussian distribution, if $\mfX\mfw^*\geq 0$, then $\mfW^*=(\mfw^*,0,\dots,0)$ is not the unique optimal solution to \eqref{min_nrm:relu_skip}. However, if $\mfX\mfw^*\geq 0$ is not satisfied, for $n>d$, $\mfW^*$ can still be the unique optimal solution of \eqref{min_nrm:relu_skip}. 


\subsubsection{Relation between Haar data and normalization layers}
Consider the minimum norm interpolation problem of two-layer ReLU networks with skip connection and normalization layer (before ReLU)
\begin{equation}
\begin{aligned}
    \min_{\mfW_1,\mfw_2,\balpha}&\pp{\|\mfW_1\|_2^2+\|\mfw_2\|_2^2+\|\balpha\|_2^2}, \\
    \text{ s.t. } \quad &\mfy=\normal_{\alpha_1}(\mfX\mfw_{1,1})w_{2,1}+\sum_{i=2}^m\pp{\normal_{\alpha_i}(\mfX\mfw_{1,i})}_+w_{2,i}.
\end{aligned}
\end{equation}
The above problem can be reformulated as a convex program \cite{ergen2021demystifying}:
\begin{equation}
\begin{aligned}
    \min_{\mfw_0,\{\mfw_j,\mfw_j'\}_{j=1}^p}\;&\pp{\|\mfw_0\|_2+\sum_{j=1}^p(\|\mfw_j\|_2+\|\mfw_j'\|_2)}, \\
\text{ s.t. } &\mfU\mfw_0+\sum_{j=1}^p \mfD_j\mfU( \mfw_j-\mfw_j
')=\mfy, \\
&(2\mfD_j-\mfI_n)\mfU\mfw_j\geq 0, (2\mfD_j-\mfI_n)\mfU\mfw_j'\geq 0, j\in[p],
\end{aligned}
\end{equation}
where the matrix $\mfU$ is computed from the compact SVD of $\mfX=\mfU\bSigma\mfV^T$. Therefore, the training problem with Gaussian random data with an additional normalization layer before the ReLU activation is equivalent to the training problem with Haar data with no normalization layers.
\subsubsection{Sub-Gaussian random data}
Our previous results on Gaussian and Haar data can be extended to a broad class of random matrices with independent entries.
In the following result, we show that for i.i.d., sub-Gaussian data distributions, ReLU neural networks with a skip connection recovers a planted linear model exactly with sufficiently many samples.
\begin{theorem}\label{thm:gauss_success}
Suppose that each entry $x_{i,j}$ of the data matrix $\mfX\in \mbR^{n\times d}$ is an i.i.d. symmetric random variable following a mean-zero sub-Gaussian distribution with variance proxy $\sigma^2$ and $\mbE[x_{i,j}^2]=\frac{1}{n}$.
Then, \eqref{irrep:grelu_skip} holds when $n\geq  C_1 d\log(n)$ with probability at least $1-4\exp(-C_2n)-2\exp(d)$ for sufficiently large $n$, where $C_1,C_2$ are absolute constants depending only on $\sigma^2$. 
\end{theorem}
We note that there exists an additional logarithmic term in the scaling $n\ge C_1 d \log (n)$ provided in Theorem \ref{thm:gauss_success}, compared to our results on Gaussian and Haar data matrices.
\subsection{Uniform (Strong) Recovery of All Linear Models}
Thus far, we considered the recovery of a fixed planted neuron and the associated exact recovery probability. To ensure the recovery of all possible planted neurons $\mfw^*$ simultaneously, we introduce the following stronger form of $\mathrm{NIC}\mydash\mathrm{L}$ specialized to column orthogonal matrices by requiring  \eqref{equ:event_hu} to hold for every $\mfw^*$:
\begin{equation}\label{irrep:white_strong}
    \max_{\mfh\in\mbR^d,\,\mfh\neq 0}\norm{\mfX^T\diag(\mbI(\mfX\mfh\geq 0))\mfX}_2<1.
    \tag{ortho-SNIC-L}
\end{equation}
For a diagonal arrangement pattern $\mfD_j$, $\|\mfX^T\mfD_j\mfX\|_2<1$ is equivalent to $\mfX^T\mfD_j\mfX\prec \mfI_d$ or, equivalently, $\mfX^T(\mfI_n-\mfD_j)\mfX\succ 0$. As $\mfX$ is column orthonormal, this is also equivalent to $\tr(\mfI_n-\mfD_j)\geq d$, or equivalently, $\tr(\mfD_j)\leq n-d$. Therefore, \eqref{irrep:white_strong} can be simplified for column orthogonal matrices as
\begin{equation}\label{max_Xh}
    (\mathrm{orth\mydash SNIC\mydash L}) \hspace{1cm} \max_{\mfh\in\mbR^d,\,\mfh\neq 0} \tr(\diag(\mbI(\mfX\mfh\geq 0)))\leq n-d.
\end{equation}

\begin{theorem}\label{thm:white_strong}
Suppose that $\mfX\in\mbR^{n\times d}$ is a Haar distributed random matrix. Let $\theta^*>0$ be the unique solution of the scalar equation $\theta+\frac{1}{2}\int_{F_{\xi^2}^{-1}(1-2\theta)}^\infty rdF_{\xi^2}(r)=\frac{1}{2}$, where $F_{\xi^2}(r)$ is the cumulative distribution function (CDF) of a $\xi^2$-random variable with $1$ degree of freedom. Then, for sufficiently large $n,d$ satisfying $\frac{d}{n}< \theta^*$, the $\mathrm{orth\mydash SNIC\mydash L}$ given in \eqref{max_Xh} holds with high probability.
\end{theorem}
\begin{remark}
It can be computed that $\theta^*\approx 0.1314$ and $\pp{\theta^{*}}^{-1}\approx 7.613$ by numerically solving the scalar equation $\theta+\frac{1}{2}\int_{F_{\xi^2}^{-1}(1-2\theta)}^\infty rdF_{\xi^2}(r)=\frac{1}{2}$. This implies that when $n>7.613d$, the $\mathrm{orth\mydash SNIC\mydash L}$ shown in \eqref{max_Xh} holds w.h.p. 
\end{remark}
\subsection{Inexact/Noisy Linear Recovery}\label{sec:noisy}
In this section, we consider inexact or noisy observations $\mfy=\mfX\mfw^*+\mfz$, where $\mfz\in\mbR^n$ is an arbitrary disturbance component. We will show that the optimal NN only learns a linear model when the regularization parameter $\beta$ is chosen from an appropriate interval. This can be understood intuitively as follows: when $\beta$ is too small, the NN trained with insufficient norm penalty overfits to the noise. On the other hand, when $\beta$ is too large, the NN underfits the observations as the norms of the parameters are over-penalized.

For two-layer ReLU networks with skip connection and normalization layer before the ReLU, the regularized training problem \eqref{prob:reg} can be equivalently cast as the convex program (see e.g., \cite{nnacr,wang2020hidden}):
\begin{equation}\label{reg:normal_before_relu_skip}
\begin{aligned}
    \min_{\mfw_0,\left\{\mfw_{j}, \mfw_{j}'\right\}_{j=1}^{p}}\quad &  \Big\|\mfU\mfw_0+\sum_{j=1}^{p} \mfD_{j} \mfU\left(\mfw_{j}-\mfw_{j}'\right)-\mfy\Big\|_2^2
    + \beta\Big(\|\mfw_0\|_2+\sum_{j=1}^{p}\left(\left\|\mfw_{j}\right\|_{2}+\left\|\mfw_{j}'\right\|_{2}\right)\Big)\\
        \text{s.t.}\quad & (2\mfD_j-\mfI_n)\mfU\mfw_j\ge0, j\in[p], (2\mfD_j-\mfI_n)\mfU\mfw_j'\ge0, j\in[p],
\end{aligned}
\end{equation}
where $\mfU$ is computed from the compact SVD of $\mfX=\mfU\bSigma\mfV^T$.

As a direct corollary of Theorem 1 in \citep{wang2020hidden}, the global optima of the nonconvex regularized training problem \eqref{prob:reg} are given by the optimal solutions of \eqref{reg:normal_before_relu_skip} up to permutation and splitting (see Appendix \ref{sec:permutation} for details). By dropping all inequality constraints, the convex program \eqref{reg:normal_before_relu_skip} reduces to the following group-Lasso problem:
\begin{equation}\label{reg:normal_before_grelu_skip}
\begin{aligned}
        \min_{\left\{\mfw_{j}\right\}_{j=0}^p}\quad&\Big\|\sum_{j=0}^{p} \mfD_{j} \mfU\mfw_{j}-\mfy\Big\|_2^2 + \beta\sum_{j=0}^{p}\left\|\mfw_{j}\right\|_{2}.
\end{aligned}
\end{equation}
As we show next, the convex NN objective \eqref{reg:normal_before_relu_skip} inherits recovery properties from its group-Lasso relaxation above. 
The following two theorems show that for sufficiently large $n$, with a suitable choice of the regularization parameter $\beta$, the optimal neural network learns only a linear component. In addition, the $\ell_2$ distance between the optimal solution and the embedded neuron can be bounded by a linear function of $\beta$.
\begin{theorem}\label{thm:noisy}

Suppose that the entries of the data matrix $\mfX$ is i.i.d. sub-Gaussian with variance proxy $\sigma^2$ as in Theorem \ref{thm:gauss_success} and the noisy observation takes the form $\mfy=\mfX\mfw^*+\mfz$.
Assume that the weight decay regularization parameter satisfies $\beta\in\bb{\norm{\mfz}_2\frac{7(\eta-\norm{\mfz}_2)}{\eta-7\norm{\mfz}_2},\eta-\norm{\mfz}_2}$, the norm of the noise component satisfies $\norm{\mfz}_2\leq \frac{1}{14} \eta$, where
$\eta\triangleq\norm{\bSigma\mfV^T\mfw^*}_2$ and $n\geq  \max\{4000\sigma^2 d\log (54n),1024 d\}$. Then, with probability at least $1-4\exp(-n/8000\sigma^2)-2\exp(-d)$, the optimal solutions to convex programs \eqref{reg:normal_before_relu_skip} and \eqref{reg:normal_before_grelu_skip} consist of only a linear neuron and no ReLU neurons, i.e.,  there exists $\mfw$ such that $\mfW=(\mfw,0,\dots,0)$ is strictly optimal. As a consequence, the non-convex weight decay regularized objective \eqref{prob:reg} has the same strictly optimal solutions up to permutation and splitting. 
Furthermore, we have the $\ell_2$ distance upper bound
\[
\norm{\mfw-\bSigma\mfV^T\mfw^*}_2\le \frac{\beta\eta}{\eta-\|\mfz\|_2} + \|\mfz\|_2.
\]
Here, $\eta$ satisfies $(1-1/16)\|\mfw^*\|_2\le \eta \le (1+1/16)\|\mfw^*\|_2$ with probability $1-2\exp(-d)$. Moreover, the optimal weights of the neural network $f(\mfX;\Theta)=\normal_{\alpha_1}(\mfX\mfw_{1,1})\mfw_{2,1}+\sum_{i=2}^m\pp{\normal_{\alpha_i}(\mfX\mfw_{1,i})}_+\mfw_{2,i}$
are given by $\mfw_{1,1}=\mu\mfw,\alpha_1=\mfw_{2,1}=\sqrt{\norm{\mfX\mfw}_2}$ and $\mfw_{1,j}=\mathbf{0},\alpha_j=\mfw_{2,j}=0$ for $2\le j\le n$, where $\mu>0$ is an arbitrary constant.
\end{theorem}

\begin{remark}[Estimating the regularization coefficient]
In practice, it may be appropriate to assume the noise component $\mfz$ follows an i.i.d. sub-Gaussian distribution with zero mean. Under this assumption, we can control the term $\|\mfz\|_2$ with high probability using classical concentration bounds. From the triangle inequality, we note that $\big \|\mfy\|_2\in[\eta-\|\mfz\|_2,\eta+\|\mfz\|_2 \big]$. Thus, we can estimate $\eta$ using the interval $\eta \in \Big [ \big(\|\mfy\|_2-\|\mfz\|_2\big)_+,\|\mfy\|_2+\|\mfz\|_2 \Big]$, and use $\eta$ to compute the required inverval for $\beta$ in Theorem \ref{thm:noisy}.
\end{remark}
\subsection{Recovering a ReLU Neuron}
Now we focus on recovering a single planted ReLU neuron. We consider the case where the training data matrix $\mfX$ is composed of i.i.d. standard Gaussian variables $\mcN(0,1/n)$. The following theorem illustrates that when $n>2d$, the $\mathrm{NNIC}\mydash 1$ will hold with high probability. In other words, ReLU NNs with normalization layer optimizing the objective \eqref{min_nrm:grelu} uniquely recovers the ReLU neuron with high probability. 
\begin{theorem}\label{thm:relu_normal}
Let $\mfw^*\in \mbR^d$ is a fixed unit-norm vector.  Suppose that each entry of $\mfX$ are i.i.d. random variables following the Gaussian distribution $\mcN(0,1/n)$. Let $\mfD_{i^*}=\diag(\mbI(\mfX\mfw^*\geq 0))$. 
Then, when $n>2d$, the $\mathrm{NNIC}\mydash 1$ given in \eqref{eq:irr_cond_normal} 
holds with probability at least $1- \exp\pp{-\frac{1}{6}\pp{\frac{n-2d}{n}}^2n}$.
\end{theorem}
The above result shows that a single normalized ReLU neuron is uniquely recovered by ReLU NN with normalization layer via \eqref{min_nrm:grelu} (up to permutation and splitting) and its convex reformulation \eqref{min_nrm:relu_normal}, regardless of the number of neurons in the NN. In other words the set of global optimum of $\eqref{min_nrm:grelu}$ only consists of networks with only a single non-zero ReLU neuron along with permutations and split versions of this neuron.

\begin{remark}
We note that the recovery condition for a single ReLU neuron and a linear neuron is the same using ReLU NNs with normalization layer and is given by $\frac{n}{d}\ge 2$.
\end{remark}
%

\section{Asymptotic analysis}
\label{sec:asymp}
In this section, we present an asymptotic analysis of the Neural Isometry Conditions when $n$ goes to infinity while $d$ is fixed. While our analysis in Section \ref{sec:relu_skip} provides sharp estimates on the recovery threshold for a single linear or ReLU neuron, the asymptotic analysis in this section proves the recovery of multiple ReLU neurons.

We begin with the case of two planted ReLU neurons in the asymptotic setting. For $\mfw,\mfv\in\mbR^d$, recall our notation for the cosine angle between $\mfw$ and $\mfv$ given by $\cos\angle(\mfw,\mfv)=\frac{\mfw^T\mfv}{\|\mfw\|_2\|\mfv\|_2}$.
In order to show that the $\mathrm{NNIC} \mydash 2$ in \eqref{eq:irr_cond_normal} holds, and consequently two-neuron recovery succeeds via ReLU NNs with normalization layer, we calculate the asymptotic limit of the left-hand-side in \eqref{eq:irr_cond_normal} when $k=2$, which is given by
\begin{align}\label{eq:asymptoticT}
T:=\norm{\mfU_j^T\bmbm{\mfU_{s_1}&\mfU_{s_2}}\bmbm{\mfU_{s_1}^T\mfU_{s_1}&\mfU_{s_1}^T\mfU_{s_2}\\
    \mfU_{s_2}^T\mfU_{s_1}&\mfU_{s_2}^T\mfU_{s_2}}^{-1}\bmbm{\tilde \mfw_1\\\tilde \mfw_2}}_2.
\end{align}
We consider the case of $\mfw_1,\mfw_2$ satisfying $\cos\angle(\mfw_1,\mfw_2)=-1$ and $\cos\angle(\mfw_1,\mfw_2)=0$ for simplicity.
    
\begin{proposition}\label{prop:asymp_two_neu}
Suppose that each entry of $\mfX\in\mbR^{n\times d}$ is an i.i.d. random variable following the normal distribution $\mcN(0,1/n)$. Suppose that $\mfy=\frac{(\mfX\mfw_1)_+}{\|(\mfX\mfw_1)_+\|_2}+\frac{(\mfX\mfw_2)_+}{\|(\mfX\mfw_2)_+\|_2}$ is the planted model, where $\mfw_1,\mfw_2\in\mbR^d$. Let $\mfD_1=\diag(\mbI(\mfX\mfw_1\geq 0))$ and $\mfD_2=\diag(\mbI(\mfX\mfw_2\geq 0))$. Consider any diagonal arrangement pattern $\mfD_j=\diag(\mbI(\mfX\mfh_j\geq 0))$. Consider the random variable $T$ defined in \eqref{eq:asymptoticT}.
\begin{itemize}
    \item Suppose that $\cos\angle(\mfw_1,\mfw_2)=-1$. Let $\gamma=\cos\angle(\mfw_1,\mfh_j)$. As $n\to\infty$, $T$ converges in probability to a univariate function $g_1(\gamma)$. 
    Here $g_1(\gamma)\leq 1$ and the equality holds if and only $\gamma=1$ or $\gamma=-1$.
    \item Suppose that $\cos\angle(\mfw_1,\mfw_2)=0$. Let $\gamma_1=\cos\angle(\mfw_1,\mfh_j)$ and $\gamma_2=\cos\angle(\mfw_2,\mfh_j)$. As $n\to\infty$, $T$ converges in probability to a bivariate function $g_2(\gamma_1,\gamma_2)$. 
    Here $g_2(\gamma_1,\gamma_2)\leq 1$ and the equality holds if and only $(\gamma_1,\gamma_2)=(1,0)$ or  $(\gamma_1,\gamma_2)=(0,1)$. 
\end{itemize}
Therefore, in both of the above cases, we have $T\le 1$ as $n\to\infty$ and consequently the $\mathrm{NNIC} \mydash 2$ holds. Moreover, the planted two-neuron NN is the unique optimal solution to \eqref{prob:min_nrm} (up to permutation and splitting) and its convex reformulation \eqref{min_nrm:relu_normal}.
\end{proposition}

We validate this asymptotic behavior in Figure \ref{fig:multi}. It can be observed that the recovery threshold for two ReLU neurons is approximately $n\ge 4d$ when $\mfw_1=e_1$ and $\mfw_2=e_2$, i.e., $\cos\angle(\mfw_1,\mfw_2)=0$ in Figure \ref{fig:multi}(a). In addition, it can be observed that the recovery threshold for three ReLU neurons is approximately $n\ge 6d$.  

\begin{figure}[H]
\setcounter{subfigure}{0}
    \subfigure[$k=2,\ \mfw_1^* = \mfe_1,\ \mfw_2^* = \mfe_2$,]{
      \centering
      \includegraphics[width=.45\textwidth]{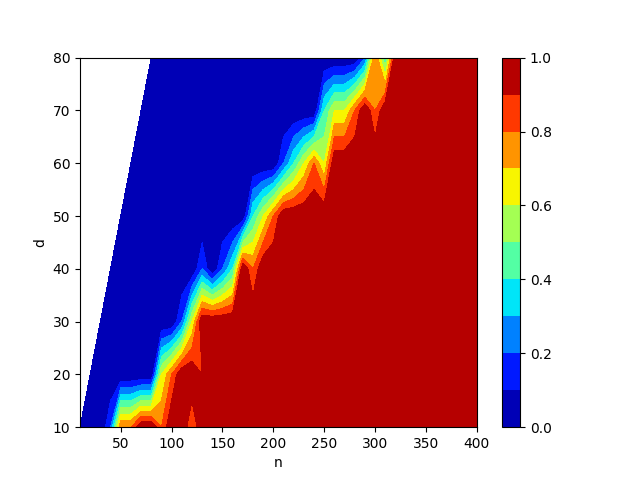} 
    }
    \subfigure[$k=3,\ \mfw_i^* = \mfe_i (i=1,2,3)$]{
      \centering
      \includegraphics[width=.45\textwidth]{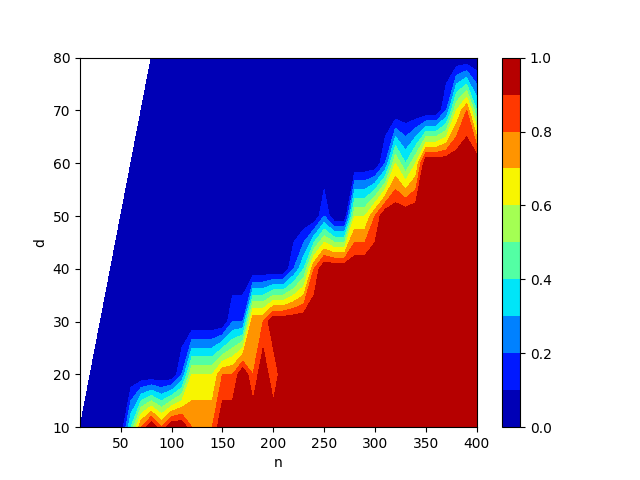} 
    }
    \caption{The empirical probability of successful recovery of the planted normalized ReLU neurons by solving the group $\ell_1$-minimization problem \eqref{min_nrm:grelu_normal} over $5$ independent trials. The label vector $\mfy$ is generated by a planted ReLU NN with $k$ neurons.}\label{fig:multi}
\end{figure}

Similar to the ReLU networks with the normalization layer, we present asymptotic analysis of plain ReLU networks. 
\begin{proposition}\label{prop:asymp_relu}
Suppose that each entry of $\mfX\in\mbR^{n\times d}$ is an i.i.d. random variable following the normal distribution $\mcN(0,1/n)$. Let $\mfD_i=\diag(\mbI(\mfX\mfw^*\geq0))$. Consider any hyperplane arrangement $\mfD_j=\diag(\mbI(\mfX\mfh_j\geq 0))$ such that $\cos\angle(\mfw^*,\mfh_j)=:\gamma<1$. Define
\begin{equation}
    R:=\norm{\mfX^T\mfD_j\mfD_i\mfX\pp{\mfX^T\mfD_i\mfX}^{-1}\hat \mfw}_2,
\end{equation}
where $\hat \mfw =\frac{\mfw^*}{\norm{\mfw^*}_2}$.
Then, as $n\to\infty$, $R$ converges in probability to $g(\gamma)$. 
Here the function $g:[-1,1]\to\mbR$ monotonically increases on $[-1,1]$ and $g(1)=1$. 
\end{proposition}
This implies that asymptotically, as $n\to\infty$, the $\mathrm{NIC}\mydash 1$ given in \eqref{irrep:grelu} holds, and plain ReLU NNs recover a single planted ReLU neuron.
\section{Numerical experiments}\label{num_res:main}
In this section, we present numerical experiments on ReLU networks with skip connection and normalization layer to validate our theoretical results on phase transitions in different NN architectures. We provide the illustration of our main results in this section and provide additional numerical results for various settings in Appendix \ref{num_res:add}. The code is available at \url{https://github.com/pilancilab/Neural-recovery}

Numerical results are divided into three parts: the first part consists of phase transition graphs for the recovery rate when the observation is noiseless by solving convex NN problems in \eqref{min_nrm:grelu_skip} and \eqref{min_nrm:grelu_normal}. 
The second and third part consist of phase transition graphs for certain types of distance measures when the observation is noisy by solving convex NN problems 
and the regularized training problem (convex and non-convex), respectively.

\subsection{ReLU networks with skip connection}\label{num_res_main:skip}
We start with phase transition graphs for successful recovery of the planted neuron by solving the convex optimization problem \eqref{min_nrm:grelu_skip}.
We compute the recovery rate for $d$ ranging from $10$ to $100$ and $n$ ranging from $10$ to $400$. For each pair of $(n,d)$, we generate 5 realizations of random training data matrices and solve the convex problem \eqref{min_nrm:grelu_skip} on each dataset. We test for four types of randomly generated data matrices: 
\begin{itemize}[itemsep=1pt]
    \item Gaussian: each entry $x_{i,j}$ of $\mfX\in\mbR^{n\times d}$ is an i.i.d. random variable following the normal distribution $\mcN(0,1/n)$.
    \item cubic Gaussian: each element $x_{i,j}$ of $\mfX\in\mbR^{n\times d}$ satisfies $x_{i,j}=z_{i,j}^3$, where $z_{i,j}$ are i.i.d. random variable following $\mcN(0,1/n)$.
    \item Haar: $\mfX\in\mbR^{n\times d}$ is drawn uniformly random from the set of column orthonormal matrices. We note that a Haar matrix can be generated by sampling an i.i.d. Gaussian matrix as above and extracting its $d$ left singular vectors of dimension $n$.
    \item whitened cubic Gaussian: $\mfX\in\mbR^{n\times d}$ is drawn non-uniformly from the set of column orthonormal matrices the matrix of left singular vectors of $\mfX'$ if $n>d$ and the matrix of right singular vectors of $\mfX'$ if $n<d$. Here $\mfX'\in\mbR^{n\times d}$ is a cubic Gaussian data matrix.
\end{itemize}
In each recovery problem, the planted neuron $\mfw^*$ is either a random vector following $\mcN(0,\mfI_d)$ or chosen as the smallest right singular vector of $\mfX$ as specified. In numerical experiments, we use a random subset $\mathcal{H}^\prime$ of the set $\mathcal{H}$ of all possible hyperplane arrangements to approximate the solution of the convex program. Here $\mathcal{H}^\prime$ is generated by
$$
\mathcal{H}'=\{\diag(\mbI(\mfX\mfh_i))|\mfh_i\in\mbR^d,i\in[\tilde n]\},
$$
where $\mfh_i$ is an i.i.d. random vector following the standard normal distribution $\mcN(0,\mfI_d)$ and $\tilde n=\max(n,50)$. On the other hand, from our theoretical analysis, the recovery will fail if there exists an all-ones hyperplane arrangement, i.e., $\mfI_n\in \mathcal{H}$. 
However, as $\mathcal{H}^\prime$ is a random subset of $\mathcal{H}$, it might not be easy to validate $\mfI_n\in \mathcal{H}$ by examining whether $\mfI_n\in \mathcal{H}'$ is satisfied.  
Therefore, we solve the following feasibility problem before solving the convex problem \eqref{min_nrm:grelu_skip}.
\begin{equation}\label{hyper_check1}
        \begin{aligned}
            &\max_{\mfw \in \mathbb{R}^d, t\in\mathbb{R}}~t\\
            &\text{s.t. } \|\mfw\|_2\le 1,\quad \mfX w\ge t1_n\,. 
        \end{aligned} 
\end{equation}
If the optimal value of the above problem is strictly greater than zero, there must exist an all-ones hyperplane arrangement, i.e., $\exists w\,:\,\mfX w>0$. In this case, we add $\mfI_n$ to the subset $\mathcal{H}^\prime$.  Otherwise, such an arrangement pattern does not exist. 

In Figure \ref{fig:lin_phase_w0}, we present the phase transition graph for the probability of successful recovery when the planted neuron $\mfw^*$ is randomly generated from $\mcN(0,\mfI_d)$. The boundaries indicate a phase transition between $n=2d$ and $n=3d$. In Appendix \ref{num_res:skip}, we will show similar phase transition graphs when the planted neuron $\mfw^*$ is the smallest right singular vector of $\mfX$ in Figure \ref{fig:lin_phase_w1}.

\newcommand{\figscale}{0.40}
\begin{figure}[t]
\centering
\setcounter{subfigure}{0}
    \subfigure[Gaussian]{
      \centering
      \includegraphics[width=\figscale\textwidth]{figs/lin/lin_phase_w0_X0.png}  
    }
    \centering
    \subfigure[Cubic Gaussian]{
      \centering
      \includegraphics[width=\figscale\textwidth]{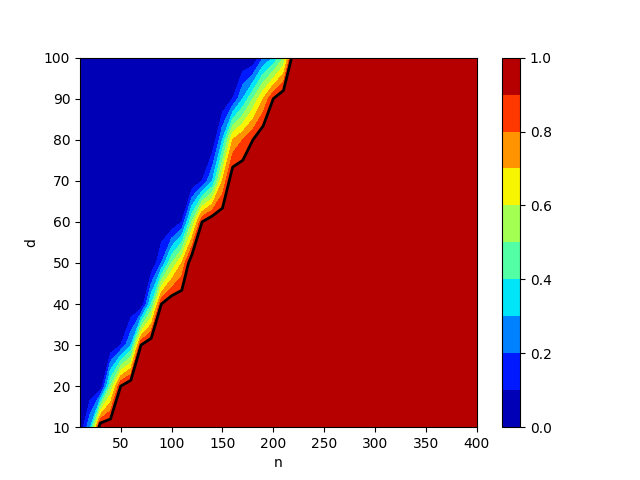}  
    }
    \subfigure[Haar]{
      \centering
      \includegraphics[width=\figscale\textwidth]{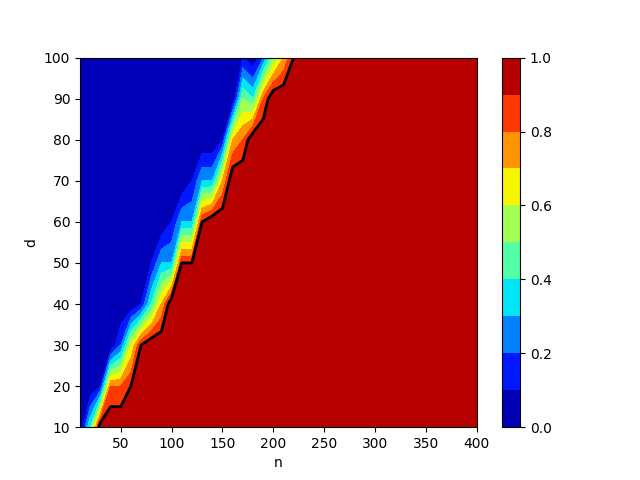}  
    }
    \subfigure[Whitened cubic Gaussian]{
      \centering
      \includegraphics[width=\figscale\textwidth]{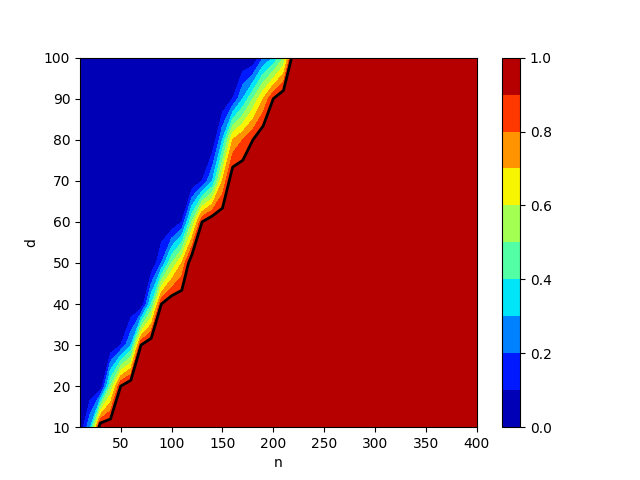}  
    }
    \caption{The probability of successful recovery of the planted linear neuron by solving the group $\ell_1$-minimization problem \eqref{min_nrm:grelu_skip} over $5$ independent trials. The black lines represent the boundaries of successful recovery with probability $1$. Here the planted neuron $\mfw^*$ is randomly generated from $\mcN(0,\mfI_d)$. }\label{fig:lin_phase_w0}
\end{figure}

The second part is phase transition under noisy observation, i.e., $\mfy=\mfX\mfw^*+\mfz$, where $\mfz\sim \mcN(0,\sigma^2/n)$. We still focus on the convex problem \eqref{min_nrm:grelu_skip}, i.e., the convex optimization formulation of gated ReLU networks with skip connection. Here we focus on Gaussian data and choose $\mfw^*$ as the smallest right singular vector of $\mfX$. We define the following two types of distance for the solution of the convex program to evaluate the performance. 
\begin{itemize}
    \item Absolute distance: the $\ell_2$ distance between the linear term $\mfw_0$ and $\mfw^*$.
    \item Test distance: generate a test set $\tilde{\mfX}$ with the same distribution as $\mfX$, then the prediction of the learned model is 
    \[
    \tilde{\mfy}=\tilde{\mfX}\mfw_0+\sum_{j=1}^{P} (\tilde{\mfX}\mfw_j)_+.
    \]
    Then the test distance is defined as the $\ell_2$ distance between the prediction $\tilde{\mfy}$ and the ground truth $\mfy^*=\tilde{\mfX}\mfw^*$.
\end{itemize}

The boundaries of red regions in Figure \ref{fig:lin_phase_noise_abs}, which represents highly unsuccessful recovery, remain around $n=2d$ for various noise levels $\sigma$. When $\sigma$ increases, the area of dark blue regions of small absolute distance/test error gradually vanishes. This implies that the linear part of the neural network no longer approximates the planted linear neuron and the gated ReLU neurons fit the noise. In Appendix \ref{num_res:skip}, we will observe the same pattern for absolute distance in Figure \ref{fig:lin_phase_noise_test}
\begin{figure}[t]
\setcounter{subfigure}{0}
\centering
    \subfigure[$\sigma=0$(noiseless)]{
      \centering
      \includegraphics[width=\figscale\textwidth]{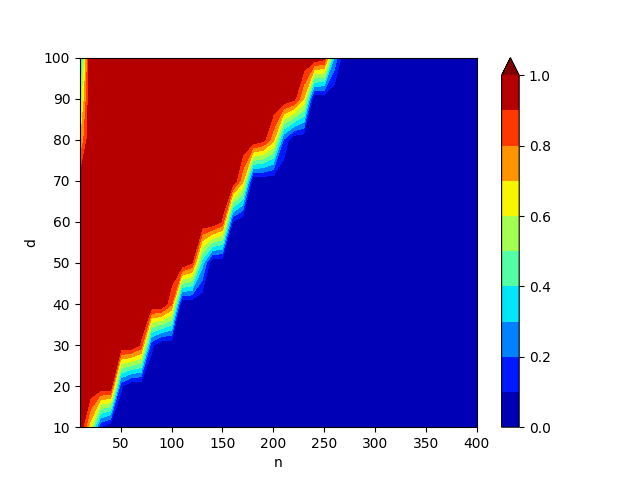}  
    }
    \subfigure[$\sigma=0.05$]{
      \centering
      \includegraphics[width=\figscale\textwidth]{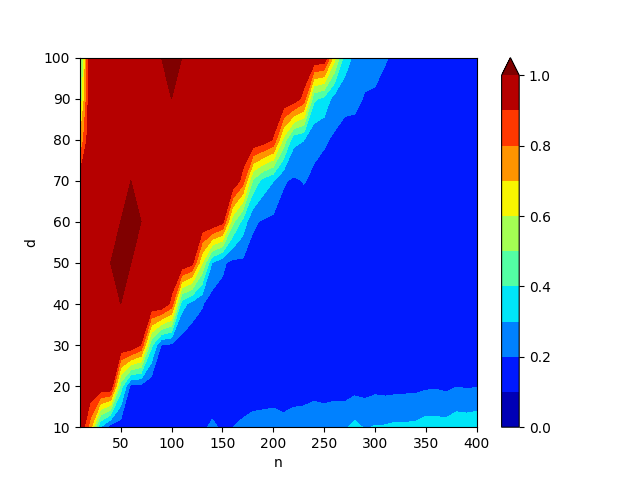}  
    }
    \subfigure[$\sigma=0.1$]{
      \centering
      \includegraphics[width=\figscale\textwidth]{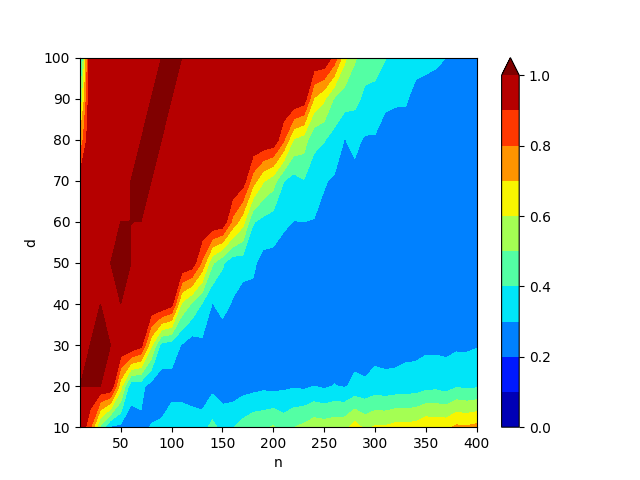}  
    }
    \subfigure[$\sigma=0.2$]{
      \centering
      \includegraphics[width=\figscale\textwidth]{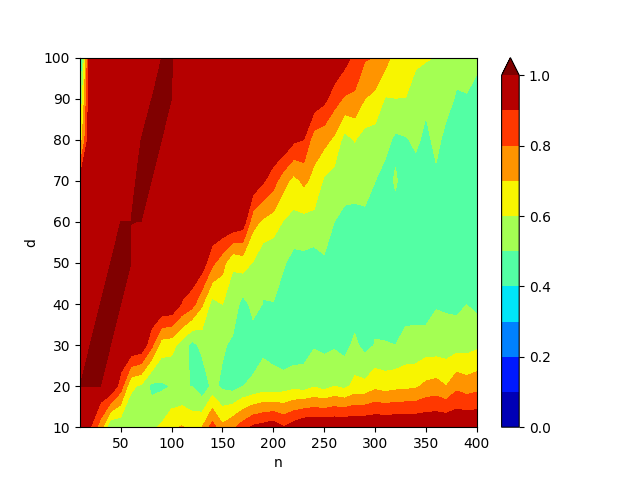}  
    }
    \caption{Averaged absolute distance to the planted linear neuron by solving the group $\ell_1$-minimization problem \eqref{min_nrm:grelu_skip} derived from training ReLU networks with skip connection over $5$ independent trials.}
    \label{fig:lin_phase_noise_abs}
\end{figure}

In the third part, we study the generalization property of ReLU networks with skip connections using convex/non-convex training methods. Results for the convex training methods are provided in
Appendix \ref{num_res:skip}.

For the nonconvex training method, we solve the regularized non-convex training problem \eqref{prob:reg} with $\beta=10^{-6}$ as an approximation of the minimum norm problem \eqref{prob:min_nrm}. We set the number of neurons to be $m=n+1$ and train the ReLU neural network with skip connection for 400 epochs. We use the AdamW optimizer and set the weight decay to be $\beta=10^{-6}$. We note that the nonconvex training may still reach local minimizers. Thus, the absolute distance to the planted linear neuron does not show a clear phase transition as the convex training. However, the transitions of test error generally follow the patterns of the group $\ell_1$-minimization problem. In Figure \ref{fig:ncvx_train_skip}, we show that the test error increases as $n/d$ increases,
and the rate of increase becomes sharper around $n=2d$ (the boundary of orange and yellow region).

\begin{figure}[t]
\setcounter{subfigure}{0}
\centering
    \subfigure[$\sigma=0$(noiseless)]{
      \centering
      \includegraphics[width=\figscale\textwidth]{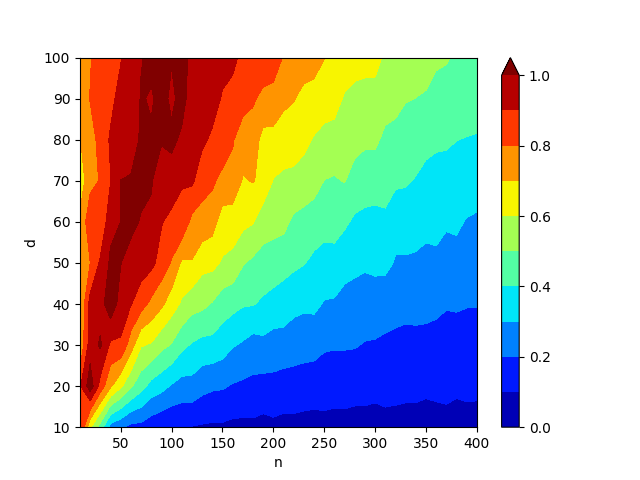}
    }
    \centering
    \subfigure[$\sigma=0.05$]{
      \centering
      \includegraphics[width=\figscale\textwidth]{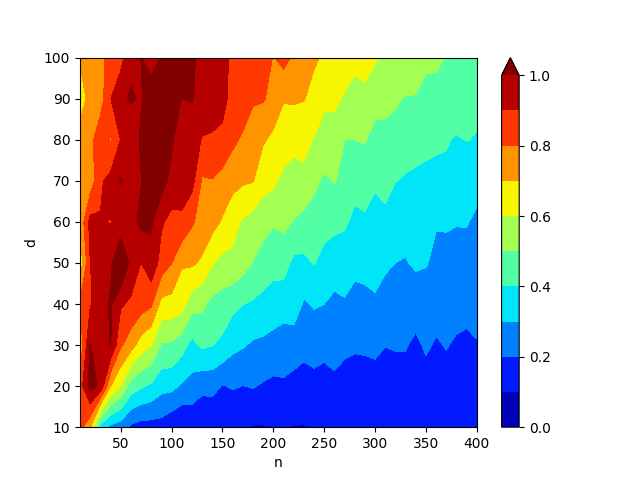}  
    }
\centering
    \subfigure[$\sigma=0.1$]{
      \centering
      \includegraphics[width=\figscale\textwidth]{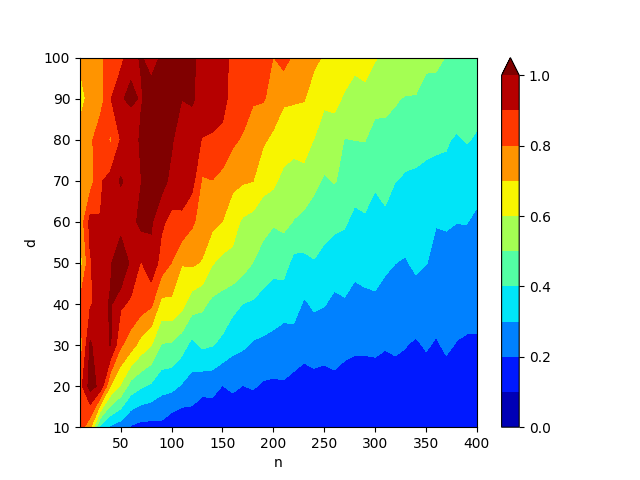}  
    }
    \centering
    \subfigure[$\sigma=0.2$]{
      \centering
      \includegraphics[width=\figscale\textwidth]{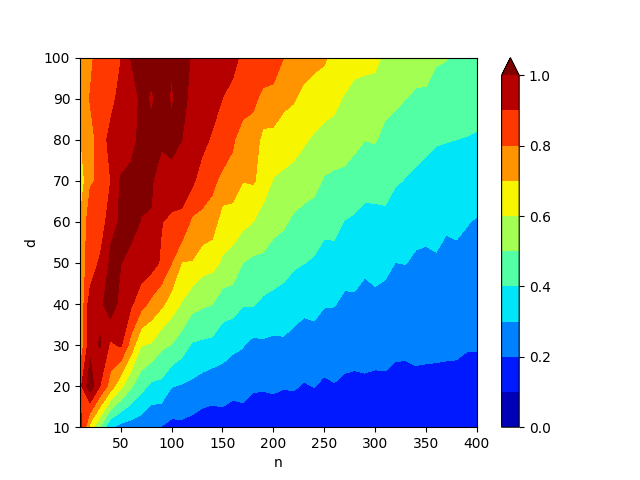}
    }
    \caption{Averaged test error by training ReLU networks with skip connection on the regularized non-convex problem \eqref{prob:reg} over $10$ independent trials.}
    \label{fig:ncvx_train_skip}
\end{figure}

\subsection{Multi-neuron recovery and irrepresentability condition}\label{num_res_main:mul}
In this subsection, we analyze the recovery for ReLU networks with normalization layer. Results for single-neuron recovery can be found in Appendix \ref{num_res:normal}. Here we focus on the case where the label vector is the combination of several normalized ReLU neurons. 
We will test for three types of planted neurons.
\begin{itemize}
    \item $k=2,\mfw_1^* = \mfw^*,\ \mfw_2^* = -\mfw^*$, where $\mfw^*\sim \mcU(\mbS^{n-1})$. In this case, the hyperplane arrangements of two neurons do not intersect.
    \item $k=2,\mfw_1^*,\mfw_2^*\sim \mcU(\mbS^{n-1})$. It is a general case where the hyperplane arrangements of two neurons can intersect.
    \item $k\ge 2,\mfw_i^*=\mfe_i$, where $\mfe_i$ is the $i$-th standard basis in $\mbR^n$.
\end{itemize}

We consider the noisy observation model, i.e., the observation is the combination of several normalized ReLU neurons and Gaussian noise, i.e.,
\[\mfy=\sum_{i=1}^{k}\frac{(\mfX\mfw_i^*)_+}{\|(\mfX\mfw_i^*)_+\|_2}+\mfz,
\]
where $\mfz\sim \mcN(0,\sigma^2/n)$. 

As a performance metric, the absolute distance is defined as 
$\left(\sum\limits_{i=1}^{k}\|\mfw_{s_i}-\tilde{\mfw}_i^*\|_2^2\right)^{1/2}$.

In Figure \ref{fig:lin_phase_noise_normal_mn21}, we show the phase transition graph when the planted neurons satisfy $\mfw_1^* = \mfw^*,\ \mfw_2^* = -\mfw^*,\ \mfw^*\sim \mcU(\mbS^{n-1})$. Results for the other two cases can be found in Appendix \ref{num_res:mul}.

\begin{figure}[H]
\setcounter{subfigure}{0}
\centering
    \subfigure[$\sigma=0$(noiseless)]{
      \centering
      \includegraphics[width=\figscale\textwidth]{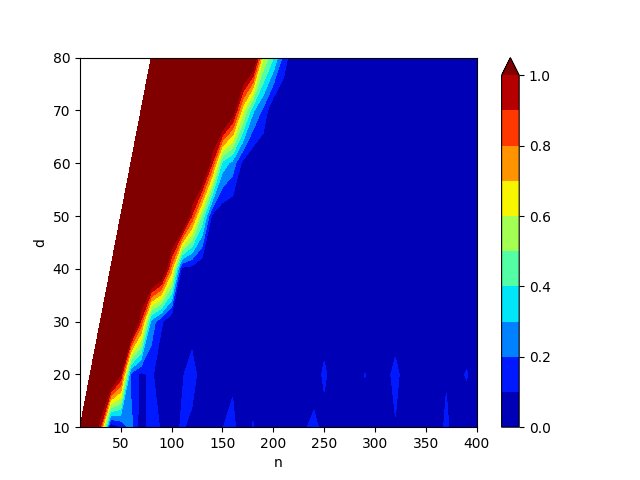}  
    }
    \centering
    \subfigure[$\sigma=0.05$]{
      \centering
      \includegraphics[width=\figscale\textwidth]{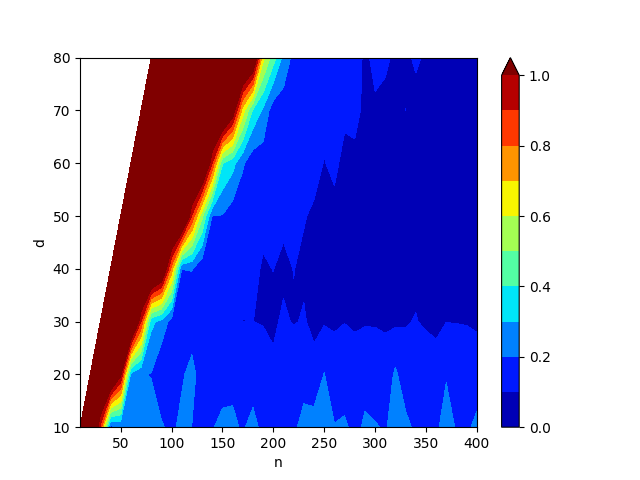}  
    }
    \centering
    \subfigure[$\sigma=0.1$]{
      \centering
      \includegraphics[width=\figscale\textwidth]{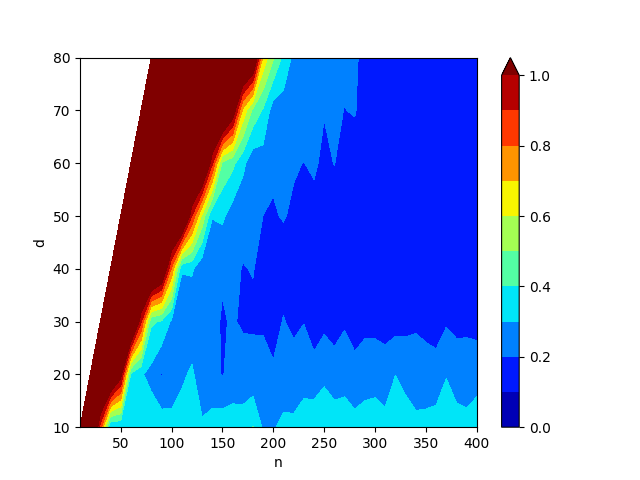}  
    }
    \centering
    \subfigure[$\sigma=0.2$]{
      \centering
      \includegraphics[width=\figscale\textwidth]{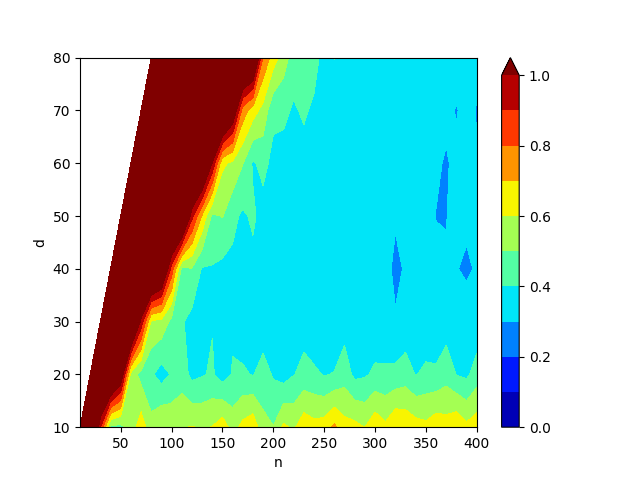}  
    }
    \caption{Averaged absolute distance to the planted normalized ReLU neurons by solving the convex problem \eqref{min_nrm:grelu_normal}
from training ReLU networks with normalization layer over 5 independent trials. Here we set $k=2$ planted neurons which satisfy $\mfw_1^* = \mfw^*,\ \mfw_2^* = -\mfw^*,\ \mfw^*\sim \mcU(\mbS^{n-1})$.}
\label{fig:lin_phase_noise_normal_mn21}
\end{figure}

It is also worth noting that we check the $\textrm{NIC}\mydash k$ given in \eqref{eq:irr_cond}, which guarantees recovery for the convex problem \eqref{min_nrm:relu_normal} and \eqref{min_nrm:grelu_normal}. We compute the probability that the $\textrm{NIC}$ holds for $d$ ranging from $10$ to 80 and $n$ ranging from 10 to 400 with 50 independent trials. For each pair of $(n,d)$. For each data matrix $\mfX\in\mbR^{n\times d}$, we use a random subset $H'$ of the set $H\backslash H_S$ to validate the inequality \eqref{eq:irr_cond}, where $H$ is the set of all possible hyperplane arrangements and $H_S$ is the set of
hyperplane arrangements generated by the planted neurons, i.e.
$$
H_S=\{\diag(\mbI(\mfX\mfw_i^*))|i\in[k]\}.
$$
In numerical experiments, we generate random hyperplane arrangements. If \eqref{eq:irr_cond} holds for all $\mfD_j\in H'$, we say $\mfX$ satisfies the $\textrm{NIC}\mydash k$ numerically.


\begin{figure}[H]
\setcounter{subfigure}{0}
\centering
    \subfigure[$k=2,\ \mfw_1^* = \mfw^*,\ \mfw_2^* = -\mfw^*,\ \mfw^*\sim \mcU(\mbS^{n-1})$]{
      \centering
      \includegraphics[width=\figscale\textwidth]{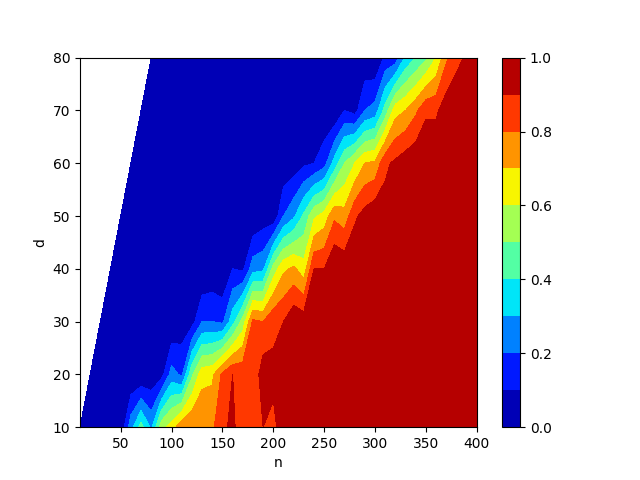}  
    }
    \centering
    \subfigure[$k=2,\ \mfw_1^*,\mfw_2^*\sim \mcU(\mbS^{n-1})$]{
      \centering
      \includegraphics[width=\figscale\textwidth]{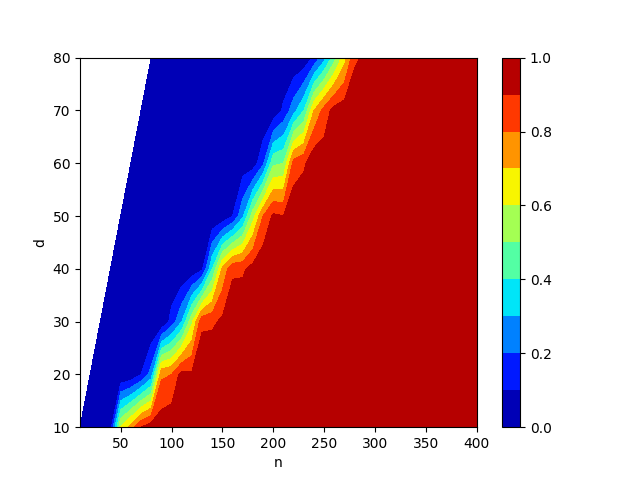} 
    }
    \end{figure}
\begin{figure}[H]
 \addtocounter{figure}{1}  
    \ContinuedFloat
\centering
    \subfigure[$k=2,\ \mfw_1^* = \mfe_1,\ \mfw_2^* = \mfe_2$,]{
      \centering
      \includegraphics[width=\figscale\textwidth]{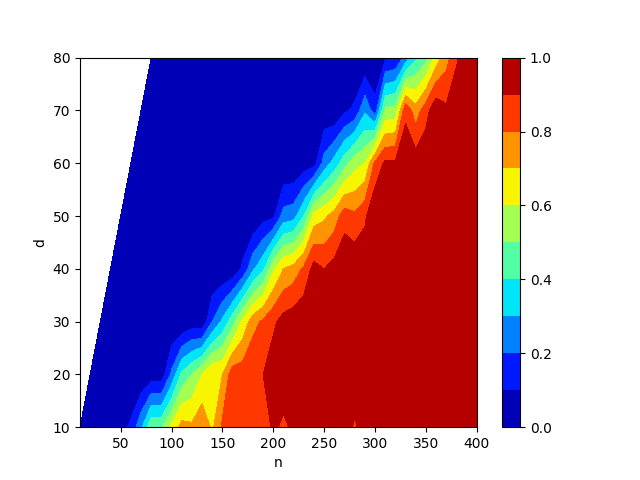} 
    }
    \centering
    \subfigure[$k=3,\ \mfw_i^* = \mfe_i(i=1,2,3)$]{
      \centering
      \includegraphics[width=\figscale\textwidth]{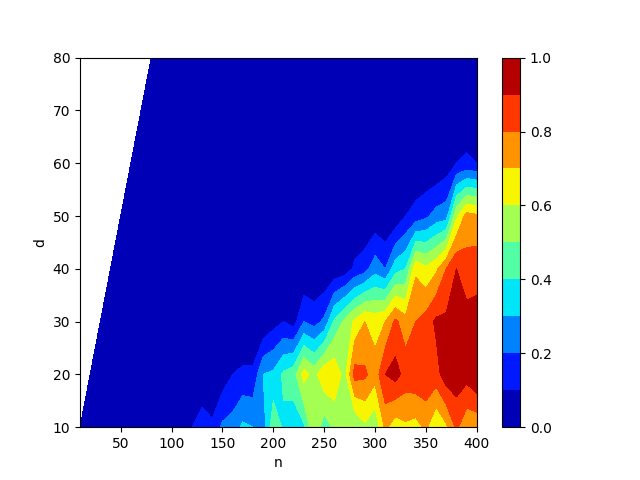} 
    }
    \caption{The probability that the irrepresentability condition \eqref{eq:irr_cond} holds for the group $\ell_1$-minimization problem \eqref{min_nrm:grelu_normal}
    over 50 independent trials. The label vector $\mfy$ is the combination of several normalized ReLU neurons.}
\end{figure}
{}
\section{Conclusion}
\label{sec:conclusion}
We presented a framework to analyze recovery properties of ReLU neural networks through convex reparameterizations. We introduced Neural Isometry Conditions, which are deterministic conditions on the training data that ensure the recovery of planted neurons via training two-layer ReLU networks and the variants with skip connection and normalization layers. Viewing the non-convex neural network training problem from a convex optimization perspective, we establish theorems analogous to sparse recovery and compressed sensing by using probabilistic methods. For randomly generated training data matrices, we showed the existence of a sharp phase transition in the recovery of simple planted models. Interestingly, ReLU neural networks with an arbitrary number of neurons exactly recover simple planted models, such as a combination of few ReLU neurons, when the number of samples exceed a critical threshold. Therefore, these models can perfectly generalize even with an extremely large number of parameters when the labels are generated by simple models. This phenomenon not only aligns with the results developed in sparse recovery theory, but is also validated by our numerical experiments. Namely, when the training data is i.i.d. Gaussian and the number of data points is smaller than a critical threshold, the convex program cannot recover the planted model. On the other hand, when the number of data points is above a critical threshold, the solution of the convex optimization problem uniquely recovers the planted solution. Our main contribution is that we explicitly characterize the data isometry conditions that imply exact recovery of planted ReLU and linear neurons, and the specific relation between the number of data points and problem dimensions for random data matrices for successful recovery with high probability. We also extend our results to the case where the observation is noisy, and show that the neural network can still learn a simple model, even with an arbitrary number of neurons, when the noise component is not too large and the regularization parameter lies in an appropriate interval.

An immediate open problem is extending our results to neural networks of depth greater than two, and investigating modern DNN structures such as convolutional and transformer layers. Moreover, the analysis of the exact recovery threshold for an arbitrary number of ReLU neurons is an important open problem. Our numerical experiments suggest the trend $n\ge 2k$ for recovering $k$ ReLU neurons when the training matrix is composed of i.i.d. Gaussian random variables.
Finally, we note that non-convex training methods might get stuck at local minimizers or stationary points that are avoided by the convex formulation. The relationship between convex and non-convex training process should be further revealed to explain this phenomenon. We leave the analysis of non-convex training process when the observation is derived from a simple model as an open research problem for future work. 
\newpage
\appendices

\section{Review of linear sparse recovery via $\ell_1$ minimization}\label{app:review_linear}
We briefly review conditions required to ensure recovery via $\ell_1$-norm minimization in the linear observation case. Suppose that $\mfy=\mfX\mfw^*$ denotes linear observations, where the matrix $\mfX\in\mbR^{n\times d}$ represents measurements and $\mfw^*\in\mbR^d$ is a $k$-sparse vector of dimension $d$. Consider the following $\ell_1$-minimization problem to recover $\mfw^*$ from measurements $\mfy$:
\begin{equation}\label{l1_min}
    \min_{\mfw\in\mbR^d}\; \|\mfw\|_1 \text{ s.t. }\mfX\mfw=\mfy.
\end{equation}
It can be shown that $\mfw^*$ can be exactly recovered from $\mfy$ when the data matrix $\mfX$ satisfies certain isometry conditions, which are analogous to the ones developed in this work. The KKT optimality conditions of the above convex optimization problem are
\begin{equation}
\begin{aligned}
   \exists \blbd\in\mbR^n \qquad \mbox{s.t.}\qquad & |\blbd^T\mfx_j^\mathrm{col}|\leq 1, &\text{ for } w_j= 0, \label{eq:KKTforL1}\\
    & \blbd^T\mfx_j^\mathrm{col}=\mathrm{sign}(w_j), &\text{ for } w_j\neq 0,
\end{aligned}
\end{equation}
where $\blbd\in\mbR^n$ is a dual variable and $\mfx_j^\mathrm{col}\in\mbR^d$ denotes the $j$-th column of $\mfX$.

Let $S=\{i\in[d]\,|\,w^*_i\neq 0\}$ denote the support set of $w^*$ and $S^c=[n]/S$ its complement. For the support set $S\subseteq [d]$ of size $|S|$, denote the subvector of $\mfw^*$ that corresponds to entries restricted to $S$ as $\mfw_S=(w_i)_{i\in S} \in \mathbb{R}^{|S|}$, and the submatrix $\mfX_S \in \mathbb{R}^{n\times |S|}$ of $\mfX$ formed with the columns $S$ that correspond to the support of $\mfw$. The irrepresentability condition is a simpler sufficient condition that implies that the KKT conditions in \eqref{eq:KKTforL1} hold for $\mfw=\mfw^*$. This is ensured by the choice $\blbd^* = \arg\min_{\blbd\,:\,\mfX^T_S \blbd = \mathrm{sign}(\mfw_S^*)} \|\blbd\|_2 = \mfX_S(\mfX^T_S\mfX_S)^{-1}\mathrm{sign}(\mfw_S^*)$ assuming that $\mfX^T_S\mfX_S$ is invertible, and leads to the condition  
\begin{equation} \label{eqn:irrepresentability}
    \|\mfX_{S^c}^T\mfX_S(\mfX^T_S\mfX_S)^{-1}\mathrm{sign}(\mfw_S^*)\|_\infty<1.
\end{equation}
Intuitively, the above condition is expected to hold under three conditions: (i) the matrix $(\mfX^T_S\mfX_S)^{-1}$ is well-conditioned, (ii) the columns of $\mfX$ have small inner-products with each other, i.e., $\|\mfX_{S^c}^T\mfX_S\|_{\infty,\infty}$ is small, and (iii) the size of the subset $S$ is not too large. Moreover,the irrepresentability condition in \eqref{eqn:irrepresentability} also ensures that $\mfw^*$ is the unique optimal solution to \eqref{l1_min} \citep{zhao2006model}.

The Restricted Isometry Property (RIP) is a stronger condition that imposes well-conditioning of submatrices uniformly over all size-$k$ subsets. RIP is stated as follows
\begin{equation*}
(1-\delta_k)\|\mfw_S\|_2^2  \le   \|\mfX_S \mfw_S \|_2^2 \le (1+\delta_k) \|\mfw_S\|_2^2\,,\qquad \forall \mfw_S \in \mbR^{|S|},\quad \forall S\subseteq [d]\,:\,|S|\le k
\end{equation*}
where $\delta_k \in (0,1)$ is the Restricted Isometry Constant for some positive integer $k$. RIP implies that all $n\times |S|$ submatrices of $\mfX\in\mathbb{R}^{n\times d}$ are well-conditioned for all subsets $S$ of size at most $k$. Examples of matrices that satisfy the RIP property include i.i.d. sub-Gaussian random matrices, random Haar matrices, as well subsampled orthonormal systems, e.g., Fourier and Hadamard matrices under conditions on the dimensions $n$ and $d$, see details in \cite{candes2008restricted,vershynin2018high}.
In \cite{van2009conditions}, it has been shown that RIP with a sufficiently small constant $\delta_k$ implies the irrepresentability condition, and hence the recovery of a linear k-sparse vector from the observations $\mfy=\mfX\mfw^*$.
\section{Permutation and splitting of neural networks}\label{sec:permutation}
In this section, we present the definition of permutation and splitting of two-layer ReLU neural networks and their variants with skip connections or normalization layers. 

For a ReLU network $\Theta=(\mfW_1,\mfw_2)$ where $\mfW_1\in\mbR^{d\times m},\mfw_2\in\mbR^{m}$, a permutation of $\Theta$ is any neural network $\Theta'=(\mfW_1',\mfw_2')$ with $\mfW_1'\in\mbR^{d\times m},\mfw_2'\in\mbR^{m}$ such that $\mfw_{1,j}'=\mfw_{1,\pi(j)}$ and $w_{2,j}'=w_{2,\pi(j)}$ for $j\in[m]$. Here $\pi:[m]\to[m]$ is a permutation of $[m]$. 

Given a neuron-pair $(\mfw_{1},w_{2})$, we say that a collection of neuron-pairs $\{(\mfw_{1,j}, w_{2,j})\}_{j=1}^k$ is a splitting of $(\mfw_{1},w_{2})$ if $(\mfw_{1,j},w_{2,j}) = (\sqrt{\gamma_j} \mfw_{1}, \sqrt{\gamma_j} w_2)$ for some $\gamma_j \geq 0$ and $\sum_{j=1}^k \gamma_j = 1$. Given a ReLU neural network $\Theta = (\mfW_1,\mfw_2)$ with $\mfW_1\in\mbR^{d\times m},\mfw_2\in\mbR^{m}$, a \emph{splitting} of $\Theta$ is any neural network $\Theta'=(\mfW_1',\mfw_2')$ with $\mfW_1'\in\mbR^{d\times m'},\mfw_2'\in\mbR^{m'}$ such that the non-zero neurons of $\Theta'$ can be partitioned into splittings of the neurons of $\Theta$.

For a ReLU network with skip connection $\Theta=(\mfW_1,\mfw_2)$ where $\mfW_1\in\mbR^{d\times m},\mfw_2\in\mbR^{m}$, the permutation and splitting of this network refers to the permutation and splitting of the ReLU neurons $(\mfw_{1,i},w_{2,i})_{i=2}^m$ in this network.

For a ReLU network with normalization layer $\Theta=(\mfW_1,\mfw_2,\balpha)$ where $\mfW_1\in\mbR^{d\times m},\mfw_2,\balpha\in\mbR^m$, a permutation of $\Theta$ is any neural network $\Theta'=(\mfW_1',\mfw_2',\balpha')$ such that $\mfw_{1,j}'=\mfw_{1,\pi(j)}$, $w_{2,j}'=w_{2,\pi(j)}$ and $\alpha_j'=\alpha_{\pi(j)}$ for $j\in[m]$. Here $\pi:[m]\to[m]$ is a permutation of $[m]$. 

Given a neuron-pair $(\mfw_{1},w_{2},\alpha)$, we say that a collection of neuron-pairs $\{(\mfw_{1,j}, w_{2,j})\}_{j=1}^k$ is a splitting of $(\mfw_{1},w_{2})$ if $(w_{2,j},\alpha_j) = (\sqrt{\gamma_j} w_2, \sqrt{\gamma_j} \alpha)$ for some $\gamma_j \geq 0$ and $\sum_{j=1}^k \gamma_j = 1$ and $\mfw_{1,j}$ is positively colinear with $\mfw_{1}$ for $j\in[m]$. Namely, there exists $\zeta_j>0$ such that $\mfw_{1,j}=\zeta_j\mfw_{1,j}$. Given a ReLU neural network with normalization layer $\Theta = (\mfW_1,\mfw_2,\balpha)$ with $\mfW_1\in\mbR^{d\times m},\mfw_2,\balpha\in\mbR^{m}$, a \emph{splitting} of $\Theta$ is any neural network $\Theta'=(\mfW_1',\mfw_2',\balpha')$ with $\mfW_1'\in\mbR^{d\times m'},\mfw_2',\balpha'\in\mbR^{m'}$ such that the non-zero neurons of $\Theta'$ can be partitioned into splittings of the neurons of $\Theta$.

\section{Justification for excluding the hyperplane arrangement induced by the zero vector}
\label{app:hyper}
Consider $\tilde H=\{\diag(\mbI(\mfX\mfh\geq 0))|\mfh\in\mbR^d\}$. Firstly, we note that $\mfI_n=\diag(\mbI(\mfX\mfh\geq 0))$. If $\mfI_n\in H$, then we have $\tilde H=H$. If $\mfI_n\notin H$, then this implies that $\{\mfw\in\mbR^d:(2\mfI_n-\mfI_n)\mfX\mfw\geq 0\}=\{0\}$. Therefore, excluding the hyperplane arrangement induced by $\mfh=0$ does not change the convex program \eqref{min_nrm:relu}. 

\section{The convex program for gated ReLU networks}
\label{app:grelu}
Consider the minimum norm interpolation problem
\begin{equation}\label{min_nrm:grelu_noncvx}
    \min_{\Theta} \frac{1}{2}\pp{\|\mfW_1\|_F^2+\|\mfw_2\|_2^2}, \text{ s.t. } f^\mathrm{gReLU}(\mfX;\Theta)=\mfy,
\end{equation}
where $f^\mathrm{gReLU}(\mfX;\Theta)$ defined in \eqref{equ:grelu} is the output of a gated ReLU network. We first reformulate \eqref{min_nrm:grelu_noncvx} in the following way.
\begin{proposition}\label{prop:reform}
The problem \eqref{min_nrm:grelu_noncvx} can be reformulated as
\begin{equation}\label{min_nrm:grelu_reform}
    \min_{\Theta} \|\mfw_2\|_1, \text{ s.t. } f^\mathrm{gReLU}(\mfX;\Theta)=\mfy, \|\mfw_{1,i}\|_2\leq 1, i\in[m].
\end{equation}
\end{proposition}
For the reformulated problem \eqref{min_nrm:grelu_reform}, we can derive the dual problem.
\begin{proposition}\label{prop:grelu_dual}
The dual problem of \eqref{min_nrm:grelu_reform} is given by
\begin{equation}\label{equ:grelu_dual}
    \max_{\blbd}  \blbd^T\mfy, \text{ s.t. } \max_{\|\mfw\|_2\leq 1, \mfh\neq 0}\left|\blbd^T\diag(\mbI(\mfX\mfh\geq 0))\mfX\mfw\right|\leq 1.
\end{equation}
\end{proposition}
Based on the hyperplane arrangement described in \eqref{equ:hyperplane}, the dual problem is also equivalent to
\begin{equation}
\begin{aligned}
    \max_{\blbd}  \blbd^T\mfy, \text{ s.t. } \max_{j\in[p],\|\mfw\|_2\leq 1}\left|\blbd^T\mfD_j\mfX\mfw\right|\leq 1
    =\max_{\blbd}  \blbd^T\mfy, \text{ s.t. }\norm{\mfX^T\mfD_j\blbd}_2\leq 1,j\in[p].
\end{aligned}
\end{equation}
Indeed, the bi-dual problem (dual of the dual problem \eqref{equ:grelu_dual}) is the group Lasso problem \eqref{min_nrm:grelu}.
\begin{proposition}\label{prop:grelu_equiv}
The dual of the problem \eqref{equ:grelu_dual} is the group Lasso problem \eqref{min_nrm:grelu}. For a sufficiently large $m$, the minimum norm interpolation problem \eqref{min_nrm:grelu_noncvx} is equivalent to \eqref{min_nrm:grelu}. 
\end{proposition}

\subsection{Proof of Proposition \ref{prop:reform}}
\begin{proof}
For $i\in[m]$, consider $\hat \mfw_{1,i}=\alpha_i \mfw_{1,i}$ and $\hat w_{2,i}=\alpha_i^{-1} w_{2,i}$, where $\alpha_i>0$. Let $\hat \Theta=(\hat \mfW_1,\hat \mfW_2,\mfH)$. Then, we note that $f^\mathrm{gReLU}(\mfX;\Theta)=f^\mathrm{gReLU}(\mfX;\hat \Theta)$. This implies that $\hat \Theta$ is feasible for \eqref{min_nrm:grelu_noncvx}. From the inequality of arithmetic and geometric mean, we note that
\begin{equation}
\begin{aligned}
    \frac{1}{2}\sum_{i=1}^m\pp{ \alpha_i^2 \|\mfw_{1,i}\|_2^2+\alpha_i^{-2}w_{2,i}^2}
    \geq \sum_{i=1}^m \|\mfw_{1,i}\|_2|w_{2,i}|.
\end{aligned}
\end{equation}
The equality is achieved when $\alpha_i=\sqrt{\frac{|w_{2,i}|}{\|\mfw_{1,i}\|_2}}$. As the scaling operation does not change $\sum_{i=1}^m \|\mfw_{1,i}\|_2|w_{2,i}|$, we can set $\|\mfw_{1,i}\|_2=1$ and then the lower bound of the objective value becomes $\sum_{i=1}^m |w_{2,i}|=\|\mfw_{2}\|_1$. This completes the proof.
\end{proof}

\subsection{Proof of Proposition \ref{prop:grelu_dual}}
Consider the Lagrangian function
\begin{equation}
\begin{aligned}
L(\mfW_1,\mfw_2,\mfH,\blbd)=\|\mfw_2\|_1+\blbd^T\pp{\mfy-\sum_{i=1}^m\diag(\mbI(\mfX\mfh_i\geq 0))\mfX\mfw_{1,i}w_{2,i}}.
\end{aligned}
\end{equation}
The problem \eqref{min_nrm:grelu_noncvx} is equivalent to
\begin{equation}
\begin{aligned}
&\min_{\mfW_1,\mfw_2,\mfH}\max_{\blbd}L(\mfW_1,\mfw_2,\mfH,\blbd),\text{ s.t. }\|\mfw_{1,i}\|_2\leq 1, \mfh_i\neq 0,i\in[m]\\
=&\min_{\mfW_1,\mfH}\max_{\blbd}\min_{\mfw_2}L(\mfW_1,\mfw_2,\mfH,\blbd),\text{ s.t. }\|\mfw_{1,i}\|_2\leq 1, \mfh_i\neq 0, i\in[m]\\
=&\min_{\mfW_1,\mfH}\max_{\blbd} \blbd^T\mfy-\sum_{i=1}^m \tilde \mbI(\left|\blbd^T\diag(\mbI(\mfX\mfh_i\geq 0))\mfX\mfw_{1,i}\right|\leq 1),\\
&\text{ s.t. }\|\mfw_{1,i}\|_2\leq 1, \mfh_i\neq 0,i\in[m].
\end{aligned}
\end{equation}
Here $\tilde \mbI(S)=0$ if the statement $S$ is correct and $\tilde \mbI(S)=+\infty$ otherwise. By exchanging the order of $\min$ and $\max$, we obtain the dual problem
\begin{equation}
\begin{aligned}
&\max_{\blbd} \min_{\mfW_1,\mfH} \blbd^T\mfy-\sum_{i=1}^m \tilde \mbI(\left|\blbd^T\diag(\mbI(\mfX\mfh_i\geq 0))\mfX\mfw_{1,i}\right|\leq 1),\text{ s.t. }\|\mfw_{1,i}\|_2\leq 1, \mfh_i\neq 0,i\in[m]\\
=&\max_{\blbd}  \blbd^T\mfy, \text{ s.t. } \max_{\substack{\|\mfw_{1,i}\|_2\leq 1, \mfh_i\neq 0}}\left|\blbd^T\diag(\mbI(\mfX\mfh_i\geq 0))\mfX\mfw_{1,i}\right|\leq 1, i\in[m]\\
=&\max_{\blbd}  \blbd^T\mfy, \text{ s.t. } \max_{\|\mfw\|_2\leq 1, \mfh\neq 0}\left|\blbd^T\diag(\mbI(\mfX\mfh\geq 0))\mfX\mfw\right|\leq 1.
\end{aligned}
\end{equation}

\subsection{Proof of Proposition \ref{prop:grelu_equiv}}
\begin{proof}
As the group lasso problem \eqref{min_nrm:grelu} is a convex problem, it is sufficient to show that the dual problem of \eqref{min_nrm:grelu} is exactly \eqref{equ:grelu_dual}. Consider the Lagrangian function
\begin{equation}
    L(\mfw_1,\dots,\mfw_p,\blbd)=\sum_{j=1}^p \|\mfw_j\|_2+\blbd^T(\mfy-\sum_{j=1}^p\mfD_j\mfX\mfw_j).
\end{equation}
The dual problem follows
\begin{equation}
\begin{aligned}
    \max_{\blbd}\min_{\mfw_1,\dots,\mfw_p} L(\mfw_1,\dots,\mfw_p,\blbd)
    =\max_{\blbd}\blbd^T\mfy, \text{ s.t. } \norm{\mfX^T\mfD_j\blbd}_2\leq 1,j\in[p],
\end{aligned}
\end{equation}
which is equivalent to \eqref{equ:grelu_dual}. This implies that the optimal value of \eqref{min_nrm:grelu} serves as a lower bound for \eqref{min_nrm:grelu_noncvx}. For a sufficiently large $m$, any feasible point $(\mfw_1,\dots,\mfw_p)$ to \eqref{min_nrm:grelu} also corresponds to a feasible neural network for \eqref{min_nrm:grelu_noncvx}. In this case, the minimum norm interpolation problem \eqref{min_nrm:grelu_noncvx} is equivalent to \eqref{min_nrm:grelu}.
\end{proof}

\section{Proofs in Section \ref{sec:preview}}
\subsection{Proof of Lemma \ref{thm:landscape_skip}} 
\begin{proof}
Let $m^*$ be the minimal number of neurons of the optimal solution to the convex program \eqref{min_nrm:relu_skip}. Similar to the proof of Theorem 1 in \citep{wang2020hidden}, assuming that $m\geq m^*$, for any globally optimal solution to the non-convex problem \eqref{prob:min_nrm}, we can merge its ReLU neurons into a minimal neural network defined in \citep{wang2020hidden}. This minimal neural network combining with the linear part corresponds to an optimal solution to the convex program \eqref{min_nrm:relu_skip}. Therefore, we can view this globally optimal neural network as a possibly split and permuted version of an optimal solution to the convex program \eqref{min_nrm:relu_skip}. For the case where $\mfy=\mfX\mfw^*$, the ReLU network with skip connection can represent $\mfy$ using the linear part. In this case, $m^*=1$.
\end{proof}

\subsection{Proof of Proposition \ref{prop:imply_linear}}
\begin{proof}
We first show that the linear neural isometry condition implies the recovery of the planted linear model by solving \eqref{min_nrm:grelu_skip}.
\begin{proposition}
\label{prop:irrep_grelu_skip}
Suppose that the linear neural isometry condition \eqref{irrep:grelu_skip} holds. Consider $\mfW^*=(\mfw_0^*,\dots,\mfw_p^*)$ such that $\mfw_{0}^*=\mfw^*$ while $\mfw_j^*=0$ for $j\neq 0$. Then, $\mfW^*$ is the unique optimal solution to \eqref{min_nrm:grelu_skip}.
\end{proposition}
\begin{proof}
The KKT conditions for \eqref{min_nrm:grelu_skip} consist of
\begin{equation}\label{kkt:grelu_skip}
\begin{aligned}
    &\norm{\mfX^T\mfD_j\blbd^T}\leq 1, &\text{ if } \mfw_j=0,\\
    &\mfX^T\mfD_j\blbd^T = \frac{\mfw_j}{\|\mfw_j\|}, &\text{ if } \mfw_j\neq 0,\\
    &\sum_{j=0}^{p} \mfD_{j} \mfX \mfw_{j}=\mfy.
\end{aligned}
\end{equation}
As a direct corollary of Proposition \ref{prop:irrep_general}, $\mfW^*$ is the unique optimal solution to \eqref{min_nrm:grelu_skip}.
\end{proof}
We then present the proof of Proposition \ref{prop:imply_linear}. As \eqref{min_nrm:grelu_skip} is derived by dropping all inequality constraints in \eqref{min_nrm:relu_skip}, the optimal value of \eqref{min_nrm:relu} is lower bounded by \eqref{min_nrm:grelu_skip}. Note that $\mfW^*$ is the unique optimal solution to \eqref{min_nrm:grelu_skip}. Hence, the optimal value of \eqref{min_nrm:grelu_skip} is $\|\mfw^*\|_2$. On the other hand, $\hat \mfW$ is feasible for \eqref{min_nrm:relu_skip} and it leads to an objective value of $\|\mfw^*\|_2$. This implies that $\hat \mfW$ is an optimal solution to \eqref{min_nrm:relu_skip} and the optimal values of \eqref{min_nrm:relu_skip} and \eqref{min_nrm:grelu_skip} are the same. Suppose that we have another solution $\bar \mfW$ which is optimal to \eqref{min_nrm:relu_skip}. Consider $\breve\mfW=\{\breve \mfw_j|j=0,1,\dots,p\}$ where $\breve\mfw_0=\tilde\mfw_0$ and $\breve \mfw_j=\bar \mfw_j-\bar \mfw_j'$ for $j=1,\dots,p$. Then, $\bar \mfW$ is also optimal to \eqref{min_nrm:grelu_skip}. As $\mfW^*$ is the unique optimal solution to \eqref{min_nrm:grelu}, this implies that $\bar \mfW=\mfW^*$. Hence, we have $\bar \mfw_0=\bar \mfw_0=\mfw^*$. For $j=1,2,\dots,p$, we have
\begin{equation}
    \|\bar \mfw_j\|_2+\|\bar \mfw_j'\|_2\geq 0=\|\mfw_j^*\|_2.
\end{equation}
The equality holds when $\bar\mfw_j=\bar\mfw_j'=0$. As $\bar \mfW$ is optimal to \eqref{min_nrm:relu_skip}, we have $\bar \mfw_j=\bar \mfw_j'=0$ for $j=1,\dots,p$. This implies that $\bar \mfW=\hat \mfW$. Thus, $\hat \mfW$ is the unique optimal solution. 
\end{proof}

\section{Proofs in Section \ref{sec:neural_recovery}}

\subsection{Proof of Theorem \ref{thm:landscape}}
\begin{proof}
Let $m^*$ be the minimal number of neurons of the optimal solution to the convex program \eqref{min_nrm:relu}. In the case there are multiple optimal solutions, we may take the one with minimal cardinality.
We can view the minimal norm problem \eqref{prob:min_nrm} as
\begin{equation}
    \min_{\Theta} \ell(f(\mfX;\Theta);\mfy)+R(\Theta),
\end{equation}
where the loss function is defined by $$\ell(\mfz;\mfy)=\begin{cases}
\begin{aligned}
&0, &\mfz=\mfy,\\
&+\infty, &\mfz\neq \mfy.\\
\end{aligned}
\end{cases}$$
Hence, by applying Theorem 1 in \citep{wang2020hidden}, for $m\geq m^*$, any globally optimal solutions of the non-convex problem \eqref{prob:min_nrm} of ReLU networks can be computed via the optimal solutions of the convex program \eqref{min_nrm:relu} up to splitting and permutation. For the case $\mfy=(\mfX\mfw^*)_+$, we have $m^*=1$. 
\end{proof}

\subsection{Proof of Theorem \ref{thm:landscape_normal}}
\begin{proof}
Because $\mfy=\frac{\mfX\mfw^*}{\|\mfX\mfw^*\|_2}$, the ReLU network with skip connection can represent $\mfy$ using $m^*=1$ neurons. Let $\mfT_1,\dots,\mfT_q$ be the enumeration of all possible diagonal arrangement patterns 
$$
\{\diag(\mathrm{sign}(\mfX\mfw))|\mfw\in\mbR^d,\mfw\neq 0\}.
$$
We denote $Q_i$ be closed convex cone of solution vectors for $\mathrm{sign}(\mfX\mfw)=\mfT_i$ for $i\in[q]$. Let $B_i=Q_i\times \mbR_{>0}$ and $B_{i+q}=Q_i\times \mbR_{<0}$ for $i\in[q]$. Similar to the definition of minimal neural networks in \citep{wang2020hidden}, we can define the minimal neural networks with normalization layer as follows:
\begin{itemize}
    \item We say that a ReLU neural network with normalization layer $\Theta=(\mfW_1,\mfw_2,\alpha)$ is minimal if (i) it is scaled, i.e., $|w_{2,i}|=|\alpha_i|$ (ii) for each cone $B_i$ where $i\in[2q]$, the minimal neural network has at most a single non-zero neuron $(\mfw_{1,j},w_{2,j},\alpha_j)$ such that $(\mfw_{1,j},w_{2,j}\alpha_j)\in B_i$.
\end{itemize}

Similar to the proof of Theorem 1 in \citep{wang2020hidden}, for any globally optimal solution to the non-convex problem \eqref{prob:min_nrm}, we can merge it into a minimal neural network with normalization layer. This minimal neural network corresponds to an optimal solution to the convex program \eqref{min_nrm:relu_normal}. Therefore, we can view this globally optimal neural network as the split and permuted version of an optimal solution to the convex program \eqref{min_nrm:relu_normal}. 
\end{proof}

\section{Proofs in Section \ref{sec:neu_iso}}
To begin with, consider a general group $\ell_1$-minimization problem
\begin{equation}\label{min_nrm:general}
\begin{aligned}
 \min_{\left\{\mfw_{j}\right\}_{j=1}^{k}}\quad  \sum_{j=1}^{k}\left\|\mfw_{j}\right\|_{2},
        \text{s.t.}\quad &\sum_{j=1}^{k} \mfA_{j}\mfw_{j}=\mfy,
\end{aligned}
\end{equation}
where $\mfA_j\in\mbR^{n\times r_j}$ and $\mfw_j\in\mbR^{r_j}$ for $j\in[k]$. The KKT condition follows
\begin{equation}\label{kkt:general}
\begin{aligned}
    &\norm{\mfA_j^T\blbd}_2\leq 1, &\text{ if } \mfw_j=0,\\
    &\mfA_j^T\blbd = \frac{\mfw_j}{\|\mfw_j\|_2}, &\text{ if } \mfw_j\neq 0,\\
    &\sum_{j=1}^{k} \mfA_j \mfw_{j}=\mfy,
\end{aligned}
\end{equation}
where $\blbd\in\mbR^n$ is the dual variable. Suppose that $\mfy=\mfA_{i^*}\mfw^*$ is the label vector, where $\mfw^*\in\mbR^{s_{i^*}}$. Assume that $\mfA_{i^*}^T\mfA_{i^*}$ is invertible. Then, the irrepresentability condition follows
\begin{equation}\label{irrep:general}
    \norm{\mfA_j^T\mfA_{i^*}(\mfA_{i^*}^T\mfA_{i^*})^{-1}\frac{\mfw^*}{\|\mfw^*\|_2}}_2<1, \forall j\neq i^*.
\end{equation}

\begin{proposition}\label{prop:irrep_general}
Let $\mfy=\mfA_{i^*}\mfw^*$. Suppose that the irrepresentability condition \eqref{irrep:general} holds. Consider $\mfW^*=(\mfw_1^*,\dots,\mfw_k^*)$ such that $\mfw_{i^*}^*=\mfw^*$ while $\mfw_j^*=0$ for $j\neq i^*$. Then, $\mfW^*$ is the unique optimal solution to \eqref{min_nrm:general}.
\end{proposition}
\begin{proof}
We first show that $\mfW^*$ is the optimal solution to \eqref{min_nrm:general}. Let $$
\blbd=\mfA_{i^*}\pp{\mfA^T_{i^*}\mfA_{i^*}}^{-1}\frac{\mfw^*}{\|\mfw^*\|_2}.
$$
We can examine that $\mfA_{i^*}^T\blbd=\frac{\mfw^*}{\|\mfw^*\|_2}$. For $j\neq i^*$, from the irrepresentability condition \eqref{irrep:general}, we have $$
\norm{\mfA_{j}^T\blbd}_2=\norm{\mfA_j^T\mfA_{i^*}(\mfA_{i^*}^T\mfA_{i^*})^{-1}\frac{\mfw^*}{\norm{\mfw^*}_2}}_2<1.
$$
Therefore, $(\mfW^*,\blbd)$ satisfies the KKT condition \eqref{kkt:general}. This implies that $\mfW^*$ is optimal to \eqref{min_nrm:general}.

Then, we prove the uniqueness. Suppose that $\mfW=(\mfw_1,\dots,\mfw_k)$ is another optimal solution to \eqref{min_nrm:grelu}. Denote $\Delta_j=\mfw_j$ for $j\neq i^*$ and $\Delta_{i^*}=\mfw_{i^*}-\mfw_{i^*}^*$. Then, we have
\begin{equation}
    \sum_{j\neq i^*}\mfA_j\Delta_j+\mfA_{i^*}\Delta_{i^*}=0.
\end{equation}
We note that
\begin{equation}
\begin{aligned}
        &\|\mfw_{i^*}\|_2+\sum_{j\neq i^*} \|\mfw_j\|_2\\
    \geq&\|\mfw^*_{i^*}\|_2+\pp{\frac{\mfw^*_{i^*}}{\|\mfw^*_{i^*}\|_2}}^T\Delta_{i^*}+\sum_{j\neq i^*} \|\mfw_j\|_2\\
    = &\|\mfw^*_{i^*}\|_2+\blbd^T\mfA_{i^*}\Delta_{i^*}+\sum_{j\neq i^*} \|\Delta_j\|_2\\
    =&\|\mfw^*_{i^*}\|_2-\blbd^T\sum_{j\neq i^*}\mfA_j\Delta_j+\sum_{j\neq i^*} \|\Delta_j\|_2\\
    \geq &\|\mfw^*_{i^*}\|_2+\sum_{j\neq i^*} \|\Delta_j\|_2\pp{1-\norm{\mfA_j^T\blbd}_2} \geq \|\mfw^*_{i^*}\|_2
\end{aligned}
\end{equation}
The equality holds when $\Delta_j=0$ for all $j\neq i^*$ and $\Delta_{i^*}=\gamma \frac{\mfw^*_{i^*}}{\|\mfw^*_{i^*}\|_2}$ for certain $\gamma\geq 0$. This implies that $\mfW^*$ is the unique optimal solution to \eqref{min_nrm:grelu}. 
\end{proof}

Consider the following weak irrepresentability condition
\begin{equation}\label{irrep:general_weak}
\begin{aligned}
    \norm{\mfA_j^T\mfA_{i^*}(\mfA_{i^*}^T\mfA_{i^*})^{-1}\mfw^*}_2<1, j\in S_1, \norm{\mfA_j^T\mfA_{i^*}(\mfA_{i^*}^T\mfA_{i^*})^{-1}\mfw^*}_2=1, j\in S_2,
\end{aligned}
\end{equation}
where $S_1\cup S_2=[k]/\{i^*\}$. We have the following results when the weak irrepresentability condition holds.
\begin{proposition}\label{prop:irrep_general_weak}
Suppose that the weak irrepresentability condition holds. Consider $\mfW^*=(\mfw_1^*,\dots,\mfw_k^*)$ such that $\mfw_{i^*}^*=\mfw^*$ while $\mfw_j^*=0$ for $j\neq i^*$. Then, $\mfW^*$ is an optimal solution to \eqref{min_nrm:general}. All optimal solutions $\mfW=(\mfw_1,\dots,\mfw_k)$ shall satisfy $\mfw_j=0$ for $j\in S_1$.
\end{proposition}
\begin{proof}
Similar to the proof of Proposition \ref{prop:irrep_general}, we can show that $\mfW^*$ is the optimal solution to \eqref{min_nrm:general}. Suppose that $\mfW=(\mfw_1,\dots,\mfw_k)$ is another optimal solution to \eqref{min_nrm:grelu}. Denote $\Delta_j=\mfw_j$ for $j\neq i^*$ and $\Delta_{i^*}=\mfw_{i^*}-\mfw_{i^*}^*$. Then, we have
\begin{equation}
    \sum_{j\neq i^*}\mfA_j\Delta_j+\mfA_{i^*}\Delta_{i^*}=0.
\end{equation}
We note that
\begin{equation}
\begin{aligned}
        &\|\mfw_{i^*}\|_2+\sum_{j\neq i^*} \|\mfw_j\|_2\\
    \geq&\|\mfw^*_{i^*}\|_2+\pp{\frac{\mfw^*_{i^*}}{\|\mfw^*_{i^*}\|_2}}^T\Delta_{i^*}+\sum_{j\neq i^*} \|\mfw_j\|_2\\
    = &\|\mfw^*_{i^*}\|_2+\blbd^T\mfA_{i^*}\Delta_{i^*}+\sum_{j\neq i^*} \|\Delta_j\|_2\\
    =&\|\mfw^*_{i^*}\|_2-\blbd^T\sum_{j\neq i^*}\mfA_j\Delta_j+\sum_{j\neq i^*} \|\Delta_j\|_2\\
    \geq &\|\mfw^*_{i^*}\|_2+\sum_{j\neq i^*} \|\Delta_j\|_2\pp{1-\norm{\mfA_j^T\blbd}_2}\\
    \geq &\|\mfw^*_{i^*}\|_2.
\end{aligned}
\end{equation}
The equality holds when $\Delta_j=0$ for all $j\in S_1$ and $\Delta_{i^*}=\gamma \frac{\mfw^*_{i^*}}{\|\mfw^*_{i^*}\|_2}$ for certain $\gamma\geq 0$. This completes the proof.
\end{proof}

We then consider the case where $\mfy=\sum_{i=1}^l\mfA_{s_i}\mfw_{s_i}^*$, where $S=\{s_1,\dots,s_l\}\subseteq [k]$. Denote $\mfA_S^\mathrm{aug}=\bmbm{\mfA_{s_1}^T\\\dots\\\mfA_{s_k}^T}$. Suppose that $\mfA_S^\mathrm{aug}(\mfA_S^\mathrm{aug})^T=:\bmbm{\mfA_{s_1}^T\mfA_{s_1}&\dots&\mfA_{s_1}^T\mfA_{s_k}\\\vdots&\ddots&\vdots\\\mfA_{s_k}^T\mfA_{s_1}&\dots&\mfA_{s_k}^T\mfA_{s_k}}$ is invertible. Then, the irrepresentability condition follows
\begin{equation}\label{irrep:general_multi}
    \norm{\mfA_j^T
    (\mfA_S^\mathrm{aug})^T(\mfA_S^\mathrm{aug}(\mfA_S^\mathrm{aug})^T)^{-1}
    \bmbm{\mfw_{s_1}^*/\norm{\mfw_{s_1}^*}_2\\\vdots\\\mfw_{s_1}^*/\norm{\mfw_{s_k}^*}_2}}_2< 1, \forall j\notin S.
\end{equation}
or equivalently,
\begin{equation}
    \norm{\mfA_j^T
    (\mfA_S^\mathrm{aug})^\dagger
    \bmbm{\mfw_{s_1}^*/\norm{\mfw_{s_1}^*}_2\\\vdots\\\mfw_{s_1}^*/\norm{\mfw_{s_k}^*}_2}}_2< 1, \forall j\notin S.
\end{equation}

\begin{proposition}
\label{prop:irrep_general_group}
Suppose that the irrepresentability condition \eqref{irrep:general_multi} holds. Consider $\hat \mfW^*=(\hat \mfw_1^*,\dots,\hat \mfw_k^*)$ such that $\hat \mfw_{i}^*=\mfw^*_i$ for $i\in S$ while $\mfw_j^*=0$ for $j\notin S$. Then, $\hat \mfW^*$ is the unique optimal solution to \eqref{min_nrm:general}.
\end{proposition}
\begin{proof}
We first show that $\hat \mfW^*$ is the optimal solution to \eqref{min_nrm:general}. Let 
$$
\blbd=(\mfA_S^\mathrm{aug})^T(\mfA_S^\mathrm{aug}(\mfA_S^\mathrm{aug})^T)^{-1}
    \bmbm{\mfw_{s_1}^*/\norm{\mfw_{s_1}^*}_2\\\vdots\\\mfw_{s_1}^*/\norm{\mfw_{s_k}^*}_2}.
$$
We can examine that $\mfA_{i}^T\blbd=\frac{\mfw^*_i}{\|\mfw^*_i\|_2}$ for $i\in S$. For $j\notin S$, from the irrepresentability condition \eqref{irrep:general_multi}, we have $\norm{\mfA_{j}^T\blbd}_2<1$. Therefore, $(\mfW^*,\blbd)$ satisfies the KKT condition \eqref{kkt:general}. This implies that $\mfW^*$ is optimal to \eqref{min_nrm:general}.

Then, we prove the uniqueness. Suppose that $\mfW=(\mfw_1,\dots,\mfw_k)$ is another optimal solution to \eqref{min_nrm:grelu}. Denote $\Delta_j=\mfw_j-\hat\mfw^*_j$ for $j\in[k]$. Then, we have
\begin{equation}
    \sum_{j\notin S}\mfA_j\Delta_j+\sum_{i\in S}\mfA_{i}\Delta_{i}=0.
\end{equation}
We note that
\begin{equation}
\begin{aligned}
        &\sum_{i\in S}\|\mfw_{i}\|_2+\sum_{j\notin S} \|\mfw_j\|_2\\
    \geq&\sum_{i\in S}\|\hat \mfw^*_{i}\|_2+\sum_{i\in S}\pp{\frac{\hat \mfw^*_{i}}{\|\hat \mfw^*_{i}\|_2}}^T\Delta_{i}+\sum_{j\notin S} \|\mfw_j\|_2\\
    = &\sum_{i\in S}\|\hat \mfw^*_{i}\|_2+\blbd^T\sum_{i\in S}\mfA_{i}\Delta_{i}+\sum_{j\notin S} \|\Delta_j\|_2\\
    =&\sum_{i\in S}\|\hat \mfw^*_{i}\|_2-\blbd^T\sum_{j\notin  S}\mfA_j\Delta_j+\sum_{j\notin S} \|\Delta_j\|_2\\
    \geq &\sum_{i\in S}\|\hat \mfw^*_{i}\|_2+\sum_{j\notin S} \|\Delta_j\|_2\pp{1-\norm{\mfA_j^T\blbd}_2}\\
    \geq &\sum_{i\in S}\|\hat \mfw^*_{i}\|_2.
\end{aligned}
\end{equation}
The equality holds when $\Delta_j=0$ for all $j\notin S$ and $\Delta_{i}=\gamma_i \frac{\mfw^*_{i^*}}{\|\mfw^*_{i^*}\|_2}$ for $i\in S$. Here $\gamma_i\geq 0$ for $i\in S$. As $\Delta_j=0$ for all $j\notin S$, we have
\begin{equation}
    \sum_{i\in S}\mfA_{i}\Delta_{i}=0.
\end{equation}
Note that $\bmbm{\mfA_{s_1}^T\mfA_{s_1}&\dots&\mfA_{s_1}^T\mfA_{s_k}\\\vdots&\ddots&\vdots\\\mfA_{s_k}^T\mfA_{s_1}&\dots&\mfA_{s_k}^T\mfA_{s_k}}$ is invertible. This implies that $\Delta_{i}=0$ for $i\in S$. Hence, $\hat \mfW^*$ is the unique optimal solution to \eqref{min_nrm:grelu}. 

\end{proof}


\subsection{Proof of Proposition \ref{prop:imply}}
\begin{proof}
We first show that the neural isometry condition \eqref{irrep:grelu} implies the recovery of the planted model via \eqref{min_nrm:grelu}.

\begin{proposition}
\label{prop:irrep_relu}
Suppose that the neural isometry condition \eqref{irrep:grelu} holds. Consider $\mfW^*=(\mfw_1^*,\dots,\mfw_p^*)$ such that $\mfw_{i^*}^*=\mfw^*$ while $\mfw_j^*=0$ for $j\neq i^*$. Then, $\mfW^*$ is the unique optimal solution of \eqref{min_nrm:grelu}.
\end{proposition}
\begin{proof}
The KKT conditions for \eqref{min_nrm:grelu} include
\begin{equation}\label{kkt:grelu}
\begin{aligned}
    &\norm{\mfX^T\mfD_j\blbd^T}\leq 1, &\text{ if } \mfw_j=0,\\
    &\mfX^T\mfD_j\blbd^T = \frac{\mfw_j}{\|\mfw_j\|}, &\text{ if } \mfw_j\neq 0,\\
    &\sum_{j=1}^{p} \mfD_j \mfX \mfw_j=\mfy,
\end{aligned}
\end{equation}
where $\blbd\in\mbR^n$ is the dual variable. As a direct corollary of Proposition \ref{prop:irrep_general}, $\mfW^*$ is the unique optimal solution to \eqref{min_nrm:grelu}.
\end{proof}
Then, we present the proof of Proposition \ref{prop:imply}. As \eqref{min_nrm:grelu} is derived by dropping all inequality constraints in \eqref{min_nrm:relu}, the optimal value of \eqref{min_nrm:relu} is lower bounded by \eqref{min_nrm:grelu}. As $\mfW^*$ is the unique optimal solution to \eqref{min_nrm:grelu}, the optimal value of \eqref{min_nrm:grelu} is $\|\mfw^*\|_2$. On the other hand, $\hat \mfW$ is feasible for \eqref{min_nrm:relu} and it leads to an objective value of $\|\mfw^*\|_2$. This implies that $\hat \mfW$ is an optimal solution to \eqref{min_nrm:relu} and the optimal values of \eqref{min_nrm:relu} and \eqref{min_nrm:grelu} are the same. Suppose that we have another solution $\bar \mfW$ which is optimal to \eqref{min_nrm:relu}. Consider $\breve\mfW=\{\breve \mfw_j|j\in[p]\}$ where $\breve \mfw_j=\bar \mfw_j-\bar \mfw_j'$. Then, $\breve \mfW$ is also optimal to \eqref{min_nrm:grelu}. As $\mfW^*$ is the unique optimal solution to \eqref{min_nrm:grelu}, this implies that $\breve \mfW=\mfW^*$. Then, we note that for $j\neq i^*$,
\begin{equation}
    \|\bar \mfw_j\|_2+\|\bar \mfw_j'\|_2\geq 0=\|\mfw_j^*\|_2.
\end{equation}
The equality holds when $\bar\mfw_j=\bar\mfw_j'=0$.
We also note that $\bar \mfw_{i^*}-\bar \mfw_{i^*}'=\mfw^*$. This implies that
\begin{equation}
    \|\bar \mfw_{i^*}\|_2+\|\bar \mfw_{i^*}'\|_2\geq \|\bar \mfw_{i^*}-\bar \mfw_{i^*}'\|_2=\|\mfw^*\|_2.
\end{equation}
The equality holds if and only if there exists $\alpha\in[0,1]$ such that $\bar \mfw_{i^*}=\alpha \mfw^*$ and $\bar \mfw_{i^*}'=-(1-\alpha) \mfw^*$. If $\alpha=1$, then $\bar\mfW=\mfW^*$. If $\alpha<1$, as $(2\mfD_{i^*}-\mfI_n)\mfX\bar \mfw_{i^*}'\geq 0$ and  $(2\mfD_{i^*}-\mfI_n)\mfX \mfw^*\geq 0$, this implies that $(2\mfD_{i^*}-\mfI_n)\mfX \mfw^*=0$. Therefore, we have $\mfX \mfw^*=0$ and $\mfX^T\mfX\mfw^*=0$. This leads to a contradiction because $\mfw^*\neq 0$ and $\mfX^T\mfX$ is invertible with probability $1$ for $n>d$.
\end{proof}

\subsection{Proof of Proposition \ref{prop:imply_normal}}
\begin{proof}
We first show that the normalized neural isometry condition implies the the recovery of the planted model via \eqref{min_nrm:relu_normal}. 
\begin{proposition}
\label{prop:irrep_grelu_normal}
Suppose that the neural isometry condition \eqref{irrep:grelu_skip} holds. Consider $\mfW^*=(\mfw_0^*,\dots,\mfw_p^*)$ such that $\mfw_{i^*}^*=\frac{\bSigma_{i^*}\mfV_{i^*}\mfw^*}{\|\bSigma_{i^*}\mfV_{i^*}\mfw^*\|_2}$ while $\mfw_{j}^*=0$ for $j\neq i^*$. Then, $\mfW^*$ is the unique optimal solution to \eqref{min_nrm:grelu_normal}.
\end{proposition}
\begin{proof}
We note that 
$$
\mfy=\frac{(\mfX\mfw^*)_+}{\|(\mfX\mfw^*)_+\|_2}=\frac{\mfD_{i^*}\mfX\mfw^*}{\norm{\mfD_{i^*}\mfX\mfw^*}_2}=\frac{\mfU_{i^*}\bSigma_{i^*}\mfV_{i^*}\mfw^*}{\norm{\mfU_{i^*}\bSigma_{i^*}\mfV_{i^*}\mfw^*}_2}=\mfU_{i^*}\mfw_{i^*}^*.
$$
The KKT conditions for \eqref{min_nrm:grelu_normal} include
\begin{equation}\label{kkt:grelu_normal}
\begin{aligned}
    &\norm{\mfU_j\blbd^T}\leq 1, &\text{ if } \mfw_j=0,\\
    &\mfU_j^T\blbd^T = \frac{\mfw_j}{\|\mfw_j\|}, &\text{ if } \mfw_j\neq 0,\\
    &\sum_{j=1}^{p} \mfU_j \mfw_{j}=\mfy,
\end{aligned}
\end{equation}
As a direct corollary of Proposition \ref{prop:irrep_general}, $\mfW^*$ is the unique optimal solution to \eqref{min_nrm:grelu_normal}.
\end{proof}
We then present the proof of Proposition \ref{prop:imply_normal}. As \eqref{min_nrm:grelu_normal} is derived by dropping all inequality constraints in \eqref{min_nrm:relu_normal}, the optimal value of \eqref{min_nrm:relu_normal} is lower bounded by \eqref{min_nrm:grelu_normal}. As $\mfW^*$ is the unique optimal solution to \eqref{min_nrm:grelu_normal}, the optimal value of \eqref{min_nrm:grelu_normal} is $1$. On the other hand, $\hat \mfW$ is feasible for \eqref{min_nrm:relu} and it leads to an objective value of $1$. This implies that $\hat \mfW$ is an optimal solution to \eqref{min_nrm:relu_normal} and the optimal values of \eqref{min_nrm:relu_normal} and \eqref{min_nrm:grelu_normal} are the same. Suppose that we have another solution $\bar \mfW$ which is optimal to \eqref{min_nrm:relu_normal}. Consider $\breve\mfW=\{\bar \mfw_j|j\in[p]\}$ where $\breve \mfw_j=\bar \mfw_j-\bar \mfw_j'$. Then, $\breve \mfW$ is also optimal to \eqref{min_nrm:grelu_normal}. As $\mfW^*$ is the unique optimal solution to \eqref{min_nrm:grelu_normal}, this implies that $\breve \mfW=\mfW^*$. Then, we note that for $j\neq i^*$,
\begin{equation}
    \|\bar \mfw_j\|_2+\|\bar \mfw_j'\|_2\geq 0=\|\mfw_j^*\|_2.
\end{equation}
The equality holds when $\bar\mfw_j=\bar\mfw_j'=0$.
We also note that
\begin{equation}
    \|\bar \mfw_{i^*}\|_2+\|\bar \mfw_{i^*}'\|_2\geq \|\bar \mfw_{i^*}-\bar \mfw_{i^*}'\|_2=\|\mfw^*_{i^*}\|_2.
\end{equation}
The equality holds if and only if there exists $\alpha\in[0,1]$ such that $\bar \mfw_{i^*}=\alpha \mfw^*$ and $\bar \mfw_{i^*}'=-(1-\alpha) \mfw^*_{i^*}$. If $\alpha=1$, then $\bar\mfW=\bar \mfW$. If $\alpha<1$, as $(2\mfD_{i^*}-\mfI_n)\mfX\mfV_i^T\bSigma_i^{-1}\bar \mfw_{i^*}'\geq 0$ and  $(2\mfD_{i^*}-\mfI_n) \mfX\mfV_i^T\bSigma_i^{-1}\mfw^*_{i^*}\geq 0$, this implies that $(2\mfD_{i^*}-\mfI_n)\mfX\mfV_i^T\bSigma_i^{-1} \mfw^*_{i^*}=0$. Therefore, we have $\mfX\mfV_i^T\bSigma_i^{-1} \mfw^*=0$ and $\mfX^T\mfX\mfV_i^T\bSigma_i^{-1}\mfw^*_{i^*}=0$. This leads to a contradiction because $\mfV_i^T\bSigma_i^{-1}\mfw^*\neq 0$ and $\mfX^T\mfX$ is invertible with probability $1$ for $n>d$.
\end{proof}

\subsection{Proof of Proposition \ref{prop:imply_multi}}
\begin{proof}
We first show that the neural isometry condition \eqref{eq:irr_cond} implies the recovery of the planted model by solving \eqref{min_nrm:grelu}. 
\begin{proposition}\label{prop:irrep_multi}
Consider $\tilde \mfW^*=(\tilde \mfw_1^*,\dots,\tilde \mfw_p^*)$ satisfies that $\tilde \mfw_{s_i}^*=r_i^*\mfw_i^*$ for $i\in \{1,\dots,k\}$ and $\tilde \mfw_j^*=0$ for $j\notin S$. Suppose that the neural isometry condition \eqref{eq:irr_cond} is satisfied. Then, $\tilde \mfW^*$ is the unique optimal solution to \eqref{min_nrm:grelu_normal}.
\end{proposition}
\begin{proof}
The KKT conditions for \eqref{min_nrm:grelu} consist of
\begin{equation}
\begin{aligned}
    &\norm{\mfD_j\mfX\blbd^T}\leq 1, &\text{ if } \tilde\mfw_j=0,\\
    &\mfD_j\mfX^T\blbd^T = \frac{\tilde \mfw_j}{\|\tilde \mfw_j\|}, &\text{ if } \tilde\mfw_j\neq 0,\\
    &\sum_{j=1}^{p} \mfD_j\mfX \tilde \mfw_{j}=\mfy,
\end{aligned}
\end{equation}
As a direct corollary of Proposition \ref{prop:irrep_general_group}, $\tilde \mfW^*$ is the unique optimal solution to \eqref{min_nrm:grelu}.
\end{proof}
We then present the proof of Proposition \ref{prop:imply_multi}. From Proposition \ref{prop:irrep_multi}, we note that $\tilde \mfW^*$ is the unique optimal solution to \eqref{min_nrm:grelu_normal}. As \eqref{min_nrm:grelu} is derived by dropping all inequality constraints in \eqref{min_nrm:relu}, the optimal value of \eqref{min_nrm:relu} is lower bounded by \eqref{min_nrm:grelu}. As $\tilde \mfW^*$ is the unique optimal solution to \eqref{min_nrm:grelu}, the optimal value of \eqref{min_nrm:grelu} is $\sum_{i=1}^k\|\mfw^*_i\|_2$. On the other hand, $\hat \mfW$ is feasible for \eqref{min_nrm:relu} and it leads to an objective value of $\sum_{i=1}^k\|\mfw^*_i\|_2$. This implies that $\hat \mfW$ is an optimal solution to \eqref{min_nrm:relu} and the optimal values of \eqref{min_nrm:relu} and \eqref{min_nrm:grelu} are the same. 

Suppose that we have another solution $\bar \mfW$ which is optimal to \eqref{min_nrm:relu}. Consider $\breve\mfW=\{\breve \mfw_j|j\in[p]\}$ where $\breve \mfw_j=\bar \mfw_j-\bar \mfw_j'$. Then, $\bar \mfW$ is also optimal to \eqref{min_nrm:grelu}. As $\mfW^*$ is the unique optimal solution to \eqref{min_nrm:grelu}, this implies that $\breve \mfW=\mfW^*$, or equivalently, $\bar \mfw_j-\bar \mfw_j'=\tilde \mfw_j$ for $j\in[p]$. Then, we note that for $j\notin S$
\begin{equation}
    \|\bar \mfw_j\|_2+\|\bar \mfw_j'\|_2\geq 0=\|\mfw_j^*\|_2.
\end{equation}
The equality holds when $\bar\mfw_j=\bar\mfw_j'=0$.
For $i\in [k]$, we have $\bar \mfw_{s_i}-\bar \mfw_{s_i}'=\tilde \mfw_{s_i}=r_i^*\mfw^*_i$. This implies that 
\begin{equation}
    \|\bar \mfw_{s_i}\|_2+\|\bar \mfw_{i}'\|_2\geq \|\bar \mfw_{i}-\bar \mfw_{i}'\|_2=\|\mfw^*_i\|_2.
\end{equation}
The equality holds if and only if there exists $\alpha\in[0,1]$ such that $\bar \mfw_{s_i}=\alpha r_i^*\mfw^*_i$ and $\bar \mfw_{s_i}'=-(1-\alpha) r_i^*\mfw^*_i$. If $0<\alpha<1$, as $(2\mfD_{s_i}-\mfI_n)\mfX\bar \mfw_{s_i}'\geq 0$ and  $(2\mfD_{s_i}-\mfI_n)\mfX\bar \mfw_{s_i}\geq 0$, this implies that $(2\mfD_{s_i}-\mfI_n)\mfX \mfw^*_i=0$. Therefore, we have $\mfX \mfw^*_i=0$ and $\mfX^T\mfX\mfw^*_i=0$. This leads to a contradiction because $\mfw^*_i\neq 0$ and $\mfX^T\mfX$ is invertible with probability $1$ for $n>d$. If $r_i^*=1$, then $\alpha=1$. Otherwise, we have $\alpha=0$ and this leads to $\bar \mfw_{s_i}'=-\mfw^*_i$, which is contradictory to $(2\mfD_{s_i}-\mfI_n)\mfX\bar \mfw_{s_i}'\geq 0$. If $r_i^*=-1$, similarly, we have $\alpha=0$. Otherwise, it follows that $\alpha=1$ and this leads to $\bar \mfw_{s_i}=-\mfw^*_i$, which is contradictory to $(2\mfD_{s_i}-\mfI_n)\mfX\bar \mfw_{s_i}\geq 0$. 
\end{proof}

\subsection{Proof of Proposition \ref{prop:imply_multi_normal}}
\begin{proof}
We first show that the neural isometry condition \eqref{eq:irr_cond_normal} implies the recovery of the planted model by solving \eqref{min_nrm:grelu_normal}. 
\begin{proposition}\label{prop:irrep_multi_normal}
Consider $\tilde \mfW^*=(\tilde \mfw_1^*,\dots,\tilde \mfw_p^*)$ satisfies that $\tilde \mfw_i^*=r_i^*\frac{\bSigma_{s_i}\mfV_{s_i}\mfw_i^*}{\|\bSigma_{s_i}\mfV_{s_i}\mfw_i^*\|_2}$ for $i\in S$ and $\tilde \mfw_j^*=0$ for $j\notin S$.  Suppose that the neural isometry condition \eqref{eq:irr_cond_normal} is satisfied. Then, $\tilde \mfW^*$ is the unique optimal solution to \eqref{min_nrm:grelu_normal}.
\end{proposition}
\begin{proof}
We note that for $i\in[k]$, we have
$$
\frac{(\mfX\mfw_i^*)_+}{\|(\mfX\mfw_i^*)_+\|_2}=\frac{\mfD_{s_i}\mfX\mfw_i^*}{\norm{\mfD_{s_i}\mfX\mfw_i^*}_2}=\frac{\mfU_{s_i}\bSigma_{s_i}\mfV_{s_i}\mfw^*_i}{\norm{\mfU_{s_i}\bSigma_{s_i}\mfV_{s_i}\mfw^*_i}_2}=\mfU_{s_i}r_i^*\tilde \mfw_{i}^*.
$$
This implies that 
$$
\mfy=\sum_{i=1}^kr_i^*\frac{(\mfX\mfw_i^*)_+}{\|(\mfX\mfw_i^*)_+\|_2}=\sum_{i=1}^k\mfU_{s_i}\mfw_{i}^*\tilde \mfw_{i}^*.
$$
The KKT conditions for \eqref{min_nrm:grelu_normal} consist of
\begin{equation}
\begin{aligned}
    &\norm{\mfU_j\blbd^T}\leq 1, &\text{ if } \tilde \mfw_j=0,\\
    &\mfU_j^T\blbd^T = \frac{\tilde \mfw_j}{\|\tilde \mfw_j\|}, &\text{ if } \tilde \mfw_j\neq 0,\\
    &\sum_{j=1}^{p} \mfU_j \tilde \mfw_{j}=\mfy,
\end{aligned}
\end{equation}
As a direct corollary of Proposition \ref{prop:irrep_general_group}, $\tilde \mfW^*$ is the unique optimal solution to \eqref{min_nrm:grelu_normal}.
\end{proof}
We then present the proof of Proposition \ref{prop:imply_multi}. From Proposition \ref{prop:irrep_general_group}, $\tilde \mfW^*$ is the unique optimal solution to \eqref{min_nrm:grelu_normal}. Because \eqref{min_nrm:grelu_normal} is derived by dropping all inequality constraints in \eqref{min_nrm:relu_normal}, the optimal value of \eqref{min_nrm:relu_normal} is lower bounded by \eqref{min_nrm:grelu_normal}. As $\tilde \mfW^*$ is the unique optimal solution to \eqref{min_nrm:grelu_normal}, the optimal value of \eqref{min_nrm:grelu_normal} is $k$. On the other hand, $\hat \mfW$ is feasible for \eqref{min_nrm:relu_normal} and it leads to an objective value of $k$. This implies that $\hat \mfW$ is an optimal solution to \eqref{min_nrm:relu_normal} and the optimal values of \eqref{min_nrm:relu_normal} and \eqref{min_nrm:grelu_normal} are the same. 

Suppose that we have another solution $\bar \mfW$ which is optimal to \eqref{min_nrm:relu}. Consider $\breve\mfW=\{\breve \mfw_j|j\in[p]\}$ where $\breve \mfw_j=\bar \mfw_j-\bar \mfw_j'$. Then, $\bar \mfW$ is also optimal to \eqref{min_nrm:grelu}. As $\mfW^*$ is the unique optimal solution to \eqref{min_nrm:grelu}, this implies that $\breve \mfW=\mfW^*$, or equivalently, $\bar \mfw_j-\bar \mfw_j'=\tilde \mfw_j$ for $j\in[p]$. Then, we note that for $j\notin S$
\begin{equation}
    \|\bar \mfw_j\|_2+\|\bar \mfw_j'\|_2\geq 0=\|\mfw_j^*\|_2.
\end{equation}
The equality holds when $\bar\mfw_j=\bar\mfw_j'=0$.
For $i\in [k]$, we have $\bar \mfw_{s_i}-\bar \mfw_{s_i}'=\tilde \mfw_{s_i}=r_i^*\frac{\bSigma_i\mfV_i\mfw_i^*}{\|\bSigma_i\mfV_i\mfw_i^*\|_2}$. This implies that 
\begin{equation}
    \|\bar \mfw_{s_i}\|_2+\|\bar \mfw_{i}'\|_2\geq \|\bar \mfw_{i}-\bar \mfw_{i}'\|_2=1.
\end{equation}
The equality holds if and only if there exists $\alpha\in[0,1]$ such that $\bar \mfw_{s_i}=\alpha r_i^*\frac{\bSigma_{s_i}\mfV_{s_i}\mfw_i^*}{\|\bSigma_{s_i}\mfV_{s_i}\mfw_i^*\|_2}$ and $\bar \mfw_{s_i}'=-(1-\alpha) r_i^*\frac{\bSigma_{s_i}\mfV_{s_i}\mfw_i^*}{\|\bSigma_{s_i}\mfV_{s_i}\mfw_i^*\|_2}$. If $0<\alpha<1$, as $(2\mfD_{s_i}-\mfI_n)\mfX\mfV_{s_i}^T\bSigma_{s_i}^{-1}\bar \mfw_{s_i}'\geq 0$ and  $(2\mfD_{s_i}-\mfI_n)\mfX\bar \mfX\mfV_{s_i}^T\bSigma_{s_i}^{-1}\mfw_{s_i}\geq 0$, this implies that $(2\mfD_{s_i}-\mfI_n)\mfX\mfV_{s_i}^T\bSigma_{s_i}^{-1} \mfw^*_i=0$. Therefore, we have $\mfX \mfV_{s_i}^T\bSigma_{s_i}^{-1} \mfw^*_i=0$ and $\mfX^T\mfX\mfV_{s_i}^T\bSigma_{s_i}^{-1} \mfw^*_i=0$. This leads to a contradiction because $\mfV_{s_i}^T\bSigma_{s_i}^{-1} \mfw^*_i\neq 0$ and $\mfX^T\mfX$ is invertible with probability $1$ for $n>d$. If $r_i^*=1$, then we have $\alpha=1$. Otherwise, we have $\alpha=0$ and this leads to $\bar \mfw_{s_i}'=-\frac{\bSigma_{s_i}\mfV_{s_i}\mfw_i^*}{\|\bSigma_{s_i}\mfV_{s_i}\mfw_i^*\|_2}$, which is contradictory to $(2\mfD_{s_i}-\mfI_n)\mfX\mfV_{s_i}^T\bSigma_{s_i}^{-1}\bar \mfw_{s_i}'\geq 0$. If $r_i^*=-1$, similarly, we have $\alpha=0$. Otherwise, we have $\alpha=1$ and this leads to $\bar \mfw_{s_i}=-\frac{\bSigma_{s_i}\mfV_{s_i}\mfw_i^*}{\|\bSigma_{s_i}\mfV_{s_i}\mfw_i^*\|_2}$, which is contradictory to $(2\mfD_{s_i}-\mfI_n)\mfX\mfV_{s_i}^T\bSigma_{s_i}^{-1}\bar \mfw_{s_i}\geq 0$. 
\end{proof}

\section{Proofs in Section \ref{sec:relu_skip}}
\subsection{Proof of Lemma \ref{lem:kinematic}}
\begin{proof}
We first illustrate Theorem 1 in \citep{amelunxen2014living} as follows:
\begin{theorem}
Fix a tolerance $\eta\in(0,1)$. Let $C$ and $K$ be convex cones in $\mbR^n$. Draw a random orthogonal basis $\mfW\in\mbR^{n\times n}$. Then, 
\begin{equation}
    P(K\cap\mfW C \neq \{0\})\begin{cases}
    \begin{aligned}
            &\leq \eta&\delta (C)+\delta (K)\leq n-a_{\eta}\sqrt{n},\\
        &\geq 1-\eta&\delta (C)+\delta (K)\geq n+a_{\eta}\sqrt{n}.\\
    \end{aligned}
    \end{cases}
\end{equation}
where $a_\eta=8\sqrt{\log(4/\eta)}$. 
\end{theorem}
We can take $C=\bbbb{t\bmbm{\mfw\\\boldsymbol{0}}|t\geq 0}$, where $\boldsymbol{0}\in \mbR^{n-d}$ is a vector of $0$s.  Then, $P(\mfX\mfw\in K)=P(\mfW C\cap K\neq \{0\})$. As $w(C)=d$, for $w(K)+d<n$, we have $a_{\eta}=\frac{n-\delta (K)-d}{\sqrt{n}}$. This implies that $\eta=4e^{-\frac{(n-\delta (K)-d)^2}{64 n}}$. Hence, by taking $\alpha=\frac{(n-\delta (K)-d)^2}{64 n^2}$, \eqref{equ:kine} holds. Similarly, for $\delta (K)+d>n$, we have $a_{\eta}=\frac{w(K)+d-n}{\sqrt{n}}$, which implies that $\eta=4e^{-\frac{(n-\delta (K)-d)^2}{64 n}}$. By taking $\alpha=\frac{(n-\delta (K)-d)^2}{64 n^2}$, \eqref{equ:kine} holds
\end{proof}

\subsection{Proof of Theorem \ref{thm:gauss_success}}
\begin{proof}
We first note that the matrix in the irrepresentability condition \eqref{irrep:grelu_skip} has the following upper bound
$$
\begin{aligned}
     &\max_{\mfh\in\mbR^d: \mfh\neq 0}\norm{\mfX^T\diag(\mbI(\mfX\mfh\geq 0))\mfX\pp{\mfX^T\mfX}^{-1}}_2
 \le \frac{1}{\sigma^2_{\min}}\max_{\mfh\in\mbR^d: \mfh\neq 0}\|\mfX^T\diag(\mbI(\mfX\mfh\geq 0))\mfX\|_2,
\end{aligned}
$$
where $\sigma_{\min}$ is the smallest singular value of $\mfX$. We introduce a lemma to bound the norm of $\mfX^T\diag(\mbI(\mfX\mfh\geq 0))\mfX$.
\begin{lemma}\label{lem:net}
Let $h\in \mbR^d$ and $t>0$ be fixed. Suppose that each element $x_{i,j}$ of $\mfX\in \mbR^{n\times d}$ are i.i.d. random variables following a mean-zero sub-Gaussian distribution with variance proxy $\sigma^2$ such that
\begin{itemize}
    \item $\mbE[x_{i,j}^2]=\frac{1}{n}$.
    \item $-x_{i,j}$ has the same distribution as $x_{i,j}$.
\end{itemize}
Then, for $d\geq 2$, with probability at least 
$$
1-4\exp\pp{-\frac{nt^2}{162\sqrt{2}\sigma^2}+\frac{d\log(54n)}{2}}
$$
we have
\begin{equation}
    \max_{\mfh\in\mbR^d,\mfh\neq 0}\|\mfX^T\diag(\mbI(\mfX\mfh\geq 0))\mfX\|_2\leq \frac{1}{2}+t.
\end{equation}
\end{lemma}
From Lemma \ref{lem:net}, by taking $t=1/4$, for $n\geq  4000\sigma^2d\log(54n)$, we have 
\begin{equation}
\begin{aligned}
   &P\pp{ \max_{\mfh\in\mbR^d: \mfh\neq 0}\|\mfX^T\diag(\mbI(\mfX\mfh\geq 0))\mfX\|_2\leq \frac{3}{4}}\geq 1-4\exp\pp{-\frac{n}{8000\sigma^2}}.
\end{aligned}
\end{equation}
On the other hand, from Theorem 4.6.1 in \citep{vershynin2018high}, we can lower bound the smallest eigenvalue of $\mfX$,
\begin{equation}
    P\pp{\sigma_\mathrm{min}\geq 1-2\sqrt{\frac{d}{n}}}\geq 1-2\exp(-d).
\end{equation}
For $n\geq 1024 d$. we have
\begin{equation}
\begin{aligned}
    &P\pp{\sigma_\mathrm{min}^2\geq 7/8}
    \geq P\pp{\sigma_\mathrm{min}\geq 1-1/16}
    \geq P\pp{\sigma_\mathrm{min}\geq 1-2\sqrt{\frac{d}{n}}}
    \geq 1-2\exp(-d).
\end{aligned}
`\end{equation}
This implies that for $n$ satisfying $n\geq \max\{1024d,4000\sigma^2d\log(54n)\}$, with probability at least $1-2\exp(-d)-4\exp\pp{-\frac{n}{8000 \sigma^2}}$, we have
\begin{equation}
     \max_{j\in[p]}\norm{\mfX^T\mfD_{j}\mfX\pp{\mfX^T\mfX}^{-1}}_2\leq \frac{3}{4\sigma^2_{\min}}\leq \frac{3}{4}\cdot \frac{8}{7}<1.
\end{equation}
Conditioned on the above event, we note that for any $\mfw^*\in\mbR^d, \mfw^*\neq 0$, we have
\begin{equation}
\begin{aligned}
    \max_{\mfh\in\mbR^d: \mfh\neq 0}\norm{\mfX^T\diag(\mbI(\mfX\mfh\geq 0))\mfX\pp{\mfX^T\mfX}^{-1}\frac{\mfw^*}{\|\mfw^*\|_2}}_2
    \leq \max_{\mfh\in\mbR^d: \mfh\neq 0}\norm{\mfX^T\diag(\mbI(\mfX\mfh\geq 0))\mfX\pp{\mfX^T\mfX}^{-1}}_2<1,
\end{aligned}
\end{equation}
i.e., the irrepresentability condition \eqref{irrep:grelu_skip} holds. This completes the proof. 
\end{proof} 

\subsection{Proof of Lemma \ref{lem:net}}
\begin{proof}
For a positive semi-definite matrix $\mfA\in\mbR^{d\times d}$, we have $\|\mfA\|_2=\max_{\mfz\in\mbR^d:\|\mfz\|_2=1}\mfz^T\mfA\mfz$. Note that
\begin{equation}
\begin{aligned}
    &\left| \max_{\mfh\in\mbR^d,\mfh\neq 0}\|\mfX^T\diag(\mbI(\mfX\mfh\geq 0))\mfX\|_2-\frac{1}{2}\right|\\
    =&\left| \max_{\mfh\in\mbR^d,\mfh\neq 0}\norm{\sum_{i=1}^n\mfx_i\mfx_i^T\mbI(\mfh^T\mfx_i\geq 0))}_2-\frac{1}{2}\right|\\
    =&\left| \max_{\mfz,\mfh\in\mbR^d:\|\mfz\|_2=1,\mfh\neq 0}\sum_{i=1}^n\mfz^T\mfx_i\mfx_i^T\mfz\mbI(\mfh^T\mfx_i\geq 0))-\frac{1}{2}\right|\\
    \leq&  \max_{\mfz,\mfh\in\mbR^d:\|\mfz\|_2=1,\mfh\neq 0}\left|\sum_{i=1}^n(\mfz^T\mfx_i)^2\mbI(\mfh^T\mfx_i\geq 0))-\frac{1}{2}\right|.
\end{aligned}
\end{equation}
By rescaling $\mfx_i=\sqrt{n}\mfx_i$, we can rewrite the above quantity as 
\begin{equation}
    \max_{\mfz,\mfh\in\mbR^d:\|\mfz\|_2=1,\mfh\neq 0}\left|\frac{1}{n}\sum_{i=1}^n(\mfz^T\mfx_i)^2\mbI(\mfh^T\mfx_i\geq 0))-\frac{1}{2}\right|.
\end{equation}
As $-x_{i,j}$ has the same distribution as $x_{i,j}$, we have
\begin{equation}
\begin{aligned}
    &P(\mfh^T\mfx_i\geq 0)=P(\mfh^T\mfx_i\leq 0),\\
    &\mbE[(\mfz^T\mfx_i)^2|\mfh^T\mfx_i\geq 0]=\mbE[(-\mfz^T\mfx_i)^2|\mfh^T(-\mfx_i)\geq 0]=\mbE[(\mfz^T\mfx_i)^2|\mfh^T\mfx_i\leq 0].
\end{aligned}
\end{equation}
We note that $\mbE[(\mfz^T\mfx_i)^2]=\|\mfz\|_2^2=1$ and
\begin{equation}
    \mbE[(\mfz^T\mfx_i)^2]=\mbE[(\mfz^T\mfx_i)^2|\mfh^T\mfx_i\geq 0]P(\mfh^T\mfx_i\geq 0)+\mbE[(\mfz^T\mfx_i)^2|\mfh^T\mfx_i\leq 0]P(\mfh^T\mfx_i\leq 0).
\end{equation}
This implies that $\mbE[(\mfz^T\mfx_i)^2|\mfh^T\mfx_i\geq 0]P(\mfh^T\mfx_i\geq 0)=\frac{1}{2}$, or equivalently, 
\begin{equation}
    \mbE[(\mfz^T\mfx_i)^2\mbI(\mfh^T\mfx_i\geq 0)]=
\frac{1}{2}.
\end{equation}
Therefore, it is sufficient to upper-bound the following probability
\begin{equation}
    P\pp{\max_{\mfz, \mfh\in\mbR^d:\|\mfz\|_2=1,\mfh\neq 0}\left|\frac{1}{n}\sum_{i=1}^n(\mfz^T\mfx_i)^2\mbI(\mfh^T\mfx_i\geq 0)-\frac{1}{2}\right|\geq t}
\end{equation}
We note that
\begin{equation}
    \begin{aligned}
        &P\pp{\max_{\mfz, \mfh\in\mbR^d:\|\mfz\|_2=1,\mfh\neq 0}\left|\frac{1}{n}\sum_{i=1}^n(\mfz^T\mfx_i)^2\mbI(\mfh^T\mfx_i\geq 0)-\mbE\bb{\frac{1}{n}\sum_{i=1}^n(\mfz^T\mfx_i')^2\mbI(\mfh^T\mfx_i'\geq 0) }\right|\geq t}\\
        &\cdot P\pp{\max_{\mfz, \mfh\in\mbR^d:\|\mfz\|_2=1,\mfh\neq 0}\left|\frac{1}{n}\sum_{i=1}^n(\mfz^T\mfx_i)^2\mbI(\mfh^T\mfx_i\geq 0)-\mbE\bb{\frac{1}{n}\sum_{i=1}^n(\mfz^T\mfx_i')^2\mbI(\mfh^T\mfx_i'\geq 0) }\right|\geq t}\\
        \geq& P\pp{\max_{\mfz, \mfh\in\mbR^d:\|\mfz\|_2=1,\mfh\neq 0}\left|\frac{1}{n}\sum_{i=1}^n(\mfz^T\mfx_i)^2\mbI(\mfh^T\mfx_i\geq 0)-\frac{1}{n}\sum_{i=1}^n(\mfz^T\mfx_i')^2\mbI(\mfv^T\mfx_i'\geq 0)\right|\geq 2t}.
    \end{aligned}
\end{equation}
where $\mfx_i'$ are i.i.d. random vectors following the same distribution of $\mfx_i$ and they are independent with $\mfx_1,\dots,\mfx_n$. This implies that
\begin{equation}
\begin{aligned}
    &P\pp{\max_{\mfz, \mfh\in\mbR^d:\|\mfz\|_2=1,\mfh\neq 0}\left|\frac{1}{n}\sum_{i=1}^n(\mfz^T\mfx_i)^2\mbI(\mfh^T\mfx_i\geq 0)-\mbE\bb{\frac{1}{n}\sum_{i=1}^n(\mfz^T\mfx_i')^2\mbI(\mfh^T\mfx_i'\geq 0) }\right|\geq t}\\
    \leq &P\pp{\max_{\mfz, \mfh\in\mbR^d:\|\mfz\|_2=1,\mfh\neq 0}\left|\frac{1}{n}\sum_{i=1}^n(\mfz^T\mfx_i)^2\mbI(\mfh^T\mfx_i\geq 0)-\frac{1}{n}\sum_{i=1}^n(\mfz^T\mfx_i')^2\mbI(\mfv^T\mfx_i'\geq 0)\right|\geq 2t}^{\frac{1}{2}},
\end{aligned}
\end{equation}
By introducing i.i.d. random variables $\epsilon_i$ uniformly distributed in $\{-1,1\}$, we have the following bound
\begin{equation}
\begin{aligned}
&P\pp{\max_{\mfz, \mfh\in\mbR^d:\|\mfz\|_2=1,\mfh\neq 0}\left|\frac{1}{n}\sum_{i=1}^n(\mfz^T\mfx_i)^2\mbI(\mfh^T\mfx_i\geq 0)-\frac{1}{n}\sum_{i=1}^n(\mfz^T\mfx_i')^2\mbI(\mfh^T\mfx_i'\geq 0)\right|\geq 2t}\\
=&P\pp{\max_{\mfz, \mfh\in\mbR^d:\|\mfz\|_2=1,\mfh\neq 0}\left|\frac{1}{n}\sum_{i=1}^n\epsilon_i\pp{(\mfz^T\mfx_i)^2\mbI(\mfh^T\mfx_i\geq 0)-(\mfz^T\mfx_i')^2\mbI(\mfh^T\mfx_i'\geq 0)}\right|\geq 2t}\\
=&P\pp{\max_{\mfz, \mfh\in\mbR^d:\|\mfz\|_2=1,\mfh\neq 0}\left|\frac{1}{2n}\sum_{i=1}^{2n}\epsilon_i(\mfz^T\mfx_i)^2\mbI(\mfh^T\mfx_i\geq 0)\right|\geq t},
\end{aligned}
\end{equation}
where $\epsilon_{n+1},\dots,\epsilon_{2n}$ are i.i.d. copies of  $\epsilon_1,\dots,\epsilon_n$. 

\textbf{Decoupling:}
Next, we apply a decoupling result in \citep[Theorem 3.4.1]{de2012decoupling} to obtain the following upper bound: 
\begin{equation}
\begin{aligned}
&P\pp{\max_{\mfz, \mfh\in\mbR^d:\|\mfz\|_2=1,\mfh\neq 0}\left|\frac{1}{2n}\sum_{i=1}^{2n}\epsilon_i(\mfz^T\mfx_i)^2\mbI(\mfh^T\mfx_i\geq 0)\right|\geq t}\\
\leq&
8 P\pp{\max_{\mfz, \mfh\in\mbR^d:\|\mfz\|_2=1,\mfh\neq0}\left|\frac{1}{2n}\sum_{i=1}^{2n}\epsilon_i(\mfz^T\mfx_i)^2\mbI(\mfh^T\mfx^\prime_i\geq 0)\right|\geq t}\,,
\end{aligned}
\end{equation}
were $\{\mfx_i^\prime\}_{i=1}^{2n}$ is an independent and identically distributed copy of the sequence $\{\mfx_i\}_{i=1}^{2n}$.

\textbf{$\epsilon$-net bound:} For each realization of $x_1',\dots,\mfx_{2n}'$, let $\mcP'=\{\diag(\mbI(\mfX'\mfh\geq 0))|\mfh\in\mbR^d,\mfh\neq 0\}$, $p'=|\mcP'|$ and we write $\mcP'=\{\mfD_1,\dots,\mfD_{p'}\}$. According to \citep{cover1965geometrical}, we have the upper bound $p'\leq 2d(2en/d)^d$.  We note that
\begin{equation}
\begin{aligned}
    &\max_{\mfz, \mfh\in\mbR^d:\|\mfz\|_2=1,\mfh\neq 0}\left|\frac{1}{2n}\sum_{i=1}^{2n}\epsilon_i(\mfz^T\mfx_i)^2\mbI(\mfh^T\mfx^\prime_i\geq 0)\right|=\max_{j\in[p']}\max_{\mfz\in\mbR^d:\|\mfz\|_2=1}\left|\frac{1}{2n}\sum_{i=1}^{2n}\epsilon_i(\mfz^T\mfx_i)^2(\mfD_j)_{i,i}\right|
\end{aligned}
\end{equation}

Consider an $\epsilon$-net of $\{\mfz\in \mbR^d:\|z\|_2=1\}$, $\{\mfz_1,\dots,\mfz_N\}$, where $N\leq (1+2/\epsilon)^d$. Namely, for any $\mfz\in\mbR^d$ satisfying $\|\mfz\|_2=1$, there exists $k\in[N]$ such that $\mfz=\mfz_k+\Delta$, where $\|\Delta\|_2\leq \epsilon$. Then, we have
\begin{equation}
\begin{aligned}
&\max_{\mfz\in\mbR^d:\|\mfz\|_2=1}\left|\frac{1}{2n}\sum_{i=1}^{2n}\epsilon_i(\mfz^T\mfx_i)^2(\mfD_j)_{i,i}\right|\\
\leq &\max_{k\in[N],\|\Delta\|_2\leq \epsilon} \left|\frac{1}{2n}\sum_{i=1}^{2n}\epsilon_i((\mfz_k+\Delta)^T\mfx_i)^2(\mfD_j)_{i,i}\right|\\
\leq &\max_{k\in[N]}\left|\frac{1}{2n}\sum_{i=1}^{2n}\epsilon_i(\mfz_k^T\mfx_i)^2(\mfD_j)_{i,i}\right|+(2\epsilon +\epsilon^2)\max_{\|\mfz\|_2\leq 1}\left|\frac{1}{2n}\sum_{i=1}^{2n}\epsilon_i(\mfz^T\mfx_i)^2(\mfD_j)_{i,i}\right|
\end{aligned}
\end{equation}
Here we utilize that for an arbitrary symmetric matrix $\mfA\in \mbR^{d\times d}$ and arbitrary vector $\mfz \in \mbR^d$ with $\|\mfz\|_2=1$,
\begin{equation}
     \max_{\|\Delta\|_2\leq \epsilon}|(\mfz+\Delta)^T\mfA(\mfz+\Delta)|\leq |\mfz^T\mfA\mfz|+(2\epsilon+\epsilon^2)\|\mfA\|_2.
\end{equation}
This implies that
\begin{equation}
\begin{aligned}
    &\max_{\mfz\in\mbR^d:\|\mfz\|_2=1}\left|\frac{1}{2n}\sum_{i=1}^{2n}\epsilon_i(\mfz^T\mfx_i)^2(\mfD_j)_{i,i}\right|
    \leq \frac{1}{1-2\epsilon-\epsilon^2} \max_{k\in[N]}\left|\frac{1}{2n}\sum_{i=1}^{2n}\epsilon_i(\mfz_k^T\mfx_i)^2(\mfD_j)_{i,i}\right|
\end{aligned}
\end{equation}
For fixed $k$, we note that $\mfz_k^T\mfx_i$ is also sub-Gaussian with variance proxy $\sigma^2$. Let $g_i=(\mfz_k^T\mfx_i)^2$. 
Therefore,  $\epsilon_ig_i^2(\mfD_j)_{i,i}$ is sub-exponential with parameters $(4\sqrt{2}\sigma^2(\mfD_j)_{i,i},4\sigma^2)$. This implies that $\sum_{i=1}^{2n}\epsilon_ig_i^2(\mfD_j)_{i,i}$ is sub-exponential with parameter $(\nu^*,4\sigma^2)$, where
\begin{equation}
    \nu^*=\sqrt{\sum_{i=1}^{2n}4\sqrt{2}\sigma^2(\mfD_j)_{i,i}}=\sqrt{4\sqrt{2}\sigma^2\tr(\mfD_j)}.
\end{equation}
Therefore, for $t\leq \frac{(\nu^*)^2}{4\sigma^2}=\sqrt{2}\tr(\mfD_j)$, we have
\begin{equation*}
\begin{aligned}
P\pp{ \left|\frac{1}{2n} \sum_{i=1}^{2n} \epsilon_i(\mfz_k^T\mfx_i)^2(\mfD_j)_{i,i} \right|>t}
   \leq 2\exp\pp{-2nt^2/(2(\nu^*)^2)}
   =2\exp\pp{-\frac{n^2t^2}{4\sqrt{2}\sigma^2\tr(\mfD_j)}}
   \leq 2\exp\pp{-\frac{nt^2}{4\sqrt{2}\sigma^2}}.
\end{aligned}
\end{equation*}
This implies that
\begin{equation}
\begin{aligned}
        P\pp{ \max_{k\in[N]}\left|\frac{1}{2n} \sum_{i=1}^{2n} \epsilon_i(\mfz_k^T\mfx_i)^2(\mfD_j)_{i,i} \right|>t}
        \leq 2N\exp\pp{-\frac{nt^2}{4\sqrt{2}\sigma^2}}
        \leq 2\exp\pp{-\frac{nt^2}{4\sqrt{2}\sigma^2}+d\log(3/\epsilon)}.
\end{aligned}
\end{equation}
Again, by applying the union bound, we have
\begin{equation}
\begin{aligned}
     &P\pp{\max_{j\in[p']} \max_{k\in[N]}\left|\frac{1}{2n} \sum_{i=1}^{2n} \epsilon_i(\mfz_k^T\mfx_i)^2(\mfD_j)_{i,i} \right|>t}\\
     \leq&p'\cdot P\pp{ \max_{k\in[N]}\left|\frac{1}{2n} \sum_{i=1}^{2n} \epsilon_i(\mfz_k^T\mfx_i)^2(\mfD_j)_{i,i} \right|>t}\\
     \leq&2\exp\pp{-\frac{nt^2}{4\sqrt{2}\sigma^2}+d\log(3/\epsilon)+d(\log (2n)-\log d +1)+\log(2d)}\\
     \leq&2\exp\pp{-\frac{nt^2}{4\sqrt{2}\sigma^2}+d\log(3/\epsilon)+d\log(6n)},
\end{aligned}
\end{equation}
where we assume that $d\geq 2$. As a result, by taking $\epsilon=1/3$, we have
\begin{equation}
\begin{aligned}
&P\pp{\max_{\mfz, \mfh\in\mbR^d:\|\mfz\|_2=1,\mfh\neq 0}\left|\frac{1}{n}\sum_{i=1}^n(\mfz^T\mfx_i)^2\mbI(\mfh^T\mfx_i\geq 0)-\frac{1}{2}\right|\geq t}^2\\
\leq &8\, P\pp{\max_{\mfz, \mfh\in\mbR^d:\|\mfz\|_2=1,\mfh\neq 0}\left|\frac{1}{2n}\sum_{i=1}^{2n}\epsilon_i(\mfz^T\mfx_i)^2\mbI(\mfh^T\mfx^\prime_i\geq 0)\right|\geq t}\\
\leq &8P\pp{ \max_{j\in[p']}\max_{k\in[N]}\left|\frac{1}{2n} \sum_{i=1}^{2n} \epsilon_i(\mfz_k^T\mfx_i)^2(\mfD_j)_{i,i} \right|>(1-2\epsilon-\epsilon^2)t}\\
\leq &16\exp\pp{-\frac{nt^2(1-2\epsilon-\epsilon^2)^2}{4\sqrt{2}\sigma^2}+d\log(3/\epsilon)+d\log (6n)}\\
= &16\exp\pp{-\frac{nt^2}{81\sqrt{2}\sigma^2}+d\log(54n)}.
\end{aligned}
\end{equation}
This completes the proof.
\end{proof}

\subsection{Proof of Theorem \ref{prop:skip_gauss_success}}
\begin{proof}
Without the loss of generality, we can let $\mfw^*$ satisfies that $\|\mfw^*\|_2=1$. 
As $n>2d$, 
$\mfX^T\mfX$ is invertible with probability $1$. Consider the event 
\begin{equation}
    E=\bbbb{\max_{\mfh\in\mbR^d:\mbI(\mfX\mfh\geq 0)\neq \bone}\|\mfX^T\diag(\mbI(\mfX\mfh\geq 0))\mfX(\mfX^T\mfX)^{-1}\mfw^*\|_2<1}.
\end{equation}
Firstly, we show that 
\begin{equation}
    P\pp{E}=1.
\end{equation}
Denote $A(\mfh)=\mfX^T\diag(\mbI(\mfX\mfh\geq 0))\mfX$ and $\tilde A(\mfh)=(\mfX^T\mfX)^{-1}A(\mfh)^TA(\mfh)(\mfX^T\mfX)^{-1}$. We note that $A(\mfh)$ and $\mfX^T\mfX$ are positive semi-definite and symmetric. As $A(\mfh)\succeq \mfX^T\mfX$, we have $A(\mfh)^TA(\mfh)\preceq (\mfX^T\mfX)^2$ and 
$$
\tilde A(\mfh)\preceq I. 
$$
The equality holds if and only if $A(\mfh)=\mfX^T\mfX$. This is equivalent to $\mfX^T\diag(\mbI(\mfX\mfh\geq 0)-\bone)\mfX=0$, which contradicts with $\mbI(\mfX\mfh\geq 0)\neq \bone$. Hence, we also have 
$$
\tilde A(\mfh)\neq I.
$$
Recall our notation $\eigmax\pp{\tilde A(\mfh)}$ used for the subspace of maximal eigenvectors of the symmetric matrix $\tilde A(\mfh)$.
We note that $\|\tilde A(\mfh)\mfw^*\|_2=1$ if and only if $\|\tilde A(\mfh)\|_2=1$ and $\mfw^*\in \eigmax\pp{\tilde A(\mfh)}$. As $\tilde A(\mfh)\neq \mfI_n$, conditioned on $\|\tilde A(\mfh)\|_2=1$, $\eigmax\pp{\tilde A(\mfh)}$ is a random subspace with dimension at most $d-1$. This implies that
\begin{equation}
    P(\|\tilde A(\mfh)\mfw^*\|_2=1)=P\pp{\|\tilde A(\mfh)\|_2=1, \mfw^*\in \eigmax\pp{\tilde A(\mfh)}}=0.
\end{equation}
For $\bsigma\in\{0,1\}^n$, define $B(\bsigma)=\mfX^T\diag(\bsigma)\mfX$ and $\tilde B(\bsigma)=(\mfX^T\mfX)^{-1} B(\bsigma)^T B(\bsigma)(\mfX^T\mfX)^{-1}$. Then, we have
\begin{equation}
\begin{aligned}
    &P\pp{\max_{\mfh\in \mbR^d:\mbI(\mfX\mfh\geq 0)\neq \bone}\|\tilde B(\mbI(\mfX\mfh\geq 0)) \mfw^*\|_2=1}\\
    =& P\pp{\max_{\bsigma\in \{0,1\}^n: \exists \mfh\neq 0,  \mbI(\mfX\mfh\geq 0)\neq \bone, \bsigma=\mbI(\mfX\mfh\geq 0)}\|\tilde B(\bsigma) \mfw^*\|_2=1}\\
    \leq & \sum_{\bsigma\in\{0,1\}^n} P(\|\tilde B(\bsigma) \mfw^*\|_2=1, \mbI(\mfX\mfh\geq 0)\neq \bone, \bsigma=\mbI(\mfX\mfh\geq 0))=0.
\end{aligned}
\end{equation}

From the kinematic formula, for $n>2d$, we have $P(\mfI_n\notin H)\geq 1-\exp(-n\alpha)$. In this case, the event $E$ implies that the neural isometry condition \eqref{irrep:grelu_skip} holds. This completes the proof. 
\end{proof}

\subsection{Proof of Proposition \ref{prop:positive_nd}}
\begin{proof}
Without the loss of generality, we can let $\mfw^*$ satisfies that $\|\mfw^*\|_2=1$. 
As $n>d$, 
$\mfX^T\mfX$ is invertible with probability $1$. Consider the event 
\begin{equation}
    E=\bbbb{\max_{\mfh\in\mbR^d:\mbI(\mfX\mfh\geq 0)\neq \bone}\|\mfX^T\diag(\mbI(\mfX\mfh\geq 0))\mfX(\mfX^T\mfX)^{-1}\mfw^*\|_2<1}.
\end{equation}
First, we show that 
\begin{equation}
    P\pp{E}=1.
\end{equation}
Denote $A(\mfh)=\mfX^T\diag(\mbI(\mfX\mfh\geq 0))\mfX$ and $\tilde A(\mfh)=(\mfX^T\mfX)^{-1}A(\mfh)^TA(\mfh)(\mfX^T\mfX)^{-1}$. We note that $A(\mfh)$ and $\mfX^T\mfX$ are positive semi-definite and symmetric. As $A(\mfh)\preceq \mfX^T\mfX$, we have $A(\mfh)^TA(\mfh)\preceq (\mfX^T\mfX)^2$ and 
$$
\tilde A(\mfh)\preceq I. 
$$
The equality holds if and only if $A(\mfh)=\mfX^T\mfX$. This is equivalent to $\mfX^T\diag(\mbI(\mfX\mfh\geq 0)-\bone)\mfX=0$, which contradicts with $\mbI(\mfX\mfh\geq 0)\neq \bone$. Hence, we also have 
$$
\tilde A(\mfh)\neq I.
$$
We note that $\|\tilde A(\mfh)\mfw^*\|_2=1$ if and only if $\|\tilde A(\mfh)\|_2=1$ and $\mfw^*\in \eigmax\pp{\tilde A(\mfh)}$. As $\tilde A(\mfh)\neq \mfI_n$, conditioned on $\|\tilde A(\mfh)\|_2=1$, $\eigmax\pp{\tilde A(\mfh)}$ is a random subspace with dimension at most $d-1$. This implies that
\begin{equation}
    P(\|\tilde A(\mfh)\mfw^*\|_2=1)=P\pp{\|\tilde A(\mfh)\|_2=1, \mfw^*\in \eigmax\pp{\tilde A(\mfh)}}=0.
\end{equation}
For $\bsigma\in\{0,1\}^n$, define $B(\bsigma)=\mfX^T\diag(\bsigma)\mfX$ and $\tilde B(\bsigma)=(\mfX^T\mfX)^{-1} B(\bsigma)^T B(\bsigma)(\mfX^T\mfX)^{-1}$. Then, we have
\begin{equation}
\begin{aligned}
    &P\pp{\max_{\mfh\in \mbR^d:\mbI(\mfX\mfh\geq 0)\neq \bone}\|\tilde B(\mbI(\mfX\mfh\geq 0)) \mfw^*\|_2=1}\\
    =& P\pp{\max_{\bsigma\in \{0,1\}^n: \exists \mfh\neq 0,  \mbI(\mfX\mfh\geq 0)\neq \bone, \bsigma=\mbI(\mfX\mfh\geq 0)}\|\tilde B(\bsigma) \mfw^*\|_2=1}\\
    \leq & \sum_{\bsigma\in\{0,1\}^n} P(\|\tilde B(\bsigma) \mfw^*\|_2=1, \mbI(\mfX\mfh\geq 0)\neq \bone, \bsigma=\mbI(\mfX\mfh\geq 0))=0.
\end{aligned}
\end{equation}
Then, according to Proposition \ref{prop:irrep_general_weak}, conditioned on the event $E$, the optimal solution $\mfW=(\mfw_0,\dots,\mfw_p)$ to \eqref{min_nrm:grelu_skip} shall satisfy that $\mfw_j=0$ for $\mfD_j\neq \mfI_n$. If there does not exist $j\in[p]$ such that $\mfD_j=\mfI_n$, then, $\mfW^*=(\mfw^*,0,\dots,0)$ is the unique optimal solution to \eqref{min_nrm:grelu_skip}. Thus, it is also the unique optimal solution to \eqref{min_nrm:relu_skip}.

If there exists $i^*\in[p]$ such that $\mfD_{i^*}=\mfI_n$. Let $\mfW=(\mfw_0,\mfw_1,\mfw_1',\dots,\mfw_p,\mfw_p')$ be an optimal solution to \eqref{min_nrm:relu_skip}. Let $\hat \mfw_j=\mfw_j-\mfw_j'$ for $j\in[p]$. Then, $\hat \mfW=(\mfw_0,\hat \mfw_1,\dots,\hat \mfw_p)$ is also optimal to \eqref{min_nrm:grelu_skip}. This implies that $\hat \mfw_j=0$ for $j\in[p]$ such that $j\neq i^*$. For $j\in[p]$ such that $j\neq i^*$, we have
\begin{equation}
    \|\hat \mfw_j\|_2+\|\hat \mfw_j'\|_2\geq 0.
\end{equation}
We also note that
\begin{equation}
    \mfX(\mfw_0+\mfw_{i^*}-\mfw_{i^*}')=\mfX\mfw^*.
\end{equation}
As $\mfX^T\mfX$ is invertible, we have $\mfw_0+\mfw_{i^*}-\mfw_{i^*}'=\mfw^*$. Thus, we have
\begin{equation}
    \|\mfw_0\|_2+\|\mfw_{i^*}\|_2+\|\mfw_{i^*}'\|_2\geq \|\mfw^*\|_2.
\end{equation}
The equality holds when there exists $\alpha_1,\alpha_2,\alpha_3\geq 0$ such that $\alpha_1+\alpha_2+\alpha_3=1$, $\mfw_0=\alpha_1\mfw^*$, $\mfw_{i^*}=\alpha_2\mfw^*$ and $\mfw_{i^*}'=-\alpha_3\mfw^*$. However, as $\mfX\mfw^*\geq 0$ does not hold and $\mfX\mfw_{i^*}\geq 0$, $\mfX\mfw_{i^*}'\geq 0$, we have $\alpha_2=\alpha_3=0$. This implies that $\mfW^*=(\mfw^*,0,\dots,0)$ is the unique optimal solution to \eqref{min_nrm:grelu_skip}. 

\end{proof}

\subsection{Proof of Theorem \ref{thm:white_strong}}
\begin{proof}
Let $K_j=\{u:D_ju\geq 0\}$ and denote $\mcC=\{j:\tr(D_j)>n-d\}$. We note that
\begin{equation}
\begin{aligned}
    &P\pp{\max_{\mfh\in\mbR^d:\mfh\neq 0} \tr(\mbI(\mfX\mfh\geq 0))\leq n-d}\\
    =&P(\exists \mfw\in\mbR^d:\mfX\mfw\in \cup_{j\in \mcC}K_j, \mfw\neq 0).
\end{aligned}
\end{equation}

As $\cup_{j\in \mcC}K_j$ is not a convex set, we cannot directly apply the kinematic formula. Let $K=\pp{\cup_{j\in \mcC}K_j}\cap S^{n-1}$, where $S^{n-1}=\{\mfz\in\mbR^{n}|\|\mfz\|_2=1\}$. It is a closed subset of $S^{n-1}$. We give the lower bound of the success probability based on the Gordon's escape through a mesh theorem.
\begin{lemma}
Let $K$ be a closed subset of $S^{n-1}$. Define the Gaussian width of $K$ by:
\begin{equation}
    w(K)=\mbE_{\mfg\sim\mcN(0,I_n)}\bb{\max_{\mfz\in K}\mfg^T\mfz}.
\end{equation}
Define $a_k=\mbE_{\mfg\sim \mcN(0,I_k)}[\|\mfg\|_2]$ for $k\in\mbN$. Then, for a $n-k$ dimensional subspace $L\subseteq \mbR^n$ drawn at random, we have
\begin{equation}
    P(L\cap K\neq \varnothing)\geq 1-\frac{7}{2}e^{-\frac{1}{18}(a_k-w(K))}.
\end{equation}
\end{lemma}
Note that $\mfX \mfw$ is a random $d$-dimensional subspace of $\mbR^n$. According to the Gordon's escape through a mesh theorem, we have
\begin{equation}
    P(\exists \mfw\in\mbR^d:\mfX\mfw\in \cup_{j\in \mcC}K_j, \mfw\neq 0)\geq 1-\frac{7}{2}e^{-\frac{1}{18}(a_{n-d}-w(K))}.
\end{equation}
To ensure that there exists $\mfw\neq 0$ such that $\mfX\mfw\in \cup_{j\in \mcC}K_j$ with high probability, we require that $a^2_{n-d}>w(K)^2$. As $\frac{k}{k+1}k\leq a^2_{k}\leq k$, it is sufficient to have $w(K)^2<n-d$. Therefore, it suffices to calculate the squared Gaussian width of $K$. We can compute that
\begin{equation}
\begin{aligned}
w(K)^2=&\pp{\mbE\bb{ \max_{j\in\mcC}\max_{\mfz\in\mbR^n: D_j\mfz\geq 0, \|\mfz\|_2=1} \mfg^T\mfz}}^2\\
= &\pp{\mbE\bb{  \max_{j\in \mcC}\|(\mfD_j\mfg)_++((\mfD_j-\mfI_n)\mfg)_+\|_2}}^2\\
\leq &\mbE\bb{\pp{ \max_{j\in \mcC}\|(\mfD_j\mfg)_++((\mfD_j-\mfI_n)\mfg)_+\|_2 }^2}\\
=&\mbE\bb{\max_{j\in \mcC}\|(\mfD_j\mfg)_+\|_2^2+\|((\mfD_j-\mfI_n)\mfg)_+\|_2^2}.
\end{aligned}
\end{equation}  
By noting that
\begin{equation*}
\begin{aligned}
&\|(\mfD_j\mfg)_+\|_2^2+\|((\mfD_j-\mfI_n)\mfg)_+\|_2^2\\
=&\|(\mfg)_+\|_2^2-\|((\mfI_n-\mfD_j)\mfg)_+\|_2^2+\|((\mfD_j-\mfI_n)\mfg)_+\|_2^2,
\end{aligned}
\end{equation*}
we have
\begin{equation}
\begin{aligned}
&\mbE\bb{\max_{j\in \mcC}\|(\mfD_j\mfg)_+\|_2^2+\|((\mfD_j-\mfI_n)\mfg)_+\|_2^2}\\
=&\mbE\bb{\|(\mfg)_+\|_2^2+ \max_{j\in \mcC}\pp{\|((\mfD_j-\mfI_n)\mfg)_+\|_2^2-\|((\mfI_n-\mfD_j)\mfg)_+\|_2^2}}\\
=&\frac{n}{2}+\mbE\bb{\max_{j\in \mcC}\sum_{i:(\mfD_j)_{i,i}=0}-\sgn(\mfg_i)\mfg_i^2}
=\frac{n}{2}+\mbE\bb{\max_{S\subseteq [n]:|S|\leq d-1}\sum_{i\in S}\sgn(\mfg_i)\mfg_i^2}\\
\end{aligned}
\end{equation}
Let $R=\sgn(G)G^2$, where $G\sim \mcN(0,1)$. Denote $F_R$ be the CDF of the random variable $R$. Suppose that $d,n\to\infty$ with a fixed ratio $d=\theta n$, $\frac{1}{n}\mbE\bb{\max_{S\subseteq [n]:|S|\leq d-1}\sum_{i\in S}\sgn(\mfg_i)\mfg_i^2}$ will converge to
\begin{equation}\label{equ:F_r}
    \int_{F_r^{-1}(1-\theta)}^\infty rdF_R(r).
\end{equation}
We note that
\begin{equation*}
    1-F_R(r)=P(R\geq r)=\frac{1}{2}P(G^2\geq r) = \frac{1}{2}(1- F_{\xi^2}(r)).
\end{equation*}
Therefore, we can rewrite \eqref{equ:F_r} as $\frac{1}{2}\int_{F_{\xi^2}^{-1}(1-2\theta)}^\infty rdF_{\xi^2}(r)$. Denote
\begin{equation}
    g(\theta)=\frac{1}{2}+\frac{1}{2}\int_{F_{\xi^2}^{-1}(1-2\theta)}^\infty rdF_{\xi^2}(r)+\theta.
\end{equation}
We note that $g(\theta)$ monotonically increases for $\theta\in[0,1/2)$. By noting that $g(0)=1/2$ and $g(0.5)=1.5$, there uniquely exists $\theta^*\in(0,1/2)$ such that $g(\theta)=1$. We also note that
\begin{equation}
    \lim_{n\to\infty,d=\theta n} \frac{1}{n}\pp{w^2(K)+d}=g(\theta).
\end{equation}
Therefore, for sufficiently large $n,d$ with $d<n\theta$, we have $w^2(K)<n-d$, which implies that \eqref{max_Xh} holds w.h.p..

We present a numerical way to compute $g(\theta)$. Denote $q=F_{\xi^2}^{-1}(1-2\theta)$.By integration by parts, we can compute that
\begin{equation}
\begin{aligned}
    &\frac{1}{2}\int_q^\infty rdF_{\xi^2}(r)=-\frac{1}{2}\int_q^\infty rd(1-F_{\xi^2}(r)) \\
    =&-\frac{1}{2}r(1-F_{\xi^2}(r))|_q^\infty+\int_q^\infty (1-F_{\xi^2}(r))dr\\
    =&q\theta +\frac{1}{2}\int_q^\infty (1-F_{\xi^2}(r))dr.
\end{aligned}
\end{equation}
By the numerical quadrature of the survival function of the $\xi^2$ random variable with 1 degree of freedom, we plot $g(\theta)$ in Figure \ref{fig:gtheta}. Note that when $\theta\approx 0.1314$, we have $g(\theta)\approx1$.

\begin{figure}[H]
\centering
\begin{minipage}[t]{0.45\textwidth}
\centering
\includegraphics[width=\linewidth]{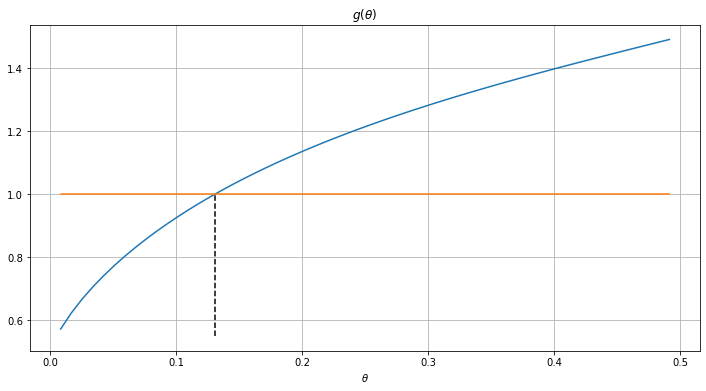}
\end{minipage}
\caption{$g(\theta)$ as a function of $\theta$ in $[0,0.5]$. }\label{fig:gtheta}
\end{figure}
\end{proof}

\subsection{Proof of Theorem \ref{thm:relu_skip_phase}}
\begin{proof}
For $\mfh\in\mbR^d$ with $\mfh\neq 0$, denote $A(\mfh)=\mfX^T\diag(\mbI(\mfX\mfh\geq 0))\mfX$. We first prove the case where $n<2d$. 
According to the kinematic formula, $P(\exists \mfh\in\mbR^d:\mfh\neq 0, A(\mfh)=\mfI_n)\geq 1-\exp^{-\alpha n}$, where $\alpha=\frac{(n/2-d)^2}{64n^2}$. 
In other words, there exists an all-ones hyperplane arrangement with probability at least $1-\exp(-n\alpha)$. In this case, let $\mfD_{j}=\mfI_n$. Construct $\tilde \mfw_i=\mfw_i'=0$ if $i\neq j$, $\tilde \mfw_j=\mfw^*$ and $\tilde \mfw_j'=0$. Then, $\mfW'=(0,\tilde \mfw_1,\tilde \mfw_1',\dots,\tilde \mfw_p,\tilde \mfw_p')$ is also a feasible solution to \eqref{min_nrm:grelu_skip}. 

Then, we consider the case where $n>2d$. We show that with probability at least $1-\exp(-\alpha n)$, the event \eqref{equ:event_hu} holds. For $\bsigma\in \{0,1\}^n$, we let $B(\bsigma)=\mfX^T\diag(\bsigma)\mfX$. 
As $\mfX^T\mfX=\mfI_d$, we note that $\|B(\bsigma)\mfw^*\|_2=1$ if and only if $\|B(\bsigma)\|_2=1$ and $\mfw^*\in\eigmax(B(\bsigma))$. For $\bsigma\in \{0,1\}^n$ with $ \bsigma\neq \bone$ and  $\|B(\bsigma)\|_2=1$, as $B(\bsigma)\neq I$, we have $\dim(\eigmax(B(\bsigma)))\leq d-1$. This implies that $P(\mfw^*\in \eigmax(B(\bsigma)|\|B(\bsigma)\|_2=1)=0$. Therefore, we have the bound:
\begin{equation}
\begin{aligned}
    &P\pp{\max_{\mfh\in \mbR^d:\mfh\neq 0}\|A(\mfh)\mfw^*\|_2=1}\\
    =& P\pp{\max_{\bsigma: \exists  \mfh\neq 0, \bsigma=\mbI(\mfX\mfh\geq 0)}\|B(\bsigma)\mfw^*\|_2=1}\\
    \leq & \sum_{\bsigma\in\{0,1\}^n} P(\|B(\bsigma)\mfw^*\|_2=1,\exists \mfh\neq 0, \bsigma=\mbI(\mfX\mfh\geq 0))\\
    \leq&P(\exists \mfh\neq 0, \bone=\mbI(\mfX\mfh\geq 0))+\sum_{\bsigma\in\{0,1\}^n, \bsigma\neq \bone }P(\|B(\bsigma)\mfw^*\|_2<1)\\
    =&P(\exists \mfh\neq 0, \bone=\mbI(\mfX\mfh\geq 0)).
\end{aligned}
\end{equation}
 According to the kinematic formula, for $n>2d$, $P(\exists \mfh\neq 0, \bone=\mbI(\mfX\mfh\geq 0))\leq \exp^{-\alpha n}$. This implies that $$P\pp{\max_{\mfh\in \mbR^d,\mfh\neq 0}\|A(\mfh)\mfw^*\|_2<1}=1-P\pp{\max_{\mfh\in \mbR^d,\mfh\neq 0}\|A(\mfh)\mfw^*\|_2=1}\geq 1-\exp^{-\alpha n}.$$

Therefore for $n>2d$, the neural isometry condition \eqref{equ:event_hu} holds with probability at least $1-\exp(-n\alpha)$. From Proposition \ref{prop:irrep_grelu_skip}, the neural isometry condition \eqref{equ:event_hu} implies that $\mfW$ is the unique optimal solution to \eqref{min_nrm:grelu_skip}.

\end{proof}
\subsection{Proof of Proposition \ref{prop:noisy}}
\begin{proof}
To ensure that $\mfW$ is an optimal solution to \eqref{gl:general}, we only require that the KKT conditions \eqref{kkt:gl_general} at $\mfW$ are satisfied, i.e., 
\begin{equation}\label{kkt:gl_general}
\begin{aligned}
    &\norm{\mfA_j^T\blbd}_2\leq \beta, &\text{ if } j\neq i^*,\\
    &\mfA_{i^*}^T\blbd = -\beta\frac{ \mfw_{i^*}}{\| \mfw_{i^*}\|_2},\\
    &\mfA_{i^*} ({\mfw}_{i^*}-\mfw^*)-\mfz=\blbd.
\end{aligned}
\end{equation}
The last two equations give
\begin{equation}
    \mfA_{i^*}^T\mfA_{i^*} {\mfw}_{i^*}+\beta \frac{\mfw_{i^*}}{\|{\mfw}_{i^*}\|_2}=\mfA_{i^*}^T\mfA_{i^*} \mfw^*+\mfA_{i^*}^T\mfz.
\end{equation}

Let us write ${\mfw}_{i^*}=r\mfw $ where $r=\|{\mfw}_{i^*}\|_2$ and $\mfw =\frac{{\mfw}_{i^*}}{\|{\mfw}_{i^*}\|_2}$. Then, the above expression gives 
\begin{equation}\label{equ:w}
    (r\mfA_{i^*}^T\mfA_{i^*} +\beta \mfI)\mfw = \mfA_{i^*}^T\mfA_{i^*} \mfw^*+\mfA_{i^*}^T\mfz.
\end{equation}
As $r\mfA_{i^*}^T\mfA_{i^*}+\beta \mfI$ is invertible, we obtain an explicit solution for $\mfw$, i.e.,
\[
\mfw=(r\mfA_{i^*}^T\mfA_{i^*}+\beta \mfI)^{-1}(\mfA_{i^*}^T\mfA_{i^*}\mfw^*+\mfA_{i^*}^T\mfz).
\]
Because $\mfw$ satisfies $\|\mfw\|_2=1$, the scalar $r$ shall satisfy
\begin{equation}\label{eq:noisy}
  1=\left\|(r\mfA_{i^*}^T\mfA_{i^*}+\beta \mfI)^{-1}(\mfA_{i^*}^T\mfA_{i^*}\mfw^*+\mfA_{i^*}^T\mfz)\right\|_2.  
\end{equation}
Because $\mfA_{i^*}^T\mfA_{i^*}=\mfI$, we have $r+\beta=\|\mfw^*+\mfA_{i^*}^T\mfz\|_2$. This implies that $r=\|\mfw^*+\mfA_{i^*}^T\mfz\|_2-\beta$. As $r\geq 0$, we require that $\beta\leq \|\mfw^*+\mfA_{i^*}^T\mfz\|_2$. A sufficient condition is that $\beta\leq \|\mfw^*\|_2-\|\mfz\|_2$. We can write the expression of $\blbd$ as
\begin{equation}
    \begin{aligned}
        \blbd =& \mfA_{i^*} ({\mfw}_{i^*}-\mfw^*)-\mfz\\
        =&\mfA_{i^*}(\mfA_{i^*}^T\mfA_{i^*}+\beta/r \mfI)^{-1}(\mfA_{i^*}^T\mfA_{i^*}\mfw^*+\mfA_{i^*}^T\mfz)-\mfA_{i^*}\mfw^*-\mfz\\
        =&\frac{r}{r+\beta}\mfA_{i^*}\mfw^*-\mfA_{i^*}\mfw+\frac{r}{r+\beta} \mfA_{i^*}\mfA_{i^*}^T\mfz-\mfz\\
        =&-\frac{\beta}{r+\beta} \mfA_{i^*}\mfw^*+\pp{\mfI-\frac{r}{r+\beta}\mfA_{i^*}\mfA_{i^*}^T}\mfz.
    \end{aligned}
\end{equation}
Because $\mfA_{i^*}\mfA_{i^*}^T$ is a projection matrix whose eigenvalues are $0$ and $1$, we have $\|\mfI-r/(r+\beta)\mfA_{i^*}\mfA_{i^*}^T\|_2\leq 1$.  Therefore, for $j\neq i^*$, we have the upper bound
\begin{equation}
    \begin{aligned}
        \|\mfA_j^T\blbd \|_2 &=\norm{-\frac{\beta}{r+\beta} \mfA_j^T\mfA_{i^*}\mfw^*+\mfA_j^T\pp{\mfI-\frac{r}{r+\beta}\mfA_{i^*}\mfA_{i^*}^T}\mfz}_2\\
        &\le \frac{\beta}{r+\beta}\|\mfA_j^T\mfA_{i^*}\mfw^*\|_2+\|\mfA_j\|_2\|\mfz\|_2\\
        &= \frac{\beta}{\norm{\mfw^*+\mfA_{i^*}^T\mfz}
        _2}\|\mfA_j^T\mfA_{i^*}\mfw^*\|_2+\|\mfA_j\|_2\|\mfz\|_2\\
        &\le \frac{\beta}{\|\mfw^*\|
        _2-\|\mfA_{i^*}^T\mfz\|_2}\|\mfA_j^T\mfA_{i^*}\mfw^*\|_2+\|\mfz\|_2\\
        &\le \frac{\beta}{\|\mfw^*\|
        _2-\|\mfz\|_2}\|\mfA_j^T\mfA_{i^*}\mfw^*\|_2+\|\mfz\|_2\\
        &\le \beta \frac{1-\gamma}{1-\|\mfz\|_2/\|\mfw^*\|_2}+\|\mfz\|_2\leq \beta.
    \end{aligned}
\end{equation}
Here we utilize that $\|\mfA_j\|_2\leq 1$ and
\begin{equation}
    \beta\geq \|\mfz\|_2\frac{1-\|\mfz\|_2/\|\mfw^*\|_2}{\gamma-\|\mfz\|_2/\|\mfw^*\|_2}.
\end{equation}
Therefore, it implies that there exists a solution $\mfW=(\mfw_1,\dots,{\mfw}_k)$ such
that $\mfw_j^*=0$ for $j\neq i^*$. In addition, we provide the upper bound for the $\ell_2$ norm $\norm{\mfw_{i^*}-\mfw^*}_2$.

\begin{equation}
    \begin{aligned}
        \norm{\mfw_{i^*}-\mfw^*}_2 & = \norm{r\mfw-\mfw^*}_2\\
        &=\norm{\frac{r}{r+\beta}\left(\mfw^*+\mfA_{i^*}^T\mfz\right)-\mfw^*}_2\\
        &\le \frac{\beta}{r+\beta}\|\mfw^*\|_2
        +\frac{r}{r+\beta}\norm{\mfA_{i^*}^T\mfz}_2\\
        &\le \frac{\beta\|\mfw^*\|_2}{\norm{\mfw^*+\mfA_{i^*}^T\mfz}_2} + \|\mfz\|_2\\
        &\le \frac{\beta\|\mfw^*\|_2}{\|\mfw^*\|_2-\|\mfz\|_2} + \|\mfz\|_2.
    \end{aligned}
\end{equation}

\end{proof}
\subsection{Proof of Theorem \ref{thm:noisy}}

In order to prove these theorems, we consider a generic group Lasso problem
\begin{equation}\label{gl:general}
 \min_{\left\{\mfw_{j}\right\}_{j=1}^{k}}\quad  \left\|\sum_{j=1}^{k} \mfA_{j}\mfw_{j}-\mfy\right\|_2^2+\beta\sum_{j=1}^{k}\left\|\mfw_{j}\right\|_{2}.
\end{equation}
The next result provides a sufficient condition on the regularization parameter and the norm of the noise component to ensure successful support recovery, as well as an estimation of the upper bound on the $\ell_2$ distance between the optimal solution and the embedded neuron.

\begin{proposition}\label{prop:noisy}
Let $\mfy=\mfA_{i^*}\mfw^*+\mfz$, where $\mfA_{i^*}$ satisfies $(\mfA_{i^*})^T\mfA_{i^*}=\mfI_d$. Assume that $\|\mfA_j\|_2\leq 1$ for $j\neq i^*$.  Suppose that the following condition holds.
$$
\max_{j\neq i^*}\|\mfA_j^T\mfA_{i^*}\mfw^*\|_2\leq (1-\gamma)\|\mfw^*\|_2
$$
for a certain scalar constant $\gamma>0$. Further, suppose that $\norm{\mfz}_2\leq \frac{\gamma}{2} \norm{\mfw^*}_2$. Then, for $\beta\in\bb{\norm{\mfz}_2\frac{\|\mfw^*\|_2-\norm{\mfz}_2}{\gamma\|\mfw^*\|_2-\norm{\mfz}_2},\|\mfw^*\|_2-\norm{\mfz}_2}$, 
there exists a solution $\mfW=(\mfw_1,\dots,{\mfw}_k)$ such
that $\mfw_j^*=0$ for $j\neq i^*$.
Moreover, the $\ell_2$ norm $\norm{\mfw_{i^*}-\mfw^*}_2$ is bounded as follows:
\[\norm{\mfw_{i^*}-\mfw^*}_2\le \frac{\beta\|\mfw^*\|_2}{\|\mfw^*\|_2-\|\mfz\|_2} + \|\mfz\|_2.
\]
\end{proposition}
In order to apply Proposition \ref{prop:noisy}, we need to estimate the upper bound of 
$$
     \max_{\mfh\in\mbR^d: \mfh\neq 0}\norm{\mfU^T\diag(\mbI(\mfX\mfh\geq 0))\mfU}_2.
$$

From Lemma \ref{lem:net}, by taking $t=1/4$, for $n\geq  4000\sigma^2d\log(54n)$, we have 
\begin{equation}
\begin{aligned}
   P\pp{ \max_{\mfh\in\mbR^d: \mfh\neq 0}\|\mfX^T\diag(\mbI(\mfX\mfh\geq 0))\mfX\|_2\leq \frac{3}{4}}
   \geq 1-4\exp\pp{-\frac{n}{8000\sigma^2}}.
\end{aligned}
\end{equation}
Note that $\mfX=\mfU\bSigma\mfV^T$, therefore
\begin{equation}
\norm{\mfX^T\mfD_j\mfX}_2= \norm{\bSigma\mfU^T\mfD_j\mfU\bSigma}_2
\ge\sigma_{\min}^2\norm{\mfU^T\mfD_j\mfU}_2,
\end{equation}
where $\sigma_{\min}$ is the smallest singular value of $\mfX$.

On the other hand, from Theorem 4.6.1 in \citep{vershynin2018high}, we have the following lower bound.
\begin{equation}
    P\pp{\sigma_\mathrm{min}\geq 1-2\sqrt{\frac{d}{n}}}\geq 1-2\exp(-d).
\end{equation}
For $n\geq 1024 d$, we have
\begin{equation}
\begin{aligned}
    &P\pp{\sigma_\mathrm{min}^2\geq 7/8}
    \geq P\pp{\sigma_\mathrm{min}\geq 1-1/16}
    \geq P\pp{\sigma_\mathrm{min}\geq 1-2\sqrt{\frac{d}{n}}}
    \geq 1-2\exp(-d).
\end{aligned}
\end{equation}
This implies that for $n$ satisfying $n\geq \max\{1024d,4000\sigma^2d\log(54n)\}$, we have
\begin{equation}
\begin{aligned}
    &P\pp{\max_{j\in[p]}\norm{\mfX^T\mfD_{j}\mfX}_2\leq \frac{3}{4}}\le P\pp{\sigma_{\min}^2\max_{j\in[p]}\norm{\mfU^T\mfD_{j}\mfU}_2\leq\frac{3}{4} }      \\
    \le & P\pp{\sigma_{\min}^2\max_{j\in[p]}\norm{\mfU^T\mfD_{j}\mfU}_2\leq\frac{3}{4}\mid \sigma_{\min}^2\ge \frac{7}{8}} P\pp{\sigma_{\min}^2\ge \frac{7}{8}}\\
    &+ P\pp{\sigma_{\min}^2\max_{j\in[p]}\norm{\mfU^T\mfD_{j}\mfU}_2\leq\frac{3}{4}\mid \sigma_{\min}^2\le \frac{7}{8}} P\pp{\sigma_{\min}^2\le \frac{7}{8}}\\
    \le & P\pp{\max_{j\in[p]}\norm{\mfU^T\mfD_{j}\mfU}_2\leq \frac{6}{7}}P\pp{\sigma_{\min}^2\ge \frac{7}{8}} + 1 - P\pp{\sigma_{\min}^2\ge \frac{7}{8}}.
\end{aligned}
\end{equation}
Therefore, we calculate that
\begin{equation}
\begin{aligned}
& P\pp{\max_{j\in[p]}\norm{\mfU^T\mfD_{j}\mfU}_2\leq \frac{6}{7}}\\
 \ge & 1 - 4\exp\pp{-\frac{n}{8000\sigma^2}} / P\pp{\sigma_{\min}^2\ge \frac{7}{8}} \\
 \ge & 1 - 4\exp\pp{-\frac{n}{8000\sigma^2}} / (1-2\exp(-d))\\
 \ge & 1 - 2\exp(-d) - 4\exp\pp{-\frac{n}{8000\sigma^2}} 
\end{aligned}
\end{equation}
Applying Proposition \ref{prop:noisy} with $\gamma=\frac{1}{7}$ and $\mfA_{i^*}=\mfU,\mfA_{j}=\mfD_j\mfU,\mfw^*=\bSigma\mfV^T\mfw^*$, we conclude that if the assumptions in Theorem \ref{thm:noisy} is satisfied, then with probability at least $1-4\exp(-n/8000\sigma^2)-2\exp(-d)$ there exists $\mfW=(\mfw,0,\dots,0)$ such that $\mfW$ is the optimal solution to both \eqref{reg:normal_before_relu_skip} and \eqref{reg:normal_before_grelu_skip} whenever $n\geq  \max\{4000\sigma^2 d\log (54n),1024 d\}$. 

Moreover, we obtain the desired upper bound
\[\norm{\mfw-\bSigma\mfV^T\mfw^*}_2
\le \frac{\beta\eta}{\eta-\|\mfz\|_2} + \|\mfz\|_2.\]

Finally, we provide high probability upper and  lower bounds for $\eta=\|\bSigma\mfV^T\mfw^*\|_2$ as follows. Again from Theorem 4.6.1 in \citep{vershynin2018high}, we know that for $n\geq 1024 d$, with probability at least $1-2\exp(-d)$ it holds that
\begin{equation}
    1-1/16\le \sigma_\mathrm{min}\le \sigma_\mathrm{max} \le 1+1/16,
\end{equation}
which immediately implies that $(1-1/16)\|\mfw^*\|_2\le \eta \le (1+1/16)\|\mfw^*\|_2$.



\subsection{Numerical Verification}
In this subsection, we numerically verify Theorem \ref{thm:noisy}. We take $n=40,d=10$ and test for 
$\sigma=0,1/8,1/4$. For each $\sigma$, we solve the regularized group Lasso problem \eqref{reg:normal_before_grelu_skip} for $\beta\in[0,2]$. Then we analyze the solution and record the number of active neurons. The recovery is regarded as success if there is exactly one active neuron. In Figure \ref{fig:noisy}, the recovery displays a failure-success-failure pattern when $\beta$ increases. Besides, the lower bound of $\beta$ that ensures successful recovery shifts right
as $\sigma$ increases, while the upper bound generally remains the same.

\begin{figure}[H]
\centering
\setcounter{subfigure}{0}
    \subfigure[$\sigma=0$]{
      \centering
      \includegraphics[width=0.3\textwidth]{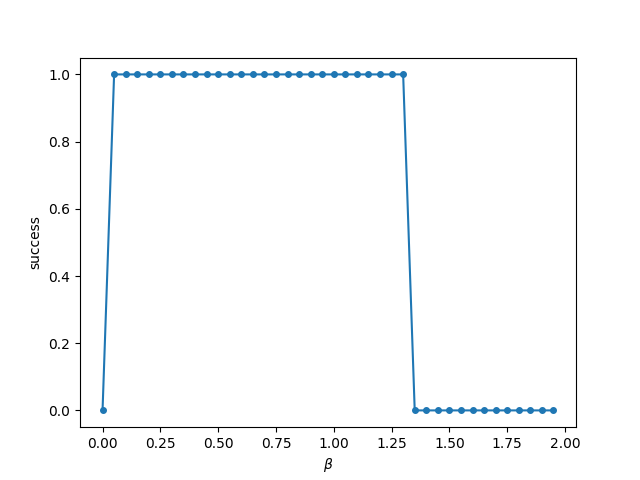}  
    }
    \centering
    \subfigure[$\sigma=1/8$]{
      \centering
      \includegraphics[width=0.3\textwidth]{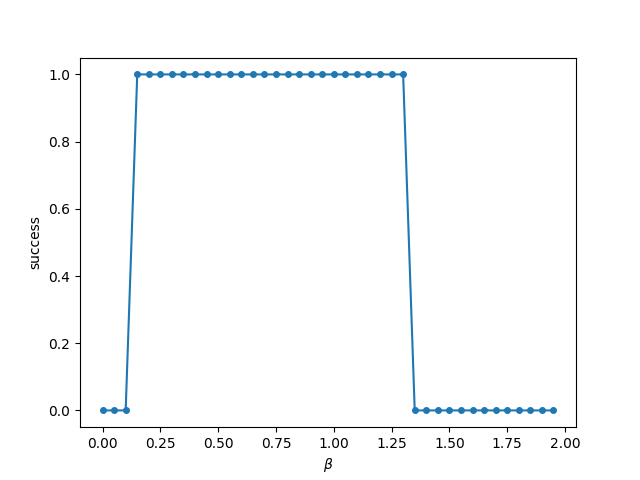}  
    }
    \subfigure[$\sigma=1/4$]{
      \centering
      \includegraphics[width=0.3\textwidth]{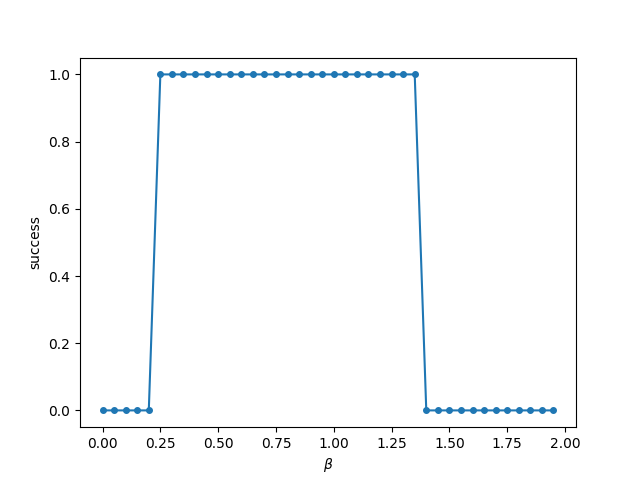}  
    }
    \caption{The pattern of successful recovery of one planted linear neuron by solving regularized group Lasso problem \eqref{gl:general} derived from training ReLU networks with skip connections. }\label{fig:noisy}
\end{figure}
\section{Proofs in Section \ref{sec:relu_normal}}

\subsection{Proof of Theorem \ref{thm:relu_normal}}
\begin{proof}
We first introduce an auxillary lemma:
\begin{lemma}\label{lem:uiuj}
Suppose that $\mfU_i, \mfU_j\in\mbR^{n\times d}$ are column orthonormal, i.e., $\mfU_i^T\mfU_i=\mfI_d$ and $\mfU_j^T\mfU_j=\mfI_d$. If $\mfU_i^T\mfU_j=\mfI_d$, then, we have $\mfU_i=\mfU_j$.
\end{lemma}
\begin{proof}
For $k\in [d]$, we let $\mfw_{i,k}$ and $\mfw_{j,k}$ be the $k$-th column of $\mfU_i$ and $\mfU_j$. Note that $1=(\mfU_i^T\mfU_j)_{k,k}=\mfw_{i,k}^T\mfw_{j,k}$. As $\mfU_i^T\mfU_i=\mfI_d$ and $\mfU_j^T\mfU_j=\mfI_d$, we have $\|\mfw_{i,k}\|_2=\|\mfw_{j,k}\|_2=1$. Therefore, $\mfw_{i,k}^T\mfw_{j,k}=1$ implies that $\mfw_{i,k}=\mfw_{j,k}$. Hence, we have $\mfU_i=\mfU_j$. 
\end{proof}
Consider the event
\begin{equation}
    E_1=\bbbb{\sum_{i=1}^n\mbI(\mfx_i^T\tilde \mfw^*\geq 0)\geq d}.
\end{equation}
We first show that for $n>2d$, $E_1$ holds with high probability. As $\mfx_i\sim \mcN(0,I_d/n)$, $\mbI(\mfx_i^T \mfw^*\geq 0)$ are i.i.d. random variables following $\Bern(1/2)$. Denote $B_i=\mbI(\mfx_i^T \mfw^*\geq 0)$ and let $B=\frac{1}{n}\sum_{i=1}^nB_i$. We note that $\mbE[B]=\frac{1}{2}$. According to the Chernoff bound, we have
\begin{equation}
    P(E_1^c)=P\pp{B< \frac{d}{n}}=P\pp{B<\pp{1-\frac{n-2d}{n}}\mbE[B]}\leq \exp\pp{-\frac{1}{6}\pp{\frac{n-2d}{n}}^2n}.
\end{equation}
This implies that $P(E_1)\geq 1-\exp\pp{-\frac{1}{6}\pp{\frac{n-2d}{n}}^2n}$. 

Denote $E_2=\{\sigma_\mathrm{min}\pp{\mfX^T\mfD_{i^*}\mfX}>0\}$, where $\mfD_{i^*}=\diag(\mbI(\mfX\mfw^*\geq 0))$.  We note that 
$$
\mfX^T\mfD_{i^*}\mfX=\sum_{i=1}^n\mfx_i\mfx_i^T\mbI(\mfx_i^T\mfw^*\geq 0). 
$$
As $\mfx_i\sim \mcN(0,I_d/n)$, we have $P(E_2^c|E_1)=\frac{P(E_2^c, E_1)}{P(E_1)}=0$, which implies that  $P(E_2|E_1)=1$. 


Conditioned on the event $E_2$, we have $\mfU_{i^*}\in\mbR^{n\times d}$. For $j\neq i^*$, we show that $\mfU_{i^*}^T\mfU_j\mfU_{j}^T\mfU_{i^*}\neq \mfI_d$. Let $\mfU_j\in\mbR^{n\times r_j}$. If $r_j<d$, then $\mfU_{i^*}^T\mfU_j\mfU_{j}^T\mfU_{i^*}$ is with rank at most $r_j<d$. Hence, we have $\mfU_{i^*}^T\mfU_j\mfU_{j}^T\mfU_{i^*}\neq \mfI_d$. If $r_j=d$, suppose that $\mfU_{i^*}^T\mfU_j\mfU_{j}^T\mfU_{i^*}=\mfI_d$. Denote $\mfP=\mfU_j^T\mfU_{i^*}\in\mbR^{d\times d}$. Then, $\mfP^T\mfP=\mfI_d$. This implies that $\mfP$ is orthogonal. Hence, $\mfU_j^T(\mfU_{i^*}\mfP^T)=\mfP\mfP^T=\mfI_d$. We note that $\mfU_{i^*}\mfP^T$ is also column orthonormal. This is because 
$$
(\mfU_{i^*}\mfP^T)^T\mfU_{i^*}\mfP^T = \mfP\mfU_{i^*}^T\mfU_{i^*}\mfP^T=\mfP\mfP^T=\mfI_d.
$$
As the matrices $\mfU_j$ and $\mfU_{i^*}\mfP^T$ are column orthonormal and $\mfU_j^T(\mfU_{i^*}\mfP^T)=\mfI_d$, from Lemma \ref{lem:uiuj}, we have $\mfU_j=\mfU_{i^*}\mfP^T$. As $\mfD_{i^*}\neq \mfD_j$, there exists $k\in[n]$ such that either of following statements will hold.
\begin{itemize}
    \item $(\mfD_{i^*})_{k,k}=1$ and $(\mfD_j)_{k,k}=0$.
    \item $(\mfD_{i^*})_{k,k}=0$ and $(\mfD_j)_{k,k}=1$.
\end{itemize}
For the first case, we note that
\begin{equation}
    \mfw_{i^*,k} = (\mfD_i)_{k,k}\mfx_{k}^T\mfV_{i^*}\bSigma_{i^*}, \mfw_{j,k} = (\mfD_j)_{k,k}\mfx_{k}^T\mfV_j\bSigma_j.
\end{equation}
Because $\mfx_k\neq 0$ and $\mfV_{i^*}\bSigma_{i^*}$ is invertible, we have $\|\mfw_{i^*,k}\|_2>0$. As $\mfU_{i^*}=\mfU_j\mfP^T$, we have $\mfw_{i^*,k}^T=\mfw_{j,k}^T\mfP^T$. As $\mfw_{j,k}=0$, this implies that $\mfw_{i^*,k}=0$, which leads to a contradiction. For the second case, we note that $\mfU_{i^*}=\mfU_j\mfP$. Similarly, this will lead to a contradiction.

Therefore, conditioned on $E_2$, for $j\neq i^*$, we have $\mfU_{i^*}^T\mfU_j\mfU_{j}^T\mfU_{i^*}\neq \mfI_d$. We note that $\mfU_{i^*}^T\mfU_j\mfU_{j}^T\mfU_{i^*}\succeq \mfU_{i^*}^T\mfU_{i^*}=\mfI_d$. This implies that $\|\mfU_{j}^T\mfU_{i^*}\tilde \mfw^*\|_2=1$ if and only if we have $\|\mfU_{i^*}^T\mfU_j\mfU_{j}^T\mfU_{i^*}\|_2=1$ and $\tilde \mfw^*\in\eigmax\{\mfU_{i^*}^T\mfU_j\mfU_{j}^T\mfU_{i^*}\}$. As $\mfU_{i^*}^T\mfU_j\mfU_{j}^T\mfU_{i^*}\neq \mfI_d$, if it follows that $\|\mfU_{i^*}^T\mfU_j\mfU_{j}^T\mfU_{i^*}\|_2=1$, then $\eigmax\{\mfU_{i^*}^T\mfU_j\mfU_{j}^T\mfU_{i^*}\}$ is with dimension at most $d-1$. In other words, conditioned on $\|\mfU_{i^*}^T\mfU_j\mfU_{j}^T\mfU_{i^*}\|_2=1$, $\eigmax\{\mfU_{i^*}^T\mfU_j\mfU_{j}^T\mfU_{i^*}\}$ is a random subspace with dimension at most $d-1$. This implies that
\begin{equation}
\begin{aligned}
 &P(\tilde \mfw^*\in \eigmax\{\mfU_{i^*}^T\mfU_j\mfU_{j}^T\mfU_{i^*}\},\|\mfU_{i^*}^T\mfU_j\mfU_{j}^T\mfU_{i^*}\|_2=1|E_2)\\
    =&P(\bSigma_{i^*}\mfV_{i^*}\mfw^*\in \eigmax\{\mfU_{i^*}^T\mfU_j\mfU_{j}^T\mfU_{i^*}\},\|\mfU_{i^*}^T\mfU_j\mfU_{j}^T\mfU_{i^*}\|_2=1|E_2)=0.
\end{aligned}
\end{equation}
Immediately, for $j\neq i^*$, we have
\begin{equation}
    P(\|\mfU_{j}^T\mfU_{i^*}\tilde \mfw^*\|_2=1)=0
\end{equation}
For $\bsigma\in\{0,1\}^n$, define $\mfU(\bsigma)$ as the left singular vector of $\diag(\bsigma)\mfX$. Then, we have
\begin{equation}
\begin{aligned}
    &P\pp{\left.\max_{\mfh\in \mbR^d:\mfh\neq 0, \mbI(\mfX\mfh\geq 0)\neq \mbI(\mfX\tilde \mfw^*\geq 0)}\|\mfU(\mbI(\mfX\mfh\geq 0))^T\mfU_{i^*}\tilde \mfw^*\|_2=1\right|E_2}\\
    =& P\pp{\left.\max_{\bsigma\in \{0,1\}^n: \exists \mfh\neq 0,  \mbI(\mfX\mfh\geq 0)\neq \mbI(\mfX\tilde \mfw^*\geq 0), \bsigma=\mbI(\mfX\mfh\geq 0)}\|\mfU(\bsigma)^T\mfU_{i^*}\tilde \mfw^*\|_2=1\right|E_2}\\
    \leq & \sum_{\bsigma\in\{0,1\}^n} P(\|\mfU(\bsigma)^T\mfU_{i^*}\tilde \mfw^*\|_2=1, \mbI(\mfX\mfh\geq 0)\neq \mbI(\mfX\tilde \mfw^*\geq 0), \bsigma=\mbI(\mfX\mfh\geq 0)|E_2)=0.
\end{aligned}
\end{equation}
We note that $P(E_2)\geq P(E_1,E_2)=P(E_1)\geq 1- \exp\pp{-\frac{1}{6}\pp{\frac{n-2d}{n}}^2n}$. This implies that 
\begin{equation}
\begin{aligned}
    &P\pp{\max_{\substack{\mfh\in \mbR^d:\mfh\neq 0, \\\mbI(\mfX\mfh\geq 0)\neq \mbI(\mfX\tilde \mfw^*\geq 0)}}\|\mfU(\mbI(\mfX\mfh\geq 0))^T\mfU_{i^*}\tilde \mfw^*\|_2<1}\\
    =&P\pp{\max_{\mfh\in \mbR^d:\mfh\neq 0, \mbI(\mfX\mfh\geq 0)\neq \mbI(\mfX\tilde \mfw^*\geq 0)}\|\mfU(\mbI(\mfX\mfh\geq 0))^T\mfU_{i^*}\tilde \mfw^*\|_2<1,E_2}\\
    =&P\pp{\left.\max_{\mfh\in \mbR^d:\mfh\neq 0, \mbI(\mfX\mfh\geq 0)\neq \mbI(\mfX\tilde \mfw^*\geq 0)}\|\mfU(\mbI(\mfX\mfh\geq 0))^T\mfU_{i^*}\tilde \mfw^*\|_2<1\right|E_2} P(E_2)\\
    =&P(E_2)\geq 1- \exp\pp{-\frac{1}{6}\pp{\frac{n-2d}{n}}^2n}.
\end{aligned}
\end{equation}
This completes the proof.
\end{proof}

\subsection{Proof of Proposition \ref{prop:asymp_two_neu}}



\begin{proof}
For simplicity, we assume that $\|\mfw_1\|_2=\|\mfw_2\|_2=\|\mfh_j\|_2=1$. We first consider the case where $\cos\angle(\mfw_1,\mfw_2)=-1$.  Then, we have $\mfw_2=-\mfw_1$. In this case, we have $\mfD_1\mfD_2=0$. Hence, it follows that
\begin{equation}
    \bmbm{\mfU_{s_1}^T\mfU_{s_1}&\mfU_{s_1}^T\mfU_{s_2}\\
    \mfU_{s_2}^T\mfU_{s_1}&\mfU_{s_2}^T\mfU_{s_2}}=\bmbm{\mfI_d&0\\0&\mfI_d}.
\end{equation}
Then, we can simply that
\begin{equation}
\begin{aligned}
&\mfU_j^T\bmbm{\mfU_{1}&\mfU_{2}}\bmbm{\mfU_{1}^T\mfU_{1}&\mfU_{1}^T\mfU_{2}\\
    \mfU_{2}^T\mfU_{1}&\mfU_{2}^T\mfU_{2}}^{-1}\bmbm{\tilde \mfw_1\\\tilde \mfw_2}\\
=    &\mfU_j^T\mfU_1\tilde \mfw_1+\mfU_j^T\mfU_2\tilde \mfw_2 \\
    = &\bSigma_j^{-1}\mfV_j^T\mfX^T\mfD_j\pp{\mfD_1\mfX\mfV_1\bSigma_1^{-1}\frac{\bSigma_1\mfV_1^T\mfw_1}{\|\bSigma_1\mfV_1^T\mfw_1\|_2}}+\bSigma_j^{-1}\mfV_j^T\mfX^T\mfD_j\pp{\mfD_2\mfX\mfV_2\bSigma_2^{-1}\frac{\bSigma_2\mfV_2^T\mfw_2}{\|\bSigma_2\mfV_2^T\mfw_2\|_2}}\\
    =&\frac{1}{\|\bSigma_1\mfV_1^T\mfw_1\|_2}\bSigma_j^{-1}\mfV_j^T\mfX^T\mfD_j\mfD_1\mfX\mfw_1+\frac{1}{\|\bSigma_2\mfV_2^T\mfw_2\|_2}\bSigma_j^{-1}\mfV_j^T\mfX^T\mfD_j\mfD_2\mfX\mfw_2\\
\end{aligned}
\end{equation}
As $n\to\infty$, $\mfX^T\mfD_i\mfD_j\mfX$ converges in probability to
\begin{equation}
    \mfM(\mfh_i,\mfh_j)=\mbE_{\mfx\sim\mcN(0,\mfI_d)}[\mfx\mfx^T\mbI(\mfx^T\mfh_i\geq 0)\mbI(\mfx^T\mfh_j\geq 0)].
\end{equation}
As $n\to \infty$, we have $\mfX^T\mfD_j\mfX \overset{p}{\to} \frac{1}{2}\mfI_d$ and $\Sigma_j \overset{p}{\to} \frac{1}{\sqrt{2}}\mfI_d$. Therefore, we have
\begin{equation}
    \norm{\mfU_j^T\mfU_1\tilde \mfw_1+\mfU_j^T\mfU_2\tilde \mfw_2}_2
    \overset{p}{\to} 2\norm{\mfM(\mfh_j,\mfw_1)\mfw_1-\mfM(\mfh_j,-\mfw_1)\mfw_1}_2.
\end{equation}
According to Lemma 7 in \citep{ghorbani2021linearized}, for $\mfh_i,\mfh_j$ satisfying $\|\mfh_i\|_2=\|\mfh_j\|_2=1$, the matrix $\mfM(\mfh_i,\mfh_j)$ takes the form 
\begin{equation}
    \mfM(\mfh_i,\mfh_j)=c_1(\gamma)\mfI_d+c_2(\gamma)(\mfh_i\mfh_j^T+\mfh_j\mfh_i^T)+c_3(\gamma)(\mfh_i\mfh_i^T+\mfh_j\mfh_j^T). 
\end{equation}
Here $c_1,c_2,c_3$ are functions of $\gamma$. 
Then, we have
\begin{equation}
\begin{aligned}
&\mfM(\mfh_j,\mfw_1)\mfw_1-\mfM(\mfh_j,-\mfw_1)\mfw_1\\
    =&(c_1(\gamma)+\gamma c_2(\gamma) +c_3(\gamma))\mfw_1+2(c_2(\gamma)+\gamma c_3(\gamma))\mfh_j\\
    &-(c_1(-\gamma)-\gamma c_2(\gamma) -c_3(\gamma))\mfw_1+2(c_2(-\gamma)-\gamma c_3(-\gamma))\mfh_j.
\end{aligned}
\end{equation}
We observe that $\|\mfM(\mfh_i,\mfh_j)\mfh_i-\mfM(-\mfh_i,\mfh_j)\mfh_i\|_2^2$ only depends on $\gamma$. Denote $g_1(\gamma)=2\|\mfM(\mfh_i,\mfh_j)\mfh_i-\mfM(-\mfh_i,\mfh_j)\mfh_i\|_2$. Therefore, it is sufficient to compute the case where $\mfh_i=\mfe_1$ and $\mfh_j=\gamma \mfe_1+\sqrt{1-\gamma^2}\mfe_2$. Utilizing Lemma \ref{lem:cal}, we have
\begin{equation}
\begin{aligned}
g_1(\gamma)^2=&4\norm{\mfM(\mfh_i,\mfh_j)\mfh_i-\mfM(-\mfh_i,\mfh_j)\mfh_i}_2^2\\
  =&4\pp{\mbE_{\mfx}\bb{\sigma'(x_1)\sigma'(\gamma x_1+\sqrt{1-\gamma^2}x_2)x_1^2}-\mbE_{\mfx}\bb{\sigma'(x_1)\sigma'(-\gamma x_1+\sqrt{1-\gamma^2}x_2)x_1^2}}^2\\
  &+4\pp{\mbE_{\mfx}\bb{\sigma'(x_1)\sigma'(\gamma x_1+\sqrt{1-\gamma^2}x_2)x_1x_2}-\mbE_{\mfx}\bb{\sigma'(x_1)\sigma'(-\gamma x_1+\sqrt{1-\gamma^2}x_2)x_1x_2}}^2\\
  =&4\pp{\int_0^\infty \pp{F\pp{\frac{\gamma}{\sqrt{1-\gamma^2}}x}-F\pp{-\frac{\gamma}{\sqrt{1-\gamma^2}}x}} p(x)x^2 dx}^2.
\end{aligned}
\end{equation}
We plot $g_1(\gamma)^2$ in Figure \ref{fig:g1_gamma}.

\begin{figure}[H]
\centering
\begin{minipage}[t]{0.45\textwidth}
\centering
\includegraphics[width=\linewidth]{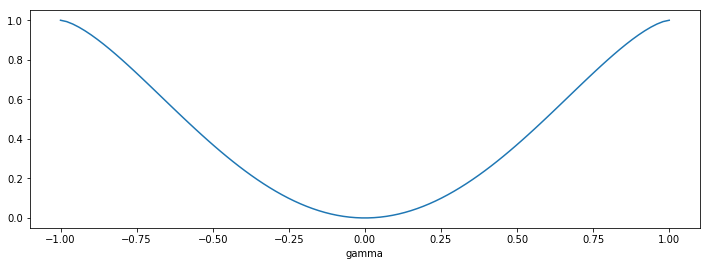}
\end{minipage}
\centering
\begin{minipage}[t]{0.45\textwidth}
\centering
\includegraphics[width=\linewidth]{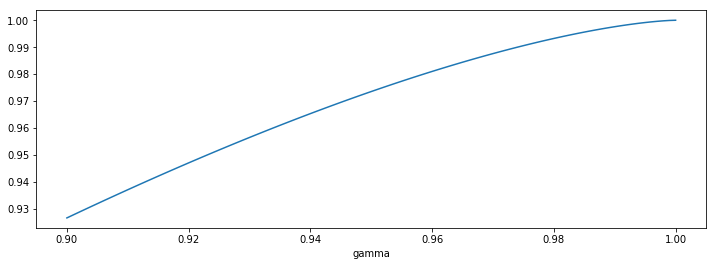}
\end{minipage}
\caption{The plot of $g_1(\gamma)^2$.}\label{fig:g1_gamma}
\end{figure}

Then, we consider the case where $\mfw_1^T\mfw_2=0$. In this case, 
\begin{equation}
\begin{aligned}
&\bmbm{\mfI_d& \mfU_1^T\mfU_2 \\
    \mfU_2^T\mfU_1&\mfI_d}
    =\bmbm{\mfI_d&\bSigma_1^{-1}\mfV_1^T\mfX^T\mfD_1\mfD_2\mfX\mfV_2\bSigma_2^{-1}\\ \bSigma_2^{-1}\mfV_2^T\mfX^T\mfD_2\mfD_1\mfX\mfV_1\bSigma_1^{-1}&\mfI_d}\\
\end{aligned}
\end{equation}
This implies that
\begin{equation}
    \bmbm{\mfV_1&0\\0&\mfV_2}\bmbm{\mfI_d& \mfU_1^T\mfU_2 \\
    \mfU_2^T\mfU_1&\mfI_d} \bmbm{\mfV_1^T&0\\0&\mfV_2^T}\\
    \overset{p}{\to}\bmbm{\mfI_d&2\mfM(\mfw_1,\mfw_2)\\ 2\mfM(\mfw_1,\mfw_2)&\mfI_d}
\end{equation}
On the other hand, we note that
\begin{equation}
\begin{aligned}
    &\mfV_j^T\mfU_j^T\bmbm{\mfU_1&\mfU_2}\bmbm{\mfV_1^T&0\\0&\mfV_2^T}\\
    =&\mfV_j^T\bSigma_j^{-1}\mfV_j\mfX^T\mfD_j\bmbm{\mfD_1\mfX\mfV_1\bSigma_1^{-1}\mfV_1^T&\mfD_2\mfX\mfV_2\bSigma_2^{-1}\mfV_2^T}
    \overset{p}{\to}2\bmbm{\mfM(\mfh_j,\mfw_1)&\mfM(\mfh_j,\mfw_2)}.
\end{aligned}
\end{equation}
We also have
\begin{equation}
\begin{aligned}
\bmbm{\mfV_1&0\\0&\mfV_2}\bmbm{\tilde\mfw_1\\\tilde\mfw_2}=\bmbm{\frac{\mfV_1^T\bSigma_1\mfV_1\mfw_1}{{\|\bSigma_1\mfV_1\mfw_1\|_2}}&\frac{\mfV_2^T\bSigma_2\mfV_2\mfw_2}{{\|\bSigma_2\mfV_2\mfw_2\|_2}}}
\overset{p}{\to}\bmbm{\mfw_1\\\mfw_2}.
\end{aligned}
\end{equation}
Therefore, we have the limit
\begin{equation*}
\begin{aligned}
\norm{\mfU_j\bmbm{\mfU_1&\mfU_2} \bmbm{\mfI_d& \mfU_1^T\mfU_2 \\
    \mfU_2^T\mfU_1&\mfI_d}^{-1}\bmbm{\tilde\mfw_1\\\tilde\mfw_2}}_2
\overset{p}{\to}2\norm{\bmbm{\mfM(\mfh_j,\mfw_1)&\mfM(\mfh_j,\mfw_2)} \bmbm{\mfI_d&2\mfM(\mfw_1,\mfw_2)\\ 2\mfM(\mfw_1,\mfw_2)&\mfI_d}^{-1}\bmbm{\mfw_1\\\mfw_2}}_2
\end{aligned}
\end{equation*}

\begin{lemma}\label{lem:gamma0}
Suppose that $\mfw_1^T\mfw_2=0$. Then, we have $c_1(0)=\frac{1}{4},c_2(0)=\frac{1}{2\pi}$ and $c_3(0)=0$. In other words, it follows that
\begin{equation}
    \mfM(\mfw_1,\mfw_2)=\frac{1}{4}\mfI_d+\frac{1}{2\pi}(\mfw_1\mfw_2^T+\mfw_2^T\mfw_1).
\end{equation}
\end{lemma}

From Lemma \ref{lem:gamma0}, we can compute that
\begin{equation}
    \bmbm{\mfI_d&2\mfM(\mfw_1,\mfw_2)\\ 2\mfM(\mfw_1,\mfw_2)&\mfI_d}= \bmbm{\mfI_d&\frac{1}{2}\mfI_d+\frac{1}{\pi}(\mfw_1\mfw_2^T+\mfw_2\mfw_1^T)\\\frac{1}{2}\mfI_d+\frac{1}{\pi}(\mfw_1\mfw_2^T+\mfw_2\mfw_1^T)&\mfI_n}.
\end{equation}
Suppose that
\begin{equation}
    \bmbm{\mfI_d&\frac{1}{2}\mfI_d+\frac{1}{\pi}(\mfw_1\mfw_2^T+\mfw_2^T\mfw_1)\\\frac{1}{2}\mfI_d+\frac{1}{\pi}(\mfw_1\mfw_2^T+\mfw_2^T\mfw_1)&\mfI_n}\bmbm{\beta_1\\\beta_2}=\bmbm{\mfw_1\\\mfw_2}.
\end{equation}
Then, we note that
\begin{equation}
\begin{aligned}
    &\beta_1+\frac{1}{2}\beta_2+\frac{\mfw_2^T\beta_2}{\pi}\mfw_1+\frac{\mfw_1^T\beta_2}{\pi}\mfw_2=\mfw_1,\quad \frac{1}{2}\beta_1+\frac{\mfw_2^T\beta_1}{\pi}\mfw_1+\frac{\mfw_1^T\beta_1}{\pi}\mfw_2+\beta_2=\mfw_2.
\end{aligned}
\end{equation}
This implies that $\beta_1,\beta_2\in \text{span}\{\mfw_1,\mfw_2\}$. Let $\bmbm{\beta_1^T\\\beta_2^T}=\mfA\bmbm{\mfw_1^T\\\mfw_2^T}$, where $\mfA=\bmbm{a_{1,1}&a_{1,2}\\a_{2,1}&a_{2,2}}$. Then, the above linear system reduces to
\begin{equation}
\begin{aligned}
&a_{1,1}+\frac{1}{2}a_{2,1}+\frac{1}{\pi} a_{2,2}=1,\\
&a_{1,2}+\frac{1}{2}a_{2,2}+\frac{1}{\pi}a_{1,2}=0,\\
&\frac{1}{2}a_{1,1}+\frac{1}{\pi} a_{2,1}+a_{2,1}=0,\\
&\frac{1}{\pi} a_{1,1}+\frac{1}{2}a_{1,2}+a_{2,2}=1.
\end{aligned}
\end{equation}
We can solve that
\begin{equation}
\begin{aligned}
    &a_{1,1}=a_{2,2}=\pp{\pp{1+\frac{1}{\pi}}^2-\frac{1}{4}}^{-1}\pp{1+\frac{1}{\pi}},a_{2,1}=a_{1,2}=\pp{\pp{1+\frac{1}{\pi}}^2-\frac{1}{4}}^{-1}\frac{1}{2}.
\end{aligned}
\end{equation}
In other words, we have $\mfA=\bmbm{1+\frac{1}{\pi}&\frac{1}{2}\\\frac{1}{2}&1+\frac{1}{\pi}}^{-1}$. 

Denote $\gamma_1=\mfh_j^T\mfw_1$ and $\gamma_2=\mfh_j^T\mfw_2$. As $\mfw_1^T\mfw_2=0$, we have $\gamma_1^2+\gamma_2^2\leq 1$. Thus, we have
\begin{equation}
\begin{aligned}
    &\bmbm{\mfM(\mfh_j,\mfw_1)&\mfM(\mfh_j,\mfw_2)} \bmbm{\mfI_d&2\mfM(\mfw_1,\mfw_2)\\ 2\mfM(\mfw_1,\mfw_2)&\mfI_d}^{-1}\bmbm{\mfw_1\\\mfw_2}\\
    =& \mfM(\mfw_1,\mfh_j)(a_{1,1}\mfw_{1}+a_{1,2}\mfw_2)+\mfM(\mfw_2,\mfh_j)(a_{2,1}\mfw_{1}+a_{2,2}\mfw_2)\\
    =&a_{1,1} (c_1(\gamma_1)+\gamma_1 c_2(\gamma_1) +c_3(\gamma_1))\mfw_1+2(c_2(\gamma_1)+\gamma_1 c_3(\gamma_1))\mfh_j)\\
    &+a_{1,2} (\gamma_2c_2(\gamma_1)\mfw_1+c_1(\gamma_1)\mfw_2+\gamma_2 c_3(\gamma_1)\mfh_j)\\
    &+a_{2,1} (\gamma_1 c_2(\gamma_1)\mfw_2+c_1(\gamma_2)\mfw_1+\gamma_1c_3(\gamma_2)\mfh_j)\\
    &+a_{2,2} (c_1(\gamma_2)+\gamma_2 c_2(\gamma_2)+c_3(\gamma_2)\mfw_1+2(c_2(\gamma_2)+\gamma_2 c_3(\gamma_2))\mfh_j)).
\end{aligned}
\end{equation}

Similarly, the norm of this quantity only depends on $\gamma_1$ and $\gamma_2$. Therefore, it is sufficient to consider $\mfw_1=\mfe_1$, $\mfw_2=\mfe_2$ and $\mfh_j=\gamma_1\mfe_1+\gamma_2\mfe_2+\sqrt{1-\gamma_1^2-\gamma_2^2}\mfe_3$. In this case, we note that
\begin{equation}\label{equ:muhj}
\begin{aligned}
    &\mfM(\mfw_1,\mfh_j)(a_{1,1}\mfw_1+a_{1,2}\mfw_2)+\mfM(\mfw_2,\mfh_j)(a_{2,1}\mfw_1+a_{2,2}\mfw_2)\\
    =&a_{1,1}\bmbm{\mbE_\mfx\bb{\sigma'(x_1)\sigma'(\gamma_1 x_1+\gamma_2 x_2+\sqrt{1-\gamma_1^2-\gamma_2^2}x_3)x_1^2}\\
    \mbE_\mfx\bb{\sigma'(x_1)\sigma'(\gamma_1 x_1+\gamma_2 x_2+\sqrt{1-\gamma_1^2-\gamma_2^2}x_3)x_1x_2}\\
    \mbE_\mfx\bb{\sigma'(x_1)\sigma'(\gamma_1 x_1+\gamma_2 x_2+\sqrt{1-\gamma_1^2-\gamma_2^2}x_3)x_1x_3}
    }\\
    &+a_{1,2}\bmbm{\mbE_\mfx\bb{\sigma'(x_1)\sigma'(\gamma_1 x_1+\gamma_2 x_2+\sqrt{1-\gamma_1^2-\gamma_2^2}x_3)x_1x_2}\\
    \mbE_\mfx\bb{\sigma'(x_1)\sigma'(\gamma_1 x_1+\gamma_2 x_2+\sqrt{1-\gamma_1^2-\gamma_2^2}x_3)x_2^2}\\
    \mbE_\mfx\bb{\sigma'(x_1)\sigma'(\gamma_1 x_1+\gamma_2 x_2+\sqrt{1-\gamma_1^2-\gamma_2^2}x_3)x_2x_3}
    }\\
    &+a_{2,1}\bmbm{\mbE_\mfx\bb{\sigma'(x_2)\sigma'(\gamma_1 x_1+\gamma_2 x_2+\sqrt{1-\gamma_1^2-\gamma_2^2}x_3)x_1^2}\\
    \mbE_\mfx\bb{\sigma'(x_2)\sigma'(\gamma_1 x_1+\gamma_2 x_2+\sqrt{1-\gamma_1^2-\gamma_2^2}x_3)x_1x_2}\\
    \mbE_\mfx\bb{\sigma'(x_2)\sigma'(\gamma_1 x_1+\gamma_2 x_2+\sqrt{1-\gamma_1^2-\gamma_2^2}x_3)x_1x_3}
    }\\
    &+a_{2,2}\bmbm{\mbE_\mfx\bb{\sigma'(x_2)\sigma'(\gamma_1 x_1+\gamma_2 x_2+\sqrt{1-\gamma_1^2-\gamma_2^2}x_3)x_1x_2}\\
    \mbE_\mfx\bb{\sigma'(x_2)\sigma'(\gamma_1 x_1+\gamma_2 x_2+\sqrt{1-\gamma_1^2-\gamma_2^2}x_3)x_2^2}\\
    \mbE_\mfx\bb{\sigma'(x_2)\sigma'(\gamma_1 x_1+\gamma_2 x_2+\sqrt{1-\gamma_1^2-\gamma_2^2}x_3)x_2x_3}
    }
\end{aligned}
\end{equation}
Let $F(x)$ and $p(x)$ represent the CDF and pdf of $\mcN(0,1)$ respectively. We introduce two lemmas for computing the expectations in \eqref{equ:muhj}.
\begin{lemma}\label{lem:cal}
Suppose that $\gamma\in[-1,1]$. We have the following computations:
\begin{equation}
\begin{aligned}
&\mbE_{\mfx}\bb{\sigma'(x_1)\sigma'(\gamma x_1+\sqrt{1-\gamma^2}x_2)x_1^2}=\int_0^\infty \pp{1-F\pp{-\frac{\gamma}{\sqrt{1-\gamma^2}}x}} p(x)x^2 dx,\\
&\mbE_{\mfx}\bb{\sigma'(x_1)\sigma'(\gamma x_1+\sqrt{1-\gamma^2}x_2)x_1x_2}=\frac{1-\gamma^2}{2\pi}.
\end{aligned}
\end{equation}
\end{lemma}

\begin{lemma}\label{lem:cal2}
Suppose that $\gamma_1^2+\gamma_2^2+\gamma_3^2=1$ and $\gamma_3\geq 0$. Then, we have
\begin{equation}
\begin{aligned}
    &\mbE\bb{\sigma'(x_1)\sigma'(\gamma_1x_1+\gamma_2x_2+\gamma_3x_3 )x_1x_2}=\frac{\sqrt{1-\gamma_1^2}|\gamma_2|}{2\pi},\\
    &\mbE\bb{\sigma'(x_1)\sigma'(\gamma_1x_1+\gamma_2x_2+\gamma_3x_3)x_2^2}=\int_{-\infty}^{\infty}\pp{\int_{0}^{\infty}\pp{1-F\pp{-\frac{\gamma_1x_1+\gamma_2x_2}{\gamma_3}}}p(x_1)dx_1}p(x_2)x_2^2dx_2,\\
    &\mbE\bb{\sigma'(x_1)\sigma'(\gamma_1x_1+\gamma_2x_2+\gamma_3x_3)x_2x_3}=\frac{|\gamma_2|\gamma_1\gamma_3}{2\pi\pp{\gamma_3^2+\gamma_2^2}^{3/2}}.
\end{aligned}
\end{equation}
\end{lemma}

Denote $\gamma_3=\sqrt{1-\gamma_1^2-\gamma_2^2}$. Hence, we can compute that
\begin{equation}
\begin{aligned}
 &\mfM(\mfw_1,\mfh_j)(a_{1,1}\mfw_1+a_{1,2}\mfw_2)+\mfM(\mfw_2,\mfh_j)(a_{2,1}\mfw_1+a_{2,2}\mfw_2)\\
 =&a_{1,1}\bmbm{\int_0^\infty \pp{1-F\pp{-\frac{\gamma_1}{\sqrt{1-\gamma_1^2}}x}} p(x)x^2 dx\\
 \frac{\sqrt{1-\gamma_2^2}|\gamma_1|}{2\pi}\\
 \frac{(1-\gamma_1^2)\gamma_3}{2\pi}}+a_{1,2}\bmbm{\frac{\sqrt{1-\gamma_2^2}|\gamma_1|}{2\pi}\\ 
  \int_{-\infty}^{\infty}\pp{\int_{0}^{\infty}\pp{1-F\pp{-\frac{\gamma_1x_1+\gamma_2x_2}{\gamma_3}}}p(x_1)dx_1}p(x_2)x_2^2dx_2 \\ 
  \frac{|\gamma_2|\gamma_1\gamma_3}{2\pi\pp{\gamma_3^2+\gamma_2^2}^{3/2}} }\\
  &+a_{2,1}\bmbm{\int_{-\infty}^{\infty}\pp{\int_{0}^{\infty}\pp{1-F\pp{-\frac{\gamma_1x_1+\gamma_2x_2}{\gamma_3}}}p(x_2)dx_2}p(x_1)x_1^2dx_1\\
  \frac{\sqrt{1-\gamma_1^2}|\gamma_2|}{2\pi}\\
  \frac{|\gamma_1|\gamma_2\gamma_3}{2\pi\pp{\gamma_3^2+\gamma_2^2}^{3/2}}
  }+a_{2,2}\bmbm{\frac{\sqrt{1-\gamma_1^2}|\gamma_2|}{2\pi}\\
\int_0^\infty \pp{1-F\pp{-\frac{\gamma_2}{\sqrt{1-\gamma_2^2}}x}} p(x)x^2 dx\\
\frac{(1-\gamma_1^2)\gamma_3}{2\pi}
}.
\end{aligned}
\end{equation}

Denote $g_2(\gamma_1,\gamma_2)=2\norm{\mfM(\mfw_1,\mfh_j)(a_{1,1}\mfw_1+a_{1,2}\mfw_2)+\mfM(\mfw_2,\mfh_j)(a_{2,1}\mfw_1+a_{2,2}\mfw_2)}_2$. We numerically verify that $g_2(\gamma_1,\gamma_2)^2\leq 1$ and it is maximized at $(0,1)$ and $(1,0)$. 

We plot $g_2(\gamma_1,\gamma_2)^2$ in Figure \ref{fig:norm_e1e2}. We note that $g_2(\gamma_1,\gamma_2)^2$ is maximized at the boundary. Hence, we plot $g_2(\cos\theta,\sin\theta)$ for $\theta\in\bb{-\frac{\pi}{4},\frac{3\pi}{4}}$ in Figure \ref{fig:norm_e1e2_radius1}. We note that $g_2(\cos\theta,\sin\theta)$ is maximized at $\theta=0$ or $\theta=\pi/2$. Therefore,  $g_2(\gamma_1,\gamma_2)$ is maximized at $(0,1)$ and $(1,0)$ and the optimal value is $1$.

\begin{figure}[H]
\centering
\begin{minipage}[t]{0.3\textwidth}
\centering
\includegraphics[width=\linewidth]{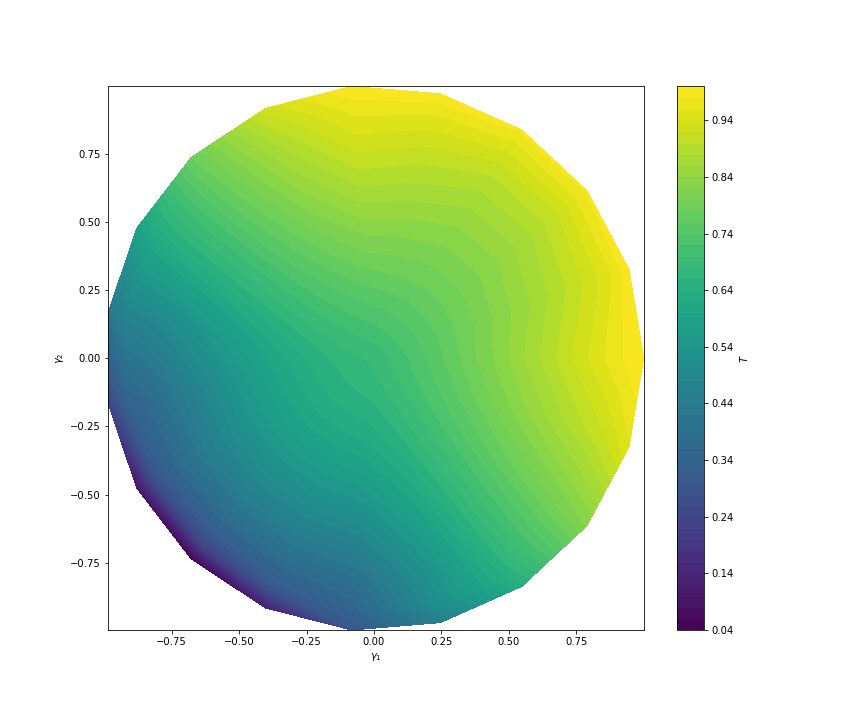}
\caption{The contour plot of $g_2(\gamma_1,\gamma_2)^2$. }\label{fig:norm_e1e2}
\end{minipage}
\centering
\begin{minipage}[t]{0.6\textwidth}
\centering
\includegraphics[width=\linewidth]{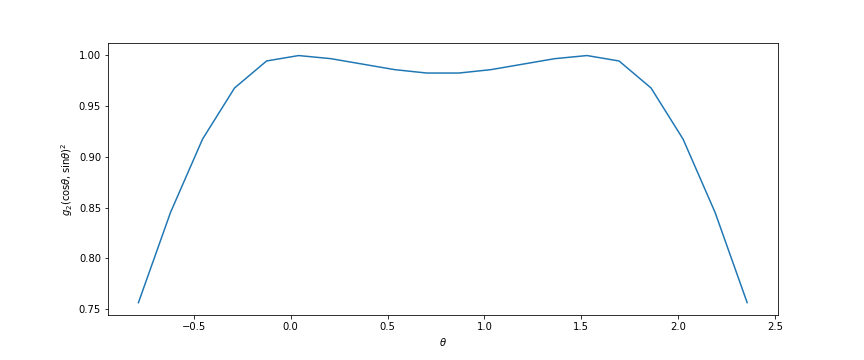}
\caption{The plot of $g_2(\cos\theta,\sin\theta)^2$. }\label{fig:norm_e1e2_radius1}
\end{minipage}
\end{figure}
\end{proof}

\subsection{Proof of Lemma \ref{lem:gamma0}}
\begin{proof}
Consider the rotation matrix $\mfP\in\mbR^{d\times d}$ such that $\mfP \mfw_1=\mfe_1$ and $\mfP\mfw_2=\mfe_2$. Then, we have 
\begin{equation}
\begin{aligned}
    &\mfM(\mfh_i,\mfh_j)\\
    =&\mbE[\mfx\mfx^T\mbI(\mfx^T\mfh_i\geq 0)\mbI(\mfx^T\mfh_j\geq 0)]\\
    =&\mbE[\mfx\mfx^T\mbI(\mfx^T\mfP^T\mfP\mfh_i\geq 0)\mbI(\mfx^T\mfP^T\mfP\mfh_j\geq 0)]\\
    =&\mfP^T\mbE[\tilde \mfx\tilde \mfx^T\mbI(\tilde \mfx^T\mfP\mfh_i\geq 0)\mbI(\tilde \mfx^T\mfP\mfh_j\geq 0)]\mfP\\
    =&\mfP^T\mfM(\mfP\mfh_i,\mfP\mfh_j)\mfP
\end{aligned}
\end{equation}
where we write $\tilde \mfx=\mfP\mfx$. Thus, it is sufficient to compute $\mfM(\mfP\mfh_i,\mfP\mfh_j)=\mfM(\mfe_1,\mfe_2)$. We note that
\begin{equation}
    \mfM(\mfe_1,\mfe_2)=\bmbm{\mfM_{1:2,1:2}&0\\0&\mbE_{\mfx}\bb{\sigma'(x_1)\sigma'(x_2)x_3^2}\mfI_{d-2}},
\end{equation}
where $\mfM_{1:2,1:2}$ follows
\begin{equation}
\begin{aligned}
    \mfM_{1:2,1:2}=&\bmbm{\mbE_{\mfx}\bb{\sigma'(x_1)\sigma'(x_2)x_1^2}&\mbE_{\mfx}\bb{\sigma'(x_1)\sigma'(x_2)x_1x_2}\\
    \mbE_{\mfx}\bb{\sigma'(x_1)\sigma'(x_2)x_1x_2}&\mbE_{\mfx}\bb{\sigma'(x_1)\sigma'(x_2)x_2^2}}=\bmbm{\frac{1}{4}&\frac{1}{2\pi}\\\frac{1}{2\pi}&\frac{1}{4}}.
\end{aligned}
\end{equation}
This implies that $c_1(0)=\mbE_{\mfx}\bb{\sigma'(x_1)\sigma'(x_2)x_3^2}=\frac{1}{4}$. We also note that
\begin{equation}
    \mfM_{1:2,1:2} = \bmbm{u_1(0)+u_3(0)&u_2(0)\\u_2(0)&u_1(0)+u_3(0)}.
\end{equation}
Hence, we have $c_1(0)=\frac{1}{4},c_2(0)=\frac{1}{2\pi}$ and $c_3(0)=0$. 

\end{proof}

\subsection{Proof of Lemma \ref{lem:cal}}
\begin{proof}
For the first equation, by integrating w.r.t. $x_2$ first, we immediately obtain that
\begin{equation}
\begin{aligned}
    &\mbE_{\mfx}\bb{\sigma'(x_1)\sigma'(\gamma x_1+\sqrt{1-\gamma^2}x_2)x_1^2}=\int_0^\infty \pp{1-F\pp{-\frac{\gamma}{\sqrt{1-\gamma^2}}x}} p(x)x^2 dx.
\end{aligned}
\end{equation}
Note that
\begin{equation}
\begin{aligned}
    &\int_{-\frac{\gamma}{\sqrt{1-\gamma^2}}x_1}^\infty xp(x)dx=\pp{\int_{-\frac{\gamma}{\sqrt{1-\gamma^2}}x_1}^0+\int_0^\infty} xp(x)dx\\
    =&\frac{1}{\sqrt{2\pi}} \pp{\int_{-\frac{\gamma}{\sqrt{1-\gamma^2}}x_1}^0+\int_0^\infty} e^{-\frac{x^2}{2}}d\frac{x^2}{2}=\frac{1}{\sqrt{2\pi}}\exp^{-\frac{\gamma^2}{2(1-\gamma^2)}x^2}.
\end{aligned}
\end{equation}
Therefore, we have
\begin{equation}
\begin{aligned}
    &\mbE_{\mfx}\bb{\sigma'(x_1)\sigma'(\gamma x_1+\sqrt{1-\gamma^2}x_2)x_1x_2}=\int_0^\infty \frac{1}{\sqrt{2\pi}}e^{-\frac{\gamma^2}{2(1-\gamma^2)}x^2} p(x)x dx\\
    =&\frac{1}{2\pi} \int_0^\infty e^{-\frac{1}{2(1-\gamma^2)}x^2}xdx=\frac{1-\gamma^2}{2\pi}.
\end{aligned}
\end{equation}
\end{proof}

\subsection{Proof of Lemma \ref{lem:cal2}}
\begin{proof}
We can compute that
\begin{equation}
\begin{aligned}
 &\mbE\bb{\sigma'(x_1)\sigma'(\gamma_1x_1+\gamma_2x_2+\gamma_3x_3\geq 0)x_1x_2}\\
 =&\int_0^\infty\int_{-\infty}^\infty\int_{\left|\frac{\gamma_1x_1+\gamma_3x_3}{\gamma_2}\right|}^\infty x_1x_2p(x_1)p(x_2)p(x_3)dx_1dx_3dx_2\\
 =&\int_0^\infty\int_{-\infty}^\infty x_1p(x_1)p(x_3)\frac{1}{\sqrt{2\pi}}\exp\pp{-\frac{(\gamma_1x_1+\gamma_3x_3)^2}{2\gamma_2^2}}dx_1dx_3.
\end{aligned}
\end{equation}
Note that
\begin{equation}
\begin{aligned}
&\int_{-\infty}^\infty  p(x_3)\frac{1}{\sqrt{2\pi}}\exp\pp{-\frac{(\gamma_1x_1+\gamma_3x_3)^2}{2\gamma_2^2}} dx_3\\
=&\frac{1}{2\pi}\int_{-\infty}^\infty \exp\pp{-\frac{(\gamma_1x_1+\gamma_3x_3)^2-\gamma_2^2x_3^2}{2\gamma_2^2}}dx_3\\
=&\frac{1}{2\pi}\int_{-\infty}^\infty \exp\pp{-\frac{(\gamma_3^2+\gamma_2^2)x_3^2+2\gamma_1\gamma_3x_1x_3+\gamma_1^2x_1^2}{2\gamma_2^2}}dx_3\\
=&\frac{1}{2\pi}\int_{-\infty}^\infty \exp\pp{-\frac{(\gamma_3^2+\gamma_2^2)\pp{x_3-\frac{\gamma_1\gamma_3}{\gamma_3^2+\gamma_2^2}x_1}^2+\gamma_1^2x_1^2-\frac{\gamma_1^2\gamma_3^2}{\gamma_3^2+\gamma_2^2}x_1^2}{2\gamma_2^2}}dx_3\\
=&\frac{1}{2\pi}\int_{-\infty}^\infty \exp\pp{-\frac{(\gamma_3^2+\gamma_2^2)\pp{x_3-\frac{\gamma_1\gamma_3}{\gamma_3^2+\gamma_2^2}x_1}^2+\frac{\gamma_1^2\gamma_2^2}{\gamma_3^2+\gamma_2^2}x_1^2}{2\gamma_2^2}}dx_3\\
=&\frac{|\gamma_2|}{\sqrt{2\pi}\sqrt{\gamma_3^2+\gamma_2^2}}\exp\pp{-\frac{\gamma_1^2}{2(\gamma_3^2+\gamma_2^2)}x_1^2}. 
\end{aligned}
\end{equation}
Then, we have
\begin{equation}
\begin{aligned}
&\int_0^\infty\int_{-\infty}^\infty x_1p(x_1)p(x_3)\frac{1}{\sqrt{2\pi}}\exp\pp{-\frac{(\gamma_1x_1+\gamma_3x_3)^2}{2\gamma_2^2}}dx_1dx_3\\
=&\int_0^\infty \frac{|\gamma_2|}{2\pi\sqrt{\gamma_3^2+\gamma_2^2}}\exp\pp{-\frac{\gamma_1^2}{2(\gamma_3^2+\gamma_2^2)}x_1^2-\frac{x_1^2}{2}}x_1dx\\
=&\int_0^\infty \frac{|\gamma_2|}{2\pi\sqrt{\gamma_3^2+\gamma_2^2}}\exp\pp{-\frac{1}{2(\gamma_3^2+\gamma_2^2)}x_1^2}x_1dx\\
=&\frac{1-\gamma_1^2}{2\pi}\frac{\gamma_2}{\sqrt{\gamma_3^2+\gamma_2^2}}=\frac{\sqrt{1-\gamma_1^2}\gamma_2}{2\pi}.
\end{aligned}    
\end{equation}
For the second equation, we note that
\begin{equation}
\begin{aligned}
&\mbE\bb{\sigma'(x_1)\sigma'(\gamma_1x_1+\gamma_2x_2+\gamma_3x_3)x_2^2}\\
=&\int_{-\infty}^{\infty}\int_{-\infty}^{\infty}\int_{\max\{0,-\frac{\gamma_2x_2+\gamma_3x_3}{\gamma_1}\}}^{\infty}p(x_1)p(x_2)p(x_3)x_2^2dx_2dx_3dx_1\\
=&\int_{-\infty}^{\infty}\pp{\int_{0}^{\infty}\pp{1-F\pp{-\frac{\gamma_1x_1+\gamma_2x_2}{\gamma_3}}}p(x_1)dx_1}p(x_2)x_2^2dx_2.
\end{aligned}
\end{equation}
For the third equation, we have
\begin{equation}
\begin{aligned}
\mbE\bb{\sigma'(x_1)\sigma'(\gamma_1x_1+\gamma_2x_2+\gamma_3x_3)x_2x_3}=\int_0^\infty\int_{-\infty}^\infty x_3p(x_1)p(x_3)\frac{1}{\sqrt{2\pi}}\exp\pp{-\frac{(\gamma_1x_1+\gamma_3x_3)^2}{2\gamma_2^2}}dx_1dx_3.
\end{aligned}
\end{equation}
Following the previous calculation, we can compute that
\begin{equation}
\begin{aligned}
&\int_{-\infty}^\infty x_3 p(x_3)\frac{1}{\sqrt{2\pi}}\exp\pp{-\frac{(\gamma_1x_1+\gamma_3x_3)^2}{2\gamma_2^2}} dx_3\\
=&\frac{1}{2\pi}\int_{-\infty}^\infty \exp\pp{-\frac{(\gamma_3^2+\gamma_2^2)\pp{x_3-\frac{\gamma_1\gamma_3}{\gamma_3^2+\gamma_2^2}x_1}^2+\frac{\gamma_1^2\gamma_2^2}{\gamma_3^2+\gamma_2^2}x_1^2}{2\gamma_2^2}}dx_3\\
=&\frac{|\gamma_2|}{\sqrt{2\pi}\sqrt{\gamma_3^2+\gamma_2^2}}\frac{\gamma_1\gamma_3}{\gamma_3^2+\gamma_2^2}x_1.
\end{aligned}
\end{equation}
Therefore, we have
\begin{equation}
\begin{aligned}
&\mbE\bb{\sigma'(x_1)\sigma'(\gamma_1x_1+\gamma_2x_2+\gamma_3x_3)x_2x_3}=\int_0^\infty \frac{|\gamma_2|}{\sqrt{2\pi}\sqrt{\gamma_3^2+\gamma_2^2}}\frac{\gamma_1\gamma_3}{\gamma_3^2+\gamma_2^2}x_1 p(x_1)dx_1=\frac{|\gamma_2|\gamma_1\gamma_3}{2\pi\pp{\gamma_3^2+\gamma_2^2}^{3/2}}.
\end{aligned}
\end{equation}
\end{proof}

\section{Proofs in Section \ref{sec:asymp}}
\subsection{Proof of Proposition \ref{prop:asymp_relu}}
\begin{proof}
We denote $\mfh_i=\frac{\mfw^*}{\|\mfw^*\|_2}$. For simplicity, we can assume that $\|\mfh_j\|_2=1$. Hence, it follows that $\gamma=\mfh_i^T\mfh_j$. As $n\to\infty$,  $\mfX^T\mfD_i\mfD_j\mfX$ converges in probability to
\begin{equation}
    \mfM(\mfh_i,\mfh_j)=\mbE_{\mfx\sim\mcN(0,\mfI_d)}[\mfx\mfx^T\mbI(\mfx^T\mfh_i\geq 0)\mbI(\mfx^T\mfh_j\geq 0)].
\end{equation}
According to Lemma 7 in \citep{ghorbani2021linearized}, the above expectation takes the form
\begin{equation}\label{equ:m}
    \mfM(\mfh_i,\mfh_j)=c_1(\gamma)\mfI_d+c_2(\gamma)(\mfh_i\mfh_j^T+\mfh_j\mfh_i^T)+c_3(\gamma)(\mfh_i\mfh_i^T+\mfh_j\mfh_j^T). 
\end{equation}
Here $c_1,c_2,c_3$ are functions of $\gamma$. 
As $n\to\infty$, we note that $\mfX^T\mfD_i\mfX\overset{p}{\to}\mbE_{\mfx\sim\mcN(0,\mfI_d)}[\mbI(\mfx^T\mfh_i\geq 0)\mfx\mfx^T]=\frac{1}{2}\mfI_d$. Thus, as $n\to\infty$, we have
\begin{equation}
\begin{aligned}
T\overset{p}{\to}&2\norm{\mfM(\mfh_i,\mfh_j)\mfh_i}_2.
\end{aligned}
\end{equation}
According to the expression \eqref{equ:m}, we have
\begin{equation}
   \mfM(\mfh_i,\mfh_j)\mfh_i =(c_1(\gamma)+\gamma c_2(\gamma)+c_3(\gamma)) \mfh_i+(c_2(\gamma)+\gamma c_3(\gamma))\mfh_j.
\end{equation}
As $\|\mfh_i\|_2=\|\mfh_j\|_2=1$, the quantity $\norm{\mfM(\mfh_i,\mfh_j)\mfh_i}_2$ only depends on $\gamma=\mfh_i^T\mfh_j$. Denote $g(\gamma)=4\norm{\mfM(\mfh_i,\mfh_j)\mfh_i}_2^2$. Thus, we can simply consider $\mfh_i=\mfe_1$ and $\mfh_j=\gamma \mfe_1+\sqrt{1-\gamma^2}\mfe_2$. According to Lemma \ref{lem:cal}, we can compute that
\begin{equation}
\begin{aligned}
g(\gamma)^2=&4\norm{\mfM(\mfh_i,\mfh_j)\mfh_i}_2^2\\
    =&4\pp{\mbE_{\mfx}\bb{\sigma'(x_1)\sigma'(\gamma x_1+\sqrt{1-\gamma^2}x_2)x_1^2}}^2+4\pp{\mbE_{\mfx}\bb{\sigma'(x_1)\sigma'(\gamma x_1+\sqrt{1-\gamma^2}x_2)x_1x_2}}^2\\
    =&4\pp{\int_0^\infty \pp{1-F\pp{-\frac{\gamma}{\sqrt{1-\gamma^2}}x}} p(x)x^2 dx}^2+4\pp{\frac{1-\gamma^2}{2\pi}}^2.
\end{aligned}
\end{equation}

We plot $g(\gamma)^2$ as a function of $\gamma$ as follows.

\begin{figure}[H]
\centering
\begin{minipage}[t]{0.45\textwidth}
\centering
\includegraphics[width=\linewidth]{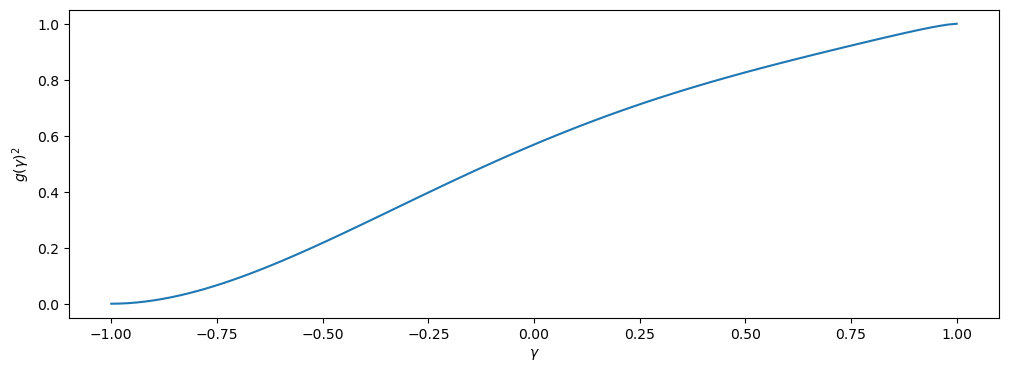}
\end{minipage}
\centering
\begin{minipage}[t]{0.45\textwidth}
\centering
\includegraphics[width=\linewidth]{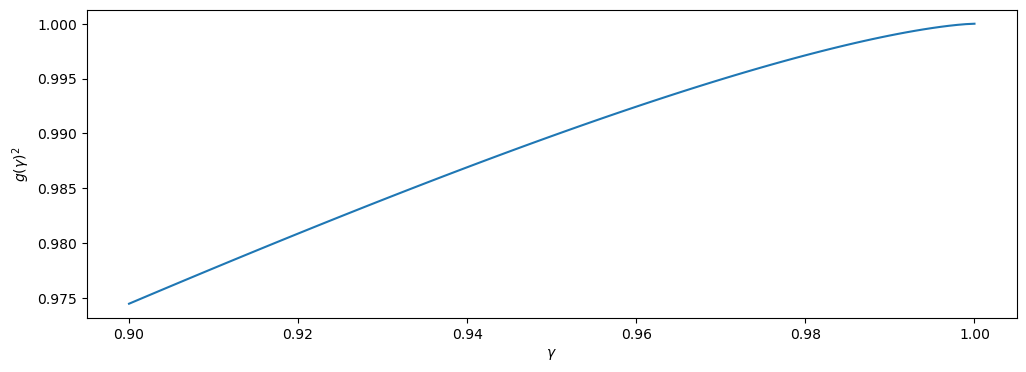}
\end{minipage}
\caption{$g(\gamma)^2$ as a function of $\gamma$.}
\end{figure}


\end{proof}

\section{Necessary condition and sufficient condition of the neural isometry condition for ReLU networks}
For the recovery of single-neuron ReLU network, we first present a necessary condition for the neural isometry condition \eqref{irrep:grelu}. 
\begin{definition}
We say that a diagonal arrangement pattern $\mfD_i\in H$ satisfies the maximal condition if for all index $j\in[p]$ and $j\neq i$, we have $\mfD_i \mfD_j \neq \mfD_{i}$.
\end{definition}
\begin{lemma}\label{lem:maximal}
A necessary condition for \eqref{irrep:grelu} is that the matrix $\mfD_{i^*}$ satisfies the maximal condition.
\end{lemma}
Then, we present a sufficient condition to ensure that the neural isometry condition \eqref{irrep:grelu} holds.
\begin{proposition}\label{prop:single_suff}
Suppose that $\mfD_{i^*}$ satisfies the maximal condition. Assume that the planted neuron $\mfw^*$ satisfies the following conditions
\begin{itemize}
    \item $\mfw^*$ is the eigenvector corresponding to the largest eigenvalue of $\mfX^T\mfD_{i^*}\mfX$, i.e., 
    $$
    \|\mfX^T\mfD_{i^*}\mfX\mfw^*\|_2=\|\mfX^T\mfD_{i^*}\mfX\|_2\|\mfw^*\|_2.
    $$
    \item  $\mfx^T_l \mfw^*> 0$ for all $l\in[n]$ satisfying $(\mfD_{i^*})_{ll}=1$
\end{itemize}
Then, the $\mathrm{NIC\mydash 1}$ given in \eqref{irrep:grelu} holds. 
\end{proposition}


\begin{remark}
The first condition on $\mfw^*$ requires that $\mfw^*$ lies in the eigenspace of $\mfX^T\mfD_{i^*}\mfX$ corresponding to its largest eigenvalue. Combining with the maximal condition on $\mfD_{i^*}$, the second condition on $\mfw^*$ implies that $\mfw^*$ lie in the interior of the cone $\{\mfw|(2\mfD_{i^*}-\mfI_n)\mfX\mfw\geq 0\}$.
\end{remark}

\subsection{Justification of the assumptions in Proposition \ref{prop:single_suff}}
\newcommand{\ddis}{\mfw}
\newcommand{\ddisidx}{q}

We justify the assumption on the planted neuron $\mfw^*$ in Proposition \ref{prop:single_suff} for the Gaussian mixture model. Let $\bmu_1,\bmu_2\in \mbR^d$ and $\sigma>0$. Suppose that $n_1$ of the $n$ training samples follow $\mcN(\bmu_1,\sigma^2\mfI_d)$ and the rest of $n_2=n-n_1$ training samples follows follow $\mcN(\bmu_2,\sigma^2\mfI_d)$. Let $\ddis\in\{0,1\}^n$ be the vector defined by
\begin{equation}
    \ddisidx_i=\begin{cases}
        \begin{aligned}
            1, \text{ s.t. } \mfx_i\sim \mcN(\bmu_1,\sigma^2\mfI_d),\\
            0, \text{ s.t. } \mfx_i\sim \mcN(\bmu_2,\sigma^2\mfI_d).
        \end{aligned}
    \end{cases}
\end{equation}

\begin{proposition}\label{prop:indicator}
Suppose that $\sigma^2>0$ and $b=:\frac{\bmu_1^T\bmu_2}{\|\bmu_1\|_2\|\bmu_2\|_2}<1$.  With probability at least $1-N_1e^{-\frac{(1-b)\|\bmu_1\|_2^2}{4\sigma}^2}-N_2e^{-\frac{(1-b)\|\bmu_2\|_2^2}{4\sigma^2}}$, there exists $i^*\in[p]$ such that $\mfD_{i^*}=\diag(\ddis)$. 
\end{proposition}

Suppose that $\delta\in(0,1)$. From Proposition \ref{prop:indicator}, for $\sigma\leq\frac{\sqrt{1-b}\min\{\|\mu_1\|_2,\|\mu_2\|_2\}}{2\sqrt{\log\pp{\frac{n}{\delta}}}}$, with probability at least $1-\delta$, there exists $i^*\in[p]$ such that $\mfD_{i^*}=\diag(\ddis)$. The next question is whether this $\mfD_{i^*}$ also satisfies the maximal condition.

\begin{proposition}\label{prop:indicator_max}
Suppose that there exists $i^*\in[p]$ such that $\mfD_{i^*}=\diag(\ddis)$. Denote $K^+=\cone(\{\mfx_i|\ddisidx_i=1,i\in[n]\})$ and $K^-=\cone(\{\mfx_i|\ddisidx_i=-1,i\in[n]\})$. Then, $D_{i^*}$ satisfying the maximal condition if and only if  $-K^{-}\subseteq \interior\pp{K^+}$. 
\end{proposition}

Finally, we show that the eigenvalue condition will hold for some specific planted neuron.
\begin{proposition}\label{prop:eig}
Let $\mfx_n=\bmu+\sigma\mfz_i$, where $\mfz_n\sim \mcN(0,I)$ for $n=1,\dots,N$. Suppose that $\delta>0$. Denote $\mfX^{(1)}=\bmbm{\mfx_1^T\\\vdots\\\mfx_{n_1}^T}\in\mbR^{n_1\times d}$. 
Let $\mfw$ satisfy that $\|\mfw\|_2=1$ and $\|\pp{\mfX^{(1)}}^T\mfX^{(1)} \mfw\|_2=\|\pp{\mfX^{(1)}}^T\mfX^{(1)}\|_2$. Then, with probability at least $1-\delta$, we have
\begin{equation}
    \norm{\mfw-\frac{\bmu}{\|\bmu\|_2}}_2^2\leq c_1 \sigma,
\end{equation}
where $c_1=\frac{2d}{n^2\|\bmu\|_2^2}2n_1\sqrt{2\log(n_1d/\delta)}\|\bmu\|_2+2d\log(n_1d/\delta)$ is a constant depending on $\delta$. 
\end{proposition}

The following proposition illustrates that for sufficiently small $\sigma$, except for the maximal condition, all conditions in Proposition \ref{prop:single_suff} will hold with probability at least $1-\delta$. 
\begin{theorem}\label{thm:relu_success}
Let $0<\delta<2\exp(-d/8)$. Suppose that $0< \sigma\leq\min\bbbb{1/(32c_1),\frac{\|\bmu\|_2}{2(d+8\log(4n/\delta))}}$, where $c_1$ is a constant defined in Proposition \ref{prop:eig} Then, there exists $i^*\in \mcP$ such that $\mfD_{i^*}=\diag(\ddis)$ with probability at least $1-\delta$. 
Let $ \mfw^*$ satisfy that $\|\mfX^T\mfD_{i^*}\mfX \mfw^*\|_2=\|\mfX^T\mfD_{i^*}\mfX\|_2\|\mfw^*\|_2$.
We also have $(2\mfD_{i^*}-I)\mfX\mfw^*> 0$ with probability at least $1-\delta$. By further assuming that the cone condition in Proposition \ref{prop:indicator_max} holds, the neural isometry condition holds, i.e., the problem \eqref{min_nrm:grelu} has a unique solution. 
\end{theorem}

\subsection{Proof of Lemma \ref{lem:maximal}}
\begin{proof}
Suppose that there exists $D_j$ such that $\mfD_{i^*} \mfD_j=\mfD_{i^*}$ and $\mfD_j\neq \mfD_{i^*}$. 
This implies that
\begin{equation}
    \mfX^T\mfD_j^T\mfD_i\mfX=\mfX^T\mfD_j\mfD_{i^*}\mfX =\mfX^T\mfD_{i^*}\mfX.
\end{equation}
Thus, we have 
\begin{equation}
  \norm{ \mfX^T\mfD_j^T\mfD_i\mfX(\mfX^T\mfD_{i^*}\mfX)^{-1}\frac{\mfw^*}{\|\mfw^*\|_2}}_2=\norm{\frac{\mfw^*}{\|\mfw^*\|_2}}_2=1.
\end{equation}
Thus, the irrepresentability condition \eqref{irrep:grelu} is violated. 
\end{proof}

\subsection{Proof of Proposition \ref{prop:single_suff}}

\begin{proof}
Consider any $j\in[p]$ and $j\neq i^*$. Let $\mcI_j=\{l\in[n]|(\mfD_{i^*})_{ll}-(\mfD_{j})_{ll}>0\}$. As $\mfD_{i^*}$ satisfies the maximal condition, we have $\mfD_j\mfD_{i^*}\neq \mfD_{i^*}$. This implies that $\mcI_j\neq \varnothing$. Let $k=|\mcI_j|$. Note that
\begin{equation}
    \mfX^T\mfD_{i^*}\mfX = \mfX^T\mfD_j \mfD_{i^*}\mfX+\sum_{l\in \mcI} \mfx_l\mfx_l^T.
\end{equation}
For simplicity, we write $\mfw=\frac{ \mfw^*}{\|\mfw^*\|_2}$,  $\mfA=\mfX^T\mfD_{i^*}\mfD_j\mfX$ and $\mfB=\sum_{l\in \mcI_j} \mfx_l\mfx_l^T$. It is sufficient to prove that 
\begin{equation}
    \|\mfA(\mfA+\mfB)^{-1}\mfw\|_2<1.
\end{equation}
As $\mfw^*$ is the eigenvector corresponding to the largest eigenvalue of $\mfA+\mfB$, we have
\begin{equation}
    (\mfA+\mfB)\mfw=\|\mfA+\mfB\|_2\mfw,
\end{equation}
which also implies that $\|\mfA+\mfB\|_2^{-1}\mfw= (\mfA+\mfB)^{-1}\mfw$. Therefore, we have 
\begin{equation}
    \|(\mfA+\mfB)^{-1}\mfw\|_2=\|(\mfA+\mfB)\|_2^{-1}.
\end{equation}
As $\mfA,\mfB$ are positive semi-definite, we have
\begin{equation}
    \|\mfA+\mfB\|_2\geq \|\mfA\|_2.
\end{equation}
This implies that 
\begin{equation}
    \|\mfA(\mfA+\mfB)^{-1}\mfw\|_2\leq \|\mfA\|_2\|(\mfA+\mfB)^{-1}\mfw\|_2=\|\mfA\|_2\|\mfA+\mfB\|_2^{-1} \leq 1.
\end{equation}
The equality holds if and only if $\|\mfA+\mfB\|_2=\|\mfA\|_2$ and $\mfw$ is the eigenvector of the largest eigenvalue of $\mfA+\mfB$. Let $\gamma=\|\mfA+\mfB\|_2=\|\mfA\|_2$. Then, we have
\begin{equation}
   (\mfA+\mfB)\mfw= \gamma \mfw, \mfA\mfw=\gamma \mfw,
\end{equation}
which implies that $\mfB\mfw=0$. Therefore, we have $\mfw^T\mfB\mfw=0$, or equivalently $\sum_{l\in \mcI_j}\mfx_l^T\mfw=0$. As $\mfx^T_l \mfw> 0$ for all $k$ satisfying $(\mfD_{i^*})_{ll}=1$, we have $\sum_{l\in \mcI_j}\mfx_l^T\mfw>0$, which leads to a contradiction. 
\end{proof}

\subsection{Proof of Proposition \ref{prop:indicator}}

\begin{proof}
Let $\mfw=\frac{\bmu_1}{\|\bmu_1\|_2}-\frac{\bmu_2}{\|\bmu_2\|_2}$. We can compute that
\begin{equation*}
    \|\mfw\|_2^2 =2-2\frac{\bmu_1^T\bmu_2}{\|\bmu_1\|_2\|\bmu_2\|_2}=2(1-b).
\end{equation*} 
This implies that $\|\mfw\|_2=\sqrt{2(1-b)}$. Note that
\begin{equation*}
    \bmu_1^Tw = \|\bmu_1\|_2-\frac{\bmu_1^T\bmu_2}{\|\bmu_2\|_2} = \|\bmu_1\|_2-b\|\bmu_1\|_2=(1-b)\|\bmu_1\|_2>0.
\end{equation*}
Similarly, we have
\begin{equation*}
    \bmu_2^T\mfw = \frac{\bmu_1^T\bmu_2}{\|\bmu_1\|_2}-\|\bmu_2\|_2=-(1-b)\|\bmu_2\|_2<0.
\end{equation*}
For index $i\in[n]$ such that $\ddisidx_i=1$, we can write $\mfx_i=\bmu_1+\sigma \mfz_i$, where $\mfz_i\sim\mcN(0,1)$. We can compute that $\mfx_i^T\mfw = \bmu_1^T\mfw+\sigma \mfz_i^T\mfw=(1-b)\|\bmu_1\|_2+\sigma \mfz_i^T\mfw$. 
According to the tail bound of Gaussian random variable, we have
\begin{equation*}
    P(\mfx_i^T\mfw\leq 0)=P(\sigma \mfz_i^T\mfw\leq -(1-b)\|\bmu_1\|_2)\leq \exp\pp{-\frac{(1-b)^2\|\bmu_1\|_2^2}{2\sigma^2\|\mfw\|_2^2}}=\exp\pp{-\frac{(1-b)\|\bmu_1\|_2^2}{4\sigma^2}},
\end{equation*}
which implies that
\begin{equation*}
    P(\mfx_n^T\mfw> 0)\geq 1- \exp\pp{-\frac{(1-b)\|\bmu_1\|_2^2}{4\sigma^2}}.
\end{equation*}
Therefore, we have
\begin{equation*}
    P(\mfx_i^T\mfw> 0, \forall i \text{ with } \ddisidx_i=1) \geq \pp{ 1-\exp\pp{-\frac{(1-b)\|\bmu_1\|_2^2}{4\sigma^2}}}^{n_1}\geq 1-n_1\exp\pp{-\frac{(1-b)\|\bmu_1\|_2^2}{4\sigma^2}}.
\end{equation*}
For index $i\in[n]$ such that $\ddisidx_i=0$, we can write $\mfx_n=\bmu_2+\sigma\mfz_i$. Similarly, we can compute that 
\begin{equation}
    P(\mfx_i^T\mfw< 0, \forall i \text{ with } y_i=0)\geq 1-n_2\exp\pp{-\frac{(1-b)\|\bmu_2\|_2^2}{4\sigma^2}}.
\end{equation}
In summary, we have
\begin{equation}
    P((2c_i-1)\mfx_i^T\mfw>0,\forall i\in[n])\geq 1-n_1\exp\pp{-\frac{(1-b)\|\bmu_1\|_2^2}{4\sigma^2}}-n_2\exp\pp{-\frac{(1-b)\|\bmu_2\|_2^2}{4\sigma^2}}.
\end{equation}
Under the event $\{\ddisidx_i\mfx_i^T\mfw>0,\forall i\in[n]\}$, the diagonal arrangement pattern $\mfD_{i^*}$ induced by the vector $\mfw$ is exactly $\diag(\mfc)$. This completes the proof. 
\end{proof}

\subsection{Proof of Proposition \ref{prop:indicator_max}}
\begin{proof}
Suppose that $D_{i^*}$ satisfies the maximal condition. Then, for any $\mfw$ satisfying that
\begin{equation}\label{w:positive}
    \mfw^T\mfx_i\geq 0,\forall i \text{ with } \ddisidx_i=1,
\end{equation}
we shall have $\mfw^Tx_i<0$ for $\ddisidx_i=0$. We note that the condition \eqref{w:positive} is equivalent to $w\in -(K^+)^\circ$. Here $(K^+)^\circ=:\{\mfw|\mfw^T\mfx\leq 0, \forall \mfx\in K^+\}$ is the polar cone of the cone $K^+$. We note that $\mfw^T\mfx_i<0$ for $c_i=0$ implies that $\mfw\in \interior\pp{(K^-)^\circ}$. Therefore, $(K^+)^\circ\subseteq \interior\pp{(K^-)^\circ}$. This is equivalent to $-K^{-}\subseteq \interior\pp{K^+}$. 

Suppose that $-K^{-}\subseteq \interior\pp{K^+}$. Then, we have $(K^+)^\circ\subseteq \interior\pp{(K^-)^\circ}$. This implies that for any $\mfw\in\mbR^d$ satisfying \eqref{w:positive}, we shall have $\mfw^T\mfx_i<0$ for $c_i=0$. Therefore, $\mfD_{i^*}$ satisfies the maximal condition.
\end{proof}

\subsection{Proof of Proposition \ref{prop:eig}}
\begin{proof}
We can write
\begin{equation}
    \mfX^{(1)} = \bone \bmu^T+\sigma\mfZ,
\end{equation}
where $\mfZ=\bmbm{\mfz_1^T\\\vdots\\\mfz_{n_1}^T}\in\mbR^{n_1\times d}$ and $\bone\in \mbR^{n_1}$ is a vector of $1$s. Each element of $\mfZ$ follows $\mcN(0,1)$. We note that the extreme eigenvalue vector of $(\bone \bmu^T)^T\bone \bmu^T=n_1^2\bmu\bmu^T$ is $\frac{\bmu}{\|\bmu\|_2}$. According to the tail bound of Gaussian random variable, with probability at least $1-\delta$, we have
\begin{equation}
    \max_{i\in[n_1],j\in[d]}|z_{i,j}|\leq \sqrt{2\log(n_1d/\delta)}.
\end{equation}
Denote $E=\{\max_{i\in[n_1],j\in[d]}|z_{i,j}|\leq \sqrt{2\log(n_1d/\delta)}\}$. Conditioned on the event $E$, we have
\begin{equation}
\begin{aligned}
        &\|(\bone \bmu^T)^T\mfZ+\mfZ^T\bone \bmu^T+\mfZ^T\mfZ\|_F
    \leq 2\sqrt{2\log(n_1d/\delta)}\|\bone \bmu^T\|_F+2d\log(n_1d/\delta)=2n_1\sqrt{2\log(n_1d/\delta)}\|\bmu\|_2+2d\log(n_1d/\delta)=:c_0.
\end{aligned}
\end{equation}

Let $\pp{\mfX^{(1)}}^T\mfX^{(1)}=\mfV\Sigma \mfV^T$ be the eigenvalue decomposition. According the Weil's theorem, we note that
\begin{equation}
    |n^2\|\mu\|_2^2-\sigma_1|\leq \sigma\|(\bone \bmu^T)^T\mfZ+\mfZ^T\bone \mu^T+\mfZ^T\mfZ\|_F\leq \sigma c_0.
\end{equation}
The other eigenvalues $\sigma_i$ of $\pp{\mfX^{(1)}}^T\mfX^{(1)}$ satisfies that 
\begin{equation}
    |\sigma_i|\leq \sigma \|(\bone \mu^T)^T\mfZ+\mfZ^T\bone \mu^T+\mfZ^T\mfZ\|_F \leq \sigma c_0.
\end{equation}
We note that 
\begin{equation}
\begin{aligned}
    \norm{\pp{\mfX^{(1)}}^T\mfX^{(1)} \bmu}_2\geq& n^2\|\bmu\|_2^3-\sigma \|(\bone \bmu^T)^T\mfZ+\mfZ^T\bone \bmu^T+\mfZ^T\mfZ\|_2\|\bmu\|_2\\
    \geq& n^2\|\bmu\|_2^3-\sigma \|(\bone \bmu^T)^T\mfZ+\mfZ^T\bone \bmu^T+\mfZ^T\mfZ\|_F\|\bmu\|_2\\
    =&(n^2\|\bmu\|_2^2-\sigma c_0)\|\bmu\|_2.
\end{aligned}
\end{equation}
As $\pp{\mfX^{(1)}}^T\mfX^{(1)} \bmu= \mfV\Sigma \mfV^T \bmu = \sum_{i=1}^d \sigma_i \mfv_i(\mfv_i^T\bmu)$, we have
\begin{equation}
\begin{aligned}
    \sigma_1 |\mfv_1^T\bmu|\geq &\norm{\mfV\Sigma \mfV^T \bmu-\sum_{i=2}^d \sigma_i \mfv_i(\mfv_i^T\bmu)}_2=\|\pp{\mfX^{(1)}}^T\mfX^{(1)} \bmu\|_2-\sum_{i=2}^d \sigma_i |\mfw_i^T\bmu|\geq (n^2\|\bmu\|_2^2-(d-1)c_0\sigma)\|\bmu\|_2.
\end{aligned}
\end{equation}
Note that $\mfw=\mfv_1$. This implies that
\begin{equation}
    \left| \mfw^T\frac{\bmu}{\|\bmu\|_2}\right|\geq \frac{n^2\|\bmu\|_2^2-(d-1)c_0\sigma }{n^2\|\bmu\|_2^2+c_0\sigma }.
\end{equation}
Note that
\begin{equation}
    \norm{\mfw-\frac{\bmu}{\|\bmu\|_2}}_2^2 = 2\pp{1-\left| \mfw^T\frac{\bmu}{\|\bmu\|_2}\right|}\leq \frac{2dc_0\sigma }{n^2\|\bmu\|_2^2+c_0\sigma}\leq c_1 \sigma,
\end{equation}
where we let $
    c_1 = \frac{2dc_0}{n^2\|\bmu\|_2^2}.
$ This completes the proof. 

\end{proof}

\subsection{Proof of Theorem \ref{thm:relu_success}}
\begin{proof}
Denote $\hat \mfw^*=\frac{\mfw^*}{\|\mfw^*\|_2}$. According to Proposition \ref{prop:eig}, there exists a constant $c_1>0$ such that $\hat \mfw^*$ satisfies that
\begin{equation}
    \norm{\hat \mfw^*-\frac{\bmu}{\|\bmu\|_2}}_2^2\leq c_1\sigma,
\end{equation}
with probability at least $1-\delta/2$. For  $\sigma\leq\frac{\|\bmu\|_2}{2(d+8\log(4n/\delta))}$, the events
\begin{equation}
    \sigma\|\mfz_i\|_2\leq \|\bmu\|_2/2,
\end{equation}
holds with probability at least $1-\delta/(2n)$ respectively for all $i\in[n]$. Denote $E=\{\sigma\|\mfz_i\|_2\leq \|\bmu\|_2/2, \forall i\in[n]\}$. Then, $P(E)\geq 1-\delta/4$. For indices $i\in[n]$ such that $\ddisidx_i=1$, conditioned on $E$, we have $\sigma\left|\mfz_i^T\frac{\bmu}{\|\bmu\|_2}\right|\leq \|\bmu\|_2/2$ and
\begin{equation}
\begin{aligned}
        \mfx_i^T\hat \mfw^*\geq&\mfx_i^T\frac{\bmu}{\|\bmu\|_2}-\|\mfx_i\|\norm{\hat \mfw^*-\frac{\bmu}{\|\bmu\|_2}}_2 
        \geq \mfx_i^T\frac{\bmu}{\|\bmu\|_2}-\sqrt{c_1\sigma}\|\mfx_i\|_2
        \geq \|\bmu\|_2+\sigma\mfz_n^T\frac{\bmu}{\|\bmu\|_2} - 2\sqrt{c_1\sigma}\|\bmu\|_2>0
\end{aligned}
\end{equation}
Here we utilize that $\sigma\leq \frac{1}{32c_1}$. Similarly, for indices $i\in[n]$ such that $\ddisidx_i=-1$, conditioned on $E$, we have $\sigma\left|\mfz_i^T\frac{\bmu}{\|\bmu\|_2}\right|\leq \|\bmu\|_2/2$ and
\begin{equation}
\begin{aligned}
    \mfx_i^T\hat \mfw^*\leq&\mfx_i^T\frac{\bmu}{\|\bmu\|_2}+\|\mfx_i\|\norm{\hat \mfw^*-\frac{\bmu}{\|\bmu\|_2}}_2 \leq\mfx_i^T\frac{\bmu}{\|\bmu\|_2}+\sqrt{c_1\sigma}\|\mfx_i\|_2\mfx_i^T\frac{\bmu}{\|\bmu\|_2}
    \le -\|\bmu\|_2-\sigma\mfz_i^T\frac{\bmu}{\|\bmu\|_2} + 2\sqrt{c_1\sigma}\|\bmu\|_2<0.
\end{aligned}
\end{equation}
This implies that $(2\mfD_{i^*}-I)\mfX\hat \mfw^*> 0$. Overall, for sufficiently small $a>0$, the event $(2\mfD_{i^*}-I)\mfX\hat \mfw^*> 0$ holds with probability at least 
\begin{equation}
     1-\delta/2-\delta/2=1-\delta.
\end{equation}
\end{proof}

\subsection{Numerical verification}
In this subsection, we numerically verify Proposition \ref{prop:indicator}. We take $n=100,n_1=n_2=50$ and $d=20,40,60,80$, and test for three types of $\bmu_1,\bmu_2$.
\begin{itemize}
    \item $\bmu_1 = \mathbf{1}_d,\ \bmu_2 = -\mathbf{1}_d$.
    \item $\bmu_1,\bmu_2 \sim \mcN(0,\mfI_d)$.
    \item $\bmu_1,\bmu_2 \sim \mcU(\mbS^{d-1})$.
\end{itemize}
 For each $d$, we compute the probability that the diagonal arrangement pattern induced by the vector
$\mfw=\frac{\bmu_1}{\|\bmu_1\|_2}-\frac{\bmu_2}{\|\bmu_2\|_2}$
equals to $\diag(\mfw)$ for $\sigma$ in a certain range. For each $\sigma$, we establish 5000 independent trials to compute the probability. The S-shaped curves shown in Figure \ref{fig:prop:indicator} correspond with formulation of the lower bound given in Proposition \ref{prop:indicator}.
\begin{figure}[H]
    \subfigure[$\bmu_1 = \mathbf{1}_d,\ \bmu_2 = -\mathbf{1}_d$]{
      \centering
      \includegraphics[width=.45\textwidth]{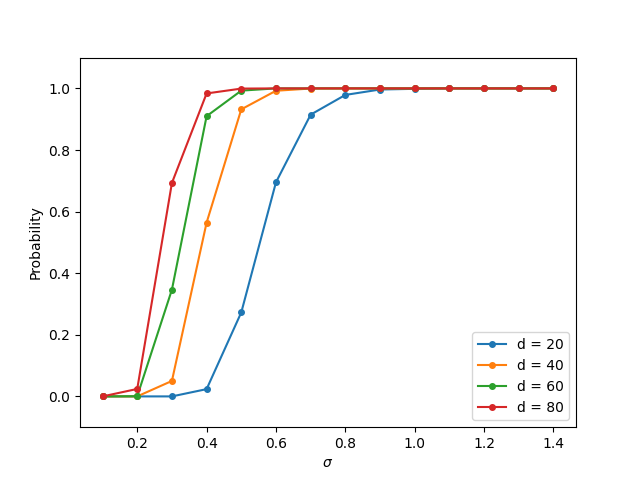}  
    }
    \subfigure[$\bmu_1,\bmu_2 \sim \mcN(0,\mcI_d)$]{
      \centering
      \includegraphics[width=.45\textwidth]{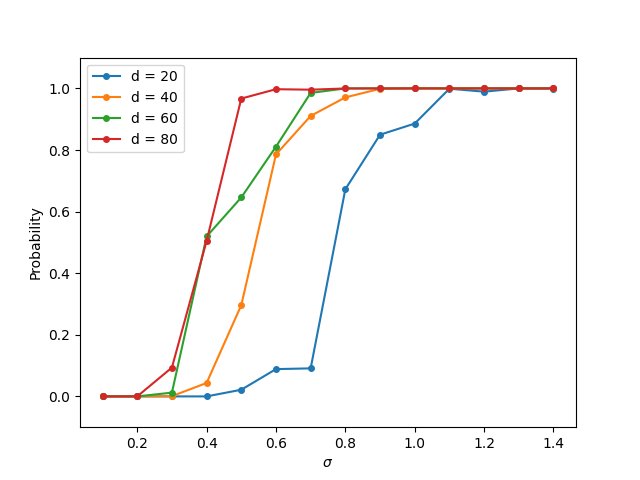}  
    }
    \end{figure}
\begin{figure}[H]
 \addtocounter{figure}{1}  
    \ContinuedFloat
    \centering
    \subfigure[$\bmu_1,\bmu_2 \sim \mcU(\mbS^{d-1})$]{
      \centering
      \includegraphics[width=.45\textwidth]{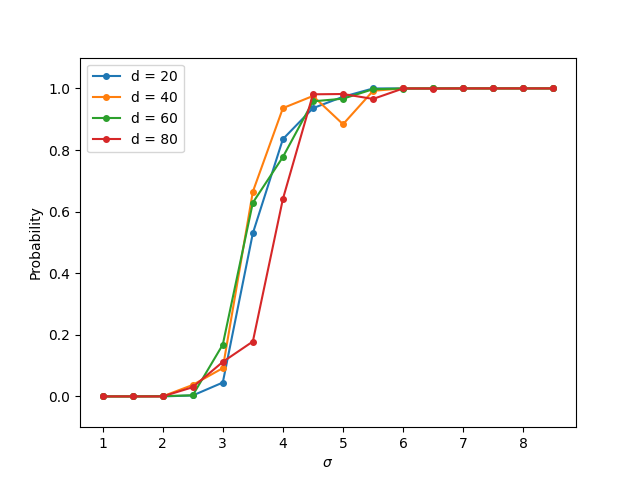}  
    }
    \caption{The probability that the statement in Proposition \ref{prop:indicator} holds over 5000 independent trials.}
    \label{fig:prop:indicator}
\end{figure}
\section{Additional numerical experiments}\label{num_res:add}
In this section, we present additional numerical results mentioned in Section \ref{num_res:main}. These results serve as a complement to our main numerical results.



\subsection{ReLU networks with skip connection}\label{num_res:skip}

In Figure \ref{fig:lin_phase_w0}, we show phase transition graph for the probability of successful recovery of the planted linear neuron by solving the group $\ell_1$-minimization problem \eqref{min_nrm:grelu_skip} when the planted neuron $\mfw^*$ is randomly generated from $\mcN(0,\mfI_d)$. In Figure \ref{fig:lin_phase_w1} below, we find similar results when the planted neuron $\mfw^*$ is the smallest right singular vector of $\mfX$.

\begin{figure}[H]
\centering
\setcounter{subfigure}{0}
    \subfigure[Gaussian]{
      \centering
      \includegraphics[width=\figscale\textwidth]{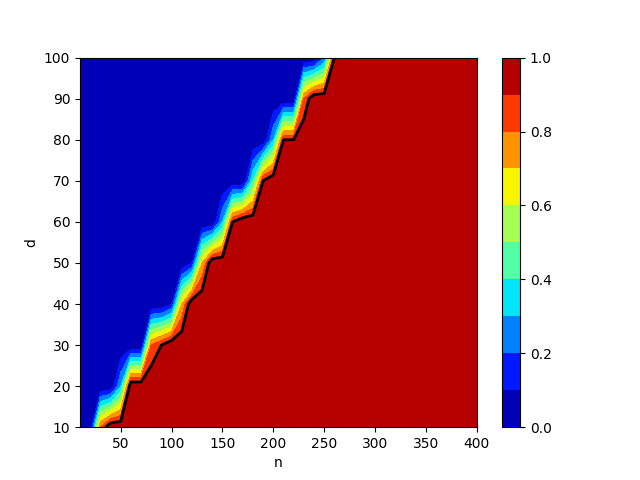}  
    }
    \subfigure[Cubic Gaussian]{
      \centering
      \includegraphics[width=\figscale\textwidth]{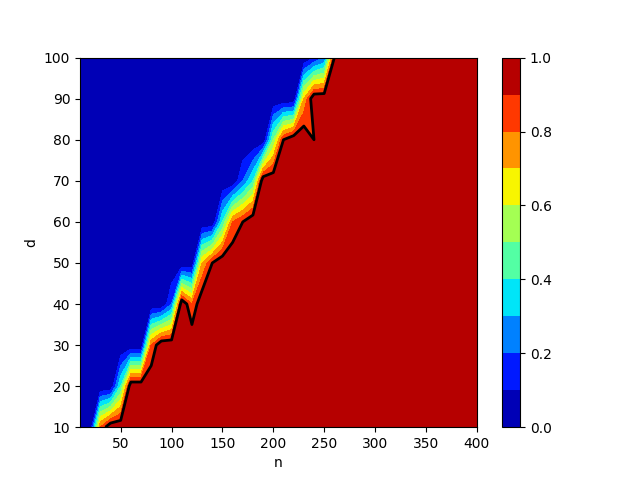}  
    }
\end{figure}
\begin{figure}[H]
 \addtocounter{figure}{1}  
    \ContinuedFloat
    \subfigure[Haar]{
      \centering
      \includegraphics[width=\figscale\textwidth]{figs/lin/lin_phase_w0_X2.png}  
    }
    \subfigure[Cubic Gaussian + whitened]{
      \centering
      \includegraphics[width=\figscale\textwidth]{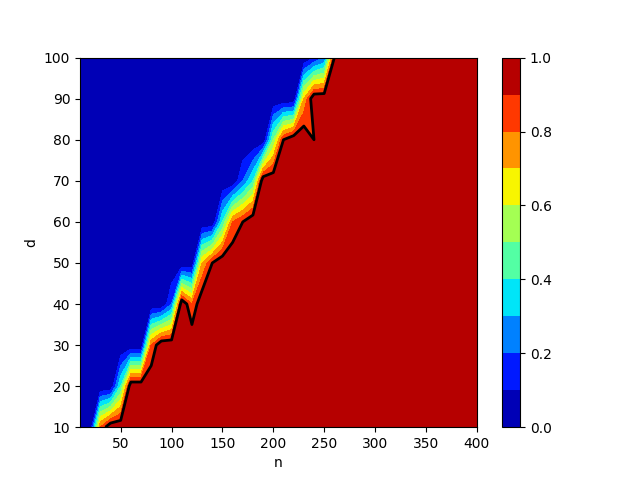}  
    }
    \caption{The probability of successful recovery of the planted linear neuron by solving the group $\ell_1$-minimization problem \eqref{min_nrm:grelu_skip} over $5$ independent trials. The black lines represent the boundaries of successful recovery with probability $1$. Here the planted neuron $\mfw^*$ is the smallest right singular vector of $\mfX$.}\label{fig:lin_phase_w1}
\end{figure}

In Figure \ref{fig:lin_phase_noise_abs}, we show phase transition graph for absolute distance by solving the group $\ell_1$-minimization problem \eqref{min_nrm:grelu_skip} derived from training ReLU networks with skip connection. In Figure \ref{fig:lin_phase_noise_test} below, we find similar pattern of the phase transition.

\begin{figure}[H]
\setcounter{subfigure}{0}
\centering
    \subfigure[$\sigma=0$(noiseless)]{
      \centering
      \includegraphics[width=\figscale\textwidth]{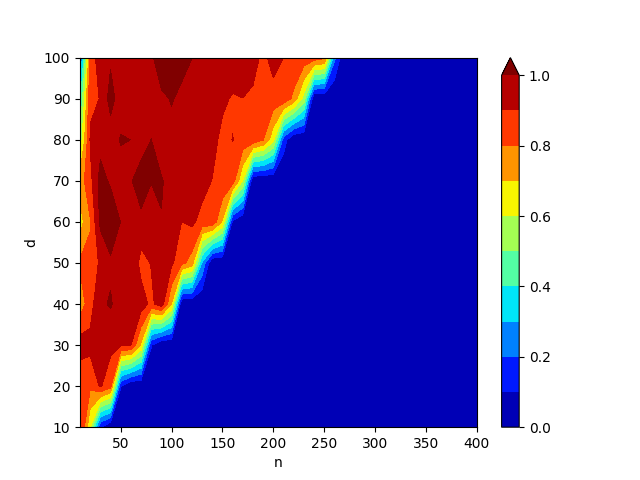}  
    }
    \centering
    \subfigure[$\sigma=0.05$]{
      \centering
      \includegraphics[width=\figscale\textwidth]{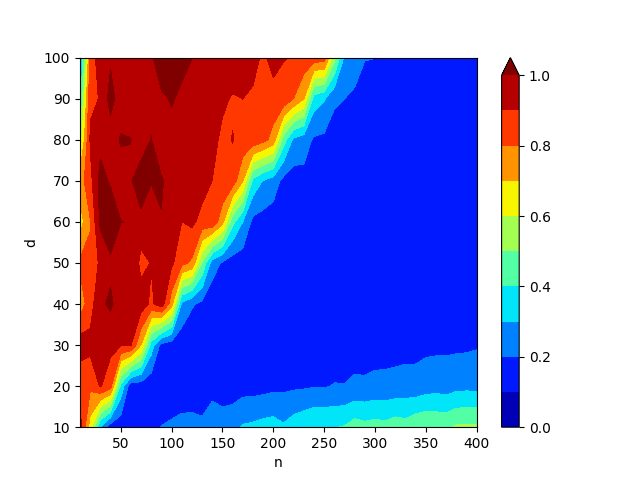}  
    }
\end{figure}
\begin{figure}[H]
 \addtocounter{figure}{1}  
    \ContinuedFloat
\centering
    \subfigure[$\sigma=0.1$]{
      \centering
      \includegraphics[width=\figscale\textwidth]{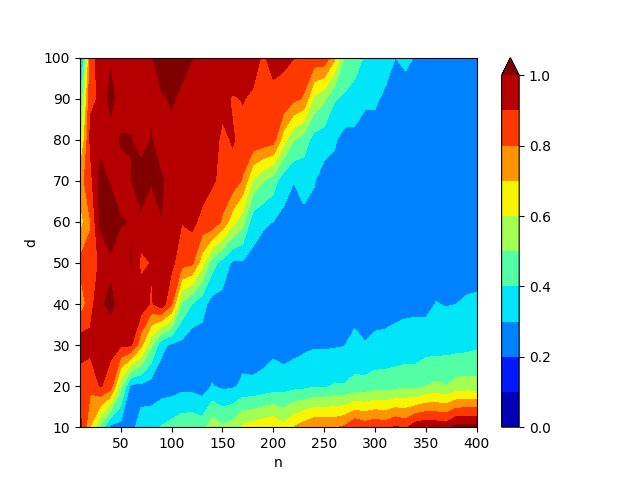}  
    }
    \centering
    \subfigure[$\sigma=0.2$]{
      \centering
      \includegraphics[width=\figscale\textwidth]{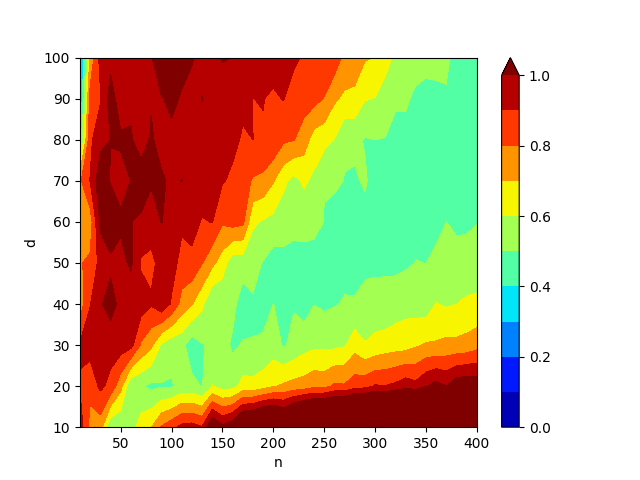}  
    }
    \caption{Averaged test distance by solving the group $\ell_1$-minimization problem \eqref{min_nrm:grelu_skip} derived from training ReLU networks with skip connection over $5$ independent trials.}
    \label{fig:lin_phase_noise_test}
\end{figure}

As a complement to Figure \ref{fig:ncvx_train_skip}, we study the generalization property of ReLU networks with skip connections using convex training methods.
We solve the following regularized training problem with small regularization parameter $\beta=10^{-6}$ as an approximation of the minimal-norm problem \eqref{min_nrm:relu_skip}.

\begin{equation}\label{cvx_train:relu_skip}
\begin{aligned}
    \min_{\mfw_0,\left\{\mfw_{j}, \mfw_{j}'\right\}_{j=1}^{p}}\quad &  \left\|\mfX\mfw_0+\sum_{j=1}^{p} \mfD_{j} \mfX\left(\mfw_{j}-\mfw_{j}'\right)-\mfy\right\|_2^2+ \beta\left(\|\mfw_0\|_2+\sum_{j=1}^{p}\left(\left\|\mfw_{j}\right\|_{2}+\left\|\mfw_{j}'\right\|_{2}\right)\right) ,\\
        \text{s.t.}\quad & (2\mfD_j-\mfI_n)\mfX\mfw_j\ge0,  (2\mfD_j-\mfI_n)\mfX\mfw_j'\ge0, j\in[p].
\end{aligned}
\end{equation}
Then we compute the corresponding absolute distance and test distance. The results shown in Figures \ref{fig:lin_cvx_abs} and \ref{fig:lin_cvx_test} have the same sharp $n=2d$ boundary of the red region. This validates Proposition \ref{prop:imply_linear}. Namely, the recovery of the group $\ell_1$-minimization problem \eqref{min_nrm:grelu_skip} implies the recovery of the convex program \eqref{min_nrm:relu_skip}.

\begin{figure}[H]
\setcounter{subfigure}{0}
\centering
    \subfigure[$\sigma=0$(noiseless)]{
      \centering
      \includegraphics[width=\figscale\textwidth]{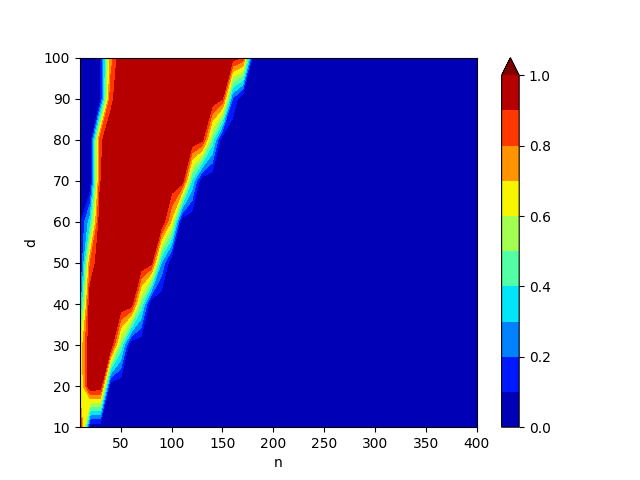}  
    }
    \centering
    \subfigure[$\sigma=0.05$]{
      \centering
      \includegraphics[width=\figscale\textwidth]{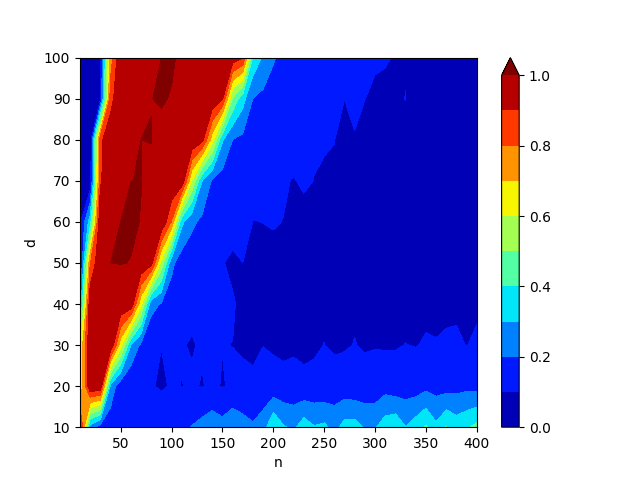}  
    }
\end{figure}
\begin{figure}[H]
 \addtocounter{figure}{1}  
    \ContinuedFloat
\centering
    \subfigure[$\sigma=0.1$]{
      \centering
      \includegraphics[width=\figscale\textwidth]{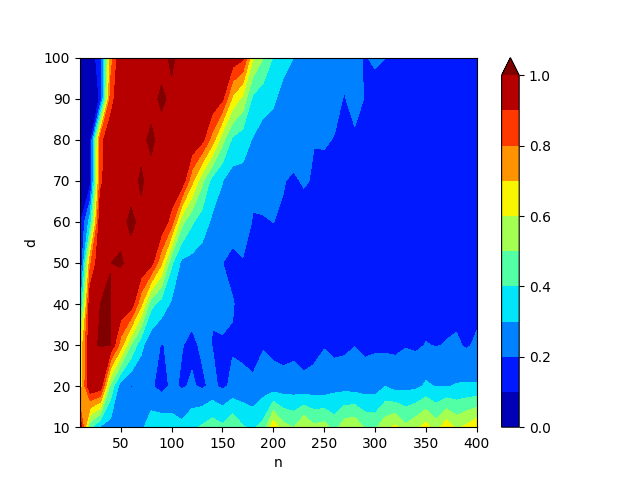}  
    }
    \centering
    \subfigure[$\sigma=0.2$]{
      \centering
      \includegraphics[width=\figscale\textwidth]{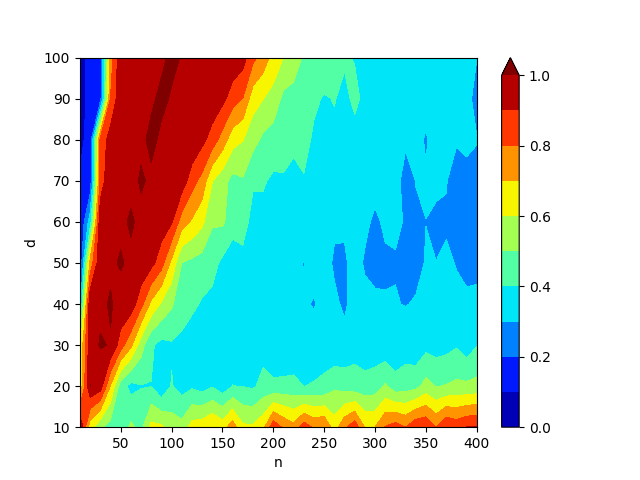}  
    }
     \caption{Averaged absolute distance to the planted linear neuron by solving the convex regularized training problem \eqref{cvx_train:relu_skip} for ReLU networks with skip connection over $5$ independent trials.}\label{fig:lin_cvx_abs}
\end{figure}
\begin{figure}[H]
\setcounter{subfigure}{0}
\centering
    \subfigure[$\sigma=0$(noiseless)]{
      \centering
      \includegraphics[width=\figscale\textwidth]{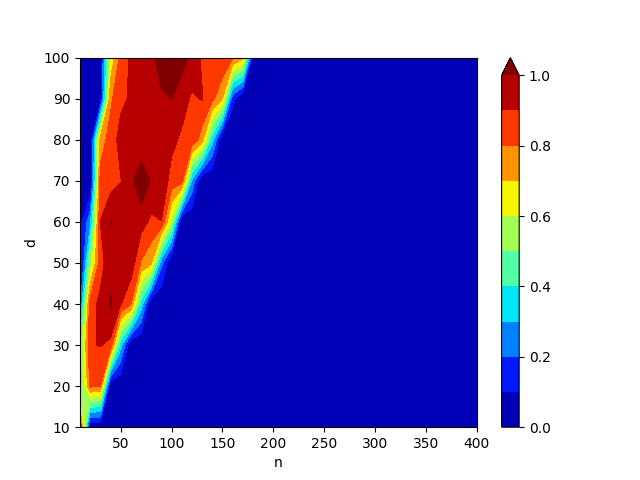}  
    }
    \centering
    \subfigure[$\sigma=0.05$]{
      \centering
      \includegraphics[width=\figscale\textwidth]{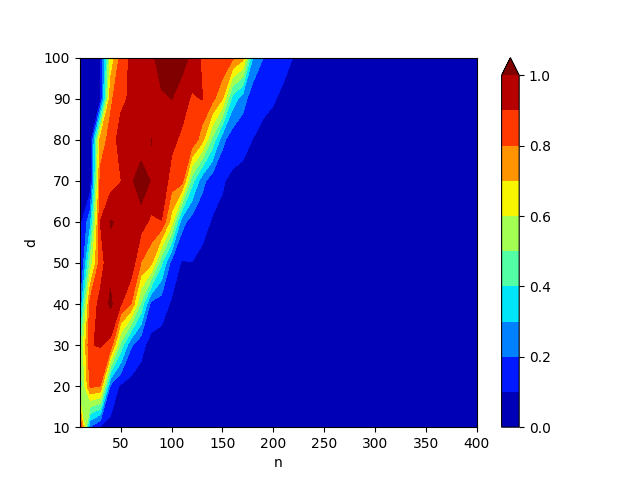}  
    }
    \centering
    \subfigure[$\sigma=0.1$]{
      \centering
      \includegraphics[width=\figscale\textwidth]{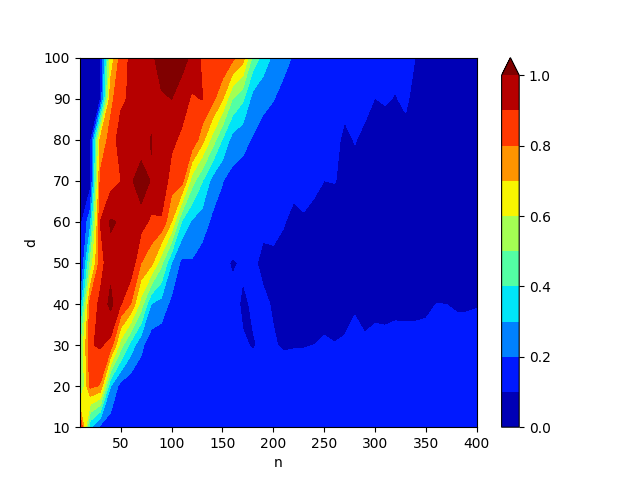}  
    }
    \centering
    \subfigure[$\sigma=0.2$]{
      \centering
      \includegraphics[width=\figscale\textwidth]{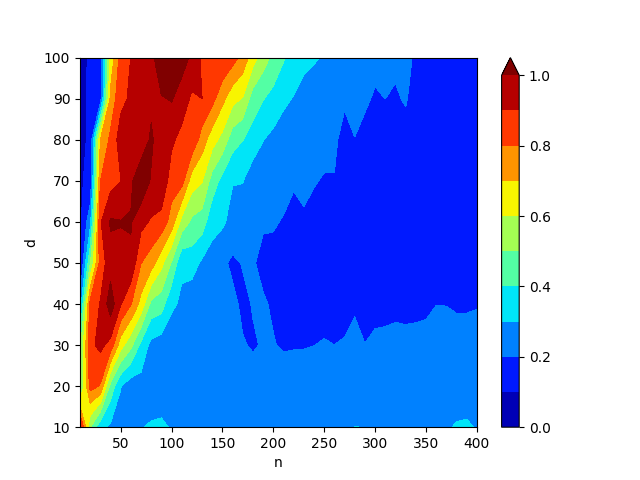}  
    }
    \caption{Averaged test distance by solving the convex regularized training problem \eqref{cvx_train:relu_skip} for ReLU networks with skip connection over $5$ independent trials.}\label{fig:lin_cvx_test}
\end{figure}
\vspace{1ex}

\subsection{ReLU networks with normalization layer}\label{num_res:normal}

We first present the phase transition graph for successful recovery of the planted normalized ReLU neuron by solving the group $\ell_1$-minimization \eqref{min_nrm:grelu_normal}.
Analogous to Section \ref{num_res_main:skip}, 
we compute the recovery rate for $d$ ranging from 10 to 100 and $n$ ranging from 10 to 400 and establish 5 independent trials for each pair of $(n,d)$. Here each element $x_{i,j}$ of the dataset $\mfX\in\mbR^{n\times d}$ is i.i.d. random variable following $\mcN(0,1/n)$. The planted neuron $\mfw^*$ is randomly generated from $\mcN(0,\mfI_d)$.
\begin{figure}[H]
      \centering
      \includegraphics[width=\figscale\textwidth]{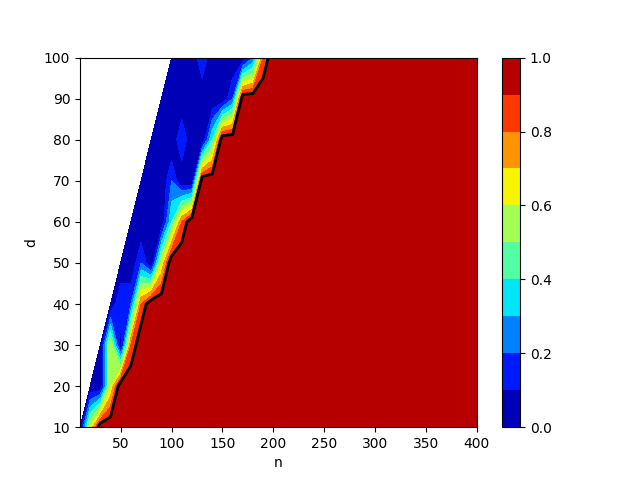}  
    \caption{The probability of successful recovery of the planted normalized ReLU neuron by solving the group $\ell_1$-minimization problem \eqref{min_nrm:grelu_normal} over $5$ independent trials. The black line represents the boundaries of successful recovery with probability $1$. Here the planted neuron $\mfw^*$ is randomly generated from $\mcN(0,\mfI_d)$.}
    \label{fig:lin_phase_normal}
\end{figure}

The second part is phase transition under noisy observation, i.e.,
$\mfy=\frac{(\mfX\mfw^*)_+}{\|(\mfX\mfw^*)_+\|_2}+\mfz$, where $\mfz\sim \mcN(0,\sigma^2/n)$. In Figure \ref{fig:lin_phase_noise_normal}, we focus on absolute distance, which is defined as the $\ell_2$ distance between the optimal solution $\mfw_{i^*}$ and $\tilde{\mfw}^*$. A sharp $n=2d$ boundary can be observed from the phase transition graphs, which corresponds with Theorem \ref{thm:relu_normal}.
\vspace{-4ex}
\begin{figure}[H]
\setcounter{subfigure}{0}
\centering
    \subfigure[$\sigma=0$(noiseless)]{
      \centering
      \includegraphics[width=\figscale\textwidth]{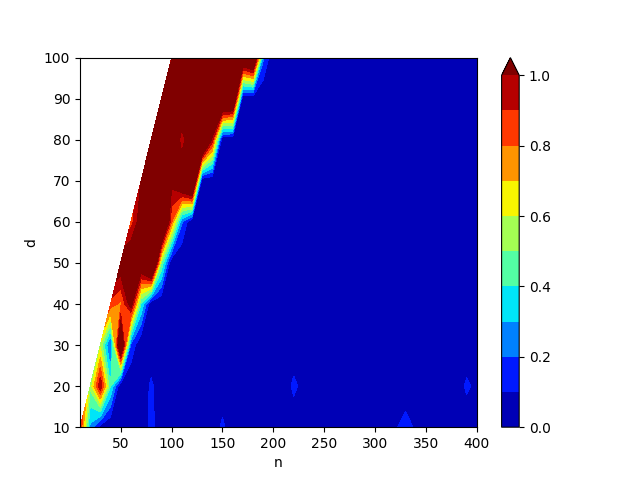}  
    }
    \centering
    \subfigure[$\sigma=0.05$]{
      \centering
      \includegraphics[width=\figscale\textwidth]{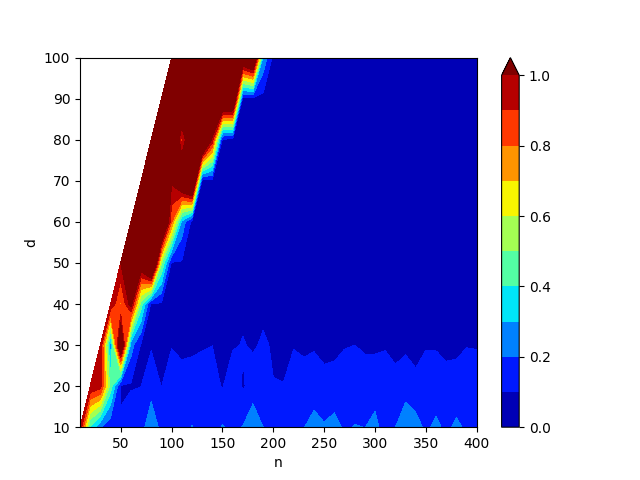}  
    }
    \newline
    \end{figure}
\begin{figure}[H]
 \addtocounter{figure}{1}  
    \ContinuedFloat
    \centering
    \subfigure[$\sigma=0.1$]{
      \centering
      \includegraphics[width=\figscale\textwidth]{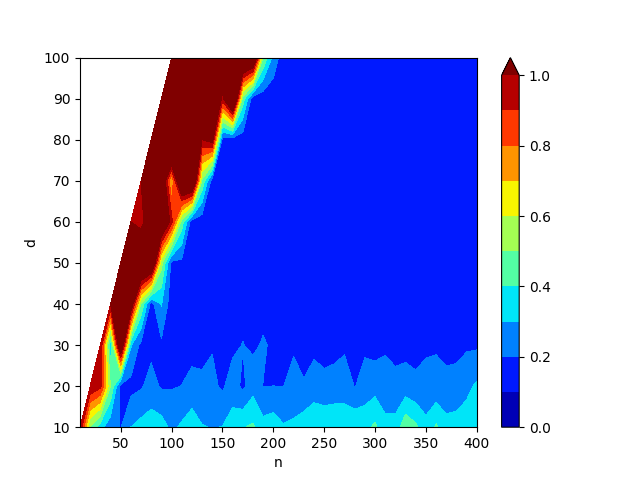}  
    }
    \centering
    \subfigure[$\sigma=0.2$]{
      \centering
      \includegraphics[width=\figscale\textwidth]{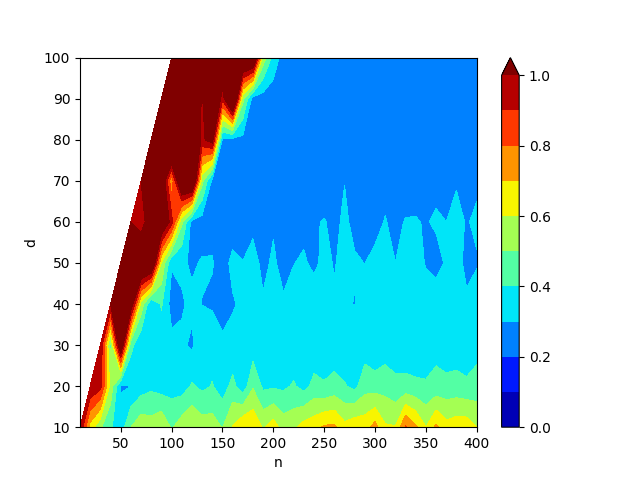}  
    }
    \caption{Averaged absolute distance to the planted normalized ReLU neuron by solving the group $\ell_1$-minimization problem \eqref{min_nrm:grelu_normal}
from training ReLU networks with normalization layer over 5 independent trials.}\label{fig:lin_phase_noise_normal}
\end{figure}

In the third part, we study the generalization property of ReLU networks with normalization layer using both the convex and non-convex
training procedures. For the non-convex training procedure, we minimize the regularized training objective with small
regularization parameter $\beta=10^{-6}$ to approximate the minimal-norm problem \eqref{min_nrm:relu_normal}.

\begin{equation}\label{cvx_train:relu_normal}
\begin{aligned}
    \min_{\left\{\mfw_{j}, \mfw_{j}'\right\}_{j=1}^{p}}\quad &  \left\|\sum_{j=1}^{p} \mfU_j\left(\mfw_{j}-\mfw_{j}'\right)-\mfy\right\|_2^2+\beta \sum_{j=1}^{p}\left(\left\|\mfw_{j}\right\|_{2}+\left\|\mfw_{j}'\right\|_{2}\right)\\
        \text{s.t.}\quad 
        & (2\mfD_j-\mfI_n)\mfX\mfV_j^T\bSigma_j^{-1}\mfw_j\ge0, (2\mfD_j-\mfI_n)\mfX\mfV_j^T\bSigma_j^{-1}\mfw_j'\ge0, j\in[p],
\end{aligned}
\end{equation}

Then we compute the corresponding absolute distance. The phase transition graphs are shown in Figure \ref{fig:cvxt_normal}. We can also observe a sharp $n=2d$ transition similar to the group $\ell_1$-minimization \eqref{min_nrm:grelu_normal}.

\begin{figure}[H]
\setcounter{subfigure}{0}
    \subfigure[$\sigma=0$(noiseless)]{
      \centering
      \includegraphics[width=.45\textwidth]{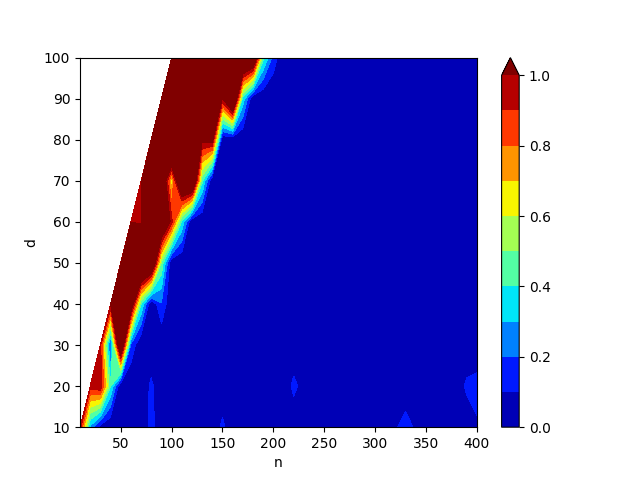}  
    }
    \subfigure[$\sigma=0.05$]{
      \centering
      \includegraphics[width=.45\textwidth]{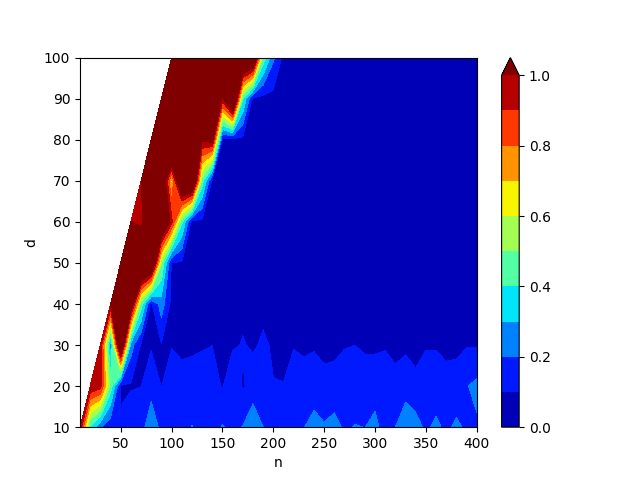}  
    }
    \newline
    \end{figure}
\begin{figure}[H]
 \addtocounter{figure}{1}  
    \ContinuedFloat
    \subfigure[$\sigma=0.1$]{
      \centering
      \includegraphics[width=.45\textwidth]{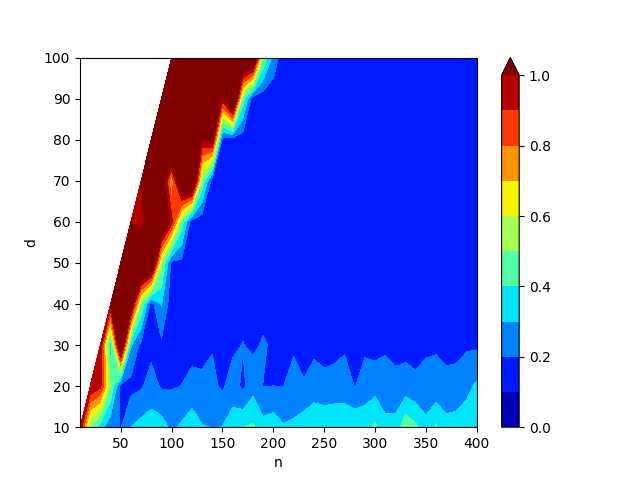}  
    }
    \subfigure[$\sigma=0.2$]{
      \centering
      \includegraphics[width=.45\textwidth]{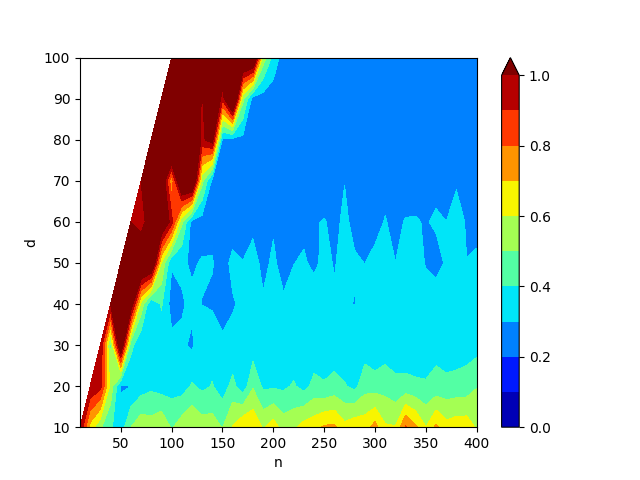}  
    }
    \caption{Averaged absolute distance by training ReLU networks with normalization layer on the regularized convex problem \eqref{prob:reg} over $5$ independent trials.}
    \label{fig:cvxt_normal}
\end{figure}

For nonconvex training method, we solve the regularized non-convex training problem \eqref{prob:reg} with the same 
setting as Section \ref{num_res_main:skip}. We set the number of neurons to be $m=n+1$ and train the ReLU neural network with normalization layer for 400 epochs. We use the AdamW optimizer with weight decay $\beta=10^{-6}$. We note that the nonconvex training may still reach local minimizers. Thus, the absolute distance do not show clear phase transition as the convex training. However, the transitions of test error generally follow the patterns of the group $\ell_1$-minimization problem.
In Figure \ref{fig:ncvx_train_normal}, we show that the test error increases as $n/d$ increases,
and the rate of increase becomes sharp around $n=2d$ (the boundary of orange region).

\begin{figure}[H]
\setcounter{subfigure}{0}
    \subfigure[$\sigma=0$(noiseless)]{
      \centering
      \includegraphics[width=\figscale\textwidth]{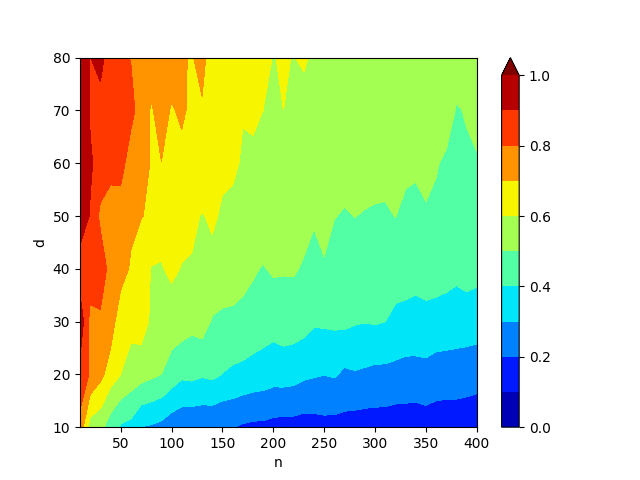}  
    }
    \subfigure[$\sigma=0.05$]{
      \centering
      \includegraphics[width=\figscale\textwidth]{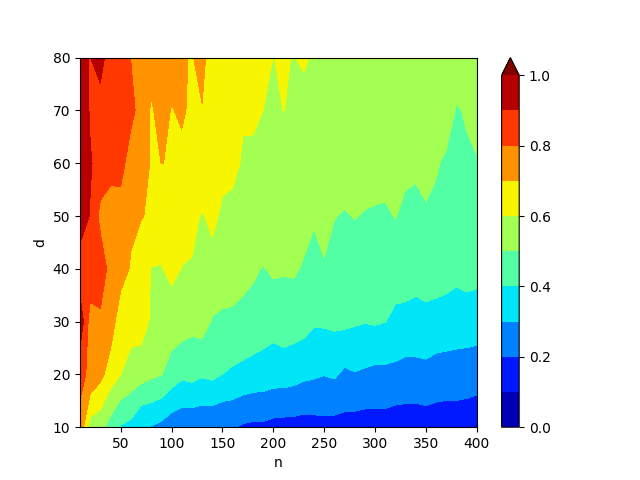}  
    }
    \newline
    \end{figure}
\begin{figure}[H]
 \addtocounter{figure}{1}  
    \ContinuedFloat
    \subfigure[$\sigma=0.1$]{
      \centering
      \includegraphics[width=\figscale\textwidth]{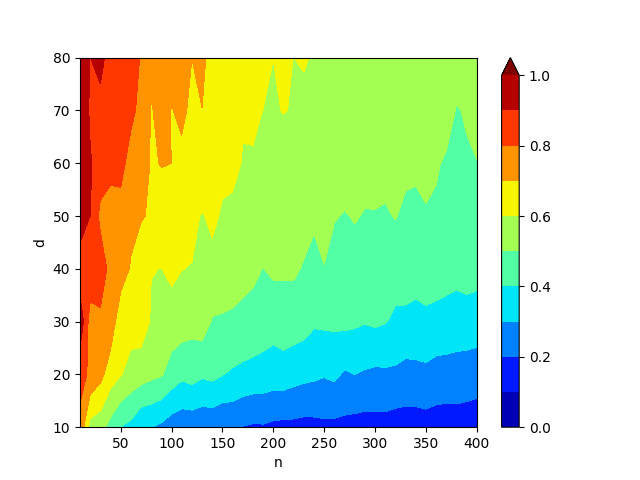}  
    }
    \subfigure[$\sigma=0.2$]{
      \centering
      \includegraphics[width=\figscale\textwidth]{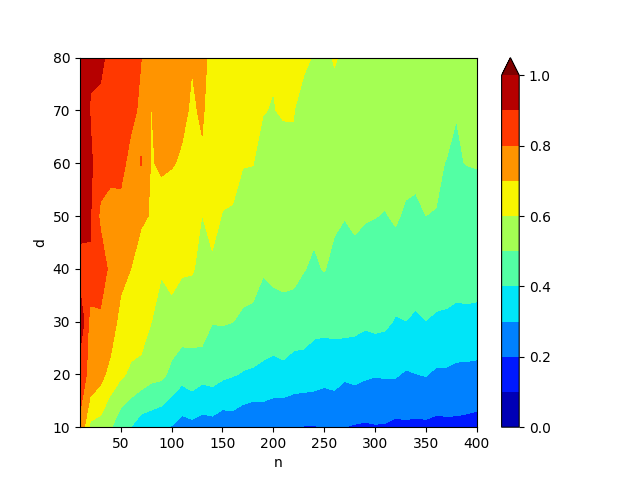}  
    }
    \caption{Averaged test error by training ReLU networks with normalization layer on the regularized non-convex problem \eqref{prob:reg} over $10$ independent trials.}
    \label{fig:ncvx_train_normal}
\end{figure}

\subsection{Multi-neuron recovery}\label{num_res:mul}
In this subsection, we present further numerical results on the recovery of ReLU networks with normalization layer when the label vector is the combination of several normalized ReLU neurons. Readers can refer to the details of the experiments in Section \ref{num_res_main:mul} and Appendix \ref{num_res:normal}.

We first show phase transition graphs for successful recovery of the planted normalized ReLU neurons.
We compute the recovery rate for $d$ ranging from 10 to 80 and $n$ ranging from 10 to 400 and establish 5 independent trials for each pair of $(n,d)$. 

\begin{figure}[H]
\setcounter{subfigure}{0}
    \centering
    \subfigure[$k=2,\ \mfw_1^* = \mfw^*,\ \mfw_2^* = -\mfw^*,\ \mfw^*\sim \mcU(\mbS^{n-1})$]{
      \centering
      \includegraphics[width=\figscale\textwidth]{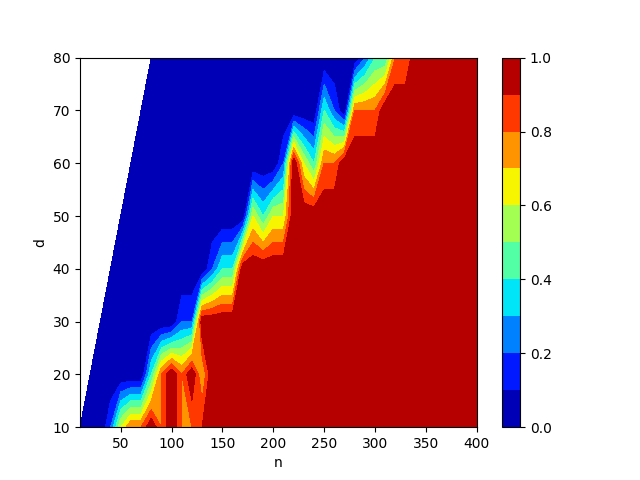}  
    }
    \centering
    \subfigure[$k=2,\ \mfw_1^*,\mfw_2^*\sim \mcU(\mbS^{n-1})$]{
      \centering
      \includegraphics[width=\figscale\textwidth]{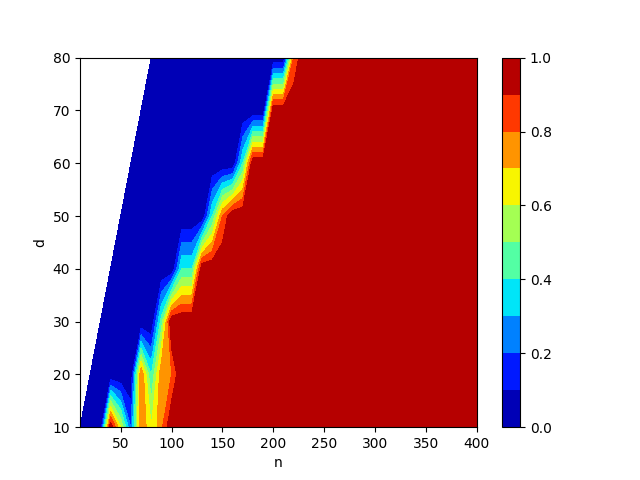} 
    }
        \end{figure}
\begin{figure}[H]
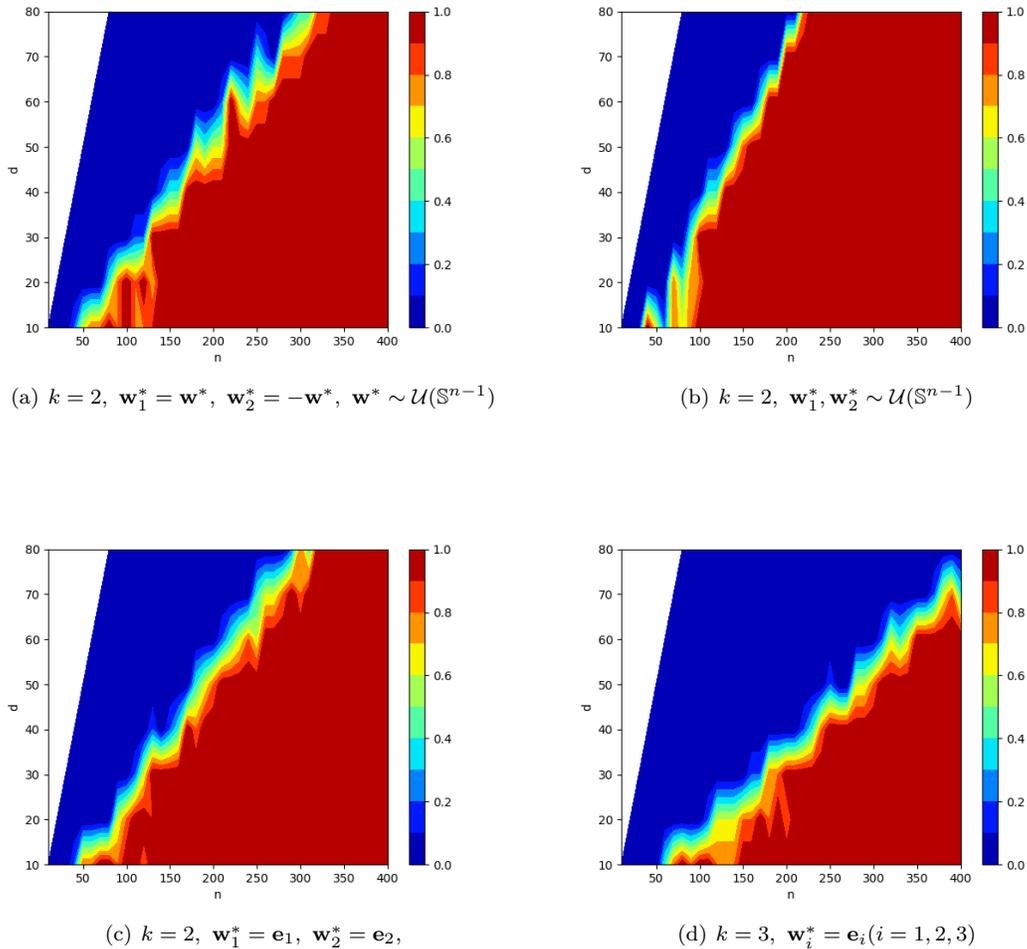
 
 \addtocounter{figure}{1}  
 \ContinuedFloat
\centering
    \centering
    \subfigure[$k=2,\ \mfw_1^* = \mfe_1,\ \mfw_2^* = \mfe_2$,]{
      \centering
      \includegraphics[width=\figscale\textwidth]{figs/bn/hard_bncvx2_mn_n400_d80_neu2_w2_X0_sig0.0_sample5.png} 
    }
    \centering
    \subfigure[$k=3,\ \mfw_i^* = \mfe_i(i=1,2,3)$]{
      \centering
      \includegraphics[width=\figscale\textwidth]{figs/bn/hard_bncvx2_mn_n400_d80_neu3_w2_X0_sig0.0_sample5.png} 
    }
    \caption{The probability of successful recovery of the planted normalized ReLU neurons by solving the group $\ell_1$-minimization problem \eqref{min_nrm:grelu_normal} over $5$ independent trials. The label vector $\mfy$ is the combination of several normalized ReLU neurons.}
\end{figure}

The second part is noisy observation model. In Figure \ref{fig:lin_phase_noise_normal_mn21}, we show the phase transition graph when the planted neurons satisfy $\mfw_1^* = \mfw^*,\ \mfw_2^* = -\mfw^*,\ \mfw^*\sim \mcU(\mbS^{n-1})$. In Figure \ref{fig:lin_phase_noise_normal_mn22}, \ref{fig:lin_phase_noise_normal_mn23} and \ref{fig:lin_phase_noise_normal_mn33} below, we show results when $\mfw_1^*,\mfw_2^*\sim \mcU(\mbS^{n-1})$, $\mfw_1^* = \mfe_1,\ \mfw_2^* = \mfe_2$ and $\mfw_i^* = \mfe_i(i=1,2,3)$, respectively.

\begin{figure}[H]
\setcounter{subfigure}{0}
\centering
    \subfigure[$\sigma=0$(noiseless)]{
      \centering
      \includegraphics[width=\figscale\textwidth]{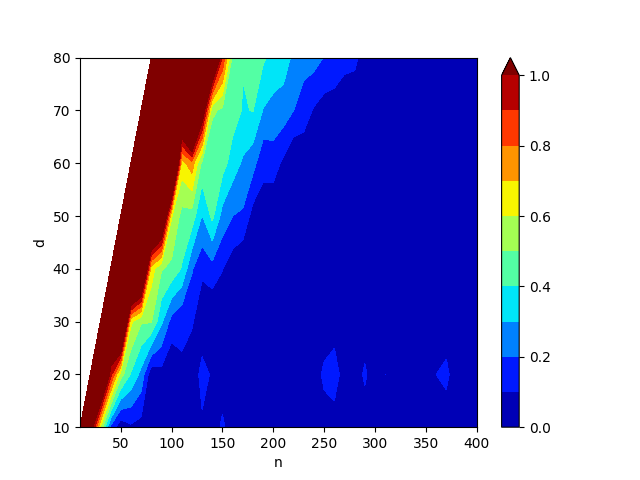}  
    }
    \centering
    \subfigure[$\sigma=0.05$]{
      \centering
      \includegraphics[width=\figscale\textwidth]{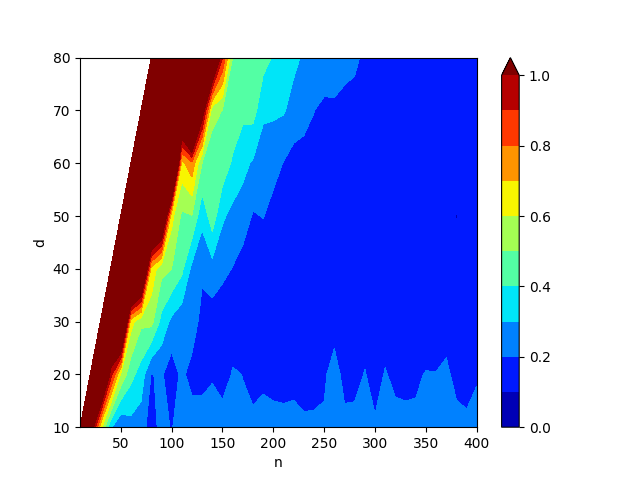}  
    }
    \end{figure}
\begin{figure}[H] 
 \addtocounter{figure}{1}  
 \ContinuedFloat
\centering
    \subfigure[$\sigma=0.1$]{
      \centering
      \includegraphics[width=\figscale\textwidth]{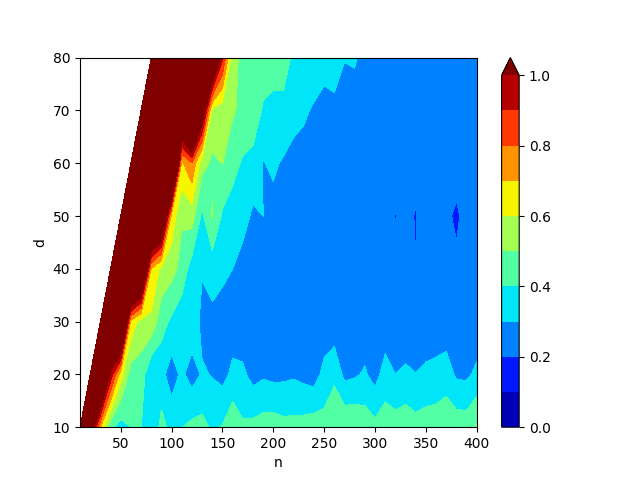}  
    }
    \centering
    \subfigure[$\sigma=0.2$]{
      \centering
      \includegraphics[width=\figscale\textwidth]{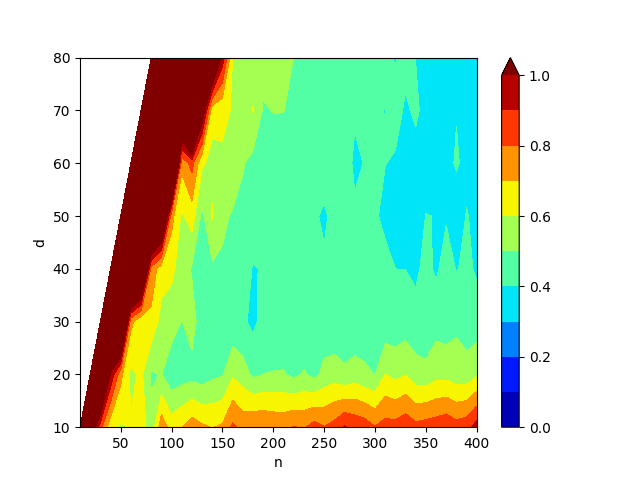}  
    }
    \caption{Averaged absolute distance to the planted normalized ReLU neurons by solving the group $\ell_1$-minimization problem \eqref{min_nrm:grelu_normal}
from training ReLU networks with normalization layer over 5 independent trials. Here we set $k=2$ planted neurons which satisfy $\mfw_1^*,\mfw_2^*\sim \mcU(\mbS^{n-1})$.}
\label{fig:lin_phase_noise_normal_mn22}
\end{figure}
\vspace{-6ex}
\begin{figure}[H]
\setcounter{subfigure}{0}
\centering
    \subfigure[$\sigma=0$(noiseless)]{
      \centering
      \includegraphics[width=\figscale\textwidth]{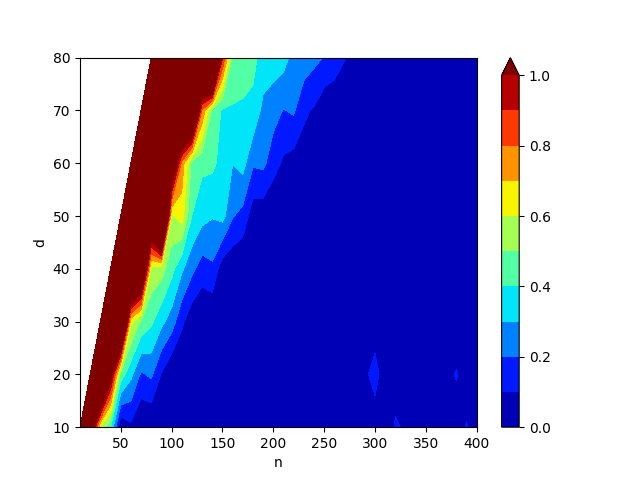}  
    }
    \centering
    \subfigure[$\sigma=0.05$]{
      \centering
      \includegraphics[width=\figscale\textwidth]{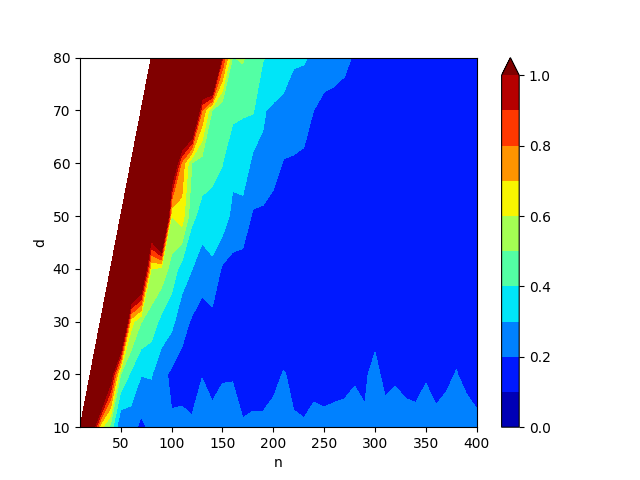}  
    }
    \centering
    \subfigure[$\sigma=0.1$]{
      \centering
      \includegraphics[width=\figscale\textwidth]{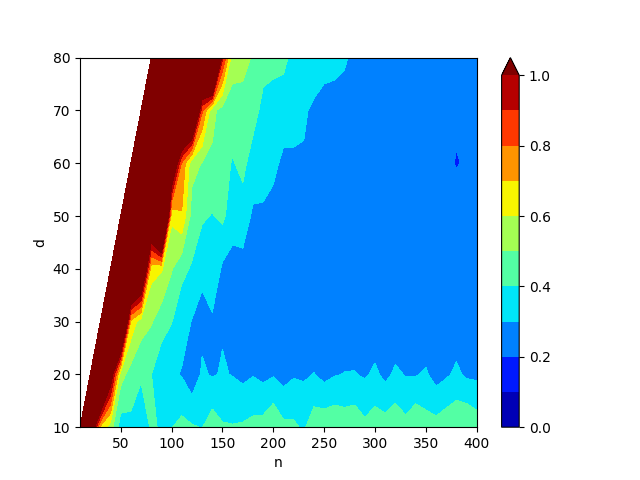}  
    }
    \centering
    \subfigure[$\sigma=0.2$]{
      \centering
      \includegraphics[width=\figscale\textwidth]{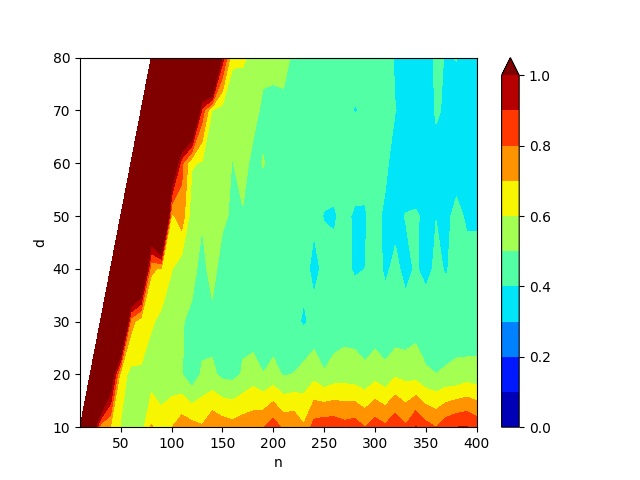}  
    }
    \caption{Averaged absolute distance to the planted normalized ReLU neurons by solving the group $\ell_1$-minimization problem \eqref{min_nrm:grelu_normal}
from training ReLU networks with normalization layer over 5 independent trials. Here we set $k=2$ planted neurons which satisfy $\mfw_1^* = \mfe_1,\ \mfw_2^* = \mfe_2$.}
\label{fig:lin_phase_noise_normal_mn23}
\end{figure}

\begin{figure}[H]
\setcounter{subfigure}{0}
\centering
    \subfigure[$\sigma=0$(noiseless)]{
      \centering
      \includegraphics[width=\figscale\textwidth]{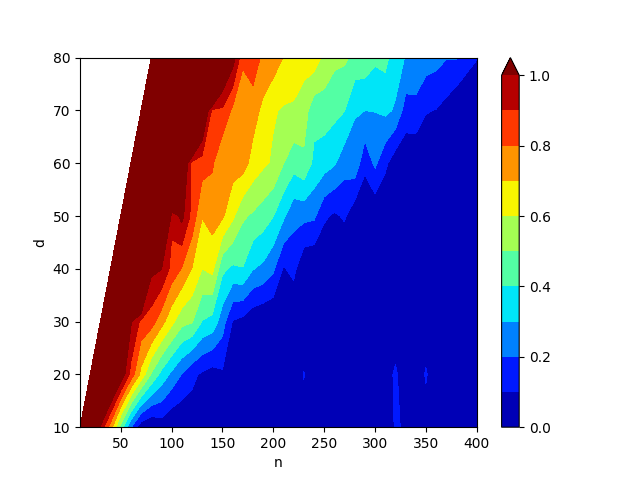}  
    }
    \centering
    \subfigure[$\sigma=0.05$]{
      \centering
      \includegraphics[width=\figscale\textwidth]{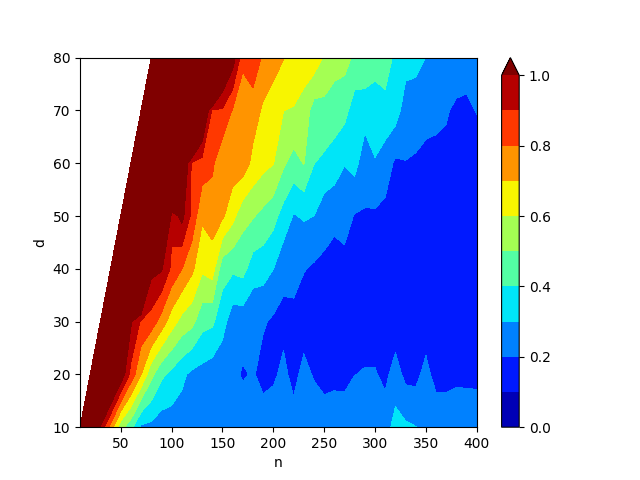}  
    }
    \centering
    \subfigure[$\sigma=0.1$]{
      \centering
      \includegraphics[width=\figscale\textwidth]{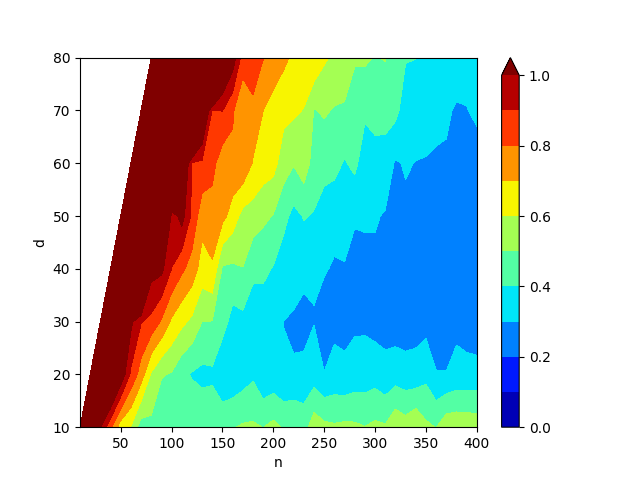}  
    }
    \centering
    \subfigure[$\sigma=0.2$]{
      \centering
      \includegraphics[width=\figscale\textwidth]{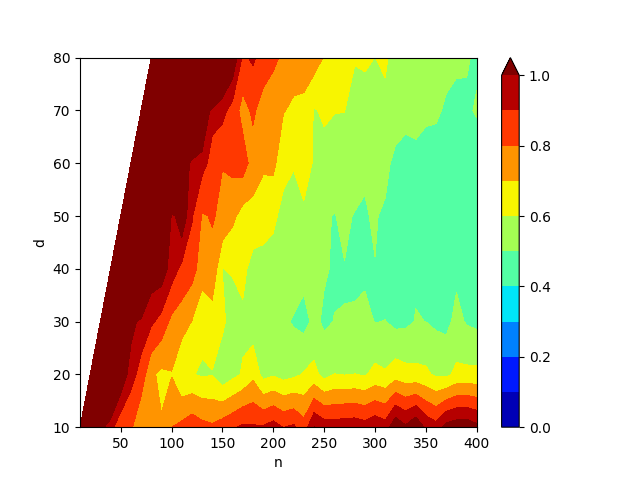}  
    }
    \caption{Averaged absolute distance to the planted normalized ReLU neurons by solving the group $\ell_1$-minimization problem \eqref{min_nrm:grelu_normal}
from training ReLU networks with normalization layer over 10 independent trials. Here we set $k=3$ planted neurons which satisfy $\mfw_i^* = \mfe_i(i=1,2,3)$.}
\label{fig:lin_phase_noise_normal_mn33}
\end{figure}
In the third part, we directly study the generalization property of ReLU networks with normalization layer using non-convex training methods. We note that the transitions of test error generally follow the patterns of the group $\ell_1$-minimization problem, analogous to our observation in Appendix \ref{num_res:normal}.
\begin{figure}[H]
\setcounter{subfigure}{0}
\centering
    \subfigure[$\sigma=0$(noiseless)]{
      \centering
      \includegraphics[width=\figscale\textwidth]{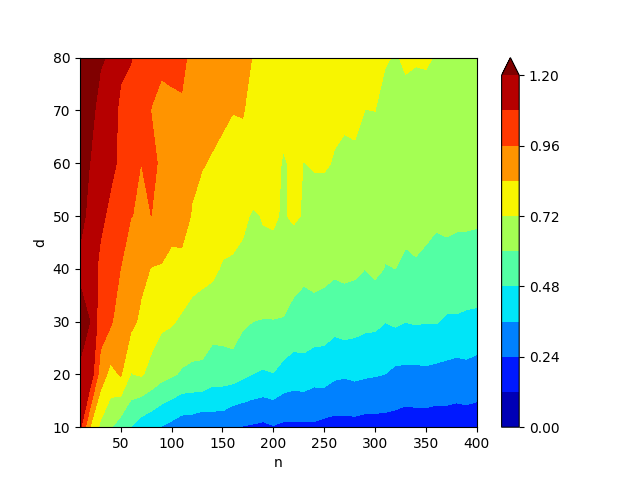}  
    }
    \centering
    \subfigure[$\sigma=0.05$]{
      \centering
      \includegraphics[width=\figscale\textwidth]{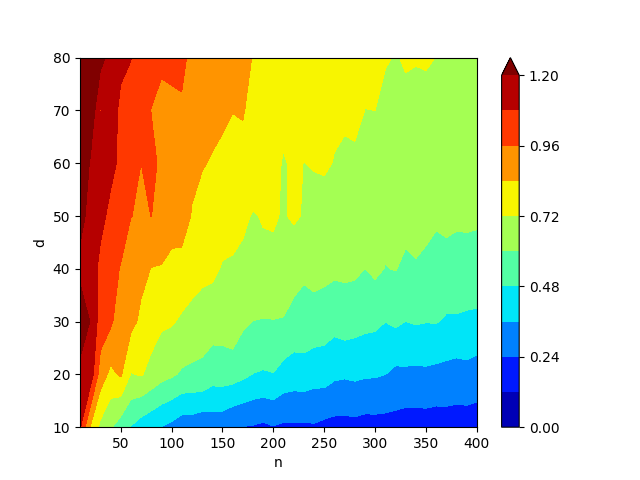}  
    }
\end{figure}
\begin{figure}[H]
 \addtocounter{figure}{1}  
    \ContinuedFloat
\centering
    \subfigure[$\sigma=0.1$]{
      \centering
      \includegraphics[width=\figscale\textwidth]{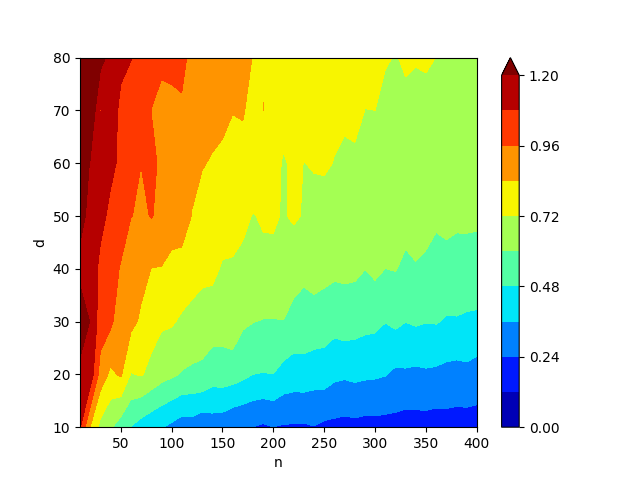}  
    }
    \centering
    \subfigure[$\sigma=0.2$]{
      \centering
      \includegraphics[width=\figscale\textwidth]{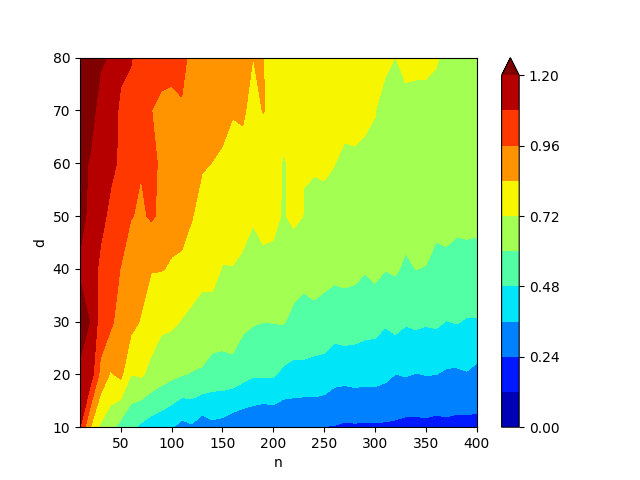}  
    }
    \caption{Averaged test error by training ReLU networks with normalization layer on the regularized non-convex problem \eqref{prob:reg} over $10$ independent trials. Here we set $k=2$ planted neurons which satisfy $\mfw_1^* = \mfw^*,\ \mfw_2^* = -\mfw^*,\ \mfw^*\sim \mcU(\mbS^{n-1})$.}
    \label{fig:ncvx_train_normal_mn21}
\end{figure}
\begin{figure}[H]
\setcounter{subfigure}{0}
\centering
    \subfigure[$\sigma=0$(noiseless)]{
      \centering
      \includegraphics[width=\figscale\textwidth]{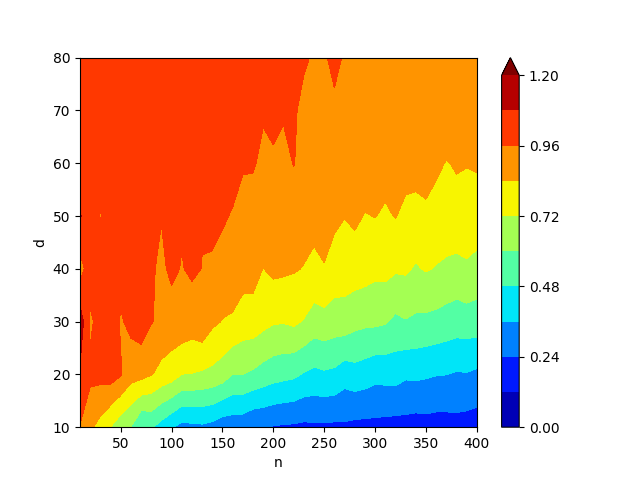}  
    }
    \centering
    \subfigure[$\sigma=0.05$]{
      \centering
      \includegraphics[width=\figscale\textwidth]{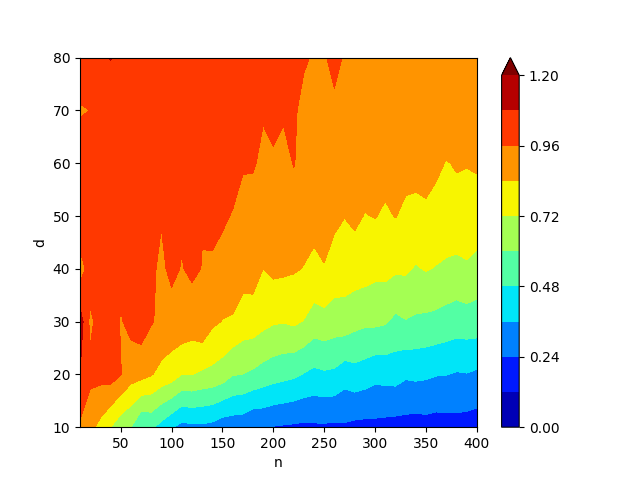}  
    }
    \centering
    \subfigure[$\sigma=0.1$]{
      \centering
      \includegraphics[width=\figscale\textwidth]{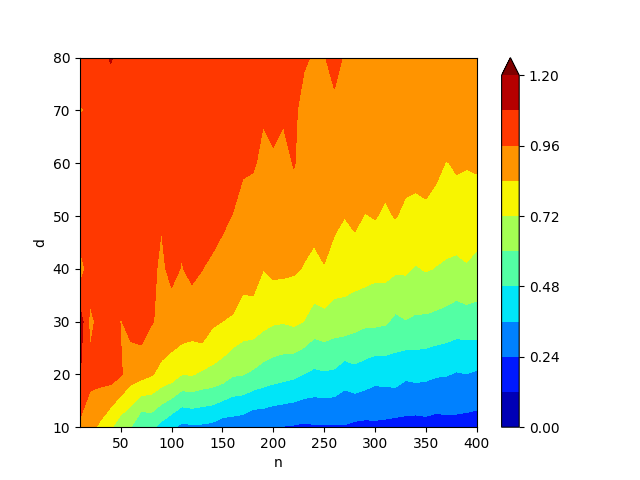}  
    }
    \centering
    \subfigure[$\sigma=0.2$]{
      \centering
      \includegraphics[width=\figscale\textwidth]{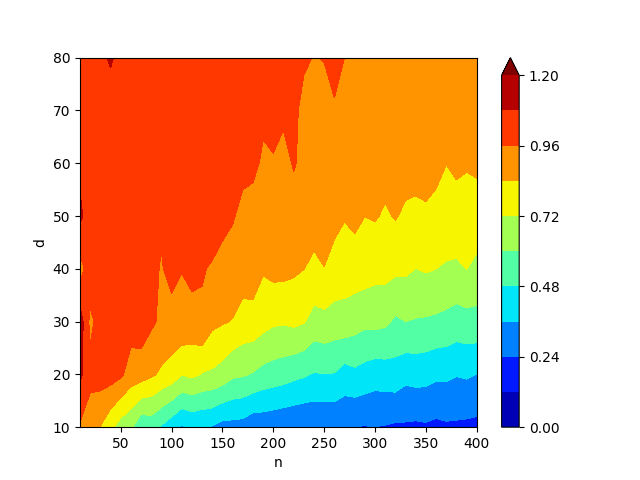}  
    }
    \caption{Averaged test error by training ReLU networks with normalization layer on the regularized non-convex problem \eqref{prob:reg} over $10$ independent trials. Here we set $k=2$ planted neurons which satisfy $\mfw_1^*,\mfw_2^*\sim \mcU(\mbS^{n-1})$.}
    \label{fig:ncvx_train_normal_mn22}
\end{figure}

\begin{figure}[H]
\setcounter{subfigure}{0}
\centering
    \subfigure[$\sigma=0$(noiseless)]{
      \centering
      \includegraphics[width=\figscale\textwidth]{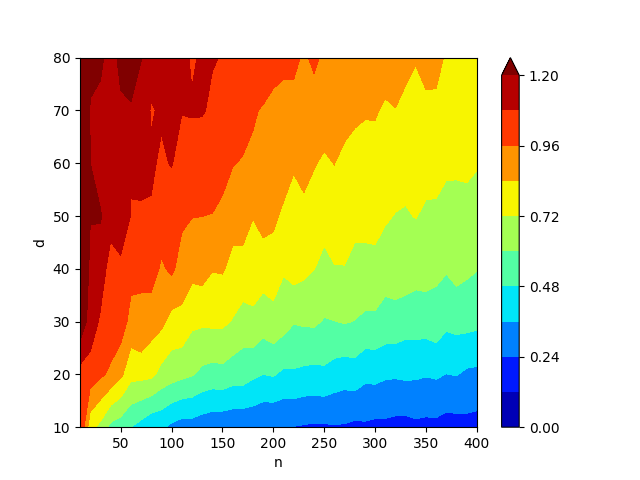}  
    }
    \centering
    \subfigure[$\sigma=0.05$]{
      \centering
      \includegraphics[width=\figscale\textwidth]{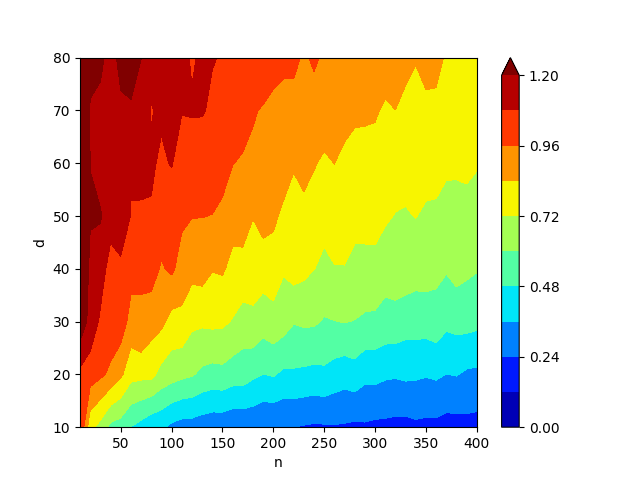}  
    }
    \centering
    \subfigure[$\sigma=0.1$]{
      \centering
      \includegraphics[width=\figscale\textwidth]{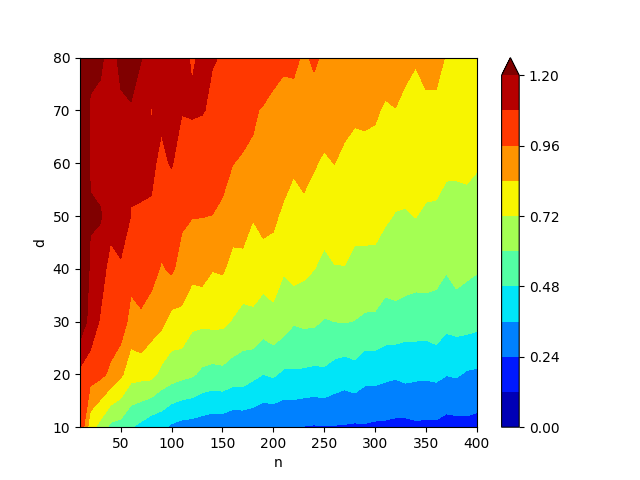}  
    }
    \centering
    \subfigure[$\sigma=0.2$]{
      \centering
      \includegraphics[width=\figscale\textwidth]{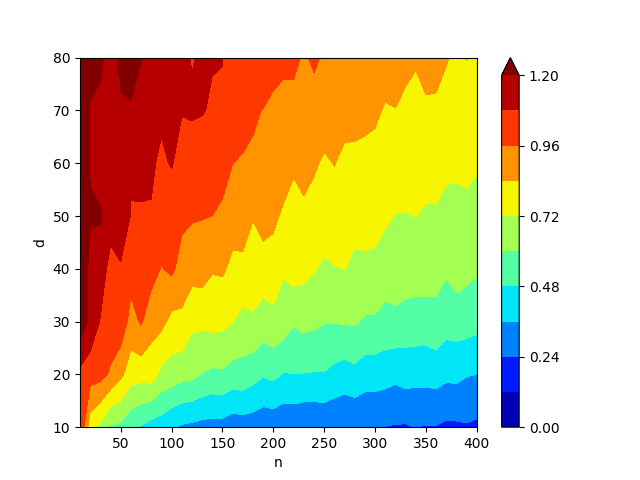}  
    }
    \caption{Averaged test error by training ReLU networks with normalization layer on the regularized non-convex problem \eqref{prob:reg} over $10$ independent trials. Here we set $k=2$ planted neurons which satisfy $\mfw_1^* = \mfe_1,\ \mfw_2^* = \mfe_2$.}
    \label{fig:ncvx_train_normal_mn23}
\end{figure}

\begin{figure}[H]
\setcounter{subfigure}{0}
\centering
    \subfigure[$\sigma=0$(noiseless)]{
      \centering
      \includegraphics[width=\figscale\textwidth]{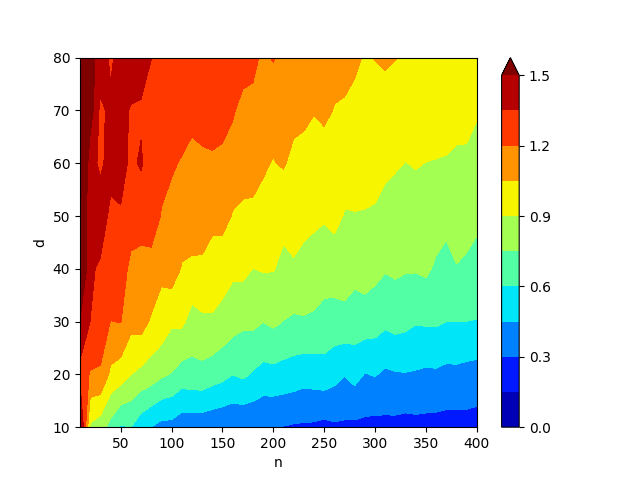}  
    }
    \centering
    \subfigure[$\sigma=0.05$]{
      \centering
      \includegraphics[width=\figscale\textwidth]{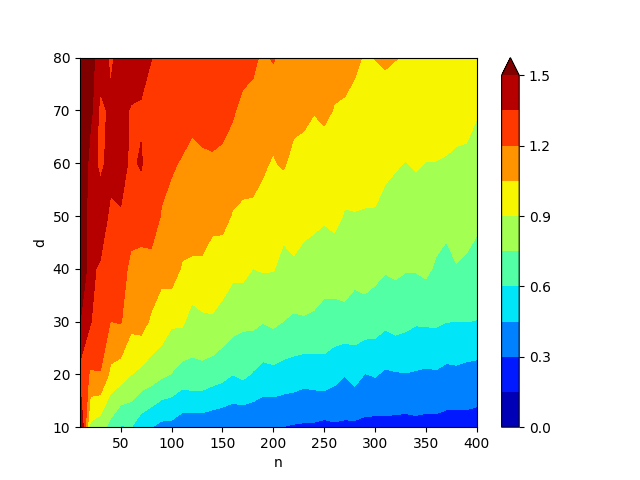}  
    }
    \end{figure}
\begin{figure}[H]
 \addtocounter{figure}{1}  
    \ContinuedFloat
\centering
    \subfigure[$\sigma=0.1$]{
      \centering
      \includegraphics[width=\figscale\textwidth]{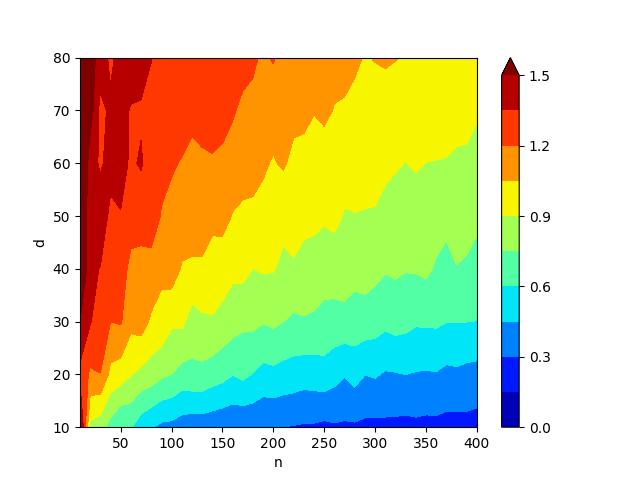}  
    }
    \centering
    \subfigure[$\sigma=0.2$]{
      \centering
      \includegraphics[width=\figscale\textwidth]{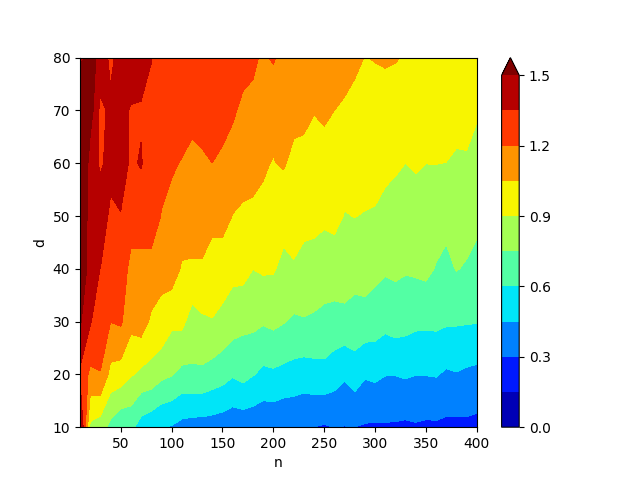}  
    }
    \caption{Averaged test error by training ReLU networks with normalization layer on the regularized non-convex problem \eqref{prob:reg} over $10$ independent trials. Here we set $k=3$ planted neurons which satisfy $\mfw_i^* = \mfe_i(i=1,2,3)$.}
    \label{fig:ncvx_train_normal_mn33}
\end{figure}

\appendices


\section*{Acknowledgment}
This work was partially supported by the National Science Foundation (NSF) under grants ECCS- 2037304, DMS-2134248, NSF CAREER award CCF-2236829, the U.S. Army Research
Office Early Career Award W911NF-21-1-0242, Stanford Precourt Institute, and the ACCESS – AI Chip Center for Emerging Smart Systems, sponsored by InnoHK funding, Hong Kong SAR.
\ifCLASSOPTIONcaptionsoff
  \newpage
\fi



%
\bibliographystyle{IEEEtran}
\bibliography{Newton}

\begin{thebibliography}{10}
\providecommand{\url}[1]{#1}
\csname url@samestyle\endcsname
\providecommand{\newblock}{\relax}
\providecommand{\bibinfo}[2]{#2}
\providecommand{\BIBentrySTDinterwordspacing}{\spaceskip=0pt\relax}
\providecommand{\BIBentryALTinterwordstretchfactor}{4}
\providecommand{\BIBentryALTinterwordspacing}{\spaceskip=\fontdimen2\font plus
\BIBentryALTinterwordstretchfactor\fontdimen3\font minus
  \fontdimen4\font\relax}
\providecommand{\BIBforeignlanguage}[2]{{%
\expandafter\ifx\csname l@#1\endcsname\relax
\typeout{** WARNING: IEEEtran.bst: No hyphenation pattern has been}%
\typeout{** loaded for the language `#1'. Using the pattern for}%
\typeout{** the default language instead.}%
\else
\language=\csname l@#1\endcsname
\fi
#2}}
\providecommand{\BIBdecl}{\relax}
\BIBdecl

\bibitem{amelunxen2014living}
D.~Amelunxen, M.~Lotz, M.~B. McCoy, and J.~A. Tropp, ``Living on the edge:
  Phase transitions in convex programs with random data,'' \emph{Information
  and Inference: A Journal of the IMA}, vol.~3, no.~3, pp. 224--294, 2014.

\bibitem{nnacr}
M.~Pilanci and T.~Ergen, ``Neural networks are convex regularizers: Exact
  polynomial-time convex optimization formulations for two-layer networks,''
  \emph{International Conference on Machine Learning (ICML)}, 2020.

\bibitem{ergen2020convex}
T.~Ergen and M.~Pilanci, ``Convex geometry and duality of over-parameterized
  neural networks,'' \emph{Journal of Machine Learning Research}, vol.~22, no.
  212, pp. 1--63, 2021.

\bibitem{ergen2021global}
------, ``Global optimality beyond two layers: Training deep relu networks via
  convex programs,'' in \emph{International Conference on Machine
  Learning}.\hskip 1em plus 0.5em minus 0.4em\relax PMLR, 2021, pp. 2993--3003.

\bibitem{wang2022convexgeometryofbp}
Y.~Wang and M.~Pilanci, ``The convex geometry of backpropagation: Neural
  network gradient flows converge to extreme points of the dual convex
  program,'' \emph{International Conference on Learning Representations
  (ICLR)}, 2022.

\bibitem{wang2020hidden}
Y.~Wang, J.~Lacotte, and M.~Pilanci, ``The hidden convex optimization landscape
  of two-layer relu neural networks: an exact characterization of the optimal
  solutions,'' \emph{International Conference on Learning Representations
  (ICLR)}, 2022.

\bibitem{zhao2006model}
P.~Zhao and B.~Yu, ``On model selection consistency of lasso,'' \emph{The
  Journal of Machine Learning Research}, vol.~7, pp. 2541--2563, 2006.

\bibitem{candes2005decoding}
E.~J. Candes and T.~Tao, ``Decoding by linear programming,'' \emph{IEEE
  transactions on information theory}, vol.~51, no.~12, pp. 4203--4215, 2005.

\bibitem{candes2008restricted}
E.~J. Candes, ``The restricted isometry property and its implications for
  compressed sensing,'' \emph{Comptes rendus mathematique}, vol. 346, no. 9-10,
  pp. 589--592, 2008.

\bibitem{sohl2015deep}
J.~Sohl-Dickstein, E.~Weiss, N.~Maheswaranathan, and S.~Ganguli, ``Deep
  unsupervised learning using nonequilibrium thermodynamics,'' in
  \emph{International Conference on Machine Learning}.\hskip 1em plus 0.5em
  minus 0.4em\relax PMLR, 2015, pp. 2256--2265.

\bibitem{ho2020denoising}
J.~Ho, A.~Jain, and P.~Abbeel, ``Denoising diffusion probabilistic models,''
  \emph{Advances in Neural Information Processing Systems}, vol.~33, pp.
  6840--6851, 2020.

\bibitem{ergen2021demystifying}
T.~Ergen, A.~Sahiner, B.~Ozturkler, J.~Pauly, M.~Mardani, and M.~Pilanci,
  ``Demystifying batch normalization in relu networks: Equivalent convex
  optimization models and implicit regularization,'' \emph{International
  Conference on Learning Representations (ICLR)}, 2022.

\bibitem{ergen2020implicit}
T.~Ergen and M.~Pilanci, ``Implicit convex regularizers of cnn architectures:
  Convex optimization of two- and three-layer networks in polynomial time,''
  \emph{International Conference on Learning Representations (ICLR)}, 2021.

\bibitem{bartan2021neural}
B.~Bartan and M.~Pilanci, ``Neural spectrahedra and semidefinite lifts: Global
  convex optimization of polynomial activation neural networks in fully
  polynomial-time,'' \emph{arXiv preprint arXiv:2101.02429}, 2021.

\bibitem{sahiner2022attention}
A.~Sahiner, T.~Ergen, B.~Ozturkler, J.~Pauly, M.~Mardani, and M.~Pilanci,
  ``Unraveling attention via convex duality: Analysis and interpretations of
  vision transformers,'' \emph{International Conference on Machine Learning},
  2022.

\bibitem{sahiner2021hidden}
A.~Sahiner, T.~Ergen, B.~Ozturkler, B.~Bartan, J.~Pauly, M.~Mardani, and
  M.~Pilanci, ``Hidden convexity of wasserstein gans: Interpretable generative
  models with closed-form solutions,'' \emph{International Conference on
  Learning Representations}, 2021.

\bibitem{resnet}
\BIBentryALTinterwordspacing
K.~He, X.~Zhang, S.~Ren, and J.~Sun, ``Deep residual learning for image
  recognition,'' \emph{CoRR}, vol. abs/1512.03385, 2015. [Online]. Available:
  \url{http://arxiv.org/abs/1512.03385}
\BIBentrySTDinterwordspacing

\bibitem{Jacot2018}
\BIBentryALTinterwordspacing
A.~Jacot, F.~Gabriel, and C.~Hongler, ``Neural tangent kernel: Convergence and
  generalization in neural networks,'' 2018. [Online]. Available:
  \url{https://arxiv.org/abs/1806.07572}
\BIBentrySTDinterwordspacing

\bibitem{Du2018}
\BIBentryALTinterwordspacing
S.~S. Du, X.~Zhai, B.~Poczos, and A.~Singh, ``Gradient descent provably
  optimizes over-parameterized neural networks,'' 2018. [Online]. Available:
  \url{https://arxiv.org/abs/1810.02054}
\BIBentrySTDinterwordspacing

\bibitem{allen2018learning}
Z.~Allen-Zhu, Y.~Li, and Y.~Liang, ``Learning and generalization in
  overparameterized neural networks, going beyond two layers,'' \emph{arXiv
  preprint arXiv:1811.04918}, 2018.

\bibitem{liu2020linearity}
C.~Liu, L.~Zhu, and M.~Belkin, ``On the linearity of large non-linear models:
  when and why the tangent kernel is constant,'' \emph{Advances in Neural
  Information Processing Systems}, vol.~33, pp. 15\,954--15\,964, 2020.

\bibitem{zou2018stochastic}
D.~Zou, Y.~Cao, D.~Zhou, and Q.~Gu, ``Stochastic gradient descent optimizes
  over-parameterized deep relu networks,'' \emph{arXiv preprint
  arXiv:1811.08888}, 2018.

\bibitem{allen2019convergence}
Z.~Allen-Zhu, Y.~Li, and Z.~Song, ``A convergence theory for deep learning via
  over-parameterization,'' in \emph{International Conference on Machine
  Learning}.\hskip 1em plus 0.5em minus 0.4em\relax PMLR, 2019, pp. 242--252.

\bibitem{allen2019learning}
Z.~Allen-Zhu, Y.~Li, and Y.~Liang, ``Learning and generalization in
  overparameterized neural networks, going beyond two layers,'' \emph{Advances
  in neural information processing systems}, vol.~32, 2019.

\bibitem{Chizat2018}
\BIBentryALTinterwordspacing
L.~Chizat, E.~Oyallon, and F.~Bach, ``On lazy training in differentiable
  programming,'' 2018. [Online]. Available:
  \url{https://arxiv.org/abs/1812.07956}
\BIBentrySTDinterwordspacing

\bibitem{arora2019fine}
S.~Arora, S.~Du, W.~Hu, Z.~Li, and R.~Wang, ``Fine-grained analysis of
  optimization and generalization for overparameterized two-layer neural
  networks,'' in \emph{International Conference on Machine Learning}.\hskip 1em
  plus 0.5em minus 0.4em\relax PMLR, 2019, pp. 322--332.

\bibitem{chen2019much}
Z.~Chen, Y.~Cao, D.~Zou, and Q.~Gu, ``How much over-parameterization is
  sufficient to learn deep relu networks?'' \emph{arXiv preprint
  arXiv:1911.12360}, 2019.

\bibitem{neyshabur2018role}
B.~Neyshabur, Z.~Li, S.~Bhojanapalli, Y.~LeCun, and N.~Srebro, ``The role of
  over-parametrization in generalization of neural networks,'' in
  \emph{International Conference on Learning Representations}, 2018.

\bibitem{hochreiter1997flat}
S.~Hochreiter and J.~Schmidhuber, ``Flat minima,'' \emph{Neural computation},
  vol.~9, no.~1, pp. 1--42, 1997.

\bibitem{jastrzkebski2017three}
S.~Jastrz{k{e}}bski, Z.~Kenton, D.~Arpit, N.~Ballas, A.~Fischer, Y.~Bengio, and
  A.~Storkey, ``Three factors influencing minima in sgd,'' \emph{arXiv preprint
  arXiv:1711.04623}, 2017.

\bibitem{wu2017towards}
L.~Wu, Z.~Zhu \emph{et~al.}, ``Towards understanding generalization of deep
  learning: Perspective of loss landscapes,'' \emph{arXiv preprint
  arXiv:1706.10239}, 2017.

\bibitem{goyal2017accurate}
P.~Goyal, P.~Doll{\'a}r, R.~Girshick, P.~Noordhuis, L.~Wesolowski, A.~Kyrola,
  A.~Tulloch, Y.~Jia, and K.~He, ``Accurate, large minibatch sgd: Training
  imagenet in 1 hour,'' \emph{arXiv preprint arXiv:1706.02677}, 2017.

\bibitem{hoffer2017train}
E.~Hoffer, I.~Hubara, and D.~Soudry, ``Train longer, generalize better: closing
  the generalization gap in large batch training of neural networks,''
  \emph{arXiv preprint arXiv:1705.08741}, 2017.

\bibitem{luo2020many}
V.~Luo and Y.~Wang, ``How many factors influence minima in sgd?'' \emph{arXiv
  preprint arXiv:2009.11858}, 2020.

\bibitem{ge2018learning}
R.~Ge, R.~Kuditipudi, Z.~Li, and X.~Wang, ``Learning two-layer neural networks
  with symmetric inputs,'' \emph{arXiv preprint arXiv:1810.06793}, 2018.

\bibitem{bakshi2019learning}
A.~Bakshi, R.~Jayaram, and D.~P. Woodruff, ``Learning two layer rectified
  neural networks in polynomial time,'' in \emph{Conference on Learning
  Theory}.\hskip 1em plus 0.5em minus 0.4em\relax PMLR, 2019, pp. 195--268.

\bibitem{tian2017analytical}
Y.~Tian, ``An analytical formula of population gradient for two-layered relu
  network and its applications in convergence and critical point analysis,'' in
  \emph{International Conference on Machine Learning}.\hskip 1em plus 0.5em
  minus 0.4em\relax PMLR, 2017, pp. 3404--3413.

\bibitem{soltanolkotabi2017learning}
M.~Soltanolkotabi, ``Learning relus via gradient descent,'' \emph{arXiv
  preprint arXiv:1705.04591}, 2017.

\bibitem{li2017convergence}
Y.~Li and Y.~Yuan, ``Convergence analysis of two-layer neural networks with
  relu activation,'' \emph{arXiv preprint arXiv:1705.09886}, 2017.

\bibitem{zhang2019learning}
X.~Zhang, Y.~Yu, L.~Wang, and Q.~Gu, ``Learning one-hidden-layer relu networks
  via gradient descent,'' in \emph{The 22nd international conference on
  artificial intelligence and statistics}.\hskip 1em plus 0.5em minus
  0.4em\relax PMLR, 2019, pp. 1524--1534.

\bibitem{zhong2017recovery}
K.~Zhong, Z.~Song, P.~Jain, P.~L. Bartlett, and I.~S. Dhillon, ``Recovery
  guarantees for one-hidden-layer neural networks,'' in \emph{International
  conference on machine learning}.\hskip 1em plus 0.5em minus 0.4em\relax PMLR,
  2017, pp. 4140--4149.

\bibitem{cover1965geometrical}
T.~M. Cover, ``Geometrical and statistical properties of systems of linear
  inequalities with applications in pattern recognition,'' \emph{IEEE
  transactions on electronic computers}, no.~3, pp. 326--334, 1965.

\bibitem{yuan2006model}
M.~Yuan and Y.~Lin, ``Model selection and estimation in regression with grouped
  variables,'' \emph{Journal of the Royal Statistical Society: Series B
  (Statistical Methodology)}, vol.~68, no.~1, pp. 49--67, 2006.

\bibitem{vershynin2018high}
R.~Vershynin, \emph{High-dimensional probability: An introduction with
  applications in data science}.\hskip 1em plus 0.5em minus 0.4em\relax
  Cambridge university press, 2018, vol.~47.

\bibitem{van2009conditions}
S.~A. Van De~Geer and P.~B{\"u}hlmann, ``On the conditions used to prove oracle
  results for the lasso,'' \emph{Electronic Journal of Statistics}, vol.~3, pp.
  1360--1392, 2009.

\bibitem{de2012decoupling}
V.~De~la Pena and E.~Gin{\'e}, \emph{Decoupling: from dependence to
  independence}.\hskip 1em plus 0.5em minus 0.4em\relax Springer Science \&
  Business Media, 2012.

\bibitem{ghorbani2021linearized}
B.~Ghorbani, S.~Mei, T.~Misiakiewicz, and A.~Montanari, ``Linearized two-layers
  neural networks in high dimension,'' \emph{The Annals of Statistics},
  vol.~49, no.~2, pp. 1029--1054, 2021.

\end{thebibliography}

\end{document}